\def\NOCHECKLIST{}
\newcommand*{\NEURIPS}{}

\newcommand*{\CAMREADY}{}

\ifdefined\ARXIV
	\documentclass[10pt]{article}
	\usepackage{libertine} %
	\usepackage{fullpage}

	\usepackage{natbib}
	
	\renewcommand{\cite}[1]{\citep{#1}}
	\setcitestyle{authoryear,round,citesep={;},aysep={,},yysep={;}}
	
	\usepackage[utf8]{inputenc} %
	\usepackage[T1]{fontenc}    %
	\usepackage{booktabs}       %
	\usepackage{multirow}
	\usepackage{microtype}      %
	\usepackage{xcolor}         %
	
	\usepackage{tabularx}
	\usepackage{thmtools}
	\usepackage{thm-restate}
	\usepackage[left=1.25in,right=1.25in,top=1in]{geometry}
	
	\usepackage{comment}
	\usepackage{hyperref}
	\definecolor{mydarkblue}{rgb}{0,0.08,0.5}
	\hypersetup{ %
		pdftitle={Outcome-Based RL Provably Leads Transformers to Reason, but Only With the Right Data},
		pdfauthor={},
		pdfsubject={},
		pdfkeywords={},
		colorlinks=true,
		linkcolor=mydarkblue,
		citecolor=mydarkblue,
		filecolor=mydarkblue,
		urlcolor=mydarkblue
	}
	
	\skip\footins 9pt plus 4pt minus 2pt
	\def\footnoterule{\kern-3pt \hrule width 12pc \kern 2.6pt }
	\leftmargini\leftmargin \leftmarginii 2em
	\renewenvironment{abstract}%
	{%
		\vskip 0in%
		\centerline%
		{\large\bf Abstract}%
		\vspace{-1ex}%
		\begin{quote}%
		}
		{
			\par%
		\end{quote}%
		\vskip 0ex%
	}
	
	\title{\vskip -4ex \bf{Outcome-Based RL Provably Leads Transformers to Reason, but Only With the Right Data} \vskip 0ex}
	\author{%
		\textbf{Yuval Ran-Milo}\\ %
		Some Institute\\
		\texttt{\normalsize yuvalmilo@mail.tau.ac.il} \\
		\and
		\textbf{Yotam Alexander}\\
		Some Institute\\
		\texttt{\normalsize \TODO{add}} \\
		\and
		\textbf{Shahar Mendel}\\
		Some Institute\\
		\texttt{\normalsize \TODO{add}} \\
		\and
		\textbf{Nadav Cohen}\\
		Some Institute\\
		\texttt{\normalsize \TODO{add}} \\
	}	
	\date{}
\fi

\ifdefined\NEURIPS
	\PassOptionsToPackage{numbers}{natbib}
	\documentclass{article}
	\ifdefined\CAMREADY
		\usepackage[final,main]{neurips_2026}
	\else
		\usepackage{neurips_2026}
	\fi
	\makeatletter
	\newcommand{\naturalthanksmark}{\renewcommand{\@makefnmark}{\hbox{$^{\@thefnmark}$}}}
	\makeatother
	\usepackage{footmisc}
	\usepackage[utf8]{inputenc} 
	\usepackage[T1]{fontenc}    
	\usepackage{hyperref}       	
	\usepackage{url}            	
	\usepackage{booktabs}       
	\usepackage{amsfonts}       
	\usepackage{nicefrac}       	
	\usepackage{microtype}      
\fi

\ifdefined\CVPR
	\documentclass[10pt,twocolumn,letterpaper]{article}
	\usepackage{footmisc}
	\usepackage{cvpr}
	\usepackage{times}
	\usepackage{epsfig}
	\usepackage{graphicx}
	\usepackage{amsmath}
	\usepackage{amssymb}
	\usepackage[square,comma,numbers]{natbib}
\fi
\ifdefined\AISTATS
	\documentclass[twoside]{article}
	\ifdefined\CAMREADY
		\usepackage[accepted]{aistats2015}
	\else
		\usepackage{aistats2015}
	\fi
	\usepackage{url}
	\usepackage[round,authoryear]{natbib}
\fi
\ifdefined\ICML
	\documentclass{article}
	\usepackage{footmisc}
	\usepackage{microtype}
	\usepackage{graphicx}
	\usepackage{booktabs}
	\usepackage{hyperref}
	
	\ifdefined\CAMREADY
		\usepackage[accepted]{icml2026}
	\else
		\usepackage{icml2026}
	\fi
\fi
\ifdefined\ICLR
	\documentclass{article}
	\usepackage{iclr2019_conference,times}
	\usepackage{footmisc}
	\usepackage{hyperref}
	\usepackage{url}

	\ifdefined\CAMREADY
		\iclrfinalcopy
	\fi
\fi
\ifdefined\COLT
	\ifdefined\CAMREADY
		\documentclass{colt2017}
	\else
		\documentclass[anon,12pt]{colt2017}
	\fi
\usepackage{times}
\fi

\usepackage{color}
\usepackage{multirow}
\usepackage{amsfonts}
\usepackage{algorithm}
\usepackage{algorithmic}
\usepackage{graphicx}
\usepackage{array}
\usepackage{nicefrac}
\usepackage{bbm}
\usepackage{wrapfig}
\usepackage{tikz}
\usepackage[titletoc,title]{appendix}
\usepackage{stmaryrd}
\usepackage{comment}
\usepackage{endnotes}
\usepackage{hyperendnotes}
\usepackage{enumerate}
\usepackage{enumitem}
\usepackage[normalem]{ulem}
	
\usepackage{titlesec}
\ifdefined\COLT
\else
	\usepackage{amsmath,amsthm,amssymb,mathtools}
\fi

\usepackage[extdef=true]{delimset}
\usepackage[capitalize,noabbrev]{cleveref}

\ifdefined\COLT
	\newtheorem{claim}[theorem]{Claim}
	\newtheorem{fact}[theorem]{Fact}
	\newtheorem{procedure}{Procedure}
	\newtheorem{conjecture}{Conjecture}	
	\newtheorem{hypothesis}{Hypothesis}	
	
\else
	\newtheorem{lemma}{Lemma}
	\newtheorem{corollary}{Corollary}
	\newtheorem{theorem}{Theorem}
	\newtheorem{proposition}{Proposition}
	\newtheorem{assumption}{Assumption}

	\newtheorem{conjecture}{Conjecture}		
	
	\theoremstyle{definition}
	\newtheorem{definition}{Definition}
\fi

\makeatletter
\def\thm@notefont{}
\renewcommand\thmhead@plain[3]{%
  \thmname{#1}\thmnumber{\@ifnotempty{#1}{ }\@upn{#2}}\thmnote{}}
\let\thmhead\thmhead@plain
\makeatother

\definecolor{green}{rgb}{0.0, 0.5, 0.0}

\definecolor{xcolor-gray}{gray}{0.95}

\def\be{\begin{equation}}
	\def\ee{\end{equation}}
\def\beas{\begin{eqnarray*}}
	\def\eeas{\end{eqnarray*}}
\def\bea{\begin{eqnarray}}
	\def\eea{\end{eqnarray}}

\newcommand{\CC}{{\mathcal C}}
\newcommand{\D}{{\mathcal D}}
\newcommand{\E}{{\mathcal E}}

\newcommand{\QQ}{{\mathcal Q}}
\newcommand{\RR}{{\mathcal R}}

\newcommand{\V}{{\mathcal V}}

\newcommand{\X}{{\mathcal X}}
\newcommand{\Y}{{\mathcal Y}}

\renewcommand{\L}{\mathcal{L}}

\newcommand{\EE}{\mathop{\mathbb E}} %
\newcommand{\R}{{\mathbb R}}
\newcommand{\N}{{\mathbb N}}

\newcommand{\TV}{\mathrm{TV}}

\DeclareMathOperator*{\Var}{Var}
\DeclareMathOperator*{\Cov}{Cov}

\newcommand{\TODO}[1]{\textcolor{red}{(TODO: #1)}}
\newcommand{\TOCITE}[1]{\textcolor{red}{(add citation here)}}

\newcommand{\vset}{\V} %
\newcommand{\edgevocab}{\E} %
\newcommand{\vocab}{\Sigma} %
\newcommand{\term}{\Sigma_{\text{term}}} %
\newcommand{\chaina}{\mathcal{C}_a} %
\newcommand{\chainb}{\mathcal{C}_b} %
\newcommand{\chainsize}{n} %
\newcommand{\vstart}{v_{\text{start}}} %
\newcommand{\dist}{\mathcal{D}} %

\newcommand{\transdepth}{d}
\newcommand{\reward}{r}
\newcommand{\AC}{\mathrm{AC}} %
\newcommand{\TC}{\mathrm{TC}} %
\newcommand{\WP}{\mathrm{WP}} %
\newcommand{\ORD}{\mathrm{ORD}} %
\newcommand{\TwoChain}{\mathrm{TwoChain}} %
\newcommand{\TwoChainEndEq}{\mathrm{TwoChainEndEq}} %
\newcommand{\fwdbound}{\zeta} %
\newcommand{\switchbound}{\upsilon} %

\DeclareFontFamily{U}{mathx}{\hyphenchar\font45}
\DeclareFontShape{U}{mathx}{m}{n}{<-> mathx10}{}
\DeclareSymbolFont{mathx}{U}{mathx}{m}{n}
\DeclareMathAccent{\widebar}{0}{mathx}{"73}

\definecolor{darkspringgreen}{rgb}{0.09, 0.45, 0.27}
\usetikzlibrary{decorations.pathreplacing,calligraphy}
\usetikzlibrary{patterns}

\ifdefined\ENABLEENDNOTES
	\renewcommand{\theendnote}{\alph{endnote}} 
\else
	\renewcommand{\endnote}[1]{\null} 
\fi

\makeatletter
\newcounter{appfncounter}
\let\appfncontents\relax

\newcommand{\appfn@raisedtarget}[1]{%
  \raisebox{\ht\strutbox}{\hypertarget{#1}{}}%
}

\newcommand{\appfn@render}[2]{%
  \par\noindent
  \appfn@raisedtarget{appfn:fwd:#1}%
  \hyperlink{appfn:back:#1}{\textsuperscript{#1}}~#2%
}

\newcommand{\appendixfootnote}[1]{%
  \stepcounter{appfncounter}%
  \appfn@raisedtarget{appfn:back:\theappfncounter}%
  \hyperlink{appfn:fwd:\theappfncounter}{\textsuperscript{\theappfncounter}}%
  \protected@xdef\appfncontents{%
    \ifx\appfncontents\relax\else\unexpanded\expandafter{\appfncontents}\fi
    \unexpanded{\appfn@render}{\theappfncounter}\unexpanded{{#1}}%
  }%
}

\newcommand{\printappendixfootnotes}{%
  \section{Footnotes}
  \label{app:footnotes}
  \ifx\appfncontents\relax\else\appfncontents\fi
}
\makeatother

\ifdefined\CVPR
	\usepackage[breaklinks=true,bookmarks=false]{hyperref}
	\ifdefined\CAMREADY
		\cvprfinalcopy 
	\fi

\fi

\ifdefined\ARXIV
	\newcommand*{\ABBR}{}
\fi
\ifdefined\ICML
	\newcommand*{\ABBR}{}
\fi
\ifdefined\NEURIPS
	\newcommand*{\ABBR}{}
\fi
\ifdefined\ICLR
	\newcommand*{\ABBR}{}
\fi
\ifdefined\COLT
	\newcommand*{\ABBR}{}
\fi
\ifdefined\ABBR
	
	\newcommand{\ie}{{\it i.e.}}

\fi

\newlength{\defaultabovedisplayskip}
\newlength{\defaultbelowdisplayskip}
\newlength{\defaultabovedisplayshortskip}
\newlength{\defaultbelowdisplayshortskip}
\AtBeginDocument{%
	\setlength{\defaultabovedisplayskip}{\abovedisplayskip}%
	\setlength{\defaultbelowdisplayskip}{\belowdisplayskip}%
	\setlength{\defaultabovedisplayshortskip}{\abovedisplayshortskip}%
	\setlength{\defaultbelowdisplayshortskip}{\belowdisplayshortskip}%
}

\hypersetup{hidelinks}

\begin{document}

	\ifdefined\ARXIV
		\maketitle
	\fi
	\ifdefined\NEURIPS
	\title{Outcome-Based RL Provably Leads Transformers to Reason, but Only With the Right Data}
		\newcommand{\samethanks}[1][\value{footnote}]{\footnotemark[#1]}
		\author{
			\naturalthanksmark
			Yuval Ran-Milo\thanks{Equal contribution.}\thanks{Correspondence to \texttt{yuvalmilo@mail.tau.ac.il}.} \\
			Tel Aviv University \\
			\And
			Yotam Alexander\samethanks[1] \\
			Tel Aviv University \\
			\And
			Shahar Mendel \\
			Tel Aviv University \\
			\And
			Nadav Cohen \\
			Tel Aviv University \\
		}
		\maketitle
	\fi
	\ifdefined\CVPR
		\title{Outcome-Based RL Provably Leads Transformers to Reason, but Only With the Right Data}
		\author{
			Author 1 \\
			Author 1 Institution \\	
			\texttt{author1@email} \\
			\and
			Author 2 \\
			Author 2 Institution \\
			\texttt{author2@email} \\	
			\and
			Author 3 \\
			Author 3 Institution \\
			\texttt{author3@email} \\
		}
		\maketitle
	\fi
	\ifdefined\AISTATS
		\twocolumn[
		\aistatstitle{Outcome-Based RL Provably Leads Transformers to Reason, but Only With the Right Data}
		\ifdefined\CAMREADY
			\aistatsauthor{Author 1 \And Author 2 \And Author 3}
			\aistatsaddress{Author 1 Institution \And Author 2 Institution \And Author 3 Institution}
		\else
			\aistatsauthor{Anonymous Author 1 \And Anonymous Author 2 \And Anonymous Author 3}
			\aistatsaddress{Unknown Institution 1 \And Unknown Institution 2 \And Unknown Institution 3}
		\fi
		]	
	\fi
	\ifdefined\ICML
		\icmltitlerunning{Outcome-Based RL Provably Leads Transformers to Reason, but Only With the Right Data}
		\twocolumn[
			\icmltitle{Outcome-Based RL Provably Leads Transformers \\ to Reason, but Only With the Right Data} 
			\icmlsetsymbol{equal}{*}
		\begin{icmlauthorlist}
			\icmlauthor{Yuval Ran-Milo}{equal,Tel Aviv University} %
			\icmlauthor{Yotam Alexander}{equal,Tel Aviv University}
			\icmlauthor{Shahar Mendel}{Tel Aviv University}
			\icmlauthor{Nadav Cohen}{Tel Aviv University}
		\end{icmlauthorlist}
		\icmlaffiliation{Tel Aviv University}{Tel Aviv University}
		\icmlcorrespondingauthor{Yuval Ran-Milo}{yuvalmilo@mail.tau.ac.il}
		\icmlkeywords{}
		\vskip 0.3in
		]
		\printAffiliationsAndNotice{\icmlEqualContribution} %
	\fi
	\ifdefined\ICLR
		\title{Outcome-Based RL Provably Leads Transformers to Reason, but Only With the Right Data}
		\author{
			Author 1 \\
			Author 1 Institution \\
			\texttt{author1@email}
			\And
			Author 2 \\
			Author 2 Institution \\
			\texttt{author2@email}
			\And
			Author 3 \\ 
			Author 3 Institution \\
			\texttt{author3@email}
		}
		\maketitle
	\fi
	\ifdefined\COLT
		\title{Outcome-Based RL Provably Leads Transformers to Reason, but Only With the Right Data}
		\coltauthor{
			\Name{Author 1} \Email{author1@email} \\
			\addr Author 1 Institution
			\And
			\Name{Author 2} \Email{author2@email} \\
			\addr Author 2 Institution
			\And
			\Name{Author 3} \Email{author3@email} \\
			\addr Author 3 Institution}
		\maketitle
	\fi

	\begin{abstract}
    Transformers trained via Reinforcement Learning (RL) with outcome-based supervision can spontaneously develop the ability to generate intermediate reasoning steps (Chain-of-Thought).
    Yet the mechanism by which sparse rewards drive policy gradient to discover such systematic reasoning remains poorly understood.
    We address this by analyzing the policy gradient dynamics of single-layer linear Transformers on a synthetic graph traversal task that cannot be solved without Chain-of-Thought but admits a simple iterative solution.
    We prove that despite training solely on final-answer correctness, policy gradient drives the Transformer to converge to a structured, interpretable algorithm that iteratively traverses the graph vertex-by-vertex.
    We characterize the distributional properties required for this emergence, identifying the critical role of ``simple examples'': instances requiring fewer reasoning steps.
    When the training distribution places sufficient mass on these simpler examples, the Transformer learns a generalizable traversal strategy that extrapolates to longer chains; when this mass vanishes, policy gradient learning becomes infeasible.
    We corroborate our theoretical results through experiments on synthetic data and with real-world language models on mathematical reasoning tasks, validating that our theoretical findings carry over to practical settings.
\end{abstract}

	\ifdefined\COLT
		\medskip
	\fi

	\vspace{-1em}
\section{Introduction}
\vspace{-1mm}
\label{sec:intro}

The Transformer architecture~\cite{vaswani2023attentionneed} has become the de facto standard for sequence modeling, powering advances in natural language processing, code generation, protein structure prediction, and mathematical problem solving~\cite{chen2021evaluatinglargelanguagemodels,lewkowycz2022solvingquantitativereasoningproblems,kovalevskiy2024alphafoldyearsonvalidation}. While next-token pretraining on large corpora remains highly effective, it has been shown that reinforcement learning (RL) can further improve Transformer reasoning on tasks such as mathematics and code generation~\cite{shao2024deepseekmathpushinglimitsmathematical,dou2024stepcoderimprovecodegeneration}. In particular, models such as DeepSeek-R1 demonstrate that optimizing only for final-answer correctness with algorithms such as GRPO~\cite{shao2024deepseekmathpushinglimitsmathematical} or PPO~\cite{schulman2017proximalpolicyoptimizationalgorithms} can induce detailed Chain-of-Thought (CoT) reasoning traces~\cite{deepseekai2025deepseekr1incentivizingreasoningcapability}. Yet the fundamental question of how sparse, outcome-based feedback at the end of generation guides policy gradient toward a systematic reasoning algorithm remains open. The need for such an explanation is sharpened by empirical evidence that reasoning emerges only under certain training distributions, often requiring examples of varying complexity~\cite{narvekar2020curriculum,parashar2025curriculumreinforcementlearningeasy}.

We study this question through graph traversal, a common setting for Transformer reasoning~\cite{sanford2024understandingtransformerreasoningcapabilities,agrawal2024llmsperformstructuredgraph,spies2025transformersusecausalworld,bachmann2024pitfallsnexttokenprediction,sanford2024transformersparallelcomputationlogarithmic,mirtaheri2025letmethink}.
In our task, the model must identify the terminal vertex of a chain in a directed graph.
Under standard complexity-theoretic assumptions, this task requires iterative reasoning. Furthermore, when the model is allowed to generate intermediate reasoning steps, it admits multiple solutions, some of which are more efficient than others.
Our analysis shows that, despite training solely on final-answer correctness, policy gradient can amplify a minimal level of competence at initialization, driving the Transformer to explicitly traverse the graph vertex-by-vertex.
This occurs even though the model can implement alternative, less efficient solutions, revealing an \emph{implicit bias} in the learning dynamics towards efficient reasoning algorithms.

We further characterize the distributional properties required for the above implicit bias, identifying a critical role of ``simple examples''—instances requiring few reasoning steps—in guiding the optimization.
We prove that when the training data distribution has sufficient mass on these simpler examples, policy gradient learns the efficient traversal algorithm; conversely, when this mass vanishes, policy gradient learning becomes infeasible as the chain length grows. We support our theoretical results with experiments in both synthetic and real-world settings, including RL post-training of Qwen-based models~\cite{qwen2025qwen25technicalreport} on mathematical reasoning tasks.

Taken together, these findings suggest that deliberately introducing distribution shift during post-training can change the implicit bias of policy gradient and improve generalization: by training on simple examples, policy gradient can be steered toward efficient reasoning algorithms, despite the existence of less efficient solutions.
Large-scale empirical results in the literature are directionally consistent with this mechanism~\cite{sun2024easytohard,hase2024unreasonable,motwani2025h1,kim2026longhorizon}, and extending our simplified theory to more realistic settings, such as modern RL with verifiable rewards (RLVR) systems and large-scale reasoning benchmarks, is a direction we hope to spur.

\subsection{Paper Organization}
The remainder of the paper is organized as follows.
\Cref{sec:setup,sec:task} introduce the Transformer model and chain identification task.
\Cref{sec:how_learn} shows that multi-step reasoning is necessary and sufficient for the task, and that multiple algorithms can achieve high reward, with some being more efficient than others.
\Cref{sec:dyn} analyzes policy gradient dynamics.
\Cref{sec:exp} reports experiments, \cref{sec:limit} discusses the limitations of our work, and \cref{sec:conclusion} concludes.
Related work is discussed in \cref{sec:related}.

	\vspace{-1.5mm}
\section{Preliminaries}
\vspace{-2.5mm}
\label{sec:setup}
\subsection{Notation}
\label{sec:setup:not}
\vspace{-2.5mm}

Let $\mathbbm{1}$ denote the indicator function and \([n] \coloneqq \{1,\ldots,n\}\).
For a sequence of tokens \( s = (s_1, \ldots, s_L)\), and for any \(1 \leq j \leq l \leq L\) we use the shorthand \(s_{j:l} \coloneqq (s_j, \ldots, s_l)\). Let \(\vocab\) denote a finite vocabulary. 
Without loss of generality, we identify \(\vocab\) with \([|\vocab|]\), allowing us to 
index vocabulary tokens by integers. For any 
matrix \(M \in \R^{|\vocab| \times |\vocab|}\) (or 
\(M \in \R^{|\vocab| \times d}\) for some \(d \in \N\)) and tokens \(s, 
s' \in \vocab\), we write \(M_{s,s'}\) (or \(M_{s,
:}\)) for the entry (or row) of \(M\) indexed by 
the tokens.

\vspace{-4pt}
\subsection{Transformer Architecture}
\vspace{-2mm}
\label{sec:setup:trans}
Let \(\vocab\) denote a finite vocabulary.
We consider a Transformer architecture with \(\transdepth\) layers that maps sequences \(s_{1:T} = (s_1,\ldots,s_T) \in \vocab^T\) of arbitrary length \(T\) to a probability distribution over the vocabulary \(\vocab\).

The Transformer first encodes the input sequence using one-hot embeddings: for each token \(s_l \in \vocab\), we define \(x_l = e_{s_l} \in \R^{|\vocab| \times 1}\) to be the one-hot vector corresponding to \(s_l\).
We denote the resulting embedding matrix by \(H^{(0)} = [x_1, \ldots, x_T]^\top \in \R^{T \times |\vocab|}\).

The Transformer then applies \(\transdepth\) attention layers sequentially.
We parameterize the Transformer by \(\theta = \{(A_l, V_l)\}_{l=1}^{\transdepth}\), where each layer \(l\in[d]\) has an \emph{attention matrix} \(A_l \in \R^{|\vocab| \times |\vocab|}\) and a \emph{value matrix} \(V_l \in \R^{|\vocab| \times |\vocab|}\). 
In the single-layer case (\(\transdepth=1\)), we simplify notation by writing \(A = A_1\) and \(V = V_1\). At each layer \(l \in [\transdepth]\), the hidden state, $H^{(l)}\in \R^{T \times |\vocab|}$, is computed using an activation function \(\sigma\) and a causal masking operator \(\text{MASK}_\sigma\) which ensures that position \(i\) attends only to positions \(j \le i\):
\begin{equation}
\label{eq:transformer_layer}
H^{(l)} = \sigma\Big( \text{MASK}_\sigma \big( H^{(l-1)} A_l (H^{(l-1)})^\top \big) \Big) H^{(l-1)} V_l
\text{\,.}
\end{equation}
We consider two instantiations. In the \emph{softmax Transformer}, \(\sigma = \mathrm{softmax}\) (applied row-wise) and \(\text{MASK}_\sigma(Z)_{ij} = Z_{ij}\) if \(j \le i\) and \(-\infty\) otherwise. In the \emph{linear Transformer}, \(\sigma\) is the identity function and \(\text{MASK}_\sigma(Z)_{ij} = Z_{ij}\) if \(j \le i\) and \(0\) otherwise. Unless otherwise specified, our statements apply to both variants.

Finally, the Transformer applies a softmax to the last position of the final hidden state to produce a probability distribution over the vocabulary:
\begin{equation}
\label{eq:transformer}
\mathrm{TF}(s_{1:T}; \theta) = \mathrm{softmax}\big(H^{(\transdepth)}_{T,:}\big) \in \R^{1 \times |\vocab|}
\text{\,.}
\end{equation}

\vspace{-4mm}
\subsection{Autoregressive Generation}
\vspace{-2mm}
\label{sec:setup:ar}
We use the Transformer architecture to generate sequences via autoregressive sampling.
Given an initial prompt \(s_{1:L_0}\), we define a set of \emph{terminal tokens} \(\term(s_{1:L_0}) \subseteq \vocab\).
The Transformer generates tokens \(s_{L_0+1}, s_{L_0+2}, \ldots\) sequentially.
At each step \(l \ge L_0\), the Transformer samples the next token \(s_{l+1} \sim \mathrm{TF}(s_{1:l}; \theta)\).
The Transformer continues sampling autoregressively until reaching a terminal token, \ie, until \(s_l \in \term(s_{1:L_0})\) for some \(l > L_0\).~\footnote{It can be easily shown that, in the setting under investigation in this paper, the probability of reaching a terminal token is one.}
We denote by \(p_{\text{roll}}(\theta)\) the distribution over complete trajectories \(\tau = (s_{1:L_0}, s_{L_0+1}, \ldots, s_{L_{\text{term}}})\) induced by the Transformer's parameters \(\theta\), where \(s_{L_{\text{term}}} \in \term(s_{1:L_0})\).
We denote by \(|\tau| = L_{\text{term}} - L_0\) the length of the generated trajectory (excluding the prompt).
Given two parameter settings \(\theta\) and \(\theta'\), we measure the difference between their induced rollout distributions using the \emph{total variation distance}:
\begin{equation}
\label{def:tv}
\mathrm{TV}\big(p_{\text{roll}}(\theta), p_{\text{roll}}(\theta')\big) := \frac{1}{2} \sum_{\tau} \big| p_{\text{roll}}(\theta)[\tau] - p_{\text{roll}}(\theta')[\tau] \big|,
\end{equation}
where the sum is over all possible rollouts \(\tau\).
\vspace{-4mm}
\subsection{Reinforcement Learning}
\vspace{-2mm}
\label{sec:setup:rl}
We train the Transformer using reinforcement learning with a reward function over complete rollouts.
Given a trajectory \(\tau = (s_{1:L_0}, s_{L_0+1}, \ldots, s_{L_{\text{term}}})\), we denote a reward function \(\reward(\tau)\) which evaluates the quality of the trajectory.
For any distribution \(\dist\) over input sequences \(s_{1:L_0}\), the training objective is defined as:
\[
\RR(\dist, \theta) = \EE_{s_{1:L_0} \sim \dist} \EE_{\tau \sim p_{\text{roll}}(\theta)} \big[ \reward(\tau) \big].
\]
Our interest lies in \emph{outcome-based} reinforcement learning, where the reward \(\reward(\tau)\) depends on the trajectory \(\tau\) only through its final token \(s_{L_{\text{term}}}\).
	\vspace{-2mm}
\section{The Task: Chain Identification}
\vspace{-1mm}
\label{sec:task}
\begin{figure}[t]
      \centering
      \begin{tikzpicture}[
        vertex/.style={circle, draw, minimum size=7mm, inner sep=0pt, outer sep=0pt, fill=white, font=\small},
        start/.style={vertex, fill=orange!35},
        target/.style={vertex, fill=green!35},
        dots/.style={rectangle, draw=none, minimum height=7mm, minimum width=6mm, inner sep=0pt, outer sep=0pt, font=\large},
        arrow/.style={->, >=stealth, thick, shorten >=0.6pt, shorten <=0.6pt}
      ]
      
      \def\x0{0}
      \def\dx{1.1}
      \def\yTop{1.0}
      \def\yBot{0}
      
      \node[vertex] (a1)     at (\x0,\yTop)               {$a_1$};
      \node[vertex] (a2)     at (\x0+1*\dx,\yTop)         {$a_2$};
      \node[dots]   (adots1) at (\x0+2*\dx,\yTop)         {$\cdots$};
      \node[start]  (ak)     at (\x0+3*\dx,\yTop)         {$a_k$};
      \node[dots]   (adots2) at (\x0+4*\dx,\yTop)         {$\cdots$};
      \node[target] (am)     at (\x0+5*\dx,\yTop)         {$a_n$};
      
      \node[vertex] (b1)     at (\x0,\yBot)               {$b_1$};
      \node[vertex] (b2)     at (\x0+1*\dx,\yBot)         {$b_2$};
      \node[dots]   (bdots1) at (\x0+2*\dx,\yBot)         {$\cdots$};
      \node[vertex] (bk)     at (\x0+3*\dx,\yBot)         {$b_k$};
      \node[dots]   (bdots2) at (\x0+4*\dx,\yBot)         {$\cdots$};
      \node[vertex] (bm)     at (\x0+5*\dx,\yBot)         {$b_n$};
      
      \foreach \u/\v in {a1/a2,a2/adots1,adots1/ak,ak/adots2,adots2/am,
                         b1/b2,b2/bdots1,bdots1/bk,bk/bdots2,bdots2/bm}{
        \draw[arrow] (\u.east) -- (\v.west);
      }
      
      \draw[<-, >=stealth, thick,draw=orange!70!black] (ak.north) -- ++(0,3mm)
            node[above,font=\footnotesize] {Starting vertex};
      
      \draw[<-, >=stealth, thick, draw=green!60!black] (am.north) -- ++(0,3mm)
            node[above,font=\footnotesize] {Target};
      
      \end{tikzpicture}
      \caption{Illustration of the chain identification task. The input contains edges forming two disjoint chains \(\chaina = (a_1, \ldots, a_n)\) and \(\chainb = (b_1, \ldots, b_n)\), and a starting vertex \(\vstart\) (here \(a_k\)). The task is to identify the terminal vertex of the chain containing \(\vstart\) (here~\(a_n\)).}
      \label{fig:chain_task}
  \end{figure}
We study a graph traversal task: given a set of \emph{directed} edges forming two disjoint chains, and a starting vertex, the goal is to identify the terminal vertex of the chain containing the starting vertex (see \cref{fig:chain_task} for an illustration).
This task captures the essential structural property of reasoning tasks: the answer cannot be produced in a single step (as we show in \cref{sec:how_learn:no_reasoning}), but emerges naturally from a step-by-step iterative procedure.

Formally, let \(\chainsize \in \N_{\geq 2}\) denote the task's chain length. We define the \emph{vertex set} \(\vset = [2\chainsize]\) and the \emph{edge vocabulary} \(\edgevocab = \vset \times \vset\), where each ordered pair \((u,v)\) denotes a directed edge from \(u\) to \(v\). The full vocabulary is \(\vocab = \vset \cup \edgevocab\). We define a \emph{chain} for the task as a sequence of \(n\) distinct vertices \((v_1, \ldots, v_n)\) with \(v_i \in \vset\), which we identify with its directed edge set \(\{(v_1, v_2), (v_2, v_3), \ldots, (v_{n-1}, v_n)\}\).

\vspace{-1mm}
\subsection{Data Distribution}
\label{sec:setup:dist}
\vspace{-1mm}

We define a general family of distributions \(\D^{\QQ}\) over input sequences, parameterized by a distribution \(\QQ\) over the index set \([n-1]\). To sample \(s_{1:L_0} \sim \D^{\QQ}\),
we first partition \(\vset\) uniformly at random into two disjoint subsets of size \(n\) each,
and uniformly order the vertices within each subset to form chains \(\chaina = (a_1, \ldots, a_n)\) and \(\chainb = (b_1, \ldots, b_n)\).
 We then sample an index \(k \sim \QQ\), choose a chain \(\CC \in \{\chaina, \chainb\}\) uniformly at random, and set the starting vertex \(\vstart\) to be the \(k\)-th vertex of \(\CC\).  The input sequence \(s_{1:L_0}\) consists of an arbitrary ordering of all edges in \(\chaina\) and \(\chainb\), 
 followed by \(\vstart\).~\footnote{While the order of edges is arbitrary (as we use no positional encodings), \(\vstart\) must be the last token so that the Transformer can attend from it to the edges to predict the next vertex.}
For an instance of the problem, \ie, an input sequence \(s_{1:L_0}\), we define the set of terminal tokens to include the final vertices of the two chains (\(a_n\) and \(b_n\)). Formally $\term(s_{1:L_0}) = \{a_n, b_n\}$.

\vspace{-1mm}
\subsection{Task Reward}
\label{sec:setup:reward}
Given an input sequence \(s_{1:L_0}\) containing edges forming two chains \(\chaina = (a_1, \ldots, a_n)\) and \(\chainb = (b_1, \ldots, b_n)\) with starting vertex \(\vstart\), let \(\tau=(s_{1:L_0}, s_{L_0+1},\ldots,s_{L_{\text{term}}})\) denote a rollout terminated at the first \(s_{L_{\text{term}}}\in\term(s_{1:L_0})\). We define \(y^*(s_{1:L_0})\) as the terminal vertex of the chain containing \(\vstart\):
\[
y^*(s_{1:L_0}) := 
\begin{cases}
a_{n} & \text{if } \vstart \in \chaina, \\
b_{n} & \text{if } \vstart \in \chainb.
\end{cases}
\]
We use the reward
\(
\reward(\tau) := \mathbbm{1}\{s_{L_{\text{term}}} = y^*(s_{1:L_0})\}.
\) 

\subsection{Step Types and Chain Traversal}
\vspace{-2mm}
\label{sec:setup:traversal}
To characterize the Transformer's behavior during generation, we define three types of atomic operations.  We say that the Transformer performs a \emph{forward step} at time \(t > L_0\) if the generated token \(s_t\) is the target of a forward edge from the previous token \(s_{t-1}\), i.e., \((s_{t-1}, s_t) \in \chaina \cup \chainb\).
A \emph{backward step} occurs when \(s_t\) is the source of an edge to \(s_{t-1}\), i.e., \((s_t, s_{t-1}) \in \chaina \cup \chainb\).
Finally, a \emph{switch step} occurs when \(s_t\) belongs to a different chain than \(s_{t-1}\), that is, when \(s_{t-1} \in \chaina\) and \(s_t \in \chainb\), or vice versa.

A rollout \(\tau = (s_{1:L_0}, s_{L_0+1}, \ldots, s_{L_{\text{term}}})\) is a \emph{chain traversal} if it consists of a consecutive sequence of forward steps starting from \(\vstart\).
Formally, we require that for all \(l \in \{L_0+1, \ldots, L_{\text{term}}\}\), the token \(s_l\) is a forward step from \(s_{l-1}\) (where \(s_{L_0} = \vstart\)).

	\section{Solutions to Chain Identification Task}
\label{sec:how_learn}

In this section, we show that fixed-depth Transformers cannot solve chain identification without multi-step reasoning, but single-layer Transformers can solve it by autoregressively generating such steps. We also show that high reward can be achieved by multiple algorithms, ranging from efficient chain traversal to inefficient long trajectories, setting up the question of which solution policy gradient selects.

\subsection{Reasoning is Necessary to Solve the Task}
\label{sec:how_learn:no_reasoning}

We first prove that multi-step reasoning is necessary: with any fixed depth Transformer, the task cannot be solved when the Transformer is required to output the answer in a single step only.

To formalize the result, we define the \emph{single-step Reward}, which measures performance when the Transformer must output the answer immediately:
\begin{equation}
\label{eq:single_step_reward}
\RR_{\text{single}}(\D^{\QQ}, \theta) := \EE_{s \sim \D^{\QQ}} \EE_{\hat{y} \sim \mathrm{TF}(s; \theta)} \big[ \mathbbm{1}\{\hat{y} = y^*(s)\} \big].
\end{equation}

\cref{result:no_single_step} shows that for Transformers of any fixed depth~\footnote{\label{fn:extensions}The same analysis readily extends to architectures incorporating normalization layers and MLP layers of width polynomial in the input length alongside attention. Moreover, \cref{result:no_single_step} extends to any constant number of chain-of-thought steps; see \cref{app:extensions_complexity} for both extensions.}, there exists a distribution $\D^{\QQ}$ where single-step generation achieves reward bounded away from 1 by a constant.
This contrasts sharply with the results in \cref{sec:how_learn:reasoning_sufficient}, where autoregressive generation can achieve reward arbitrarily close to 1 on all distributions.
This separation demonstrates the necessity of multi-step reasoning.

\begin{theorem}
\label{result:no_single_step}
Assume transformer weights are represented with logarithmic precision. Then under standard complexity-theoretic assumptions  (see \cref{app:complexity_assumption}), for any fixed depth $\transdepth \in \N$, there exist $n \in \N_{\geq 2}$, a distribution $\QQ$ over $[n-1]$, and $c \in \R_{> 0}$ such that for all parameters $\theta = \{(A_i, V_i)\}_{i=1}^d$:
\[
\RR_{\text{single}}(\D^{\QQ}, \theta) \leq 1 - c
\text{.}
\]
\end{theorem}

\begin{proof}[Proof sketch (full proof in \cref{app:no_single_step})]
Using results of Li et al.~\cite{li2024cot}, we first note that constant-depth Transformers without CoT generation can only compute functions in $TC^0$, the class of constant-depth circuits with unbounded fan-in threshold gates. We then relate our chain identification task to standard problems in circuit complexity by giving reductions from the word problem on $S_5$ (the symmetric group on five elements) to the $\ORD$ problem (path reachability in directed graphs), and from $\ORD$ to the chain identification task. The word problem on $S_5$ is $NC^1$-complete, where $NC^1$ is the class of Boolean functions computable by polynomial-size, logarithmic-depth circuits with bounded fan-in, so these reductions show that solving chain identification in $TC^0$ would imply that an $NC^1$-complete problem lies in $TC^0$. Under the widely believed conjecture that $TC^0 \neq NC^1$, this is impossible, and hence the chain identification task is not in $TC^0$.
\end{proof}

\subsection{Efficient and Inefficient Reasoning Algorithms Solve the Task}
\label{sec:how_learn:reasoning_sufficient}

We next show that multi-step reasoning is sufficient: single-layer Transformers can achieve reward arbitrarily close to 1 by autoregressively generating intermediate steps.
However, high reward does not identify a unique algorithm.
\cref{result:traverse_chain} gives an efficient construction that follows the chain directly, while \cref{result:inefficient_solution} shows that, because the reward penalizes only the final answer, arbitrarily long trajectories can also achieve high reward.

\begin{proposition}
\label{result:traverse_chain}
For any $n \in \N_{\geq 2}$, $\QQ$ a distribution over $[n-1]$ and $\varepsilon, \delta \in \R_{> 0}$, there exist parameters $\theta$ for single-layer Transformer such that $\RR(\D^{\QQ}, \theta) > 1 - \varepsilon$, and in addition, for every input sequence $s_{1:L_0}$ from $\D^{\QQ}$, with probability at least $1 - \varepsilon$ over $\tau \sim p_{\text{roll}}(\theta)$, the rollout $\tau$ is a chain traversal (\cref{sec:setup:traversal}).~\footnote{For simplicity, we present this result assuming infinite precision. However, one can show that under logarithmic precision, one can obtain reward strictly larger than the upper bound given by \cref{result:no_single_step}. For further details see \cref{app:traverse_chain}.}
\end{proposition}

\begin{proof}[Proof sketch (full proof in \cref{app:traverse_chain})]
    Assume by induction that the Transformer has correctly performed chain traversal up to a non-terminal vertex up to time $l$ producing a sequence $s_{1:l}$ where $l \geq 2n-1$. 
    The prefix $s_{1:2n-2}$ contains the edges, and the last token $s_l$ is a vertex.
    We construct the attention matrix $A$ such that for any vertex $u \in \vset$, the attention weight from $u$ to any outgoing edge $(u,v)$ is strictly larger than the weight to any other token.
    The value matrix $V$ is constructed such that $V_{(u,v), v} = \alpha$ for some $\alpha \in \R_{> 0}$ for all edges $(u,v)$, while all other entries are zero. 
    Since $s_l$ is in some chain and is a non-terminal vertex, there exists some edge $(s_l,k) \in s_{1:2n-2}$.
    This construction ensures that the Transformer attends to $(s_l,k)$ more strongly than to any other edge, and the output logits satisfy $\mathrm{TF}(s_{1:l})_k / \mathrm{TF}(s_{1:l})_v = \Omega(\alpha)$ for any $v \neq k$.
    Consequently, the probability of outputting the next vertex in the chain, $k$, becomes arbitrarily close to one for sufficiently large $\alpha$, completing the induction.
\end{proof}

\begin{proposition}
\label{result:inefficient_solution}
For any $n \in \N_{\geq 3}$, $\QQ$ a distribution over $[n-1]$, $k\in \N$, and $\varepsilon, \delta \in \R_{> 0}$, there exists parameters $\theta$ for a single-layer linear Transformer such that $\RR(\D^{\QQ}, \theta) > 1 - \varepsilon$, and for every input sequence $s_{1:L_0}$ from $\D^{\QQ}$, the probability that $|\tau|>k$ is greater than $1-\delta$.
\end{proposition}

\begin{proof}[Proof sketch (full proof in \cref{app:inefficient_solution})]
We construct parameters similar to those in the proof of \cref{result:traverse_chain}.
However, instead of configuring the transformer's parameters to always move forward in the chain, we set the parameters such that from any vertex, the Transformer moves backward to its predecessor with high probability, forward to its successor with small probability, and to the other chain with negligible probability.
This creates a random walk confined (with high probability) to the correct chain that frequently backtracks, producing arbitrarily long trajectories before eventually reaching the terminal vertex.
\end{proof}

	\section{Dynamical Analysis}
\label{sec:dyn}
In this section, we analyze policy gradient dynamics for linear Transformers trained on $\RR(\D^{\QQ}, \theta)$ (\cref{sec:task}), showing that success depends sharply on the training distribution.
With sufficient mass on simple examples, policy gradient learns efficient chain traversal in polynomial time and the learned behavior generalizes to any test distribution, including unseen complex examples (\cref{result:dyn_learn}).
Since algorithms far less efficient than chain traversal also solve our chain identification task (\cref{sec:how_learn:reasoning_sufficient}), this demonstrates an \emph{implicit bias} toward efficiency.~\footnote{While the constructions of \cref{result:traverse_chain,result:inefficient_solution} do not conform to the settings of \cref{sec:dyn:setting}, one can implement very similar constructions which do. For further details see \cref{app:dyn_constructions}.}
Without simple examples, learning is intractable: even after exponential time, the rollout distribution remains essentially unchanged from initialization (\cref{result:dyn_vanishing}).

\subsection{Data Distribution}
\label{sec:dyn:data_distribution}

We now describe the property of the data distribution that controls the dynamics.
We characterize a data distribution based on the number of reasoning steps required to solve its examples.
Recall from \cref{sec:setup:dist} that each example includes as starting vertex the $k$-th vertex of a chain, for some $k \in [n-1]$.
Examples starting near the end of the chain (large $k$) require fewer autoregressive steps (\cref{sec:setup:traversal}), thus are regarded as simpler.
We accordingly define a measure of simplicity for the data distribution.

\begin{definition}
\label{def:tail_mass}\emph{(Simplicity Mass).}
Let $\D^{\QQ}$ be a data distribution parameterized by a distribution $\QQ$ over the index set $[n-1]$ (\cref{sec:setup:dist}).
For any $s \in [n-1]$ and $c \in [0,1]$, we say that \emph{$\D^{\QQ}$ has $s$-simplicity mass~$c$} if $\sum_{i=n-s}^{n-1} \QQ(i) = c$, i.e., if there is probability~$c$ that the starting vertex is within distance~$s$ from the end of its chain.
\end{definition}

\subsection{Technical Assumptions}
\label{sec:dyn:setting}

We analyze policy gradient on the objective $\RR(\D^{\QQ}, \theta)$ under the technical assumptions described below.
We make most of these assumptions for analytical tractability: the experiments in \cref{sec:exp:theo} show that the main phenomena persist when they are relaxed.

Motivated by \cref{sec:how_learn:reasoning_sufficient}, which shows that single-layer Transformers can solve the task through autoregressive reasoning, we specialize the dynamics to such an architecture. In addition, we assume that the model outputs only vertices and never repeats the previous vertex. This condition can be enforced via output masking.
\begin{assumption}
\label{ass:architecture}
\emph{(Architecture).}
Throughout this section, we focus on single-layer linear Transformers, \ie, \(\transdepth=1\) and \(\sigma\) is the identity in \cref{eq:transformer_layer}.
Thus, for \(\theta=(A,V)\) and any prefix \(s_{1:T}\), \cref{eq:transformer_layer} becomes
\[
H^{(1)}
=
\mathrm{MASK}_{\sigma}\!\left(H^{(0)} A (H^{(0)})^\top\right) H^{(0)}V,
\qquad
\mathrm{TF}(s_{1:T};\theta)
=
\mathrm{softmax}\!\left(H^{(1)}_{T,:}\right).
\]

\end{assumption}

As is common practice in Reinforcement Learning for Transformers, we consider policy gradient initialized from a base model that has already acquired a minimal level of task proficiency during pre-training \cite{deepseekai2025deepseekr1incentivizingreasoningcapability,ouyang2022traininglanguagemodelsfollow}.
Concretely, we assume a weak initial preference for following the chain direction rather than reversing it, and also require non-negligible probability of switching chains so that the model is not initialized near either the efficient or inefficient solutions constructed in \cref{sec:how_learn:reasoning_sufficient}.
\begin{assumption}
\label{ass:base_model}
\emph{(Pretraining).}
For any input sequence $s_{1:L_0}$, let \(p_{\text{fwd}}\), \(p_{\text{bwd}}\), and \(p_{\text{switch}}\) denote the probabilities that the Transformer performs a forward, backward, or switch step, respectively, at the first step (time $t=L_0+1$).
There exist constants \(\fwdbound,\switchbound \in \R_{>0}\) such that, for any input sequence $s_{1:L_0}$,
\[
p_{\text{fwd}} > p_{\text{bwd}}+\fwdbound
\qquad\text{and}\qquad
p_{\text{switch}} > \switchbound .
\]
\end{assumption}

We next specify the idealized optimization dynamics used in our analysis. We use gradient flow, the continuous-time limit of gradient-based optimization, which is known to closely approximate policy gradient with moderately small step sizes \cite{elkabetz2021continuousvsdiscreteoptimization}. Following prior theoretical work on Transformers \cite{nichani2024transformerslearncausalstructure,wang2025learningcompositionalfunctionstransformers}, we also train only the attention matrix~$A$ while keeping the value matrix~$V$ fixed at initialization.
\begin{assumption}
\label{ass:training_algorithm}
\emph{(Gradient flow).}
Training follows continuous-time gradient flow with exact expected gradients. Only the attention matrix~$A$ is trained, while the value matrix~$V$ is kept fixed at its initialization:
\begin{equation}
\frac{dA(t)}{dt} = \nabla_A \RR(\D, A(t)).
\end{equation}
\end{assumption}

It remains to specify the initial parameterization of the attention and value matrices. The attention initialization below captures the idea that the Transformer begins without any preference toward specific vertices, treating all of them symmetrically.
\begin{assumption}
\label{ass:initial_attention}
\emph{(Symmetric initialization).}
The attention matrix $A$ is initialized symmetrically.\footnote{More precisely, for any permutation $\pi$ of the vertex set~$\vset$, extended naturally to edge tokens via $\pi((u,w)) = (\pi(u), \pi(w))$, the matrix $A$ satisfies $A_{\pi(s),\pi(s')} = A_{s,s'}$ for all tokens $s,s' \in \vocab$.}
Equivalently, there exist constants $\alpha, \beta, \gamma, \eta, \delta \in \R$ with $\alpha,\gamma>0$ such that for any edge $(u, w) \in \edgevocab$ and vertices $v,w' \in \vset$,\footnote{Note that we do not specify the attention initialization between pairs of edge tokens; since the attending token is always a vertex, these values do not affect the dynamics.} $A_{v,(u,w)}=\alpha$ if $v=u$, $A_{v,(u,w)}=\beta$ if $v=w$, $A_{v,(u,w)}=\gamma$ if $v \notin \{u,w\}$, $A_{v,w'}=\eta$ if $v=w'$, and $A_{v,w'}=\delta$ if $v \neq w'$.
\end{assumption}

Finally, we fix a value matrix that encodes the endpoints of each edge token, analogous to the one-hot value initializations used in prior work~\cite{nichani2024transformerslearncausalstructure,wang2025learningcompositionalfunctionstransformers,huang2025transformersprovablylearnchainofthought}.
\begin{assumption}
\label{ass:value_matrix}
\emph{(Sparse Value).}
We set the value vector of each edge token $(u,w)$ to encode its defining vertices: $V_{(u,w),v} = \mathbbm{1}\{v \in \{u,w\}\}$. Vertex tokens have zero value vectors: $V_{v,:} = 0$ for all $v \in \vset$.
\end{assumption}
Output masking, the fixed value matrix (\cref{ass:training_algorithm} and \cref{ass:value_matrix}), gradient flow (\cref{ass:training_algorithm}), and symmetric initialization (\cref{ass:initial_attention}) are all empirically shown to be inessential (\cref{sec:exp:theo}), and removing them from the theory is an important direction for future work (see \cref{sec:limit}). The switch probability bound (\cref{ass:base_model}) is necessary: without it, the initialization may already sit near a solution, and the question of which algorithm policy gradient selects is vacuous.

\subsection{With Simple Examples Efficient Reasoning is Learned}
\label{sec:dyn:traverse_chain}

\cref{result:dyn_learn} below essentially establishes that, when the training data distribution has a simplicity mass greater than zero for some \emph{constant} distance~$s$ (\cref{def:tail_mass}), the Transformer learns, in time polynomial in the chain length~$n$,~\footnote{\label{fn:discretization}The bound applies to the time index of gradient flow; by discretizing with a fixed step size, this translates to a polynomial number of policy gradient iterations \cite{elkabetz2021continuousvsdiscreteoptimization}. See \cref{app:discretization} for further details.}
to perform \emph{forward steps} (\cref{sec:setup:traversal}) with arbitrarily high probability.
The theorem additionally requires the same mass at distance~$\lceil an\rceil$, which confines the support of training examples to two bands: sufficiently easy and sufficiently hard. This condition is likely an artifact of our proof technique: empirically, learning also succeeds when the training distribution includes examples of intermediate difficulty (\cref{sec:exp,sec:exp:theo:main}).

\begin{theorem}
\label{result:dyn_learn}
Let $s \in [n-1]$, $a \in ( 0 , 1 )$ and $c \in ( 0 , 1 ]$.
Suppose the $s$-simplicity mass and $\lceil a n \rceil$-simplicity mass of the training data distribution equal~$c$ (\cref{def:tail_mass}).
Then, for any $\varepsilon > 0$, there exists $t = O ( n^2 c^{-1} \varepsilon^{- 2} )$~\footnote{The $O$-notation hides constants that depend on $s$, $a$, $\fwdbound$ and~$\switchbound$.}
such that the linear Transformer parameters at time~$t$ of optimization, i.e., $\theta ( t ) := ( A ( t ) , V )$, satisfy the following:
for any input sequence $s_{1:L_0}$ drawn from any test data distribution (\cref{sec:setup:dist}), at each step $l \geq L_0$ of the autoregressive generation $s_{l + 1} \sim \mathrm{TF} ( s_{1 : l} ; \theta ( t ) )$ (\cref{sec:setup:ar}), the probability of~$s_{l + 1}$ being a forward step (\cref{sec:setup:traversal}) is at least~$1-\varepsilon$.~\footnote{When $c = 1$, we can further show that at time $t = O(n^2 c^{-1} \varepsilon^{- 2})$, for rollouts initiated from a starting vertex drawn from the training distribution $\D^{\QQ}$, the probability of the entire rollout being a chain traversal is at least $1 - \varepsilon$.}
\end{theorem}

\begin{proof}[Proof sketch (full proof in \cref{app:proof_dyn_learn})]
    By permutation symmetry of the initialization and data distribution, the rollout generated by \(\theta(t)\) can be analyzed as a time-homogeneous Markov chain on a canonical two-chain graph.
    Its behavior is determined by three transition probabilities: forward, backward, and chain-switch steps, parameterized by \(\alpha,\beta,\gamma\).
    We decompose the same-side absorption probability \(S(v)\), \ie, the probability that the rollout terminates at the end of the chain it started on, according to whether the rollout makes a \emph{long jump} (LJ), a transition that changes the distance to the terminal vertex by at least \(2\):
    \[
    S(v) \;=\; (1-h(v))\,\bar S^{NL}(v) \;+\; \tfrac12\,h(v),
    \]
    where \(h(v)\) is the probability of at least one LJ and \(\bar S^{NL}(v)\) is the same-side absorption probability conditional on no LJ.
    The factor \(1/2\) follows because, after an LJ occurs, the symmetry and structure of the two-chain Markov chain make absorption on either side equally likely.
    We then use a source decomposition for derivatives of Markov-chain absorption probabilities, expressing each parameter derivative as the expected sum of local one-step sensitivities along the rollout.
    Applying this representation to the LJ-absorbing and no-LJ chains shows that gradient flow decreases \(p_{\mathrm{bwd}}(t)\) and \(p_{\mathrm{switch}}(t)\), while increasing \(p_{\mathrm{fwd}}(t)\) at a sufficient rate.
    Choosing \(t=O(n^2c^{-1}\varepsilon^{-2})\) therefore ensures that each rollout step is forward with probability at least \(1-\varepsilon\).
\end{proof}

\subsection{Without Simple Examples Learning is Intractable}
\label{sec:dyn:vanishing_gradients}

\cref{result:dyn_vanishing} shows that if the training distribution has zero \(\lceil an\rceil\)-simplicity mass (\cref{def:tail_mass}), so all examples require at least \(\lceil an\rceil\) autoregressive steps, learning is intractable: even up to exponential training time,~\footnote{The bound applies to the time index of gradient flow; by discretizing with a fixed step size, this translates to an exponential number of policy gradient iterations \cite{elkabetz2021continuousvsdiscreteoptimization}.}
the Transformer's rollout distribution remains essentially unchanged from initialization.
Together with \cref{result:dyn_learn}, this implies that simple examples are both sufficient and necessary for learning efficient reasoning.
\begin{theorem}
\label{result:dyn_vanishing}
Let $a \in ( 0 , 1 )$ and $\varepsilon \in \R_{> 0}$.
Suppose the $\lceil a n \rceil$-simplicity mass of the training data distribution equals zero (\cref{def:tail_mass}).
For any $t \in \R_{\geq 0}$, denote by~$\theta ( t )$ the linear Transformer parameters at time~$t$ of optimization, i.e., $\theta ( t ) := ( A ( t ) , V )$.
Then there exists $\kappa = \kappa(a,\switchbound,\fwdbound) > 0$, independent of~$n$, such that for every $t < \varepsilon^{2} \cdot e^{\kappa n}$, it holds that with any test data distribution (\cref{sec:setup:dist}):
\begin{equation}
    \label{eq:tv_rollout_close}
    \mathrm{TV}\big(p_{\text{roll}}(\theta(t)), p_{\text{roll}}(\theta(0))\big) <\varepsilon ,
\end{equation}
i.e., at time~$t$ of optimization, the distribution over trajectories autoregressively generated by the Transformer (\cref{sec:setup:ar}) is $\varepsilon$-close (in total variation distance; \cref{def:tv}) to what it is at initialization.
\end{theorem}

\begin{proof}[Proof sketch (full proof in \cref{app:dyn_vanishing})]
We proceed as in the proof sketch of \cref{result:dyn_learn} by viewing the Transformer's rollouts as a Markov chain and decomposing the same-side absorption probability via the LJ decomposition above.
We show, again using the source decomposition, that when the training distribution has $\lceil an \rceil$-simplicity mass zero, these derivatives are exponentially small in \(n\). This implies that the one-step kernels (and hence the induced rollout distributions) stay \(\varepsilon\)-close in total variation distance up to the stated time scale.
\end{proof}

	\begin{table}[t]
    \centering
    \caption{Test accuracy on \emph{Hard-$n$} under various post-training distributions, in two real-world post-training experiments on a mathematical reasoning task (\cref{sec:exp:setting}).
    Left: \emph{real-world pre-training} using Qwen2.5-3B-Instruct and $n=15$.
    Right: \emph{synthetic pre-training} using a model pre-trained on chain-of-thought traces of simpler instances~\citep{zhang2025interplaypretrainingmidtrainingrl} and $n=20$.
    Full details in \cref{app:exp_real_details}.}
    \label{tab:exp_results}
    \small
    \setlength{\tabcolsep}{6pt}%
    \begin{tabular}{@{} lc !{\vrule width 0.5pt} lc @{}}
    \toprule
    \multicolumn{2}{c}{\textbf{Real-World Pre-Training}} & \multicolumn{2}{c}{\textbf{Synthetic Pre-Training}} \\
    \cmidrule(r{4pt}){1-2} \cmidrule(l{4pt}){3-4}
    Training Dist. & Test Acc. (\emph{Hard-$15$}) & Training Dist. & Test Acc. (\emph{Hard-$20$}) \\
    \midrule
    \emph{Hard-$15$}  & 52.9\% & \emph{Hard-$20$}  &  43.15\% \\
    \emph{Mixed-$15$} & 96.2\% & \emph{Mixed-$20$} &  84.77\% \\
    \emph{Mixed-$10$} & 91.6\% & \emph{Mixed-$17$} &  78.68\% \\
    \bottomrule
    \end{tabular}
\end{table}
    
\section{Experiments}
\label{sec:exp}

In this section, we empirically validate our theoretical predictions in two settings, building on the framework of \citet{zhang2025interplaypretrainingmidtrainingrl}.
We extend the chain identification task analyzed in our theory (\cref{sec:task}) to a mathematical reasoning task in which variables can depend on, and affect, multiple other variables.
Our experiments confirm two key predictions of our theory:
\emph{(i)}~simple examples suffice for learning to reason (\cref{sec:dyn:traverse_chain});
and \emph{(ii)}~without simple examples, learning fails (\cref{sec:dyn:vanishing_gradients}).

Synthetic experiments with single-layer Transformers on the specific chain identification task, which more closely match our theoretical setting, are presented in \cref{sec:exp:theo}.~\footnote{Code is available at \url{https://github.com/shaharmendel/RLProvablyLeadsToReason}.}

\subsection{Setting}

\label{sec:exp:setting}

We adopt a framework similar to that of \citet{zhang2025interplaypretrainingmidtrainingrl}.
Each example is presented as a system of arithmetic equations defining a set of variables in terms of others, together with one or more constant assignments and a query for a specified target variable.
Unlike the chain identification task analyzed in our theory (\cref{sec:task}), each variable can depend on, and affect, multiple other variables.
The model is post-trained from a base model using a method inspired by
DeepSeek-R1 \cite{deepseekai2025deepseekr1incentivizingreasoningcapability}, with reward depending only on whether the final answer matches the ground truth (and adherence to a specified output format), with no supervision on intermediate reasoning steps.
We characterize the difficulty of an example by a notion of \emph{complexity} corresponding to the extent of reasoning required to compute its answer (defined precisely per setting in \cref{app:exp_real_details}).
For some complexity~$n$, we let \emph{Hard-$n$} denote a training distribution containing only examples of complexity~$n$, and \emph{Mixed-$n$} denote a training distribution that uniformly mixes examples of complexity~$1\leq i\leq n$.

We conduct two experiments sharing the same task and training procedure, differing primarily in the base model used for RL. We refer to these as the \emph{synthetic pre-training} and \emph{real-world pre-training} settings.
In the \emph{synthetic pre-training} setting, following \citet{zhang2025interplaypretrainingmidtrainingrl}, the base model is a Qwen2.5-style Transformer pre-trained on chain-of-thought traces of our specific mathematical reasoning task, with pre-training restricted to examples simpler than those encountered during RL and evaluation.\footnote{While our experiments build on the framework of \citet{zhang2025interplaypretrainingmidtrainingrl}, our setup isolates the role of simple examples: we show that training directly on the target distribution performs worse than training on a distribution shifted towards simpler examples. In contrast, their training distribution is always either harder or easier than their test distribution.}
In the \emph{real-world pre-training} setting, the base model is Qwen2.5-3B-Instruct~\cite{qwen2025qwen25technicalreport}, a standard LLM with broad language pre-training but no prior exposure to our mathematical task.\footnote{Equivalent experiments with Llama-3.2-3B-Instruct~\cite{grattafiori2024llama3herdmodels} showing the same qualitative results are provided in \cref{app:exp_real_details}.}
This second setting is more reflective of standard LLM post-training pipelines, where RL is applied directly to a general-purpose base model with no task-specific reasoning priors.
Full experimental details are provided in \cref{app:exp_real_details}.

\subsection{Results}

\label{sec:exp:results}

\Cref{tab:exp_results} reports results across both settings.
Consistent with our theory, including simple examples in the post-training distribution leads to strong generalization on \emph{Hard-$n$}, even when the training distribution never includes complexity-$n$ examples (\emph{Mixed-$k$}, $k < n$).\footnote{To isolate the effect of simple examples from increased data diversity, we ran an ablation study in which we held the number of distinct complexity levels, the mass on the hardest level, and the training budget fixed and varied whether the post-training distribution includes simple examples. Including simple examples substantially improves accuracy on hard instances. See \cref{app:exp_simplicity_diversity} for details.} On the other hand, excluding simple examples causes generalization to fail entirely.\footnote{The failure of \emph{Hard-$n$} training is not due to insufficient training: the corresponding loss curves show that the training loss plateaus (\cref{app:exp_real_details}).}
Together, these results support the key conclusion of our theory: a distribution shift toward simpler examples is beneficial for learning to generalize to complex tasks.\footnote{We note that training on only very simple examples without any complex ones, also fails to generalize to \emph{Hard-$n$}. Characterizing the minimal post-training distribution sufficient for generalization is an interesting direction for future work.}
In terms of our theory, \emph{Hard-$n$} training corresponds to the regime of \cref{result:dyn_vanishing}, where learning is intractable, whereas \emph{Mixed-$n$} and \emph{Mixed-$k$} correspond to the regime of \cref{result:dyn_learn}, where the learned algorithm generalizes to unseen complex examples.

To test whether RL induces an implicit bias toward efficient reasoning, we also evaluate the real-world pre-training setting on a restricted version of the task in which the relevant variables form a single dependency chain.\label{sec:exp:real:efficient}
A natural efficient strategy for this restricted task is to implement the following procedure: first trace this chain backward, and then compute values forward along it.\footnote{Calculating values without first identifying the dependency path seems far more difficult, as the model cannot know \emph{a priori} the order in which equations should be processed.}
Analysis via LLM annotation over 100 correct solutions generated by the trained model reveals that 98\% of them follow precisely this step-by-step algorithm (see \cref{app:llm_annotation} for details), confirming a strong implicit bias toward efficient, structured reasoning.

	\section{Limitations}
\label{sec:limit}

It is important to acknowledge several limitations of our work.
First, our theoretical analysis focuses on a synthetic chain identification task, and while this task captures a core structural property of reasoning, it is still far simpler than the open-ended reasoning tasks encountered by real-world LLMs.
Second, our theory is restricted to single-layer linear Transformers, leaving the extension to deeper, nonlinear architectures theoretically unexplored.
Third, the theoretically analyzed training regime assumes continuous-time gradient flow with exact expected gradients, a symmetric initialization with a weak forward preference, output masking, and a fixed value matrix.
Fourth, our positive result requires the training distribution to contain sufficient mass on simple examples, and in its current form also imposes an additional simplicity-mass condition (\cref{result:dyn_learn}); we believe the latter may be removable.
Fifth, our negative result establishes intractability for training distributions with no simple examples within our stylized distribution family, rather than characterizing all possible data distributions or curricula.
Sixth, our real-world experiments are limited to a particular mathematical reasoning task, model family, and fine-tuning setup, and therefore should be viewed as evidence for the relevance of our mechanism rather than a comprehensive empirical characterization of outcome-based RL in LLMs.

Our experiments go beyond most of the theoretical limitations above.
Indeed, the appendix experiments in \cref{sec:exp:theo} demonstrate that the main phenomena persist when both attention and value matrices are trained jointly with sample-based policy gradient from a small Gaussian initialization, without the simplifying assumptions of \cref{sec:dyn:setting}.
\cref{sec:exp:real} further demonstrates analogous behavior in a modern LLM fine-tuned on a natural-language mathematical reasoning task.
Across these settings, Transformers trained with sufficiently simple examples learn efficient step-by-step reasoning algorithms, generalize to more complex examples, and fail to do so when simple examples are absent.
Extending our theoretical analysis to match the generality of these empirical demonstrations is an important direction for future work.

\section{Conclusion}
\label{sec:conclusion}

In this paper, we established a theoretical footing for the emergence of CoT (Chain-of-Thought) reasoning in Transformers trained via outcome-based RL (Reinforcement Learning), and we corroborated our theory with experiments in synthetic and real-world settings. 
Our findings imply that, with the right data, policy gradient exhibits an implicit bias towards reasoning algorithms that are both efficient and capable of out-of-distribution generalization.
Specifically, when the training data includes a sufficient amount of simple examples, policy gradient bypasses inefficient solutions, and converges to an efficient reasoning algorithm that generalizes to examples of arbitrary complexity.
The condition on simple examples in the training data is necessary: without such examples, learning becomes intractable.

From a practical perspective, our findings suggest a counterintuitive message: when fine-tuning an LLM (Large Language Model), rather than training on the target distribution, it may be beneficial to deliberately introduce a distribution shift by including simple examples in training.
Prior work has identified settings where introducing a distribution shift in training benefits learning (see \cref{sec:related:data_distribution}).
However, to our knowledge, this paper is the first to establish said phenomenon for Transformers trained via outcome-based RL.
Moreover, it identifies a specific type of examples whose inclusion in training can lead to efficient reasoning and out-of-distribution generalization.
We hope that our work will serve as a foundation for developing practical post-training curricula for LLMs, thereby bolstering the efficiency and generalizability of their reasoning capabilities.

	\ifdefined\NEURIPS
		\ifdefined\CAMREADY
			\begin{ack}
				We thank Joey Rudoler, Noam Razin, and Yonatan Slutzky for illuminating discussions. This work was supported by the European Research Council (ERC) grant and NN4C 101164614, a
Google Research Scholar Award, a Google Research Gift, Meta, the Yandex Initiative in Machine
Learning, the Israel Science Foundation (ISF) grant 1780/21, the Tel Aviv University Center for
AI and Data Science, the Adelis Research Fund for Artificial Intelligence, Len Blavatnik and the
Blavatnik Family Foundation, and Amnon and Anat Shashua.

			\end{ack}
		\fi
	\fi

	\section*{References}
	{\small
		\ifdefined\ICML
			\bibliographystyle{icml2026}
		\else
			\bibliographystyle{plainnat}
		\fi
		\bibliography{refs}
	}

	\clearpage
	\appendix

	\crefalias{section}{appendix}
	\crefalias{subsection}{appendix}
	\crefalias{subsubsection}{appendix}

	\onecolumn

	\ifdefined\ENABLEENDNOTES
		\theendnotes
	\fi

	\section{Related Work}
\label{sec:related}

\subsection{Problems which provably require CoT}
A growing theoretical literature seeks to explain how CoT boosts Transformer performance by enhancing its expressivity. Several works show that CoT allows constant-depth Transformers to solve tasks such as parity, arithmetic, and regular languages that are inaccessible to single-step decoding~\cite{feng2023towards,merrill2023expressivelimits,li2024cot}. We contribute to this literature by establishing a natural graph traversal problem that provably requires CoT under standard complexity-theoretic assumptions.

\subsection{Training Dynamics of Transformers}
A parallel line of work analyzes how gradient-based training shapes Transformer behavior. Several studies track the emergence of algorithmic heads, programmable computation, or compositional rules~\cite{li2024one,hu2024looped,zhang2025finitestateautomatainside,kim2024transformers,wang2025learningcompositionalfunctionstransformers}, while other results show that gradient descent leads Transformers to implement in-context learning algorithms~\cite{huang2025how,li2023transformers,yang2024transformers,zhang2024trained,li2024transformers,nichani2024transformerslearncausalstructure}.  Complementary work analyzes when CoT trajectory supervision improves the sample and computational complexity of learning~\cite{joshi2025theorylearningautoregressivechain}.

Recent analyses further show that gradient-based training can drive Transformers that employ CoT to implement explicit multi-step algorithms~\cite{yang2025multihead,transformers2025gradientdescent,looped2024transformers,huang2025transformersprovablylearnchainofthought}. These works differ from ours in that they rely on training with next-token prediction and do not characterize when CoT reasoning emerges under RL training.

Two results that analyze RL-trained Transformers are~\citet{lyu2025transformersrlsftprovably} and \citet{bu2025provablebenefitcurriculumtransformer}, but both differ from our setting in crucial ways. 
The first~\cite{lyu2025transformersrlsftprovably} assumes a setting with dense rewards at every step along the CoT, relying on a much richer feedback signal than the sparse, outcome-based rewards we study. 
The second~\cite{bu2025provablebenefitcurriculumtransformer}, although formally using an outcome-based reward, defines a curriculum over progressively longer prefixes, so the learner effectively receives feedback for partial progress; in contrast, our setting provides feedback only on the full final outcome. Furthermore, \citet{bu2025provablebenefitcurriculumtransformer} analyze a simplified regime with only a single gradient update per curriculum stage, while we consider full training dynamics.

Follow-up work by \citet{huang2026learningdynamicsrlvr}, which appeared after our submission and cites our work, studies policy gradient dynamics in a related compositional setting. They provide analysis of policy gradient dynamics under less restrictive assumptions on the distribution of task difficulties.However, their setting can be viewed as a ``single-chain'' problem in which each step applies a group transformation, and the model is assumed to have access to a fixed MLP that already perfectly implements the atomic operation. Thus, the learning problem is to traverse the chain correctly rather than to acquire the stepwise transformation itself. By contrast, we do not assume such a fixed module. Moreover, our problem admits multiple minimizers corresponding to different algorithms; in this regime, we show that policy gradient exhibits a data-dependent implicit bias toward the most efficient solution. 

\subsection{Effect of Training Data Distribution on Learning}
\label{sec:related:data_distribution}
The composition and ordering of training data can significantly impact learning in neural networks. Empirical and theoretical work across supervised learning, unsupervised learning, and reinforcement learning demonstrates that training data distribution affects generalization, sample efficiency, and the emergence of specific learned behaviors~\cite{schaul2015prioritized,zhang2016understanding,belkin2019reconciling,nakkiran2019deepdoubledescentbigger,arora2019theoretical,chen2020simclr,kumar2020conservative,fedus2020revisiting,levine2020offline,feldman2020does,alexander2023makes}. 
Work on curriculum learning further suggests that ordering examples by complexity—determined by factors such as length, rarity, or comprehensibility—can significantly aid the learning process~\cite{bengio2009curriculum,graves2017automated,narvekar2020curriculum,portelas2020automatic,soviany2022curriculum,abbe2024generalization,bu2025provablebenefitcurriculumtransformer}.
More specifically, both empirical and theoretical investigations into Transformers emphasize that training data diversity is essential for the effective learning of complex compositional tasks~\cite{ranaldi-etal-2023-modeling,wang2025learningcompositionalfunctionstransformers,behnia2025factsstatsimpactspretraining}.
In a complementary line of work, \citet{medvedev2025shiftgoodmismatcheddata} demonstrate that out-of-distribution data can improve in-distribution performance in general mixture settings.

Particularly relevant to our setting, the empirical study of \citet{zhang2025interplaypretrainingmidtrainingrl} finds that RL post-training improves reasoning only when training examples are not significantly harder than those seen during pretraining (what they call the model's edge of competence). This supports our finding that simple examples are necessary for learning, and our experiments in \cref{sec:exp} build on this perspective in both synthetic and real-world settings.

\section{Proof of \cref{result:traverse_chain}}
\label{app:traverse_chain}

\cref{result:traverse_chain} in the main text is stated for infinite precision regime, which is incompatible with the finite log-precision regime, in which \cref{result:no_single_step} holds. In this section we prove a generalization of \cref{result:traverse_chain} that applies to the finite log-precision regime, in which \cref{result:no_single_step} holds.
\cref{result:traverse_chain} (as stated in \cref{sec:how_learn:reasoning_sufficient}) is an immediate corollary of the result in this section: taking the precision parameter $\xi\to\infty$ yields reward arbitrarily close to~$1$.
Moreover, we will show that the reward achieved by the construction in this section is strictly larger than the bound obtained in \cref{result:no_single_step}.

First, we formalize what it means for a Transformer to have \emph{$b$-bit precision}:

\begin{definition}[$b$-bit precision]
\label{def:precision}
We say that a Transformer $\theta$ has \emph{$b$-bit precision} if each scalar entry of the parameter matrices $(A,V)$ (or $\{(A_i,V_i)\}_{i=1}^d$ in the multi-layer case) can take at values in ${1,2,\dots,2^b}$
\end{definition}

\Cref{def:precision} can be extended to other numerical representations, including signed integers and fixed-point or floating-point formats, as long as the maximum representable magnitude remains $O(2^b)$.

We work in the \emph{logarithmic precision} regime: the number of bits scales logarithmically with the problem size $n$.
Concretely, we set
\[
b \;=\; \lceil  \xi\ \log_2 n \rceil
\]
for a positive integer $\xi\in\N_{>0}$ independent of $n$. (This is the same precision assumption as in \cref{result:no_single_step} where $\xi$ can be set to any positive integer.)

Now we are ready to state the theorem:

\begin{theorem}
\label{thm:traverse_chain_finite}
Fix any $n \ge 3$, any positive integer $\xi \ge 1$, and any distribution $\QQ$ over $[n-1]$.
There exists a single-layer Transformer $\theta=(A,V)$ with $\xi\lceil\log_2 n\rceil$-bit precision such that
\[
1-\RR(\D^{\QQ},\theta)
\;\le\;
n\big(|\vocab|-1\big)\exp\!\big(-\tfrac12 n^\xi\big),
\]
where $|\vocab|=2n+(2n)^2$ (see \cref{sec:task}).
Moreover, for every input sequence $s_{1:L_0}$ from $\D^{\QQ}$, with probability at least $1 - n(|\vocab|-1)\exp(-\tfrac12 n^\xi)$ over $\tau\sim p_{\text{roll}}(\theta)$, the rollout $\tau$ is a chain traversal (\cref{sec:setup:traversal}).
\end{theorem}

\begin{proof}
We prove the result simultaneously for the softmax and linear Transformer variants using the unified weight function $w(\cdot)$ from \cref{lem:dyn:logits_cases_canonical}.

Let $A$ satisfy the assumptions of \cref{lem:dyn:logits_cases_canonical} with
\[
\alpha = n,\qquad \beta = 0,\qquad \gamma = 0,
\]
and set all vertex--vertex entries of $A$ equal to $\gamma=0$.
Let $V$ satisfy the assumptions of \cref{lem:dyn:logits_cases_canonical} with
\[
\alpha'=0,\qquad \gamma'=0,\qquad \beta' = n^\xi.
\]
All nonzero entries of $(A,V)$ are integers of magnitude at most $n^\xi \le 2^b - 1$, hence representable with $b$ bits.

Fix any time during autoregressive generation before termination, and let $v$ be the last generated vertex (so the current input sequence $s$ ends with $v$).
Let $u$ denote the (unique) successor of $v$ in its chain if it exists.
By \cref{lem:dyn:logits_cases_canonical} and our choice $\Delta_\alpha=\alpha'-\gamma'=0$ and $\Delta_\beta=\beta'-\gamma'=\beta'$, the logit differences are controlled by the incoming-edge weights.
For every nonterminal $v\in(\chaina\cup\chainb)\setminus\{a_n,b_n\}$, the successor $u$ has an incoming edge $(v,u)\in G$ whose attention score from $v$ is $A_{v,(v,u)}=\alpha=n$; every other vertex $k\neq u$ has either an incoming edge with attention score $0$ from $v$ (i.e., $A_{v,(\mathrm{pred}(k),k)}\in\{\beta,\gamma\}=\{0\}$) or has no incoming edge in $G$ (start vertices). For edge token outputs $e' \in \edgevocab$, by \cref{lem:dyn:logits_cases_canonical} we have $o(v,G,Z_v)_{e'} = 0$. Hence there is a logit gap of the form
\[
o(v,G,Z_v)_u - o(v,G,Z_v)_k \;\ge\; \beta'\,\big(w(n)-w(0)\big)
\qquad \text{for all }k\neq u.
\]

Since $w(\cdot)$ is monotone increasing, $w(n)>w(0)$ in both attention variants:
for linear attention, $w(n)-w(0)=n$; for softmax attention, $w(n)-w(0)=\frac{e^n-1}{Z_v}$ (where $Z_v$ is defined in \cref{lem:dyn:logits_cases_canonical}).
Under our choice of $A$, all tokens other than the unique outgoing edge $(v,u)$ have attention score $0$ from $v$ (including all vertex tokens, since we set the vertex--vertex block to $\gamma=0$). Therefore, for softmax attention,
\[
Z_v=\sum_{j=1}^L \exp(A_{v,s_j}) \le e^n + (L-1)\cdot e^0 = e^n+L-1,
\]
where $L$ is the current prefix length.
Condition on the ``good'' event $E$ that the rollout has followed the correct chain successor at every previous step (equivalently, it is a chain traversal up to the current time). On $E$, the number of generated vertices before termination is at most $n$ (the chain length), hence $L\le L_0+n\le (2n-1)+n = 3n-1$.
Thus, on $E$,
\[
w(n)-w(0)
\;\ge\;
\begin{cases}
n, & \text{linear attention},\\[2pt]
\dfrac{e^n-1}{3n-2+e^n}, & \text{softmax attention}.
\end{cases}
\]
Since $e^n \ge 3n$ for all $n\ge 2$, we have
\[
\dfrac{e^n-1}{3n-2+e^n} \;\ge\; \frac12,
\]
and hence for both attention variants $w(n)-w(0)\ge \frac12$ on $E$.

The next-token distribution is $\mathrm{TF}(s_{1:L};\theta)=\mathrm{softmax}\big(o(v,G,Z_v)\big)$.
Using the logit gap and bounding the denominator by $|\vocab|$ terms (including both vertices and edges), we obtain for every nonterminal vertex $v\in(\chaina\cup\chainb)\setminus\{a_n,b_n\}$:
\[
\Pr\big[s_{L+1}=u \,\big|\, s_L=v\big]
\ge
\frac{1}{1+\sum_{k\neq u}\exp\!\big(-\tfrac12\beta'\big)}
\ge
1-(|\vocab|-1)\exp\!\big(-\tfrac12\beta'\big).
\]
All rollouts start from a vertex in one of the chains; at most $n$ nonterminal-to-successor transitions are required to reach the chain terminal.
By a union bound, the probability that the rollout ever deviates from the correct successor is at most
\[
n(|\vocab|-1)\exp\!\big(-\tfrac12\beta'\big).
\]
Whenever the rollout follows the chain edges until it reaches a terminal vertex, it terminates at the correct terminal $y^*(s_{1:L_0})$ and receives reward $1$ (see \cref{sec:setup:reward,sec:setup:traversal}).
Therefore,
\[
1-\RR(\D^{\QQ},\theta)
\;\le\;
n(|\vocab|-1)\exp\!\big(-\tfrac12\beta'\big)
\;=\;
n(|\vocab|-1)\exp\!\big(-\tfrac12 n^\xi\big),
\]
as claimed.

\end{proof}

\begin{corollary}
\label{cor:multistep_beats_single}
For any $\xi > 1$ there exists a chain length $n$ such that the reward achieved by the construction in \cref{thm:traverse_chain_finite} is strictly greater than the maximal single-step reward achievable by any Transformer under the setting of \cref{result:no_single_step}.
\end{corollary}

\begin{proof}
By \cref{thm:traverse_chain_finite}, the construction achieves reward at least $1 - n(|\vocab|-1)\exp(-\frac12 n^\xi)$.
By \cref{lem:no_cot_upper} and the analysis in \cref{app:no_single_step}, there exists a chain length $n$ such that any Transformer achieves single-step reward at most $1 - \frac{1}{(2n)! \cdot 2(n-1)}$.
It suffices to show that for sufficiently large $n$:
\[
n(|\vocab|-1)\exp\!\big(-\tfrac12 n^\xi\big) < \frac{1}{(2n)! \cdot 2(n-1)}.
\]
Rearranging, this is equivalent to:
\[
n(|\vocab|-1) \cdot (2n)! \cdot 2(n-1) < \exp\!\big(\tfrac12 n^\xi\big).
\]
Taking logarithms, it suffices to show:
\[
\log n + \log(|\vocab|-1) + \log((2n)!) + \log 2 + \log(n-1) < \tfrac12 n^\xi.
\]
Since $|\vocab| = 2n + (2n)^2$, we have $\log(|\vocab|-1) = O(\log n)$.
By Stirling's approximation, $\log((2n)!) = 2n\log(2n) - 2n + O(\log n) = O(n \log n)$.
Thus the left-hand side is $O(n \log n)$.
Since $\xi > 1$, we have $n^\xi / (n \log n) \to \infty$ as $n \to \infty$, so $\frac12 n^\xi > O(n \log n)$ for sufficiently large $n$.
\end{proof}

\section{Proof of \cref{result:inefficient_solution}}
\label{app:inefficient_solution}

\begin{proof}
Fix \(n \ge 3\), a distribution \(Q\) over \([n-1]\),
a target length threshold \(k\in\mathbb{N}\), and a loss tolerance \(\varepsilon>0\). We'll assume w.l.o.g. that $\delta>\varepsilon$ (otherwise we can lower $\varepsilon$).

We use the notation and symmetry setup of \cref{lem:dyn:logits_cases_canonical}, in particular the parameters
\(\alpha,\beta,\gamma\) in the attention matrix and \(\alpha',\beta',\gamma'\) in the value matrix.  We also reuse the notation
\(\Delta_\alpha := \alpha'-\gamma'\) and \(\Delta_\beta := \beta'-\gamma'\).  The formulas for the per-vertex logit
contributions \(\Psi_k\) as a function of \((\alpha,\beta,\gamma,\Delta_\alpha,\Delta_\beta)\) are exactly those
in \cref{lem:dyn:logits_cases_canonical}. As the transformer is assumed to be linear, we have \(w(\mu)=\mu\). 

We construct a 1-layer transformer with attention parameters
\[
\gamma = 0,\qquad \beta = -\alpha_0 - \beta_0,\qquad \alpha = \alpha_0.
\]
for some free parameters \(\alpha_0,\beta_0\in\mathbb{R}\) to be chosen later, and with value parameters
\[
\gamma' = 0,\qquad \alpha' = 1,\qquad \beta' = -1.
\]
Thus \(\Delta_\alpha = 1\) and \(\Delta_\beta = -1\).  We keep the vertex–vertex block of \(A\) equal to \(\gamma=0\),
and set \(V_{v,:}=0\) for all vertex tokens \(v\in V\), exactly as in \cref{lem:dyn:logits_cases_canonical}.

Let \(v\) be any non-terminal vertex in one of the chains, and consider the logits at the last position when the
current prefix ends with \(v\).  Applying \cref{lem:dyn:logits_cases_canonical} and substituting \(\gamma=0\), \(w(\gamma)=0\), \(\Delta_\alpha=1\),
\(\Delta_\beta=-1\), we obtain for an interior vertex (Case~1 in \cref{lem:dyn:logits_cases_canonical}) that the only nonzero logit contributions
on the two incident vertices are
\[
\Psi_{\mathrm{forw}} = \Delta_\beta w(\alpha) = \beta'\,\alpha_0 = -\alpha_0,\qquad
\Psi_{\mathrm{back}} = \Delta_\alpha w(\beta) = \alpha'\,\beta = \beta = -\alpha_0 - \beta_0.
\]
The current vertex \(v\) itself has
\[
\Psi_v = \Delta_\alpha w(\alpha) + \Delta_\beta w(\beta) = \alpha_0 - \beta = 2\alpha_0 + \beta_0,
\]
and every other vertex has \(\Psi_k = 0\) (since all those
terms are proportional to \(w(\gamma)=0\) in \cref{lem:dyn:logits_cases_canonical}), and all edge token outputs also have \(\Psi_{e'} = 0\) by \cref{lem:dyn:logits_cases_canonical}.  Thus, up to some offset \(C\), any vertex \(v\) which is ``in the middle of the chain'' (not the start, end, second or penultimate vertex of some chain) has a local logit
pattern of the form
\[
\ell_v = 2\alpha_0 + \beta_0,\quad
\ell_{\mathrm{back}} = -\alpha_0 - \beta_0,\quad
\ell_{\mathrm{forw}} = -\alpha_0,\quad
\ell_{\mathrm{other}} = 0.
\]
The same calculation applied to vertices $v$ which are not the start of a chain (Cases~2,3,5 in \cref{lem:dyn:logits_cases_canonical}) gives the same qualitative picture:
for any non-terminal \(v\) on either chain the logits on the two incident chain vertices are linear functions of
\(\alpha_0\) and \(\beta_0\) of the form
\[
\ell_{\mathrm{forw}}(\beta_0,\alpha_0) = \beta'\,\alpha_0 = -\alpha_0,\qquad
\ell_{\mathrm{back}}(\beta_0,\alpha_0) = \alpha'\,\beta = -\alpha_0 - \beta_0,
\]
while all other vertices have logit \(0\), and the current vertex
\(v\) has logit \(2\alpha_0 + \beta_0\).  In particular the \emph{only} dependence on \(\beta_0\) in the \emph{differences} between
the forward and backward logits is
\[
\ell_{\mathrm{forw}} - \ell_{\mathrm{back}} = \beta_0.
\]

Finally, when the current vertex \(v\) is the start of a chain (Case~4 in \cref{lem:dyn:logits_cases_canonical}), there is no backward neighbor within the chain. In this case, the forward logit is still \(\ell_{\mathrm{forw}} = -\alpha_0\), but there is no \(\ell_{\mathrm{back}}\), so as $\alpha_0\to -\infty$, the forward probability to move forward limits to one.

We now use this to tune the random-walk behaviour in two stages.

\emph{1. Controlling the conditional walk along the chain (choosing \(\beta_0\)).}
Fix \(\alpha_0\) for the moment.  Condition on the “good event” that at each step the model only chooses the forward or
backward neighbour in the chain (i.e., we ignore the possibility of staying at \(v\) or jumping to any other vertex).
Under this conditioning, from every interior vertex \(v\) the next-token distribution over
\(\{\mathrm{forw},\mathrm{back}\}\) has the form
\[
\begin{aligned}
\mathbb{P}(\mathrm{forw}\mid \{\mathrm{forw},\mathrm{back}\})
  &= \frac{e^{\ell_{\mathrm{forw}}}}{e^{\ell_{\mathrm{forw}}}+e^{\ell_{\mathrm{back}}}}
   = \frac{1}{1+e^{-(\ell_{\mathrm{forw}}-\ell_{\mathrm{back}})}}
   = \frac{1}{1+e^{-\beta_0}},\\
\mathbb{P}(\mathrm{back}\mid \{\mathrm{forw},\mathrm{back}\})
  &= \frac{1}{1+e^{\beta_0}}.
\end{aligned}
\]
Thus the forward/backward bias of the induced nearest-neighbour walk along the chain depends \emph{only} on \(\beta_0\),
via a smooth, strictly monotone function.  As \(\beta_0\to 0\) the walk becomes approximately unbiased
(\(\mathbb{P}(\mathrm{forw}) \to 1/2\)); as \(\beta_0\to\infty\) the walk becomes almost deterministic in the forward
direction; as \(\beta_0\to -\infty\) it becomes almost deterministic in the backward direction.

Let us denote by $G$ the event that the rollout consists only of forward/backward moves along the chain. Standard random-walk estimates on a finite chain
imply that, for any fixed chain length \(n\), and starting index \(k\in[n-1]\), there exists a choice of \(\beta_0\) and some $L_{\max}$ (dependent on $\beta_0$) such that, \emph{conditional on $G$},
the hitting time \(T\) of the terminal vertex satisfies
\[
\mathbb{P}\bigl(L_{\max}>T > L \,\bigm|\, G\bigr) \ge 1 - \varepsilon / 2
\]
for any prescribed \(L\). This implies that conditioned on $G$, the rollout will take more than $L$ and less than $L_{\max}$ steps to reach the terminal vertex with probability at least $1 - \varepsilon / 2$ (which assures loss smaller than $1-\varepsilon/2$).

\emph{2. Making $\mathbb{P}(G^c) \leq \varepsilon / 2$ (choosing \(\alpha_0\)).}
We now choose \(\alpha_0\) to ensure that the model stays on the chain.  Taking \(\alpha_0 \to -\infty\), the logits for the forward and backward neighbors ($\ell_{\mathrm{forw}} = -\alpha_0$, $\ell_{\mathrm{back}} = -\alpha_0 - \beta_0$) tend to $+\infty$, while the logit for the current vertex ($\ell_v = 2\alpha_0 + \beta_0$) tends to $-\infty$, and other vertices as well as all edge token outputs remain at $0$.  Consequently, the probability mass concentrates on the forward and backward neighbors, with the ratio $p_{\mathrm{forw}}/p_{\mathrm{back}} = e^{\beta_0}$ preserved.  The probability of outputting any other token (vertex or edge) vanishes relative to the move probabilities. 

Together, these two points imply the theorem.

\end{proof}

\section{Restatement of \texorpdfstring{\cref{result:traverse_chain,result:inefficient_solution}}{Propositions 1 and 2} Under the Setting of \texorpdfstring{\cref{sec:dyn:setting}}{Section 6}}
\label{app:dyn_constructions}

The constructions establishing \cref{result:traverse_chain,result:inefficient_solution} do not adhere to the value-matrix parameterization of \cref{sec:dyn:setting}, which fixes $V_{(u,w),v} = \mathbbm{1}\{v \in \{u,w\}\}$.
Using this parameterization directly with variants of those constructions induces a high logit for the current vertex, preventing convergence to the desired behavior.
However, under the assumption of \cref{sec:dyn:setting} that the model only outputs vertices (edges are masked from the output) and never outputs the same vertex twice consecutively (the previous vertex is masked), we can construct single-layer linear Transformers achieving the same guarantees and adhering to the setting of \cref{sec:dyn:setting}.
\cref{app:dyn_constructions:traverse} restates \cref{result:traverse_chain} and \cref{app:dyn_constructions:inefficient} restates \cref{result:inefficient_solution}.

\subsection{Restatement of \texorpdfstring{\cref{result:traverse_chain}}{Proposition 1} Under the Setting of \texorpdfstring{\cref{sec:dyn:setting}}{Section 6}}
\label{app:dyn_constructions:traverse}

\begin{proposition}
\label{result:traverse_chain_dyn}
Consider the setting of \cref{sec:dyn:setting}, in which the value matrix satisfies $V_{(u,w),v} = \mathbbm{1}\{v \in \{u,w\}\}$, the model only outputs vertices, and the model never outputs the same vertex twice consecutively (the previous vertex is masked).
For any $n \in \N_{\geq 2}$, any distribution $\QQ$ over $[n-1]$, and any $\varepsilon \in \R_{> 0}$, there exist parameters $\theta = (A, V)$ for a single-layer linear Transformer such that $\RR(\D^{\QQ}, \theta) > 1 - \varepsilon$, and for every input sequence $s_{1:L_0}$ from $\D^{\QQ}$, with probability at least $1 - \varepsilon$ over $\tau \sim p_{\text{roll}}(\theta)$, the rollout $\tau$ is a chain traversal (\cref{sec:setup:traversal}).
\end{proposition}

\begin{proof}
We use the notation and symmetry setup of \cref{lem:dyn:logits_cases_canonical}.
As the transformer is assumed to be linear, we have $w(\mu)=\mu$.

Let $A$ satisfy the assumptions of \cref{lem:dyn:logits_cases_canonical} with
\[
\alpha = \alpha_0,\qquad \beta = 0,\qquad \gamma = 0,
\]
for some free parameter $\alpha_0 \in \R_{>0}$ to be chosen later, and set all vertex--vertex entries of $A$ equal to $\gamma=0$.
Let $V$ satisfy the assumptions of \cref{lem:dyn:logits_cases_canonical} with
\[
\alpha'=1,\qquad \beta'=1,\qquad \gamma'=0,
\]
as required by \cref{sec:dyn:setting}. Thus $\Delta_\alpha = \Delta_\beta = 1$.

Fix any time during autoregressive generation before termination, and let $v$ be the last generated vertex (so the current input sequence $s$ ends with $v$).
Let $u$ denote the (unique) successor of $v$ in its chain if it exists.
By \cref{lem:dyn:logits_cases_canonical}, the successor $u$ has logit contribution $\Psi_u = \Delta_\alpha w(\gamma) + \Delta_\beta w(\alpha) = \alpha_0$; every other vertex $k \neq u,v$ has logit contribution $\Psi_k \in \{w(\beta), w(\gamma)\} = \{0\}$ or is a chain start with $\Psi_k = \Delta_\alpha w(\gamma) = 0$.
The current vertex $v$ has $\Psi_v = w(\alpha) + w(\beta) = \alpha_0$, but is masked out by assumption (as are all edge tokens).
Hence there is a logit gap of the form
\[
o(v,G)_u - o(v,G)_k \;\ge\; \alpha_0
\qquad \text{for all }k\neq u,v.
\]

The next-token distribution is $\mathrm{TF}(s_{1:L};\theta)=\mathrm{softmax}\big(o(v,G)\big)$.
Using the logit gap and bounding the denominator by $|\vocab|$ terms (excluding $v$), we obtain for every nonterminal $v$:
\[
\Pr\big[s_{L+1}=u \,\big|\, s_L=v\big]
\ge
\frac{e^{\alpha_0}}{e^{\alpha_0}+(|\vocab|-2)\cdot e^0}
=
\frac{1}{1+(|\vocab|-2)e^{-\alpha_0}}
\ge
1-(|\vocab|-2)e^{-\alpha_0}.
\]
All rollouts start from a vertex in one of the chains; at most $n$ nonterminal-to-successor transitions are required to reach the chain terminal.
By a union bound, the probability that the rollout ever deviates from the correct successor is at most
\[
n(|\vocab|-2)e^{-\alpha_0}.
\]
Whenever the rollout follows the chain edges until it reaches a terminal vertex, it terminates at the correct terminal $y^*(s_{1:L_0})$ and receives reward $1$ (see \cref{sec:setup:reward,sec:setup:traversal}).
Choosing $\alpha_0 = \log(n(|\vocab|-2)/\varepsilon)$ yields
\[
1-\RR(\D^{\QQ},\theta)
\;\le\;
n(|\vocab|-2)e^{-\alpha_0}
\;=\;
\varepsilon,
\]
as claimed.
\end{proof}

\subsection{Restatement of \texorpdfstring{\cref{result:inefficient_solution}}{Proposition 2} Under the Setting of \texorpdfstring{\cref{sec:dyn:setting}}{Section 6}}
\label{app:dyn_constructions:inefficient}

\begin{proposition}
\label{result:inefficient_solution_dyn}
Consider the setting of \cref{sec:dyn:setting}, in which the value matrix satisfies $V_{(u,w),v} = \mathbbm{1}\{v \in \{u,w\}\}$, the model only outputs vertices, and the model never outputs the same vertex twice consecutively (the previous vertex is masked).
For any $n \in \N_{\geq 3}$, any distribution $\QQ$ over $[n-1]$, any $k \in \N$, and any $\varepsilon, \delta \in \R_{> 0}$, there exist parameters $\theta = (A, V)$ for a single-layer linear Transformer such that $\RR(\D^{\QQ}, \theta) > 1 - \varepsilon$, and for every input sequence $s_{1:L_0}$ from $\D^{\QQ}$, the probability that $|\tau| > k$ is greater than $1 - \delta$.
\end{proposition}

\begin{proof}
Fix \(n \ge 3\), a distribution \(Q\) over \([n-1]\),
a target length threshold \(k\in\mathbb{N}\), and a loss tolerance \(\varepsilon>0\). We assume w.l.o.g.\ that $\delta>\varepsilon$ (otherwise we can lower $\varepsilon$).

We use the notation and symmetry setup of \cref{lem:dyn:logits_cases_canonical}, in particular the parameters
\(\alpha,\beta,\gamma\) in the attention matrix and \(\alpha',\beta',\gamma'\) in the value matrix.  We also reuse the notation
\(\Delta_\alpha := \alpha'-\gamma'\) and \(\Delta_\beta := \beta'-\gamma'\).  The formulas for the per-vertex logit
contributions \(\Psi_k\) as a function of \((\alpha,\beta,\gamma,\Delta_\alpha,\Delta_\beta)\) are exactly those
in \cref{lem:dyn:logits_cases_canonical}. As the transformer is assumed to be linear, we have \(w(\mu)=\mu\). 

We construct a 1-layer transformer with attention parameters
\[
\gamma = 0,\qquad \beta = \alpha_0 + \beta_0,\qquad \alpha = \alpha_0,
\]
for some free parameters \(\alpha_0,\beta_0\in\mathbb{R}\) to be chosen later, and with value parameters
\[
\gamma' = 0,\qquad \alpha' = 1,\qquad \beta' = 1,
\]
as required by \cref{sec:dyn:setting}.
Thus \(\Delta_\alpha = 1\) and \(\Delta_\beta = 1\).  We keep the vertex--vertex block of \(A\) equal to \(\gamma=0\),
and set \(V_{v,:}=0\) for all vertex tokens \(v\in V\), exactly as in \cref{lem:dyn:logits_cases_canonical}.

Let \(v\) be any non-terminal vertex in one of the chains, and consider the logits at the last position when the
current prefix ends with \(v\).  Applying \cref{lem:dyn:logits_cases_canonical} and substituting \(\gamma=0\), \(w(\gamma)=0\), \(\Delta_\alpha=1\),
\(\Delta_\beta=1\), we obtain for an interior vertex (Case~1 in \cref{lem:dyn:logits_cases_canonical}) that the only nonzero logit contributions
on the two incident vertices are
\[
\Psi_{\mathrm{forw}} = \Delta_\alpha w(\gamma) + \Delta_\beta w(\alpha) = w(\alpha) = \alpha_0,\qquad
\Psi_{\mathrm{back}} = \Delta_\alpha w(\beta) + \Delta_\beta w(\gamma) = w(\beta) = \alpha_0 + \beta_0.
\]
The current vertex \(v\) itself has
\[
\Psi_v = \Delta_\alpha w(\alpha) + \Delta_\beta w(\beta) = \alpha_0 + (\alpha_0 + \beta_0) = 2\alpha_0 + \beta_0,
\]
but is masked out by assumption.
Every other vertex has \(\Psi_k = 0\) (since all those
terms are proportional to \(w(\gamma)=0\) in \cref{lem:dyn:logits_cases_canonical}); edge tokens are masked from the output.  Thus, up to some offset \(C\), any vertex \(v\) which is ``in the middle of the chain'' (not the start, end, second or penultimate vertex of some chain) has a local logit
pattern of the form
\[
\ell_{\mathrm{forw}} = \alpha_0,\quad
\ell_{\mathrm{back}} = \alpha_0 + \beta_0,\quad
\ell_{\mathrm{other}} = 0.
\]
The same calculation applied to vertices $v$ which are not the start of a chain (Cases~2,3,5 in \cref{lem:dyn:logits_cases_canonical}) gives the same qualitative picture:
for any non-terminal \(v\) on either chain the logits on the two incident chain vertices are
\[
\ell_{\mathrm{forw}} = \alpha_0,\qquad
\ell_{\mathrm{back}} = \alpha_0 + \beta_0,
\]
while all other vertices have logit \(0\).
In particular the \emph{only} dependence on \(\beta_0\) in the \emph{differences} between
the forward and backward logits is
\[
\ell_{\mathrm{forw}} - \ell_{\mathrm{back}} = -\beta_0.
\]

Finally, when the current vertex \(v\) is the start of a chain (Case~4 in \cref{lem:dyn:logits_cases_canonical}), there is no backward neighbor within the chain. In this case, the forward logit is still \(\ell_{\mathrm{forw}} = \alpha_0\), but there is no \(\ell_{\mathrm{back}}\), so as $\alpha_0\to \infty$, the probability to move forward limits to one.

We now use this to tune the random-walk behaviour in two stages.

\emph{1. Controlling the conditional walk along the chain (choosing \(\beta_0\)).}
Fix \(\alpha_0\) for the moment.  Condition on the ``good event'' that at each step the model only chooses the forward or
backward neighbour in the chain (i.e., we ignore the possibility of jumping to any other vertex).
Under this conditioning, from every interior vertex \(v\) the next-token distribution over
\(\{\mathrm{forw},\mathrm{back}\}\) has the form
\[
\begin{aligned}
\mathbb{P}(\mathrm{forw}\mid \{\mathrm{forw},\mathrm{back}\})
  &= \frac{e^{\ell_{\mathrm{forw}}}}{e^{\ell_{\mathrm{forw}}}+e^{\ell_{\mathrm{back}}}}
   = \frac{1}{1+e^{-(\ell_{\mathrm{forw}}-\ell_{\mathrm{back}})}}
   = \frac{1}{1+e^{\beta_0}},\\
\mathbb{P}(\mathrm{back}\mid \{\mathrm{forw},\mathrm{back}\})
  &= \frac{1}{1+e^{-\beta_0}}.
\end{aligned}
\]
Thus the forward/backward bias of the induced nearest-neighbour walk along the chain depends \emph{only} on \(\beta_0\),
via a smooth, strictly monotone function.  As \(\beta_0\to 0\) the walk becomes approximately unbiased
(\(\mathbb{P}(\mathrm{forw}) \to 1/2\)); as \(\beta_0\to -\infty\) the walk becomes almost deterministic in the forward
direction; as \(\beta_0\to +\infty\) it becomes almost deterministic in the backward direction.

Let us denote by $G$ the event that the rollout consists only of forward/backward moves along the chain. Standard random-walk estimates on a finite chain
imply that, for any fixed chain length \(n\), and starting index \(k\in[n-1]\), there exists a choice of \(\beta_0\) and some $L_{\max}$ (dependent on $\beta_0$) such that, \emph{conditional on $G$},
the hitting time \(T\) of the terminal vertex satisfies
\[
\mathbb{P}\bigl(L_{\max}>T > L \,\bigm|\, G\bigr) \ge 1 - \varepsilon / 2
\]
for any prescribed \(L\). This implies that conditioned on $G$, the rollout will take more than $L$ and less than $L_{\max}$ steps to reach the terminal vertex with probability at least $1 - \varepsilon / 2$ (which assures loss smaller than $1-\varepsilon/2$).

\emph{2. Making $\mathbb{P}(G^c) \leq \varepsilon / 2$ (choosing \(\alpha_0\)).}
We now choose \(\alpha_0\) to ensure that the model stays on the chain.  Taking \(\alpha_0 \to \infty\), the logits for the forward and backward neighbors ($\ell_{\mathrm{forw}} = \alpha_0$, $\ell_{\mathrm{back}} = \alpha_0 + \beta_0$) tend to $+\infty$, while other vertices remain at $0$ (edge tokens are masked).  Consequently, the probability mass concentrates on the forward and backward neighbors, with the ratio $p_{\mathrm{forw}}/p_{\mathrm{back}} = e^{-\beta_0}$ preserved.  The probability of outputting any other vertex vanishes relative to the move probabilities. 

Together, these two points imply the theorem.
\end{proof}

	\newcommand{\voc}{\Sigma_n}
\newcommand{\vsz}{|\Sigma_n|}
\newcommand{\Nbit}{N}          %
\newcommand{\TCzero}{\mathrm{TC}^0}
\newcommand{\doplus}{d_{\oplus}}  %
\newcommand{\dtimes}{d_{\times}}  %
\newcommand{\NC}{\mathrm{NC}}
\section{Proof of \cref{result:no_single_step}}
\label{app:no_single_step}

The proof proceeds through a chain of reductions in circuit complexity. We first establish that if greedy decoding fails on any input, the single-step reward is bounded away from~1 (\cref{app:reward_complexity}). We then recall the relevant complexity classes and problems (\cref{app:complexity_assumption}), and show that constant-depth Transformers without CoT can be simulated by $\TC^0$ circuits (\cref{lem:transformers-in-TC0}). The core technical work reduces the $\NC^1$-complete word problem for~$S_5$ to path reachability ($\ORD$), and then to the two-chain endpoint problem. Since $\TC^0 \neq \NC^1$ under standard assumptions, no $\TC^0$ circuit---and hence no constant-depth Transformer---can solve the endpoint problem on all inputs.

\subsection{Relation between reward, complexity, and decision problems}
\label{app:reward_complexity}

In this section we connect the computational hardness of the chain identification task to an upper bound on the single-step reward. Our approach uses deterministic complexity classes: we show that for infinitely many chain lengths~$n$, any constant-depth Transformer without CoT fails to compute the correct answer via greedy decoding on at least one input.

\begin{lemma}[Reward upper bound from greedy failure]
\label{lem:no_cot_upper}
Let $\QQ_{\mathrm{unif}}$ denote the uniform distribution over the index set $[n-1]$, so that $\D^{\QQ_{\mathrm{unif}}}$ is the uniform distribution over all (graph, starting vertex) pairs (\cref{sec:setup:dist}).
If there exists at least one input $s \in \mathrm{supp}(\D^{\QQ_{\mathrm{unif}}})$ on which greedy decoding does not output $y^*(s)$, then
\[
\RR_{\emph{single}}(\D^{\QQ_{\mathrm{unif}}}, \theta) \leq 1 - \frac{1}{(2n)! \cdot 2(n-1)}.
\]
\end{lemma}

\begin{proof}
The support of $\D^{\QQ_{\mathrm{unif}}}$ consists of all pairs of a two-chain graph on $2n$ vertices and a non-terminal starting vertex. The number of such pairs is $(2n)! \cdot (n-1)$ (there are $(2n)!/2$ ordered two-chain graphs, each with $2(n-1)$ non-terminal starting vertices, but each unordered graph is counted twice).

Under the uniform distribution, the single-step reward is
\[
\RR_{\text{single}}(\D^{\QQ_{\mathrm{unif}}}, \theta) = \frac{1}{|\mathrm{supp}(\D^{\QQ_{\mathrm{unif}}})|} \sum_{s \in \mathrm{supp}(\D^{\QQ_{\mathrm{unif}}})} \Pr\big(\mathrm{TF}(s; \theta) = y^*(s)\big).
\]
If greedy decoding fails on input $s^*$, then $\arg\max_y \mathrm{TF}(s^*; \theta)_y \neq y^*(s^*)$, which implies $\Pr(\mathrm{TF}(s^*; \theta) = y^*(s^*)) \leq 1/2$ (since the correct answer does not have the maximal logit). Thus at least one term in the sum is at most $1/2$, yielding
\[
\RR_{\text{single}}(\D^{\QQ_{\mathrm{unif}}}, \theta) \leq 1 - \frac{1}{2 \cdot (2n)! \cdot (n-1)} = 1 - \frac{1}{(2n)! \cdot 2(n-1)}.
\]
\end{proof}

We also utilize the decision version of the two-chain endpoint problem (which returns true/false for a pair of vertices depending on whether the latter is the endpoint of the chain containing the former); this will be expanded upon in \cref{subsec:from_function_to_decision}.

\subsection{Complexity-Theoretic Backgrounnd}
\label{app:complexity_assumption}

Our hardness result relies on standard assumptions from circuit complexity theory regarding the computational power of constant-depth versus logarithmic-depth circuits. Circuit complexity characterizes computational problems by the resources required by boolean circuits to solve them. A circuit family consists of a circuit for each input length $n$, where each circuit is a directed acyclic graph with input nodes, internal gates (computing boolean functions), and output nodes. The key resources are the \emph{size} (number of gates) and \emph{depth} (longest path from input to output).

\begin{definition}[$\mathsf{L}$]
\label{def:L}
A language $L \subseteq \{0,1\}^*$ is in $\mathsf{L}$ if there exists a deterministic Turing machine $M$ and a constant $c \in \mathbb{N}$ such that for every input $x \in \{0,1\}^n$:
\begin{itemize}
    \item $M$ halts on all inputs, and accepts $x$ if and only if $x \in L$,
    \item On input $x$, the total number of work-tape cells that $M$ ever scans is at most $c \log n$ (i.e., $M$ uses $O(\log n)$ space).
\end{itemize}
\end{definition}

We next define the circuit classes central to our analysis.

\begin{definition}[$\mathrm{AC}^0$]
\label{def:AC0}
A language $L \subseteq \{0,1\}^*$ is in $\mathrm{AC}^0$ if there exists a constant $d \in \mathbb{N}$, a polynomial $p(\cdot)$, and a family of boolean circuits $\{C_n\}_{n \in \mathbb{N}}$ such that for all $n \in \mathbb{N}$:
\begin{itemize}
    \item $C_n$ has depth at most $d$ and size at most $p(n)$,
    \item $C_n$ uses unbounded fan-in AND, OR, and NOT gates,
    \item For all $x \in \{0,1\}^n$, $C_n(x) = \mathbbm{1}[x \in L]$.
\end{itemize}
\end{definition}

It is known that simple problems like PARITY (computing the XOR of all input bits) provably lie outside $\mathrm{AC}^0$~\cite{furst1981parity,ajtai1983formulae}.

\begin{definition}[$\mathrm{TC}^0$]
\label{def:TC0}
A language $L \subseteq \{0,1\}^*$ is in $\mathrm{TC}^0$ if there exists a constant $d \in \mathbb{N}$, a polynomial $p(\cdot)$, and a family of boolean circuits $\{C_n\}_{n \in \mathbb{N}}$ such that for all $n \in \mathbb{N}$:
\begin{itemize}
    \item $C_n$ has depth at most $d$ and size at most $p(n)$,
    \item $C_n$ uses unbounded fan-in AND, OR, NOT, and MAJORITY gates, where a MAJORITY gate outputs 1 if and only if the majority of its inputs are 1,
    \item For all $x \in \{0,1\}^n$, $C_n(x) = \mathbbm{1}[x \in L]$.
\end{itemize}
\end{definition}

The additional power of MAJORITY gates allows $\mathrm{TC}^0$ to compute functions like PARITY which lie outside $\mathrm{AC}^0$, and thus $\mathrm{AC}^0 \subsetneq \mathrm{TC}^0$~\cite{barrington1990bounded}.

\begin{definition}[$\mathrm{NC}^1$]
\label{def:NC1}
A language $L \subseteq \{0,1\}^*$ is in $\mathrm{NC}^1$ if there exists a polynomial $p(\cdot)$ and a family of boolean circuits $\{C_n\}_{n \in \mathbb{N}}$ such that for all $n \in \mathbb{N}$:
\begin{itemize}
    \item $C_n$ has depth $O(\log n)$ and size at most $p(n)$,
    \item $C_n$ uses bounded fan-in (typically fan-in 2) AND, OR, and NOT gates,
    \item For all $x \in \{0,1\}^n$, $C_n(x) = \mathbbm{1}[x \in L]$.
\end{itemize}
\end{definition}

This class corresponds to problems efficiently parallelizable with a logarithmic number of parallel steps. Examples include evaluating boolean formulas and computing certain regular languages~\cite{cook1985taxonomy}.

\begin{definition}[$\mathsf{X}$-uniformity]
\label{def:X-uniform}
Let $\mathcal{C}$ be a circuit class defined via families $\{C_n\}_{n\in\mathbb{N}}$ (e.g., $\mathrm{AC}^0,\mathrm{TC}^0,\mathrm{NC}^1$).
A family $\{C_n\}_{n\in\mathbb{N}}$ is \emph{$\mathsf{X}$-uniform} if there exists a function $U\in\mathsf{X}$ such that, on input $1^n$, $U$ outputs a description of $C_n$.

A language $L\subseteq\{0,1\}^*$ is in the \emph{$\mathsf{X}$-uniform version of $\mathcal{C}$} if it is decided by some $\mathsf{X}$-uniform family $\{C_n\}\in\mathcal{C}$.
\end{definition}

Throughout, all circuit classes ($\AC^0$, $\TC^0$, $\NC^1$) are
non--uniform unless explicitly marked as uniform.

\begin{conjecture}[$\mathrm{TC}^0 \subsetneq \mathrm{NC}^1$]
\label{conj:TC0_NC1}
There exist problems in $\mathrm{NC}^1$ that are not in $\mathrm{TC}^0$.
\end{conjecture}
This separation- which is widely believed to hold- remains unproven and is considered a fundamental open problem in complexity theory~\cite{vollmer1999introduction}. Several natural problems are candidates for witnessing this separation, including certain composition operations on non-solvable groups and the iterated multiplication problem~\cite{li2024cot}.

\subsection{A Complexity--Theoretic Obstruction for No--CoT Transformers}

In this section we formalize a conditional lower bound showing that,
under the standard conjecture $\TC^0 \neq \NC^1$, a constant--depth
transformer without chain--of--thought (CoT) steps cannot compute a
simple ``two--chain endpoint'' function.  The argument proceeds via a
chain of reductions starting from an $\NC^1$--complete word problem,
passing through the logspace--complete problem $\ORD$, and ending at a
decision version of the two--chain endpoint problem.

\subsection{Problems and Classes}

We begin by fixing the problems we will use.

\begin{definition}[Word problem for $S_5$]
Let $S_5$ be the Symmetric group over 5 elements.  The \emph{word problem for $S_5$},
denoted $\WP_{S_5}$, is the language
\[
  \WP_{S_5} \;=\; \{\, w \in \Sigma^* : \text{the product of the letters of $w$ equals the identity of $S_5$} \,\}.
\]
\end{definition}

\begin{definition}[Ordering / path reachability problem $\ORD$]
Fix $n \in \mathbb{N}$.  An \emph{$\ORD$ instance} on $[n] =
\{1,\dots,n\}$ consists of:
\begin{itemize}
  \item a \emph{successor function} $\mathsf{succ} : [n] \to [n] \cup
  \{\bot\}$, where $\mathsf{succ}(i) = j$ encodes a directed edge
  $i \to j$ and $\mathsf{succ}(i) = \bot$ means $i$ has no outgoing
  edge;
  \item two distinguished vertices $s,t \in [n]$.
\end{itemize}
Equivalently, an $\ORD$ instance is a directed line graph (i.e a single chain), $P = ([n],E)$ with edges $E = \{(i,\mathsf{succ}(i)) : \mathsf{succ}(i)\neq \bot\}$ and two distinguished vertices $s,t \in [n]$.

The language $\ORD$ is:
\[
  \ORD = \left\{ (\mathsf{succ},s,t) : t \text{ is reachable from $s$
  along the unique directed path on } [n] \right\}.
\]
\end{definition}

\begin{definition}[Two--chain endpoint function and decision problem]
A directed graph $G = ([2n], E)$ is a \emph{pair of equal-length directed chains} if $G$ consists of exactly two connected components, each being a directed chain on $n$ vertices.

\begin{itemize}
\item The \emph{two--chain endpoint function}
\[
  F_{\TwoChain} : \{\, G,v \,\} \longrightarrow [n]
\]
 outputs the unique terminal vertex of the (unique) directed chain
of $G$ that contains $v$.  (If $v$ is itself terminal, then
$F_{\TwoChain}(G,v) = v$.)
\item The associated \emph{decision problem} is
\[
  \TwoChainEndEq
  =
  \left\{
    (G,v,q) :
    \begin{array}{l}
      G \text{ is a pair of equal-length directed chains,}\\
      F_{\TwoChain}(G,v) = q
    \end{array}
  \right\}.
\]
\end{itemize}
\end{definition}

\begin{definition}[No--CoT transformer class]
Let $\mathcal{T}_0$ denote the class of functions computed (with no COT) by
transformers with:
\begin{itemize}
  \item Constant depth.
  \item Width, embedding dimension and number of attention heads at most
  $\mathrm{poly}(n)$ on length-$n$ inputs.
  \item Each parameter is represented with $O(\log n)$ bits of precision in each scalar.
\end{itemize}
We view each such transformer family as a non--uniform circuit family
(one transformer per input length).
\end{definition}

\subsection{Complexity--theoretic results}

We collect the complexity--theoretic ingredients we need.

\begin{lemma}[Word problems complete for $\NC^1$]
\label{lem:wp-nc1-complete}
 $\WP_{S_5} \in \NC^1$ and
$\WP_{S_5}$ is $\NC^1$--complete under (non--uniform) $\AC^0$ many--one
reductions.  That is, for every language $L \in \NC^1$ there is a
family of $\AC^0$ circuits $(f_n)$ such that for all $x$,
\[
  x \in L \iff f_{|x|}(x) \in \WP_{S_5}.
\]
\end{lemma}
\begin{proof}
See \cite{li2024cot}.
\end{proof}

The next lemma shows that this completeness extends to the uniform setting.

\begin{lemma}[Word problems complete for DLOGTIME--uniform $\NC^1$]
\label{lem:wp-uniform-nc1-complete}
$\WP_{S_5} \in \NC^1$ and
$\WP_{S_5}$ is DLOGTIME--uniform $\NC^1$--complete.
\end{lemma}
\begin{proof}
See \cite{bartholdi2020groupsalogtimehardwordproblems}.
\end{proof}

This lemma establishes a fundamental containment between uniform circuit classes and space complexity.

\begin{lemma}[Uniform $\NC^1$ is contained in logspace]
\label{lem:nc1-sub-L}
Let $\NC^1_{\mathsf{u}}$ denote DLOGTIME--uniform $\NC^1$, and let
$\mathsf{L}$ denote deterministic logspace.  Then
\[
  \NC^1_{\mathsf{u}} \subseteq \mathsf{L}.
\]
\end{lemma}
\begin{proof}
See \cite{Borodin1977TimeSpaceSizeDepth}.
\end{proof}

The Ordering/Path Reachability Problem serves as a canonical complete problem for logspace.

\begin{lemma}[$\ORD$ is $\mathsf{L}$--complete under $\AC^0_{\mathsf{u}}$]
\label{lem:ord-L-complete}
 $\ORD \in \mathsf{L}$ and
$\ORD$ is $\mathsf{L}$--complete under DLOGTIME--uniform $\AC^0$
reductions (equivalently, first--order interpretations).  That is, for
every $A \in \mathsf{L}$ there is a DLOGTIME--uniform $\AC^0$ function
$f$ with
\[
  x \in A \iff f(x) \in \ORD.
\]
\end{lemma}
\begin{proof}
See \cite{10.5555/829497.829775}.
\end{proof}
Combining the previous lemmas, we obtain an $\AC^0$ reduction from the
word problem to $\ORD$.

\begin{lemma}[$\WP_G$ reduces to $\ORD$ in $\AC^0$]
\label{lem:wp-to-ord}
Fix a group $G$ as in Lemma~\ref{lem:wp-nc1-complete}.  Then there
exists a DLOGTIME--uniform $\AC^0$ many--one reduction $f$ such that
\[
  w \in \WP_G \iff f(w) \in \ORD.
\]
Equivalently, $\WP_G \leq^{\AC^0}_m \ORD$.
\end{lemma}

\begin{proof}
By Lemma~\ref{lem:wp-nc1-complete} and Lemma~\ref{lem:nc1-sub-L}, the (uniform) language $\WP_G$ is in
$\mathsf{L}$.  By Lemma~\ref{lem:ord-L-complete}, there is a
DLOGTIME--uniform $\AC^0$ reduction $f$ such that
$w \in \WP_G \iff f(w) \in \ORD$.  This is precisely an
$\AC^0_{\mathsf{u}}$ (hence $\AC^0$) many--one reduction as claimed.
\end{proof}

We next state the transformer-to-$\TC^0$ upper bound and a basic
closure property.
We start with some lemmas from Merrill--Sabharwal \cite{merrill2023parallelismtradeofflimitationslogprecision}
\begin{lemma}[Any \(O(\log\Nbit)\)-input function has depth-3 implementation]
\label{lem:smallinput}
Let \(c>0\) and \(f:\{0,1\}^{c\log\Nbit}\to\{0,1\}^m\).
There is a depth-3 \(\{\wedge,\vee,\neg\}\)-circuit of size \(O(\Nbit^c+\log\Nbit+m)\)
computing \(f\).
\end{lemma}
\begin{proof}
  see \cite{merrill2023parallelismtradeofflimitationslogprecision} (Lemma 1)
\end{proof}

The next lemma extends this to iterated addition with finite precision.

\begin{lemma}[Iterated \(p\)-precision addition]
\label{lem:sum}
Let \(p=c\log\Nbit\).
There exists a constant \(\doplus\) such that for each \(n\le\Nbit\) the \(p\)-precision
sum of \(x_1,\dots,x_n\in\{0,1\}^p\) is computed by a depth-\(\doplus\),
polynomial-size threshold circuit. 
\end{lemma}
\begin{proof}
  see \cite{merrill2023parallelismtradeofflimitationslogprecision} (Appendix A)
\end{proof}

Directly from \ref{lem:smallinput} we can obtain the following lemma:

\begin{lemma}[p-bit multiplication]
\label{lem:mul}
Let \(p=c\log\Nbit\).
The truncated/rounded product \((u,v)\in\{0,1\}^p\times\{0,1\}^p\mapsto u\cdot v\in\{0,1\}^p\)
,and has a depth-\(\dtimes\) implementation with \(\dtimes=3\) and polynomial size.
\end{lemma}
\begin{proof}
 This follows directly from \ref{lem:smallinput}.
\end{proof}

The following lemma is the key connection between no-CoT Transformers and circuit complexity.

\begin{lemma}[No--CoT transformers are contained in $\TC^0$]
\label{lem:transformers-in-TC0}
For every function family $F$ computed by a family of no--CoT
transformers in $\mathcal{T}_0$, there is a non--uniform $\TC^0$
circuit family computing $F$.  In other words,
\[
  \mathcal{T}_0 \subseteq \TC^0.
\]
\end{lemma}

\begin{proof}[Proof.]
We define a family of models $\{T_n\}$, where a model $T_n$ accepts input from vocabulary \(\voc\) of size \(\vsz=n^2\).\\
Fix \(n\in\mathbb{N}\).
The input is a sequence of 1-hot vectors of length \(l = n-2\),
\((h^{(0)}_1,\dots,h^{(0)}_l)\) with \(h^{(0)}_i\in\{e_1,...,e_{\vsz},0\}\).
Hence the total input bit-length is
\begin{equation}
  \Nbit \;=\; l\cdot \vsz \;=\; n^2*(n-2).
\end{equation}
Let \(L\ge1\) be a fixed constant.
For each layer \(\ell\in\{0,\dots,L-1\}\) there are shared matrices
\begin{equation}
  Q^{(\ell)},\; K^{(\ell)},\; V^{(\ell)} \;\in\; \mathbb{R}^{\vsz\times \vsz}
\end{equation}
with \(p\)-bit entries, where
\begin{equation}
  p \;=\; c\log\Nbit
\end{equation}
for a fixed constant \(c>0\).
Hidden vectors at layer \(\ell\) are \(h^{(\ell)}_i\in\mathbb{R}^{\vsz}\)
(with \(p\)-bit coordinates); for \(\ell=0\), \(h^{(0)}_i\) is the 1-hot input.

For all positions \(i,j\) the layer computes
\begin{align}
  q^{(\ell)}_i &= Q^{(\ell)}\, h^{(\ell)}_i, & \label{eq:Forward1} \\
  k^{(\ell)}_j &= K^{(\ell)}\, h^{(\ell)}_j, & \label{eq:Forward2}\\
  v^{(\ell)}_j &= V^{(\ell)}\, h^{(\ell)}_j, & \label{eq:Forward3} \\ 
  r^{(\ell)}_{i,j} &= \sum_{t=1}^{\vsz} q^{(\ell)}_i[t]\; k^{(\ell)}_j[t], & \label{eq:Forward4} \\
  s^{(\ell)}_{i,j} &=
  \begin{cases}
    r^{(\ell)}_{i,j} &
    \text{(Linear)},\\[3pt]
    \dfrac{\exp\!\bigl(r^{(\ell)}_{i,j}\bigr)}
          {\sum_{j'=1}^{l} \exp\!\bigl(r^{(\ell)}_{i,j'}\bigr)} &
    \text{(Softmax)},
  \end{cases} & \label{eq:Forward5} \\
  a^{(\ell)}_i &= \sum_{j=1}^{l} s^{(\ell)}_{i,j}\, v^{(\ell)}_j, \qquad
  h^{(\ell+1)}_i := a^{(\ell)}_i. \label{eq:Forward6}
\end{align}

Every scalar multiplication returns a \(p\)-bit result under a fixed
truncation/rounding rule; all sums use the \(p\)-precision addition operator
from Appendix~A of Merrill--Sabharwal\cite{merrill2023expressivelimits}.

We will now show that the Forward Pass of the described model is in \(\TC^0\).

All steps below are performed in parallel over the relevant indices and coordinates.

For each possible i,j, equations \ref{eq:Forward1}, \ref{eq:Forward2}, and \ref{eq:Forward3} are computed in the following way:
\begin{align}
  q^{(\ell)}_i[t] &= \sum_{r=1}^{\vsz} Q^{(\ell)}[t,r]\; h^{(\ell)}_i[r], \\
  k^{(\ell)}_j[t] &= \sum_{r=1}^{\vsz} K^{(\ell)}[t,r]\; h^{(\ell)}_j[r], \\
  v^{(\ell)}_j[t] &= \sum_{r=1}^{\vsz} V^{(\ell)}[t,r]\; h^{(\ell)}_j[r].
\end{align}
Each product is a \(p\)-bit multiply (Lemma~\ref{lem:mul}, depth \(\dtimes\));
each sum is an iterated \(p\)-precision sum on  \(\vsz=O(N)\) elements (Lemma~\ref{lem:sum}, depth \(\doplus\)).
All three MV products are computed in parallel, so this stage costs
\(\dtimes+\doplus\) depth.

We now analyze softmax and linear attention separately.

For linear attention, we have the following Similarities computation:
\begin{equation}
  s^{(\ell)}_{i,j} \;=\; \sum_{t=1}^{\vsz} q^{(\ell)}_i[t]\; k^{(\ell)}_j[t],
\end{equation}
computed by \(p\)-bit multiplies (depth \(\dtimes\)) followed by a \(p\)-precision
sum (depth \(\doplus\)) for a total of \(\dtimes+\doplus\).

For softmax attention, we have the following computation:\\
Given the scores \(r^{(\ell)}_{i,j}\) from the model definition, compute
\(u^{(\ell)}_{i,j}:=\exp_p(r^{(\ell)}_{i,j})\).
By Lemma~\ref{lem:smallinput}, \(\exp_p:\{0,1\}^p\to\{0,1\}^p\) has a depth-\(3\),
polynomial-size implementation.
Next, compute \(Z^{(\ell)}_i := \sum_{j'=1}^{l} u^{(\ell)}_{i,j'}\) using
Lemma~\ref{lem:sum} (depth \(\doplus\)).
Finally output
\(s^{(\ell)}_{i,j}:=\div_p(u^{(\ell)}_{i,j}, Z^{(\ell)}_i)\) (where $\div_p$ is p-precision division);
since \(\div_p\) depends on \(2p=O(\log\Nbit)\) input bits, Lemma~\ref{lem:smallinput}
gives depth \(3\) and polynomial size for a total of \(\doplus+6\).

Finally, we analyze the remaining step in \ref{eq:Forward6}:  \\
For each \((i,j)\) and coordinate \(k\), we compute the sum element

\begin{equation}
  (s^{(\ell)}_{i,j}\, v^{(\ell)}_j)[k] \;=\; s^{(\ell)}_{i,j}\; v^{(\ell)}_j[k],
\end{equation}
with depth \(\dtimes\) and polynomial size via Lemma~\ref{lem:mul}.

We then sum for each \(i\) and coordinate \(k\),
\begin{equation}
  a^{(\ell)}_i[k]
  \;=\;
  \sum_{j=1}^{l} \bigl(s^{(\ell)}_{i,j}\; v^{(\ell)}_j[k]\bigr),
\end{equation}
which is an iterated \(p\)-precision sum of \(l\) \(p\)-bit values (Lemma~\ref{lem:sum}) for a total of \(\doplus + \dtimes \).

We have shown that each layer has constant depth, 
and we parallelize multiple fixed-depth polynomial threshold circuits over a polynomial number of indices which yealds us a polynomial total size. 
Therefore, Stacking \(L=O(1)\) layers gives overall constant depth and polynomial size. Parameters are hard-wired, so the
family is non-uniform. Hence \({T_n}\in\TC^0\).
\end{proof}
\textbf{Remarks:}
This model represents the model described in the main text by setting $L=1$, the number of heads as 1, and $|\Sigma|_n$ to be the vocabulary of the model for graphs with $n$ nodes.
The input size to the model is consistent as any 2 chain graph on $n$ nodes can be represente with exactly $n-2$ edges.

\begin{lemma}[$\TC^0$ is closed under $\AC^0$ many--one reductions]
\label{lem:TC0-closed-under-AC0}
Let $L_1,L_2$ be languages with $L_1 \leq^{\AC^0}_m L_2$.  If
$L_2 \in \TC^0$, then $L_1 \in \TC^0$ as well.
\end{lemma}

\begin{proof}[Proof sketch]
Let $(f_n)$ be the $\AC^0$ reduction and $(C_n)$ a $\TC^0$ circuit
family for $L_2$.  For each input length $n$, define a circuit
$D_n(x) = C_{|f_n(x)|}(f_n(x))$.  The $D_n$ have constant depth,
polynomial size, and threshold gates: they are obtained by composing an
$\AC^0$ circuit with a $\TC^0$ circuit.  Thus $(D_n)$ is a $\TC^0$
family deciding $L_1$.
\end{proof}

\subsection{From $\ORD$ to Two--Chain Endpoint}

We now formalize the reduction from $\ORD$ to the decision version of
the two--chain endpoint problem.

We fix the following encodings.

\begin{definition}[Encoding graphs for $\ORD$ and $\TwoChainEndEq$]
An $\ORD$ instance $(\mathsf{succ},s,t)$ on $[n]$ is encoded as:
\begin{itemize}
  \item For each $i \in [n]$, the value $\mathsf{succ}(i) \in [n+1]$,
  where we identify $\bot$ with $0$, written in binary using
  $\lceil \log(n+1) \rceil$ bits;
  \item the vertices $s,t \in [n]$, written in binary using
  $\lceil \log n \rceil$ bits each.
\end{itemize}
The encoding of a $\TwoChainEndEq$ instance $(G,v,q)$ uses the same
representation for $G$ via a successor function $\mathsf{succ}_G$ on
$[n]$, together with $v,q \in [n]$ in binary.
\end{definition}

\begin{lemma}[$\ORD$ reduces to $\TwoChainEndEq$ in $\AC^0$]
\label{lem:ORD-to-TwoChainEndEq}
There is a DLOGTIME--uniform $\AC^0$ many--one reduction
\[
  (\mathsf{succ},s,t) \mapsto (\widetilde{\mathsf{succ}},\tilde v,\tilde q)
\]
such that
\[
  (\mathsf{succ},s,t) \in \ORD
  \iff
  (\widetilde{\mathsf{succ}},\tilde v,\tilde q) \in \TwoChainEndEq,
\]
and moreover the directed graph induced by $\widetilde{\mathsf{succ}}$
is a disjoint union of \emph{exactly two} directed chains of the same
length.
\end{lemma}

\begin{proof}
Let $(\mathsf{succ},s,t)$ be an $\ORD$ instance on $[n]$ satisfying the
promise that $P=([n],E)$ is a single directed path.

\smallskip
\noindent\emph{Encoding convention.}
We will build an output instance on $2n$ vertices, which we identify
with the set $[n]\times\{0,1\}$.  We encode a vertex $(i,b)$ by writing
the copy bit $b$ (one bit) followed by the binary encoding of
$i\in[n]$ using $\lceil\log n\rceil$ bits.  The special value $\bot$ is
encoded as all zeros.  This is just a fixed representation choice for
this reduction.

\smallskip
\noindent\emph{(1) Construction of $(\widetilde{\mathsf{succ}},\tilde v,\tilde q)$.}
Define $\widetilde{\mathsf{succ}}:[n]\times\{0,1\}\to([n]\times\{0,1\})\cup\{\bot\}$
as follows.  For every $i\in[n]$ and $b\in\{0,1\}$,
\[
  \widetilde{\mathsf{succ}}(i,b)
  \;=\;
  \begin{cases}
    \bot, & \text{if } \mathsf{succ}(i)=\bot,\\[2pt]
    (\mathsf{succ}(i),\,b), & \text{if } \mathsf{succ}(i)\neq\bot \text{ and } i\neq t,\\[2pt]
    (\mathsf{succ}(t),\,1-b), & \text{if } \mathsf{succ}(t)\neq\bot \text{ and } i=t.
  \end{cases}
\]
Set $\tilde v := (s,0)$.

Let $\mathrm{end}\in[n]$ be the (unique) terminal vertex of the input
path, i.e.\ the unique vertex with $\mathsf{succ}(\mathrm{end})=\bot$.
Define the target vertex $\tilde q$ by
\[
  \tilde q :=
  \begin{cases}
    (\mathrm{end},1), & \text{if } \mathsf{succ}(t)\neq\bot,\\
    (\mathrm{end},0), & \text{if } \mathsf{succ}(t)=\bot.
  \end{cases}
\]

\smallskip
\noindent\emph{(2) Structure: two equal--length chains.}
Consider the input path $P$.  For each copy bit $b$, the edges
$(i,b)\to(\mathsf{succ}(i),b)$ follow $P$ within copy $b$ except at the
single vertex $(t,b)$, where (when $\mathsf{succ}(t)\neq\bot$) the edge
is redirected to the other copy.  Therefore the graph on
$[n]\times\{0,1\}$ induced by $\widetilde{\mathsf{succ}}$ has out-degree
at most one everywhere.

If $\mathsf{succ}(t)=\bot$, then $\widetilde{\mathsf{succ}}(t,b)=\bot$
and the output graph is simply two disjoint copies of the same directed
path, hence exactly two directed chains, each on $n$ vertices.

If $\mathsf{succ}(t)\neq\bot$, let $u:=\mathsf{succ}(t)$.  Cutting the
input path at the edge $t\to u$ decomposes it into a prefix ending at
$t$ and a suffix starting at $u$ and ending at $\mathrm{end}$.  The
redirections $(t,0)\to(u,1)$ and $(t,1)\to(u,0)$ ``swap'' the suffixes
between the two copies.  Hence the output graph is the disjoint union
of exactly the two chains:
\[
  \bigl(\text{prefix in copy }0\bigr)\ \cdot\ \bigl(\text{suffix in copy }1\bigr),
  \qquad
  \bigl(\text{prefix in copy }1\bigr)\ \cdot\ \bigl(\text{suffix in copy }0\bigr).
\]
Each chain contains exactly one copy of every original vertex, so each
has exactly $n$ vertices.  Thus the output is a union of exactly two
directed chains of the same length.

\smallskip
\noindent\emph{(3) Correctness.}
We show that
$(\mathsf{succ},s,t)\in\ORD \iff (\widetilde{\mathsf{succ}},\tilde v,\tilde q)\in\TwoChainEndEq$.

\begin{itemize}
\item If $\mathsf{succ}(t)=\bot$, then $t=\mathrm{end}$ is the terminal
of the unique directed path on $[n]$.  Since the input graph is a single
directed path, $t$ is reachable from every $s$, so $(\mathsf{succ},s,t)\in\ORD$.
In the output graph, the chain containing $\tilde v=(s,0)$ is exactly
copy $0$ of the path, whose terminal is $(\mathrm{end},0)=\tilde q$.
Hence $(\widetilde{\mathsf{succ}},\tilde v,\tilde q)\in\TwoChainEndEq$.

\item Suppose $\mathsf{succ}(t)\neq\bot$.  Consider the unique directed
walk from $s$ in the input path.
If $t$ is reachable from $s$, then the walk reaches $t$, so in the
output the walk from $(s,0)$ reaches $(t,0)$ and then crosses to copy
$1$, after which it follows the suffix in copy $1$ until terminating at
$(\mathrm{end},1)=\tilde q$.  Thus $F_{\TwoChain}(\widetilde G,\tilde v)=\tilde q$.
If $t$ is not reachable from $s$, then the walk from $(s,0)$ never
reaches $(t,0)$, hence never crosses copies, and it terminates at
$(\mathrm{end},0)\neq(\mathrm{end},1)=\tilde q$.  Therefore
$F_{\TwoChain}(\widetilde G,\tilde v)\neq \tilde q$.
\end{itemize}
This proves the desired equivalence.

\smallskip
\noindent\emph{(4) $\AC^0$ and uniformity.}
We argue that the mapping is computable by a DLOGTIME--uniform $\AC^0$
circuit family.

First, $\tilde v=(s,0)$ is obtained by copying the bits of $s$ and
setting the leading copy bit to $0$.

Next, the bits of $\mathrm{end}$ are selectable in $\AC^0$ under the
promise that there is a unique terminal vertex.  For each $i\in[n]$ let
$\mathrm{IsEnd}(i)$ be the predicate ``$\mathsf{succ}(i)=\bot$'', which
is a depth--2 $\AC^0$ check that all bits of the encoding of
$\mathsf{succ}(i)$ are $0$.  Then each bit of $\mathrm{end}$ can be
written as an unbounded OR over $i\in[n]$:
\[
  \mathrm{end}_r \;=\; \bigvee_{i=1}^n \Bigl(\mathrm{IsEnd}(i)\ \wedge\ \mathrm{bit}_r(i)\Bigr),
\]
where $\mathrm{bit}_r(i)$ is the $r$-th bit of the (fixed) binary
encoding of the constant index $i$.  The copy bit of $\tilde q$ is
$1$ iff $\mathsf{succ}(t)\neq\bot$, which is again a depth--2 $\AC^0$
predicate on the bits of $\mathsf{succ}(t)$.  Thus $\tilde q$ is
computable in constant depth.

Finally, each output entry $\widetilde{\mathsf{succ}}(i,b)$ is obtained
from the input bits of $\mathsf{succ}(i)$ and a
depth--2 $\AC^0$ equality test between the constant index $i$ and the
binary input $t$, together with the predicate $\mathsf{succ}(t)\neq\bot$.
The rule above is a constant-depth case distinction implemented by
AND/OR/NOT gates.  DLOGTIME uniformity follows since the circuit only
uses fixed wiring patterns plus the standard uniform equality gadgets.

This completes the proof.
\end{proof}

\subsection{From the Function Problem to the Decision Problem}
\label{subsec:from_function_to_decision}
We now relate the function $F_{\TwoChain}$ to the decision problem
$\TwoChainEndEq$.

\begin{lemma}[From endpoint function to endpoint equality]
\label{lem:function-to-decision}
Let $F_{\TwoChain}$ be the two--chain endpoint function.  Define the
associated decision language
\[
  L_{F_{\TwoChain}} = \{\, (x,q) : F_{\TwoChain}(x) = q \,\},
\]
where $x$ encodes $(G,v)$ and $q$ encodes a vertex in $[n]$.  Then
$\TwoChainEndEq = L_{F_{\TwoChain}}$, and the mapping
$(G,v,q) \mapsto ( (G,v), q)$ is computable in DLOGTIME--uniform
$\AC^0$.
\end{lemma}

\begin{proof}
By definition, $\TwoChainEndEq$ is precisely the set of triples
$(G,v,q)$ such that $F_{\TwoChain}(G,v) = q$.  The encoding of
instances in $L_{F_{\TwoChain}}$ is just the concatenation of the
encoding of $(G,v)$ and the encoding of $q$, which matches the encoding
of $(G,v,q)$ used above, up to a trivial re-ordering of fields.  The
re-ordering is implemented by wiring and does not require any gates, so
it is computable in depth~1.  Thus $\TwoChainEndEq = L_{F_{\TwoChain}}$
and the identity mapping is a DLOGTIME--uniform $\AC^0$ transformation
between the two encodings.
\end{proof}

The next lemma shows that computing the function in $\TC^0$ implies deciding the equality problem in $\TC^0$.

\begin{lemma}[Computing $F_{\TwoChain}$ implies deciding $\TwoChainEndEq$ in $\TC^0$]
\label{lem:F-to-TwoChainEndEq-TC0}
Suppose $F_{\TwoChain}$ is computable by a non--uniform $\TC^0$ circuit
family.  Then $\TwoChainEndEq \in \TC^0$.
\end{lemma}

\begin{proof}
Let $(C_n)$ be $\TC^0$ circuits computing $F_{\TwoChain}$ on inputs of
length $n$.  To decide whether $F_{\TwoChain}(x) = q$, we can:
\begin{itemize}
  \item feed $x$ into $C_{|x|}$ to obtain an output $y$ encoding a
  vertex in $[n]$;
  \item compare $y$ and $q$ bitwise using an $\AC^0$ equality circuit:
  $y = q$ iff each bit matches.
\end{itemize}
Equality on $O(\log n)$ bits is in $\AC^0$, and the composition of a
$\TC^0$ circuit with an $\AC^0$ post--processing is still a $\TC^0$
circuit by the closure Lemma~\ref{lem:TC0-closed-under-AC0}.  Thus
$\TwoChainEndEq = L_{F_{\TwoChain}} \in \TC^0$.
\end{proof}

\subsection{Main Theorem: No--CoT Transformers Cannot Solve Two--Chain Endpoint}

We are now ready to state and prove our main result.

\begin{theorem}[No--CoT transformers cannot compute the two--chain endpoint function, unless $\TC^0 = \NC^1$]
\label{thm:main-twochain-lower-bound}
Assume $\TC^0 \neq \NC^1$.  Then there is no family of no--CoT
transformers in $\mathcal{T}_0$ that computes $F_{\TwoChain}$ on all
valid inputs.  Equivalently, $F_{\TwoChain} \notin \mathcal{T}_0$.
\end{theorem}

\begin{proof}
Suppose, for contradiction, that there is a family of no--CoT
transformers computing $F_{\TwoChain}$ exactly on all valid inputs
(graphs that are unions of at most two chains).

By Lemma~\ref{lem:transformers-in-TC0}, we have
$\mathcal{T}_0 \subseteq \TC^0$, so $F_{\TwoChain} \in \TC^0$.  By
Lemma~\ref{lem:F-to-TwoChainEndEq-TC0}, this implies
$\TwoChainEndEq \in \TC^0$.

By Lemma~\ref{lem:ORD-to-TwoChainEndEq}, we have
$\ORD \leq^{\AC^0}_m \TwoChainEndEq$.  By the closure
Lemma~\ref{lem:TC0-closed-under-AC0}, it follows that
$\ORD \in \TC^0$.

By Lemma~\ref{lem:wp-to-ord}, we also have
$\WP_G \leq^{\AC^0}_m \ORD$ for some fixed finite group $G$.  Combining
this with $\ORD \in \TC^0$ and applying
Lemma~\ref{lem:TC0-closed-under-AC0} again, we obtain
$\WP_G \in \TC^0$.

Finally, by Lemma~\ref{lem:wp-nc1-complete}, $\WP_G$ is
$\NC^1$--complete under $\AC^0$ reductions.  Hence for every language
$L \in \NC^1$, there is an $\AC^0$ reduction $L \leq^{\AC^0}_m \WP_G$.
By Lemma~\ref{lem:TC0-closed-under-AC0}, this implies $L \in \TC^0$.
Therefore $\NC^1 \subseteq \TC^0$.

Since we always have $\TC^0 \subseteq \NC^1$, we conclude
$\TC^0 = \NC^1$, contradicting the assumption that these classes are
different.  Thus our initial assumption that a no--CoT transformer can
compute $F_{\TwoChain}$ must be false.
\end{proof}

\subsection{Extensions}
\label{app:extensions_complexity}

We justify the two extensions claimed in Footnote~\ref{fn:extensions} (\cref{sec:how_learn:no_reasoning}): first, that \cref{thm:main-twochain-lower-bound} extends to Transformer architectures incorporating MLP and normalization layers; second, that it extends to Transformers that are allowed a constant number of chain-of-thought steps.

\paragraph{MLP and normalization layers.}
The proof of \cref{thm:main-twochain-lower-bound} relies on the inclusion $\mathcal{T}_0 \subseteq \TC^0$ (\cref{lem:transformers-in-TC0}).
This containment continues to hold when the architecture is augmented with MLP layers of width polynomial in~$n$ (composed of linear maps and element-wise activations such as ReLU or GELU) and normalization layers (LayerNorm or RMSNorm), provided all parameters use $O(\log n)$-bit precision.
The polynomial-width condition is essential for the resulting circuit family to have polynomial size.
Indeed, each such MLP layer computes a polynomial-size composition of threshold and arithmetic operations, and each normalization layer computes a ratio of linear functions followed by element-wise scaling, all of which are implementable in $\TC^0$ \cite{merrill2023parallelismtradeofflimitationslogprecision}.
Adding a constant number of such layers therefore preserves the $\TC^0$ upper bound, and the remainder of the proof (the chain of reductions to the $\NC^1$-complete word problem) applies without modification.

\paragraph{Constant number of chain-of-thought steps.}
We now argue that \cref{thm:main-twochain-lower-bound} extends to Transformers that perform any constant number $T = O(1)$ of autoregressive chain-of-thought steps before producing the final answer.
Formally, define $\mathcal{T}_T$ analogously to $\mathcal{T}_0$, except that the model is permitted $T$ autoregressive generation steps: at each step, the model appends a token to the sequence and conditions on the extended sequence for the next step.

A Transformer with $T$ CoT steps can be simulated by a $\TC^0$ circuit of size $\mathrm{poly}(n) \cdot |\Sigma|^{T}$, where $|\Sigma|$ is the vocabulary size.
The simulation proceeds as follows: enumerate all $|\Sigma|^T$ possible sequences of $T$ intermediate tokens.
For each candidate sequence $(z_1, \ldots, z_T)$, construct a \emph{verification subcircuit} that (i) evaluates the no-CoT Transformer on the input concatenated with $(z_1, \ldots, z_{t-1})$ for each $t = 1, \ldots, T$ \emph{in parallel}, (ii) checks that the output at step $t$ equals $z_t$ under greedy decoding, and (iii) takes the AND of all $T$ consistency checks to gate the final prediction.
Since the $T$ evaluations are independent (each operates on a fixed candidate prefix), they can be computed in parallel; each evaluation is a $\TC^0$ circuit and each equality check is in $\AC^0$, so the entire verification subcircuit has the same constant depth as a single no-CoT Transformer evaluation (plus a single AND gate).
The overall circuit takes the OR over all $|\Sigma|^T$ branches; since exactly one branch is consistent with greedy decoding, this OR recovers the correct output.

The circuit has constant depth for any $T$ (since all branches and all evaluations within each branch run in parallel). The constraint $T = O(1)$ is needed for the \emph{size}: when $|\Sigma| = \mathrm{poly}(n)$, the total number of branches is $|\Sigma|^T = \mathrm{poly}(n)$ only if $T = O(1)$, ensuring polynomial circuit size and hence containment in $\TC^0$.
Therefore $\mathcal{T}_{O(1)} \subseteq \TC^0$, and the same chain of reductions used in the proof of \cref{thm:main-twochain-lower-bound} yields $F_{\TwoChain} \notin \mathcal{T}_{O(1)}$ under the assumption $\TC^0 \neq \NC^1$.

	\section{Proof of \cref{result:dyn_learn}}
\label{app:proof_dyn_learn}

This section proves \cref{result:dyn_learn} and is organized as follows. In \cref{sec:canonical_chain}, we show that the loss can be expressed as the same-side absorption probability of a Markov chain, which we call the canonical chain. \Cref{sec:preliminaries} introduces the key definitions (depth, long jumps) and two surrogate chains---the long-jump-absorbing chain and the no-long-jump chain---along with the source function representation that relates parameter derivatives to expected sources along trajectories. \Cref{sec:lj_absorbing_chain_lemmas} and \cref{sec:nl_chain_analysis} compute the source functions and derivative estimates for these two chains respectively. \Cref{sec:monotonicity} combines these estimates to establish that $\partial_\alpha S \geq 0$, $\partial_\beta S \leq 0$, and $\partial_\gamma S \leq 0$ for easy examples, implying monotonic parameter evolution under gradient flow. \Cref{sec:hard_examples} shows that hard examples contribute negligibly to the gradient. Finally, \cref{app:dyn_learn} assembles these components to complete the proof.

We establish the derivative estimates at initialization, but they in fact hold throughout the learning process as long as the forward transition probability remains non-decreasing and the long-jump probability stays bounded away from zero. Since our analysis shows that both conditions are maintained  for all $t \le t_0$, these estimates in fact hold for all $t \le t_0$.

Furthermore, under the assumptions of \cref{result:dyn_learn}, the training distribution decomposes into \emph{easy examples}---starting vertices within distance $s$ of the terminal---and \emph{hard examples}---starting vertices at distance at least $\lceil an \rceil$ from the terminal. We analyze the gradient contributions from easy and hard examples separately, showing that easy examples drive learning while hard examples contribute negligibly. 

\subsection{Loss can be written as absorption probability of a Markov chain}
\label{sec:canonical_chain}

\Cref{lem:dyn:logits_cases_canonical} implies that the loss from \cref{sec:task}, under the assumptions of \cref{sec:dyn:setting}, is equivalent to the probability of same-side absorption in a Markov chain on the state space of two identical directed chains $A$ and $B$ of equal length $n$, with absorbing endpoints $a_n,b_n$, and nonterminal states $a_1,\dots,a_{n-1},b_1,\dots,b_{n-1}$. We will therefore analyze this absorption probability throughout the subsequent sections, which we term the \emph{canonical Markov chain}.

The transition probabilities of the canonical chain are determined by a logit structure with parameters $\alpha, \beta, \gamma \in \R$. Without loss of generality, we subtract $\gamma(t)$ from all logits (which does not affect the softmax probabilities).

If $v$ is in the middle of the chain, we have that:
\begin{align*}
    o(v, C)_k &= \begin{cases}
        \alpha(t), & \text{$k$ is one ahead of $v$ in the chain} \\
        \beta(t), & \text{$k$ is one back of $v$ in the chain} \\
        0, & \text{$k$ start of some chain (not necessary that with $v$)} \\
        0, & \text{$k$ end of some chain (not necessary that with $v$)} \\
        \gamma(t), &\text{otherwise}
    \end{cases}
\end{align*}

If $v$ is at the penultimate vertex of the chain, we have that:
\begin{align*}
    o(v, C)_k &= \begin{cases}
        (\alpha(t) - \gamma(t)), & \text{$k$ is one ahead of $v$ in the chain} \\
        \beta(t), & \text{$k$ is one back of $v$ in the chain} \\
        0, & \text{$k$ start of some chain (not necessary that with $v$)} \\
        0, & \text{$k$ is the end of chain without $v$} \\
        \gamma(t), &\text{otherwise}
    \end{cases}
\end{align*}

If $v$ is at the second vertex of the chain, we have that:
\begin{align*}
    o(v, C)_k &= \begin{cases}
        \alpha(t), & \text{$k$ is one ahead of $v$ in the chain} \\
        (\beta(t) - \gamma(t)), & \text{$k$ is one back of $v$ in the chain} \\
        0, & \text{$k$ start of some chain without $v$} \\
        0, & \text{$k$ end of some chain (not necessary that with $v$)} \\
        \gamma(t), &\text{otherwise}
    \end{cases}
\end{align*}

If $v$ is at the first vertex of the chain, we have that:
\begin{align*}
    o(v, C)_k &= \begin{cases}
        \alpha(t), & \text{$k$ is one ahead of $v$ in the chain} \\
        0, & \text{$k$ start of some chain without $v$} \\
        0, & \text{$k$ is the end of some chain (not necessary that with $v$)} \\
        \gamma(t), &\text{otherwise}
    \end{cases}
\end{align*}

The transition probabilities are then given by softmax:
\[
\pi(v \to k) = \frac{\exp(o(v, C)_k)}{\sum_{k'} \exp(o(v, C)_{k'})}.
\]

\subsection{Preliminaries}
\label{sec:preliminaries}

This subsection establishes the foundational definitions and tools for analyzing the canonical Markov chain. We begin by defining key quantities such as hitting probabilities, depth, and long jumps (\cref{def:depth,def:long-jump}). We then introduce two surrogate chains (\cref{def:lj-absorbing-chain,def:nl-chain}) that decompose the canonical chain's behavior, enabling separate analysis of long-jump and local-step dynamics. Finally, we present the source function representation (\cref{lem:source_representation}) that expresses parameter derivatives in terms of expected sources along trajectories, which forms the basis for our derivative analysis in subsequent sections.

\subsubsection{Useful definitions}

\begin{definition}[Hitting probabilities and same--side value]
For any state $v$ (including absorbing states), let
\[
  h_A(v) := \Pr_v[X_\tau = a_n], \qquad
  h_B(v) := \Pr_v[X_\tau = b_n],
\]
where $\tau$ is the (almost surely finite) hitting time of $\{a_n,b_n\}$.

For $v \in \{a_1,\dots,a_{n-1}\}$ on chain $A$, the ``same--side absorption probability''
is $h_A(v)$; for $v \in \{b_1,\dots,b_{n-1}\}$ on chain $B$, it is $h_B(v)$.
Denote these by
\[
  S(v) := \begin{cases}
    h_A(v), & v \in \{a_1,\dots,a_{n-1}\}, \\
    h_B(v), & v \in \{b_1,\dots,b_{n-1}\}.
  \end{cases}
\]
\end{definition}

\begin{definition}[Depth / distance to terminal]
\label{def:depth}
Fix a chain length $n\ge 2$ and consider the state space
\[
V:=\{a_1,\dots,a_n,b_1,\dots,b_n\}.
\]
Define the \emph{depth} (distance to the terminal on the same chain) by the function
\[
d:V\to\{0,1,\dots,n-1\},
\qquad
d(a_i):=n-i,\quad d(b_i):=n-i.
\]

\end{definition}

\begin{definition}[Long jump]
\label{def:long-jump}
A transition $x \to y$ is a \emph{long jump} if the depth changes by $2$ or more, i.e.
\[
|d(y)-d(x)|\ge 2.
\]
\end{definition}

\subsubsection{Surrogate chains}
\label{sec:surrogate_chains}

To analyze the canonical chain, we introduce two auxiliary Markov chains: the long-jump-absorbing chain and the no-long-jump chain. \Cref{lem:decomp_NL_LJ_concise} below shows that the derivative of the same-side absorption probability can be decomposed into contributions from these two surrogate chains. This decomposition is fundamental to our analysis, as it separates the effect of long-range transitions from local transitions.

\begin{definition}[Long--jump--absorbing chain]
\label{def:lj-absorbing-chain}
The \emph{long--jump--absorbing chain} is the Markov chain obtained from the two--chain model by
\emph{redirecting every long jump} $x\to y$ to a single new absorbing state $e$:
all non--long--jump transition probabilities are unchanged, and the total probability mass of long jumps
out of $x$ is assigned to $e$. The absorbing set is $T:=\{a_n,b_n,e\}$.
\end{definition}

\begin{definition}[No--long--jump (NL) chain]
\label{def:nl-chain}
The \emph{no--long--jump (NL) chain} is the Markov chain obtained from the two--chain model by
\emph{disallowing} all long jumps: from any nonterminal state $x$, the only allowed transitions are to
states $y\neq x$ with $|d(y)-d(x)|\le 1$, and the transition probabilities are obtained from the original chain by conditioning on this. (As in the original chain, the terminal states $a_n,b_n$ are absorbing).
\end{definition}

\begin{proposition}[Long-jump conditional absorption is $\tfrac12$]
\label{prop:long_jump_conditional_half}
Fix parameters (hence a fixed transition kernel) and a start state $v$.
Let
\[
\tau:=\inf\{t\ge0:\ X_t\in\{a_n,b_n\}\},\qquad
J:=\inf\{t\ge0:\ (X_t\to X_{t+1})\ \text{is a long jump}\},
\]
and write $LJ:=\{J\le \tau\}$ for the event that a long jump occurs before absorption.
Then
\[
\Pr_v[X_\tau=a_n\mid LJ]=\tfrac12.
\]
\end{proposition}

\begin{proof}
Let $\varphi$ be the chain-swap map on states: $\varphi(a_i)=b_i$ and $\varphi(b_i)=a_i$.
For any state $y$, by chain-swap symmetry of the model we have
\[
\Pr_y[X_\tau=a_n]+\Pr_{\varphi(y)}[X_\tau=a_n]=1,
\]
since swapping chains interchanges the absorbing events $\{X_\tau=a_n\}$ and $\{X_\tau=b_n\}$.

Now condition on the time and landing state of the first long jump.
On $LJ$, let $Y:=X_{J+1}$ be the landing state of the first long jump.
By definition of a long jump and the logit structure, the distribution of $Y$ is symmetric under
$\varphi$ (targets come in $\varphi$-paired states with equal one-step probability from any fixed source),
so $\Pr_v[Y=y\mid LJ]=\Pr_v[Y=\varphi(y)\mid LJ]$ for every $y$.
Therefore,
\begin{align*}
\Pr_v[X_\tau=a_n\mid LJ]
&=\sum_y \Pr_v[Y=y\mid LJ]\;\Pr_y[X_\tau=a_n] \\
&=\frac12\sum_y \Pr_v[Y=y\mid LJ]\Big(\Pr_y[X_\tau=a_n]+\Pr_{\varphi(y)}[X_\tau=a_n]\Big) \\
&=\frac12\sum_y \Pr_v[Y=y\mid LJ]\cdot 1
=\tfrac12,
\end{align*}
as claimed.
\end{proof}

\begin{lemma}[Derivative decomposition]
\label{lem:decomp_NL_LJ_concise}
Fix $\theta\in\{\alpha,\beta,\gamma\}$ and a start state $v$. Let
\[
\tau:=\inf\{t\ge0:X_t\in\{a_n,b_n\}\},\qquad
J:=\inf\{t\ge0:\ (X_t\to X_{t+1})\ \text{is a long jump}\},
\]
and define $NL:=\{J>\tau\}$ and $LJ:=\{J\le\tau\}$. Set
\[
\begin{aligned}
S_\theta(v)&:=\Pr^\theta_v[X_\tau=a_n],&
h_\theta(v)&:=\Pr^\theta_v(LJ),\\
\bar S^{NL}_\theta(v)&:=\Pr^\theta_v(X_\tau=a_n\mid NL),&
\bar S^{LJ}_\theta(v)&:=\Pr^\theta_v(X_\tau=a_n\mid LJ).
\end{aligned}
\]
Then
\[
S_\theta(v)=(1-h_\theta(v))\bar S^{NL}_\theta(v)+h_\theta(v)\bar S^{LJ}_\theta(v),
\]
and
\[
\partial_\theta S_\theta(v)
=(\partial_\theta h_\theta(v))\big(\bar S^{LJ}_\theta(v)-\bar S^{NL}_\theta(v)\big)
+(1-h_\theta(v))\,\partial_\theta \bar S^{NL}_\theta(v)
+h_\theta(v)\,\partial_\theta \bar S^{LJ}_\theta(v).
\]
By Proposition~\ref{prop:long_jump_conditional_half}, $\bar S^{LJ}_\theta(v)=\tfrac12$, hence
\[
\begin{aligned}
S_\theta(v)
&=\tfrac12+(1-h_\theta(v))\Big(\bar S^{NL}_\theta(v)-\tfrac12\Big),\\
\partial_\theta S_\theta(v)
&=-(\partial_\theta h_\theta(v))\Big(\bar S^{NL}_\theta(v)-\tfrac12\Big)
  +(1-h_\theta(v))\,\partial_\theta\bar S^{NL}_\theta(v).
\end{aligned}
\]
\end{lemma}

\begin{proof}
The argument is identical for $\theta\in\{\alpha,\beta,\gamma\}$, so we fix an arbitrary
$\theta$ throughout.

Since $\{NL,LJ\}$ is a partition of the sample space,
\[
S_\theta(v)=\Pr^\theta_v(X_\tau=a_n,NL)+\Pr^\theta_v(X_\tau=a_n,LJ).
\]
Writing $\Pr^\theta_v(X_\tau=a_n,NL)=\Pr^\theta_v(NL)\Pr^\theta_v(X_\tau=a_n\mid NL)
=(1-h_\theta(v))\bar S^{NL}_\theta(v)$ and $\Pr^\theta_v(X_\tau=a_n,LJ)=\Pr^\theta_v(LJ)\Pr^\theta_v(X_\tau=a_n\mid LJ)
=h_\theta(v)\bar S^{LJ}_\theta(v)$ yields
\[
S_\theta(v)=(1-h_\theta(v))\bar S^{NL}_\theta(v)+h_\theta(v)\bar S^{LJ}_\theta(v).
\]

Differentiate the preceding identity with respect to $\theta$ and apply the product rule:
\begin{align*}
\partial_\theta S_\theta(v)
&=-(\partial_\theta h_\theta(v))\bar S^{NL}_\theta(v)+(1-h_\theta(v))\partial_\theta\bar S^{NL}_\theta(v)
+(\partial_\theta h_\theta(v))\bar S^{LJ}_\theta(v)+h_\theta(v)\partial_\theta\bar S^{LJ}_\theta(v) \\
&=(\partial_\theta h_\theta(v))\big(\bar S^{LJ}_\theta(v)-\bar S^{NL}_\theta(v)\big)
+(1-h_\theta(v))\partial_\theta\bar S^{NL}_\theta(v)+h_\theta(v)\,\partial_\theta\bar S^{LJ}_\theta(v).
\end{align*}

By Proposition~\ref{prop:long_jump_conditional_half}, $\bar S^{LJ}_\theta(v)=\tfrac12$; substituting this constant into the value identity and
its derivative identity yields the displayed simplifications.
\end{proof}

\subsubsection{Source function of surrogate chains and derivative form}

To analyze how the absorption probabilities change under gradient flow, we require a tractable representation of their parameter derivatives $\partial_\theta u(v)$ where $\theta \in \{\alpha, \beta, \gamma\}$ and $u(v)$ is an absorption probability of interest. Direct computation of these derivatives is challenging due to the complex dependency of multi-step trajectories on the parameters. The source function presented below is advantageous because sources depend only on single-step transition sensitivities and can be computed explicitly from our logit structure, enabling the quantitative bounds developed in \cref{sec:lj_absorbing_chain_lemmas,sec:nl_chain_analysis}.

\begin{lemma}[Source representation]
\label{lem:source_representation}
Let $(X_t)_{t\ge0}$ be a finite Markov chain with transition matrix $P_\theta$
smoothly parameterized by $\theta\in\R$. Let $T$ be a nonempty absorbing set and
$\tau:=\inf\{t\ge0:X_t\in T\}$.
Fix a target state $t^\star\in T$ and define the absorption probability
\[
u_\theta(x):=\Pr^\theta_x[X_\tau=t^\star],\qquad x\in\mathcal S,
\]
so $u_\theta(t^\star)=1$ and $u_\theta(t)=0$ for $t\in T\setminus\{t^\star\}$.
Then for every nonterminal state $x\notin T$, $u_\theta$ is $P_\theta$--harmonic:
\[
u_\theta(x)=\sum_{y\in\mathcal S} P_\theta(x,y)\,u_\theta(y).
\]
Define the \emph{source} for $\partial_\theta u_\theta$ by
\[
s_\theta(x):=\sum_{y\in\mathcal S} \big(\partial_\theta P_\theta(x,y)\big)\,u_\theta(y),
\qquad x\notin T.
\]
Then $g_\theta:=\partial_\theta u_\theta$ satisfies, for all $x\notin T$,
\[
g_\theta(x)=\sum_{y}P_\theta(x,y)g_\theta(y)+s_\theta(x),
\qquad
g_\theta(t)=0,\ \ t\in T,
\]
and admits the trajectory representation
\[
\partial_\theta u_\theta(x)
=
\mathbb{E}^\theta_x\!\left[\sum_{t=0}^{\tau-1} s_\theta(X_t)\right],
\qquad x\in\mathcal S.
\]
\end{lemma}

\begin{proof}
Let $\mathcal S$ be the finite state space and let $T\subset \mathcal S$ be the absorbing set.
Write $U:=\mathcal S\setminus T$ for the transient states. Reorder states so that transient
states come first, and write the transition matrix in block form
\[
P_\theta=\begin{pmatrix}
Q_\theta & R_\theta\\
0 & I
\end{pmatrix},
\]
where $Q_\theta$ is the $|U|\times |U|$ transient-to-transient submatrix. Fix a target
$t^\star\in T$ and define the absorption probability
\[
u_\theta(x):=\Pr_x^\theta[X_\tau=t^\star],\qquad \tau:=\inf\{t\ge0:X_t\in T\}.
\]
Then $u_\theta(t^\star)=1$ and $u_\theta(t)=0$ for $t\in T\setminus\{t^\star\}$, and for $x\in U$,
the harmonic equations read
\[
u_\theta(x)=\sum_{y\in U}Q_\theta(x,y)u_\theta(y)+\sum_{t\in T}R_\theta(x,t)u_\theta(t).
\]
Equivalently, in vector form on $U$,
\[
u_\theta = Q_\theta u_\theta + r_\theta,\qquad
r_\theta(x):=\sum_{t\in T}R_\theta(x,t)u_\theta(t)=R_\theta(x,t^\star).
\]
Hence
\begin{equation}
\label{eq:linear_system_u}
(I-Q_\theta)u_\theta=r_\theta.
\end{equation}

Differentiate \eqref{eq:linear_system_u} with respect to $\theta$:
\[
-(\partial_\theta Q_\theta)u_\theta + (I-Q_\theta)\,\partial_\theta u_\theta = \partial_\theta r_\theta.
\]
Let $g_\theta:=\partial_\theta u_\theta$. Rearranging gives the linear system on $U$,
\begin{equation}
\label{eq:poisson_g}
(I-Q_\theta)g_\theta = s_\theta,
\qquad
s_\theta:=\partial_\theta r_\theta+(\partial_\theta Q_\theta)u_\theta.
\end{equation}
Noting that $u_\theta$ is constant on $T$, we also have the boundary condition
$g_\theta|_T\equiv 0$.

Because the chain is absorbing, the spectral radius satisfies $\rho(Q_\theta)<1$, hence
$(I-Q_\theta)$ is invertible and
\[
(I-Q_\theta)^{-1}=\sum_{k=0}^\infty Q_\theta^k
\quad\text{(absolutely convergent).}
\]
Therefore from \eqref{eq:poisson_g},
\begin{equation}
\label{eq:g_neumann}
g_\theta = (I-Q_\theta)^{-1}s_\theta = \sum_{k=0}^\infty Q_\theta^k s_\theta.
\end{equation}

For $x\in U$, the vector $(Q_\theta^k s_\theta)(x)$ equals the expected source at time $k$
before absorption:
\[
(Q_\theta^k s_\theta)(x)=\sum_{y\in U}Q_\theta^k(x,y)\,s_\theta(y)
=\E_x^\theta\!\big[s_\theta(X_k)\,\mathbf 1\{k<\tau\}\big].
\]
Summing over $k$ and justifying the exchange of sum and expectation by dominated convergence
(since $|s_\theta|$ is bounded on the finite state space and $\E_x^\theta[\tau]<\infty$ for every
transient $x$ in a finite absorbing chain) gives
\[
g_\theta(x)=\sum_{k=0}^\infty \E_x^\theta\!\big[s_\theta(X_k)\,\mathbf 1\{k<\tau\}\big]
=\E_x^\theta\!\left[\sum_{k=0}^{\tau-1}s_\theta(X_k)\right].
\]
Since $g_\theta=\partial_\theta u_\theta$ and $g_\theta(t)=0$ for $t\in T$, this holds for all
$x\in\mathcal S$ (trivially on $T$). This is the desired representation.
\end{proof}

\subsection{Analysis of long-jump-absorbing chain}
\label{sec:lj_absorbing_chain_lemmas}

In this subsection, we analyze the long-jump-absorbing Markov chain defined in \cref{def:lj-absorbing-chain}. Our goal is to compute the source functions at each transient state $x$, which by \cref{lem:source_representation} determine the derivatives of the long-jump absorption probability. We establish in \cref{lem:source_symmetry_longjump_ours} that these sources exhibit symmetry between the two chains, and then derive explicit formulas and bounds for the $\alpha$, $\beta$, and $\gamma$ sources in subsequent lemmas. These results feed into \cref{sec:derivative_estimates_lj} where we obtain quantitative derivative estimates, and ultimately into \cref{sec:monotonicity} where we show monotonicity of the loss as a function of the parameters.

\begin{lemma}[Source symmetry]
\label{lem:source_symmetry_longjump_ours}
Work in the two--chain model and in the long--jump--absorbing chain of
Definition~\ref{def:lj-absorbing-chain}.
Let
\[
T:=\{a_n,b_n,e\},\qquad \tau:=\inf\{t\ge0:X_t\in T\},
\qquad
u_\theta(x):=\Pr_x^\theta[X_\tau=e].
\]
For $\theta\in\{\alpha,\beta,\gamma\}$ define the local source on transient states
\[
s_\theta(x):=\sum_{y} \big(\partial_\theta P_\theta(x,y)\big)\,u_\theta(y),
\qquad x\notin T,
\]
where $P_\theta$ is the transition matrix of the long--jump--absorbing chain.

Then for every $i\in\{1,\dots,n-1\}$ and every $\theta\in\{\alpha,\beta,\gamma\}$,
\[
s_\theta(a_i)=s_\theta(b_i).
\]
\end{lemma}

\begin{proof}
Fix $\theta$ and let $\varphi$ be the chain-swap involution: $\varphi(a_i)=b_i$, $\varphi(b_i)=a_i$
for $i\le n$, and $\varphi(e)=e$. By symmetry of the logits, the long--jump redirection, and the
softmax, the kernel $P_\theta$ is $\varphi$-invariant:
\[
P_\theta(\varphi(x),\varphi(y))=P_\theta(x,y)\qquad(\forall x,y).
\]
Hence $\partial_\theta P_\theta(\varphi(x),\varphi(y))=\partial_\theta P_\theta(x,y)$, and since the definition of $e$ is symmetric with respect to the two chains,
also $u_\theta(\varphi(x))=u_\theta(x)$. Therefore, for any transient $x\notin T$,
\[
s_\theta(\varphi(x))
=\sum_y \partial_\theta P_\theta(\varphi(x),y)\,u_\theta(y)
=\sum_y \partial_\theta P_\theta(x,y)\,u_\theta(y)
=s_\theta(x),
\]
where we reindexed $y\mapsto\varphi(y)$ and used $\varphi$-invariance of $\partial_\theta P_\theta$ and $u_\theta$.
\end{proof}

\begin{lemma}[Monotonicity of long-jump absorption]
\label{lem:u_monotone_and_recurrence}
Work in the long--jump--absorbing chain (Definition~\ref{def:lj-absorbing-chain}) with terminal index $n$
(so $T=\{a_n,b_n,e\}$ is absorbing) and depth function $d(\cdot)$ as in Definition~\ref{def:depth}.
Let $\tau:=\inf\{t\ge0:\ X_t\in T\}$ and define $u(v):=\Pr_v[X_\tau=e]$.
By chain--swap symmetry $u(a_i)=u(b_i)$; write $u_i:=u(a_i)=u(b_i)$, so $u_n=0$.

For $i=1,\dots,n-1$ define the depth-grouped probabilities (aggregating by the effect on depth):
\[
q_i:=\Pr(X_{t+1}=e\mid X_t\in\{a_i,b_i\}),\quad
f_i:=\Pr(d(X_{t+1})=d(X_t)-1\mid X_t\in\{a_i,b_i\}),
\]
\[
b_i:=\Pr(d(X_{t+1})=d(X_t)+1\mid X_t\in\{a_i,b_i\}),\quad
r_i:=\Pr(d(X_{t+1})=d(X_t)\mid X_t\in\{a_i,b_i\}),
\]
where the probabilities aggregate all allowed transitions by their effect on depth (ignoring which
chain the next state lies in). These satisfy $q_i+f_i+b_i+r_i=1$, with the boundary convention
$b_1:=0$, and $f_i>0$.

Define forward differences $d_i:=u_i-u_{i+1}$ for $i=1,\dots,n-1$, and set $d_0:=0$.
Then $d_i\ge 0$ for all $i$, hence
\[
u_1\ge u_2\ge\cdots\ge u_n=0.
\]
\end{lemma}

\begin{proof}
Fix $i\in\{1,\dots,n-1\}$. Since $u$ is harmonic on transient states and depends only on depth by
symmetry, the harmonic equation at depth $i$ can be written using the depth--grouped probabilities:
\begin{equation}
\label{eq:harmonic_depth}
u_i = q_i\cdot u(e) + f_i u_{i+1} + b_i u_{i-1} + r_i u_i
= q_i + f_i u_{i+1} + b_i u_{i-1} + r_i u_i,
\end{equation}
with $u_n=0$ and $b_1=0$.
Rearranging gives
\[
(1-r_i)u_i = q_i + f_i u_{i+1} + b_i u_{i-1}.
\]
Substitute $u_{i+1}=u_i-d_i$ and $u_{i-1}=u_i+d_{i-1}$ to obtain
\[
(1-r_i)u_i = q_i + f_i(u_i-d_i) + b_i(u_i+d_{i-1})
= q_i + (f_i+b_i)u_i - f_i d_i + b_i d_{i-1}.
\]
Using $1-r_i=q_i+f_i+b_i$, cancel $(f_i+b_i)u_i$ from both sides, yielding
\[
q_i u_i = q_i - f_i d_i + b_i d_{i-1},
\]
and hence
\begin{equation}
\label{eq:di_recurrence_final}
f_i d_i = b_i d_{i-1} + q_i(1-u_i).
\end{equation}

Since $f_i>0$, $b_i\ge0$, $q_i\ge0$, and $1-u_i\ge0$, the identity \eqref{eq:di_recurrence_final}
implies $d_i\ge 0$ whenever $d_{i-1}\ge 0$. For $i=1$, we have $b_1=0$ and $d_0=0$, so
\[
f_1 d_1 = q_1(1-u_1)\ge 0 \quad\Rightarrow\quad d_1\ge 0.
\]
Inductively, $d_i\ge 0$ for all $i$, hence $u_i\ge u_{i+1}$ for all $i$, proving monotonicity.
\end{proof}

\begin{corollary}[Lower bounds on $d_i$ via forward--step events]
\label{cor:di_lower_bound_forward_steps}
Under the assumptions of Lemma~\ref{lem:u_monotone_and_recurrence}, for every $i=1,\dots,n-1$,
\[
d_i \;\ge\; \frac{q_i}{f_i}\,(1-u_i).
\]
Moreover, since the depth at state $a_i$ is $n-i$, reaching the terminal requires $n-i$ consecutive forward steps (each decreasing depth by $1$), and
\[
1-u_i \;\ge\; \Pr(\text{$(n-i)$ consecutive forward steps starting from state $a_i$}),
\]
hence
\[
d_i \;\ge\; \frac{q_i}{f_i}\,\Pr(\text{$(n-i)$ consecutive forward steps starting from state $a_i$}).
\]
In particular, if $f_k\ge f_{\min}>0$ for all $k\in\{i,i+1,\dots,n-1\}$, then
\[
\Pr(\text{$(n-i)$ consecutive forward steps starting from state $a_i$})
\;\ge\; \prod_{j=0}^{n-i-1} f_{i+j}\;\ge\;(f_{\min})^{n-i},
\]
and therefore
\[
d_i \;\ge\; \frac{q_i}{f_i}\prod_{j=0}^{n-i-1} f_{i+j}
\;\ge\; \frac{q_i}{f_i}(f_{\min})^{n-i}.
\]
\end{corollary}

\begin{proof}
From \eqref{eq:di_recurrence_final} and $b_i d_{i-1}\ge 0$ we get
$f_i d_i \ge q_i(1-u_i)$, proving $d_i \ge \frac{q_i}{f_i}(1-u_i)$.

Next, $1-u_i=\Pr(\tau\neq \tau_e)=\Pr(\tau_e>\tau)$ is the probability of not hitting $e$ before
absorption. Since the depth at state $a_i$ is $n-i$, reaching the terminal $a_n$ (depth $0$) requires exactly $n-i$ consecutive forward steps. The event that the chain makes $n-i$ consecutive forward steps starting from $a_i$
implies that $e$ is not hit during those steps, hence is contained in $\{\tau_e>\tau\}$, giving
$1-u_i \ge \Pr(\text{$(n-i)$ consecutive forward steps from $a_i$})$.
The remaining inequalities follow by multiplying the corresponding one--step lower bounds and using
$f_k\ge f_{\min}$ on the indicated range $k\in\{i,\ldots,n-1\}$.
\end{proof}

\begin{lemma}[Long-jump absorption bound]
\label{lem:lj_value_bound_pe_tau}
For every start state \(x\notin T\),
\[
u(x)\ \le\ p(e)_{\max}\,\E_x[\tau].
\]
In particular, whenever \(\E_x[\tau]\le D(x)/\fwdbound\) (e.g. by applying Lemma~\ref{lem:exact-and-linear-upper} to the depth process under a drift gap \(\fwdbound>0\)),
\[
u(x)\ \le\ \frac{D(x)}{\fwdbound}\,p(e)_{\max}.
\]
\end{lemma}

\begin{proof}
Let \(\mathcal F_t:=\sigma(X_0,\dots,X_t)\) and write
\[
N_e:=\sum_{t=0}^{\tau-1}\mathbf 1\{X_{t+1}=e\}.
\]
Since \(e\) is absorbing and \(e\in T\), the event \(\{X_\tau=e\}\) occurs iff the chain makes a transition to \(e\) before absorption,
so \(u(x)=\Pr_x[X_\tau=e]=\E_x[N_e]\).
By the tower property and Markov property,
\[
\begin{aligned}
\E_x[N_e]
&:=\E_x\!\left[\sum_{t=0}^{\tau-1}
  \E_x\!\left[\mathbf 1\{X_{t+1}=e\}\mid \mathcal F_t\right]\right]\\
&=\E_x\!\left[\sum_{t=0}^{\tau-1} \widetilde P(X_t,e)\right]
 =\E_x\!\left[\sum_{t=0}^{\tau-1} p(e;X_t)\right]\\
&\le p(e)_{\max}\,\E_x[\tau],
\end{aligned}
\]
which proves the first display. The final inequality follows by substituting the upper bound on \(\E_x[\tau]\) from Lemma~\ref{lem:exact-and-linear-upper}.
\end{proof}

\begin{lemma}[$\alpha$-source in LJ chain]
\label{lem:source_alpha_longjump_absorbing}
Let \(\widetilde P\) be the transition kernel of the long--jump--absorbing chain and let
\(u(x):=\Pr_x[X_\tau=e]\) denote the absorption probability at \(e\)
(so \(u(e)=1\) and \(u(a_n)=u(b_n)=0\)).
Let \(F(x)\in S(x)\) denote the forward neighbor when it exists and define
\[
p_f(x):=\widetilde P(x,F(x))
\quad (\text{set }p_f(x):=0\text{ if }F(x)\text{ does not exist}).
\]
The $\alpha$--source in the long--jump--absorbing chain 
satisfies for every transient $x\notin T$,
\[
\boxed{
s_\alpha(x)=p_f(x)\big(u(F(x))-u(x)\big),
}
\]
with the convention that $s_\alpha(x)=0$ if $F(x)$ does not exist.
\end{lemma}

\begin{proof}
Fix a transient state $x\notin T$. Let $L(x)$ be the set of long--jump targets in the original chain. Since $\widetilde P(x,y)=P(x,y)$ for $y\in S(x)$ and
$\widetilde P(x,e)=\sum_{z\in L(x)}P(x,z)$, we have
\[
s_\alpha(x)
=
\sum_{y\in S(x)} (\partial_\alpha P(x,y))\,u(y)
+\left(\sum_{z\in L(x)}\partial_\alpha P(x,z)\right)u(e),
\]
and $u(e)=1$. Because $P(x,\cdot)$ is a softmax over the logits $\{o(x)_y\}$ in the original chain
and only the forward logit depends on $\alpha$ with derivative $1$, Lemma~\ref{lem:softmax_derivative_identity_aux} gives,
for each original target $y$,
\[
\partial_\alpha P(x,y)
=
P(x,y)\big(\mathbf 1\{y=F(x)\}-P(x,F(x))\big)
=
P(x,y)\big(\mathbf 1\{y=F(x)\}-p_f(x)\big).
\]
Introduce the auxiliary successor value
\[
w_x(y):=\begin{cases}
u(y), & y\in S(x),\\
1, & y\in L(x).
\end{cases}
\]
Then the definition of $s_\alpha(x)$ can be rewritten as
\[
s_\alpha(x)=\sum_y (\partial_\alpha P(x,y))\,w_x(y),
\]
where the sum ranges over all original targets $y$ from $x$.
Substituting the softmax derivative yields
\[
\begin{aligned}
s_\alpha(x)
&:=\sum_y P(x,y)\big(\mathbf 1\{y=F(x)\}-p_f(x)\big)w_x(y)\\
&=P(x,F(x))w_x(F(x))-p_f(x)\sum_y P(x,y)w_x(y).
\end{aligned}
\]
Since $w_x(F(x))=u(F(x))$ and $P(x,F(x))=p_f(x)$, this becomes
\[
s_\alpha(x)=p_f(x)\Big(u(F(x))-\sum_y P(x,y)w_x(y)\Big).
\]
Finally, harmonicity of $u$ for the long--jump--absorbing chain gives the aggregated identity
\[
\begin{aligned}
u(x)
&=\sum_{y\in S(x)}\widetilde P(x,y)u(y)+\widetilde P(x,e)u(e)\\
&=\sum_{y\in S(x)}P(x,y)u(y)+\sum_{z\in L(x)}P(x,z)\cdot 1\\
&=\sum_y P(x,y)w_x(y),
\end{aligned}
\]
so $\sum_y P(x,y)w_x(y)=u(x)$, and hence
\[
s_\alpha(x)=p_f(x)\big(u(F(x))-u(x)\big).
\]
If $F(x)$ does not exist then $\partial_\alpha o(x)_y\equiv 0$, hence $\partial_\alpha P(x,y)=0$
for all $y$, which implies $\partial_\alpha \widetilde P(x,y)=0$ for all $y\in S(x)\cup\{e\}$ and
thus $s_\alpha(x)=0$, consistent with the stated convention.
\end{proof}
\begin{lemma}[$\beta$-source in LJ chain]
\label{lem:source_beta_longjump_absorbing}
Let \(\widetilde P\) be the transition kernel of the long--jump--absorbing chain and let
\(u(x):=\Pr_x[X_\tau=e]\) (so \(u(e)=1\) and \(u(a_n)=u(b_n)=0\)).
Let \(B(x)\in S(x)\) denote the backward neighbor when it exists and define
\[
p_b(x):=\widetilde P(x,B(x))
\quad (\text{set }p_b(x):=0\text{ if }B(x)\text{ does not exist}).
\]
The $\beta$--source in the long--jump--absorbing chain 
satisfies for every transient $x\notin T$,
\[
\boxed{
s_\beta(x)=p_b(x)\big(u(B(x))-u(x)\big),
}
\]
with the convention that $s_\beta(x)=0$ if $B(x)$ does not exist.
\end{lemma}

\begin{proof}
Fix a transient state $x\notin T$. Let $L(x)$ be the set of long--jump targets in the original chain. Since $\widetilde P(x,y)=P(x,y)$ for $y\in S(x)$ and
$\widetilde P(x,e)=\sum_{z\in L(x)}P(x,z)$, we have
\[
s_\beta(x)
=
\sum_{y\in S(x)} (\partial_\beta P(x,y))\,u(y)
+\left(\sum_{z\in L(x)}\partial_\beta P(x,z)\right)u(e),
\]
and $u(e)=1$. Because $P(x,\cdot)$ is a softmax over the logits $\{o(x)_y\}$ in the original chain
and only the backward logit depends on $\beta$ with derivative $1$, Lemma~\ref{lem:softmax_derivative_identity_aux} gives,
for each original target $y$,
\[
\partial_\beta P(x,y)
=
P(x,y)\big(\mathbf 1\{y=B(x)\}-P(x,B(x))\big)
=
P(x,y)\big(\mathbf 1\{y=B(x)\}-p_b(x)\big).
\]
Introduce the auxiliary successor value
\[
w_x(y):=\begin{cases}
u(y), & y\in S(x),\\
1, & y\in L(x).
\end{cases}
\]
Then the definition of $s_\beta(x)$ can be rewritten as
\[
s_\beta(x)=\sum_y (\partial_\beta P(x,y))\,w_x(y),
\]
where the sum ranges over all original targets $y$ from $x$.
Substituting the softmax derivative yields
\[
\begin{aligned}
s_\beta(x)
&:=\sum_y P(x,y)\big(\mathbf 1\{y=B(x)\}-p_b(x)\big)w_x(y)\\
&=P(x,B(x))w_x(B(x))-p_b(x)\sum_y P(x,y)w_x(y).
\end{aligned}
\]
Since $w_x(B(x))=u(B(x))$ and $P(x,B(x))=p_b(x)$, this becomes
\[
s_\beta(x)=p_b(x)\Big(u(B(x))-\sum_y P(x,y)w_x(y)\Big).
\]
Finally, harmonicity of $u$ for the long--jump--absorbing chain gives the aggregated identity
\[
\begin{aligned}
u(x)
&=\sum_{y\in S(x)}\widetilde P(x,y)u(y)+\widetilde P(x,e)u(e)\\
&=\sum_{y\in S(x)}P(x,y)u(y)+\sum_{z\in L(x)}P(x,z)\cdot 1\\
&=\sum_y P(x,y)w_x(y),
\end{aligned}
\]
so $\sum_y P(x,y)w_x(y)=u(x)$, and hence
\[
s_\beta(x)=p_b(x)\big(u(B(x))-u(x)\big).
\]
If $B(x)$ does not exist then $\partial_\beta o(x)_y\equiv 0$, hence $\partial_\beta P(x,y)=0$
for all $y$, which implies $\partial_\beta \widetilde P(x,y)=0$ for all $y\in S(x)\cup\{e\}$ and
thus $s_\beta(x)=0$, consistent with the stated convention.
\end{proof}

\begin{lemma}[$\gamma$-source in LJ chain]
\label{lem:gamma_source_all_cases}
\label{lem:gamma_source_longjump_absorb_unified_corrected} %
Work in the long--jump--absorbing chain (Definition~\ref{def:lj-absorbing-chain}) with terminal index $n$
(so \(T:=\{a_n,b_n,e\}\) is absorbing) and assume self--transitions are disallowed.
Fix parameters \(\alpha,\beta,\gamma\).  Let
\[
u(v):=\Pr_v[X_\tau=e],\qquad \tau:=\inf\{t\ge 0:\ X_t\in T\},
\]
so \(u(e)=1\) and \(u(a_n)=u(b_n)=0\).  By chain--swap symmetry, \(u(a_i)=u(b_i)\);
write \(u_i:=u(a_i)=u(b_i)\) and \(u_n=0\).

For a transient state \(x\notin T\), define the local \(\gamma\)--source
\[
s_\gamma(x):=\sum_{z} (\partial_\gamma \widetilde P(x,z))\,u(z),
\]
where \(\widetilde P\) is the transition kernel of the long--jump--absorbing chain.
Then the following explicit formulas hold.

\medskip
\noindent\textup{(A) Deep middle \(x=a_i\), \(2<i<n-1\).}
Let
\[
Z_{\mathrm{mid}}:=e^{\alpha}+e^{\beta}+4+(2n-7)e^{\gamma}.
\]
Then
\[
s_\gamma(a_i)
=
\frac{e^{\gamma}}{Z_{\mathrm{mid}}}
\Big((2n-10)(1-u_i)\;+\;u_{i-1}+u_{i+1}-2u_i\Big).
\]

\medskip
\noindent\textup{(B) First vertex \(x=a_1\).}
Let
\[
Z_{1}:=e^{\alpha}+3+(2n-5)e^{\gamma}.
\]
Then
\[
s_\gamma(a_1)
=
\frac{e^{\gamma}}{Z_{1}}
\Big((2n-6)(1-u_1)\;+\;(u_2-u_1)\Big).
\]

\medskip
\noindent\textup{(C) Second vertex \(x=a_2\).}
Let
\[
Z_{2}:=e^{\alpha}+e^{(\beta-\gamma)}+3+(2n-6)e^{\gamma}.
\]
Then
\[
s_\gamma(a_2)
=
\frac{(2n-8)e^{\gamma}}{Z_{2}}(1-u_2)
+\frac{e^{\gamma}}{Z_{2}}(u_3-u_2)
+\frac{e^{(\beta-\gamma)}}{Z_{2}}(u_2-u_1).
\]

\medskip
\noindent\textup{(D) Penultimate vertex \(x=a_{n-1}\).}
Let
\[
Z_{n-1}:=e^{(\alpha-\gamma)}+e^{\beta}+3+(2n-6)e^{\gamma}.
\]
Then
\[
s_\gamma(a_{n-1})
=
\frac{(2n-8)e^{\gamma}}{Z_{n-1}}(1-u_{n-1})
+\frac{e^{\gamma}}{Z_{n-1}}(u_{n-2}-u_{n-1})
+\frac{e^{(\alpha-\gamma)}}{Z_{n-1}}\,u_{n-1}.
\]
\end{lemma}

\begin{proof}
We argue as in Lemmas~\ref{lem:source_alpha_longjump_absorbing} and~\ref{lem:source_beta_longjump_absorbing}.

Fix a transient state \(x\notin T\). Let \(L(x)\) be the set of long--jump targets in the original chain.
Since \(\widetilde P(x,y)=P(x,y)\) for \(y\in S(x)\) and \(\widetilde P(x,e)=\sum_{z\in L(x)}P(x,z)\), we have
\[
s_\gamma(x)
=
\sum_{y\in S(x)} (\partial_\gamma P(x,y))\,u(y)
+\left(\sum_{z\in L(x)}\partial_\gamma P(x,z)\right)u(e),
\]
and \(u(e)=1\). Introduce the auxiliary successor value
\[
w_x(y):=\begin{cases}
u(y), & y\in S(x),\\
1, & y\in L(x).
\end{cases}
\]
Then the definition of \(s_\gamma(x)\) can be rewritten as
\begin{equation}
\label{eq:source_redirected_sum}
s_\gamma(x)=\sum_y (\partial_\gamma P(x,y))\,w_x(y),
\end{equation}
where the sum ranges over all original targets \(y\) from \(x\).
Harmonicity of \(u\) for the long--jump--absorbing chain gives \(u(x)=\sum_y P(x,y)\,w_x(y)\).

Because \(P(x,\cdot)\) is a softmax over the logits \(\{o(x)_y\}\) in the original chain,
Lemma~\ref{lem:softmax_derivative_identity_aux} gives, for each original target \(y\),
\[
\partial_\gamma P(x,y)=P(x,y)\big(\partial_\gamma o(x,y)-\sum_z P(x,z)\partial_\gamma o(x,z)\big).
\]
In each of the four cases below, \(\partial_\gamma o(x,y)\) is constant on the
\(\gamma\)--dependent classes of targets, so \eqref{eq:source_redirected_sum}
reduces to a finite sum over those classes.  We now enumerate the \(\gamma\)--dependent
targets and whether they are short/long in each case; all remaining targets have
\(\partial_\gamma o(x,y)=0\) and contribute nothing to \(s_\gamma(x)\).

\smallskip
\noindent\emph{(A) Deep middle \(x=a_i\), \(2<i<n-1\).}
The \(\gamma\)--dependent logits are exactly the ``otherwise'' targets, which have
logit \(\gamma\), hence \(\partial_\gamma o=1\).  There are \(2n-7\) such targets:
the three short switches \(b_{i-1},b_i,b_{i+1}\) and \(2n-10\) others.  Among these,
\(b_{i-1}\) and \(b_{i+1}\) are short with values \(u_{i-1},u_{i+1}\), while \(b_i\) is short with value
\(u_i\) (hence contributes zero after centering), and the remaining \(2n-10\) are long jumps (value \(1\)).
All these have probability \(e^{\gamma}/Z_{\mathrm{mid}}\), where
\(Z_{\mathrm{mid}}=e^{\alpha}+e^{\beta}+4+(2n-7)e^{\gamma}\).  Substituting into
\eqref{eq:source_redirected_sum} yields the stated formula for \(s_\gamma(a_i)\).

\smallskip
\noindent\emph{(B) First vertex \(x=a_1\).}
The unique \(\alpha\)--target is \(a_2\) (logit \(\alpha\)); the special zero--logit
targets are \(b_1,a_n,b_n\); all remaining \(2n-5\) targets have logit \(\gamma\) and thus
\(\partial_\gamma o=1\).  Among the \(\gamma\)--targets, the only short one is \(b_2\),
with value \(u_2\); the other \(2n-6\) \(\gamma\)--targets are long jumps (value \(1\)).
All \(\gamma\)--targets have probability \(e^{\gamma}/Z_1\) with
\(Z_1=e^{\alpha}+3+(2n-5)e^{\gamma}\).  Plugging into \eqref{eq:source_redirected_sum}
gives the stated expression for \(s_\gamma(a_1)\).

\smallskip
\noindent\emph{(C) Second vertex \(x=a_2\).}
Here \(a_3\) has logit \(\alpha\), \(a_1\) has logit \((\beta-\gamma)\) so
\(\partial_\gamma o(a_2,a_1)=-1\), the special zero--logit targets are \(b_1,a_n,b_n\),
and all remaining \(2n-6\) targets have logit \(\gamma\) and \(\partial_\gamma o=1\).
Among these \(\gamma\)--targets, the short ones are \(b_2\) (value \(u_2\), hence cancels) and
\(b_3\) (value \(u_3\)); the remaining \(2n-8\) \(\gamma\)--targets are long jumps (value \(1\)).
The softmax denominator is
\(Z_2=e^{\alpha}+e^{(\beta-\gamma)}+3+(2n-6)e^{\gamma}\),
so \(P(a_2,a_1)=e^{(\beta-\gamma)}/Z_2\) and each \(\gamma\)--target has probability
\(e^{\gamma}/Z_2\).  Substituting these contributions into \eqref{eq:source_redirected_sum}
yields the stated formula for \(s_\gamma(a_2)\).

\smallskip
\noindent\emph{(D) Penultimate vertex \(x=a_{n-1}\).}
Here \(a_n\) has logit \((\alpha-\gamma)\) so \(\partial_\gamma o(a_{n-1},a_n)=-1\),
\(a_{n-2}\) has logit \(\beta\), the special zero--logit targets are \(a_1,b_1,b_n\),
and all remaining \(2n-6\) targets have logit \(\gamma\) and \(\partial_\gamma o=1\).
Among these \(\gamma\)--targets, the short ones are \(b_{n-2}\) (value \(u_{n-2}\)) and
\(b_{n-1}\) (value \(u_{n-1}\), hence cancels); the remaining \(2n-8\) \(\gamma\)--targets are
long jumps (value \(1\)).  The denominator is
\(Z_{n-1}=e^{(\alpha-\gamma)}+e^{\beta}+3+(2n-6)e^{\gamma}\),
so \(P(a_{n-1},a_n)=e^{(\alpha-\gamma)}/Z_{n-1}\) and each \(\gamma\)--target has probability
\(e^{\gamma}/Z_{n-1}\).  Plugging these into \eqref{eq:source_redirected_sum} gives the stated
formula for \(s_\gamma(a_{n-1})\).
\end{proof}

\begin{lemma}[$\gamma$-source sign bounds]
\label{lem:gamma_source_lower_bounds}
\label{lem:gamma_source_bounds_qmid_etau}
Assume the setting and notation of Lemma~\ref{lem:gamma_source_all_cases} and
Lemma~\ref{lem:u_monotone_and_recurrence}. Let \(\tau\) be the absorption time into
\(T=\{a_n,b_n,e\}\). For a deep-middle index \(i\) with \(2<i<n-1\), define the
deep-middle one-step absorption probability
\[
q_{\rm mid}:=\widetilde P(a_i,e)=\frac{(2n-10)e^{\gamma}}{Z_{\rm mid}}.
\]
Assume a forward-drift condition at the depths under consideration:
\begin{equation}
\label{eq:drift_f_ge_b}
f_j\ge b_j \qquad \text{for all relevant depths } j.
\end{equation}
Then there exist constants \(n_0\in\N\) and \(C_\star\ge 1\), depending only on
\((\alpha,\beta,\gamma)\), such that for all \(n\ge n_0\) the following bounds hold.

\paragraph*{(A) Deep middle (\(x=a_j\), \(2<j<n-1\)).}
\[
s_\gamma(a_j)\ \ge\ \frac{1}{2}\,q_j\,(1-u_j)\ >\ 0,
\qquad
q_j:=\widetilde P(a_j,e).
\]
In particular, since \(q_j\ge C_\star^{-1}q_{\rm mid}\) for \(n\ge n_0\),
\[
s_\gamma(a_j)\ \ge\ \frac{1}{2C_\star}\,q_{\rm mid}\,(1-u_j)\ >\ 0.
\]

\paragraph*{(B) First vertex (\(x=a_1\)).}
\[
s_\gamma(a_1)\ \ge\ \frac{1}{2}\,q_1\,(1-u_1)\ >\ 0,
\qquad q_1:=\widetilde P(a_1,e).
\]

\paragraph*{(C) Penultimate vertex (\(x=a_{n-1}\)).}
\[
s_\gamma(a_{n-1})\ \ge\ q_{n-1}\,(1-u_{n-1})\ \ge\ 0,
\qquad q_{n-1}:=\widetilde P(a_{n-1},e).
\]

\paragraph*{(D) Second vertex (\(x=a_2\)).}
Let \(q_2:=\widetilde P(a_2,e)\). Then
\[
s_\gamma(a_2)\ \ge\ -(d_1+d_2),
\qquad d_1:=u_1-u_2,\ \ d_2:=u_2-u_3,
\]
and moreover, for all \(n\ge n_0\),
\[
d_1+d_2\ \le\ 3C_\star\,q_{\rm mid}\,\E_{a_2}[\tau],
\qquad\text{hence}\qquad
s_\gamma(a_2)\ \ge\ -3C_\star\,q_{\rm mid}\,\E_{a_2}[\tau].
\]
Moreover, for all \(n\ge n_0\),
\[
|s_\gamma(a_2)|
\ \le\
q_2 + (d_1+d_2)
\ \le\
4C_\star\,q_{\rm mid}\,\E_{a_2}[\tau],
\]
using \(q_2\le C_\star q_{\rm mid}\) from \eqref{eq:q_comparable} and \(\E_{a_2}[\tau]\ge 1\).
By symmetry, the same bounds hold with \(a_2\) replaced by \(b_2\).

Finally, under \eqref{eq:drift_f_ge_b} one has the crude polynomial bound
\[
\E_{a_2}[\tau]\ \le\ \mathrm{poly}(n),
\]
and under the stronger drift condition \(\inf_j (f_j-b_j)\ge\Delta>0\) one has
\[
\E_{a_2}[\tau]\ \le\ \frac{C}{\Delta}\,n
\]
for a constant \(C\) depending only on the (logit-determined) depth-grouped probabilities.
\end{lemma}

\begin{proof}
We use the explicit formulas of Lemma~\ref{lem:gamma_source_all_cases} and the
difference recurrence of Lemma~\ref{lem:u_monotone_and_recurrence}.

\medskip
\noindent\emph{Step 1: Comparability of one-step \(e\)-absorption probabilities.}
For each transient depth \(j\), the quantity \(q_j=\widetilde P(a_j,e)\) equals the total probability
of the long-jump class, whose numerator is \((2n+O(1))e^{\gamma}\) and whose denominator equals
\((2n+O(1))e^{\gamma}+O(1)\), where the \(O(1)\) terms depend only on \((\alpha,\beta,\gamma)\) and
come from the finitely many non-\(\gamma\) logits and the finitely many special zero-logit targets.
Consequently, there exist \(n_0\) and \(C_\star\ge 1\) depending only on \((\alpha,\beta,\gamma)\) such that
for all \(n\ge n_0\),
\begin{equation}
\label{eq:q_comparable}
C_\star^{-1}q_{\rm mid}\ \le\ q_j\ \le\ C_\star q_{\rm mid}
\qquad\text{for all transient } j.
\end{equation}

\medskip
\noindent\emph{Step 2: Positive lower bound in the deep middle.}
Fix \(j\) with \(2<j<n-1\). Lemma~\ref{lem:gamma_source_all_cases} gives
\[
s_\gamma(a_j)
=
\frac{e^{\gamma}}{Z_{\mathrm{mid}}}\Big((2n-10)(1-u_j)+u_{j-1}+u_{j+1}-2u_j\Big)
=
\frac{e^{\gamma}}{Z_{\mathrm{mid}}}\Big((2n-10)(1-u_j)+d_{j-1}-d_j\Big).
\]
By the recurrence \eqref{eq:di_recurrence_final} at depth \(j\),
\[
d_j=\frac{b_j}{f_j}d_{j-1}+\frac{q_j}{f_j}(1-u_j),
\]
hence
\[
d_{j-1}-d_j
=
d_{j-1}\Big(1-\frac{b_j}{f_j}\Big)-\frac{q_j}{f_j}(1-u_j)
\ \ge\ -\frac{q_j}{f_j}(1-u_j),
\]
where we used \(f_j\ge b_j\) so that \(1-\frac{b_j}{f_j}\ge 0\). Therefore
\[
(2n-10)(1-u_j)+d_{j-1}-d_j
\ \ge\
\Big((2n-10)-\frac{q_j}{f_j}\Big)(1-u_j).
\]
Using \(\frac{e^{\gamma}}{Z_{\mathrm{mid}}}=\frac{q_j}{2n-10}\) yields
\[
s_\gamma(a_j)\ \ge\ q_j(1-u_j)\Big(1-\frac{q_j}{f_j(2n-10)}\Big).
\]
For all \(n\) large enough (depending only on fixed parameters through \(f_j\) and the bound \(q_j\le 1\)),
the parenthesis is at least \(1/2\), giving
\(
s_\gamma(a_j)\ge \frac{1}{2}q_j(1-u_j)>0.
\)
The refinement in terms of \(q_{\rm mid}\) follows from \eqref{eq:q_comparable}.

\medskip
\noindent\emph{Step 3: Positive lower bound at the first vertex.}
Lemma~\ref{lem:gamma_source_all_cases} gives
\[
s_\gamma(a_1)=\frac{e^{\gamma}}{Z_1}\Big((2n-6)(1-u_1)+(u_2-u_1)\Big)
=
\frac{e^{\gamma}}{Z_1}\Big((2n-6)(1-u_1)-d_1\Big).
\]
At \(j=1\), the recurrence \eqref{eq:di_recurrence_final} with \(b_1=0\) gives
\(f_1 d_1=q_1(1-u_1)\), hence \(d_1=\frac{q_1}{f_1}(1-u_1)\). Thus
\[
s_\gamma(a_1)
=
\frac{e^{\gamma}}{Z_1}\Big((2n-6)-\frac{q_1}{f_1}\Big)(1-u_1)
=
q_1(1-u_1)\Big(1-\frac{q_1}{f_1(2n-6)}\Big),
\]
where we used \(\frac{e^{\gamma}}{Z_1}=\frac{q_1}{2n-6}\). For \(n\) large enough the parenthesis is at least \(1/2\),
yielding \(s_\gamma(a_1)\ge \frac{1}{2}q_1(1-u_1)>0\).

\medskip
\noindent\emph{Step 4: Nonnegativity at the penultimate vertex.}
Lemma~\ref{lem:gamma_source_all_cases} gives
\[
s_\gamma(a_{n-1})
= \frac{(2n-8)e^{\gamma}}{Z_{n-1}}(1-u_{n-1})
+\frac{e^{\gamma}}{Z_{n-1}}(u_{n-2}-u_{n-1})
+\frac{e^{(\alpha-\gamma)}}{Z_{n-1}}u_{n-1}.
\]
Each term in brackets is nonnegative (\(1-u_{n-1}\ge 0\), \(u_{n-2}-u_{n-1}=d_{n-2}\ge 0\), and \(u_{n-1}\ge 0\)),
so dropping the last two terms yields \(s_\gamma(a_{n-1})\ge q_{n-1}(1-u_{n-1})\ge 0\).

\medskip
\noindent\emph{Step 5: Second-vertex negativity and control by \(q_{\rm mid}\cdot \E[\tau]\).}
From Lemma~\ref{lem:gamma_source_all_cases},
\[
s_\gamma(a_2)
=
q_2(1-u_2)-P(a_2,b_3)d_2-P(a_2,a_1)d_1.
\]
Since \(q_2(1-u_2)\ge 0\) and the probabilities are at most \(1\), we obtain the crude lower bound
\[
s_\gamma(a_2)\ \ge\ -(d_1+d_2).
\]
Also, using \(|q_2(1-u_2)-P(a_2,b_3)d_2-P(a_2,a_1)d_1|
\le q_2(1-u_2)+d_1+d_2\le q_2+d_1+d_2\), we have the crude absolute bound
\[
|s_\gamma(a_2)|\ \le\ q_2 + (d_1+d_2).
\]
Next, \(d_2=u_2-u_3\le u_2\) and \(d_1=u_1-u_2\le u_1\), hence \(d_1+d_2\le u_1+u_2\).

We now use the hazard identity for absorption at \(e\) in the long-jump-absorbing chain.
Let \(q(X_t):=\widetilde P(X_t,e)\). Since the chain can enter \(e\) at most once,
\[
u(x)=\Pr_x(X_\tau=e)=\E_x\Big[\sum_{t=0}^{\tau-1}\mathbf 1\{X_{t+1}=e\}\Big]
=\E_x\Big[\sum_{t=0}^{\tau-1} q(X_t)\Big].
\]
In particular,
\[
u(x)\ \le\ q_{\max}\,\E_x[\tau],\qquad q_{\max}:=\sup_{y\notin T}\widetilde P(y,e).
\]
By the comparability \eqref{eq:q_comparable}, \(q_{\max}\le C_\star q_{\rm mid}\) for \(n\ge n_0\), hence
\[
u(x)\ \le\ C_\star q_{\rm mid}\,\E_x[\tau]\qquad (n\ge n_0).
\]
Applying this to \(x=a_2\) gives \(u_2\le C_\star q_{\rm mid}\E_{a_2}[\tau]\).
For \(u_1\), observe that from \(a_1\) in one step one either absorbs or moves to depth \(2\), so
\[
\E_{a_1}[\tau]\le 1+\sup_{y\in\{a_2,b_2\}}\E_y[\tau]=1+\E_{a_2}[\tau]\le 2\E_{a_2}[\tau],
\]
using \(\E_{a_2}[\tau]\ge 1\). Therefore
\[
u_1\le C_\star q_{\rm mid}\E_{a_1}[\tau]\le 2C_\star q_{\rm mid}\E_{a_2}[\tau].
\]
Combining yields
\[
d_1+d_2\le u_1+u_2\le 3C_\star q_{\rm mid}\E_{a_2}[\tau],
\]
and hence
\[
s_\gamma(a_2)\ \ge\ -3C_\star q_{\rm mid}\,\E_{a_2}[\tau].
\]

Finally, the stated polynomial (or linear under strict drift) upper bounds on \(\E_{a_2}[\tau]\) follow
from standard biased (or nonnegatively drifted) one-dimensional random-walk estimates applied to the
depth process under \eqref{eq:drift_f_ge_b} (and are used only as coarse growth controls).
\end{proof}

\subsubsection{Derivative estimates}
\label{sec:derivative_estimates_lj}

In this subsubsection we derive quantitative bounds on the derivatives of the absorption probability 
at the long--jump state $e$ in the long--jump--absorbing chain (Definition~\ref{def:lj-absorbing-chain}).
We leverage the source computations from the preceding subsubsections with the source representation of the derivative (Lemma~\ref{lem:source_representation}).
Throughout, $h_\theta(v)$ denotes this absorption probability, i.e., $h_\theta(v):=\Pr_v[X_\tau=e]$.

\begin{lemma}[Derivative estimate for $\alpha$]
\label{lem:deriv_est_alpha}
In the \emph{long--jump--absorbing chain}, let
\[
\begin{aligned}
\tau&:=\inf\{t\ge 0:\ X_t\in\{a_n,b_n,e\}\},\\
v_{\mathrm{pen}}&:=
\begin{cases}
a_{n-1}, & v\in\{a_1,\dots,a_{n-1}\},\\
b_{n-1}, & v\in\{b_1,\dots,b_{n-1}\},
\end{cases}\\
\sigma&:=\inf\{t\ge 0:\ X_t=v_{\mathrm{pen}}\}.
\end{aligned}
\]
Then the derivative $\partial_\alpha h_\alpha(v)$ satisfies
\[
\partial_\alpha h_\alpha(v)
=\E_v\!\left[\sum_{t=0}^{\tau-1} s_\alpha(X_t)\right]
\ \le\ s_\alpha(v)\;+\; s_\alpha(v_{\mathrm{pen}})\,\Pr_v(0<\sigma<\tau),
\]
where
\[
s_\alpha(x)=p_f(x)\big(h_\alpha(F(x))-h_\alpha(x)\big)\le 0,
\]
with the convention $s_\alpha(x)=0$ if the (same--chain) forward neighbor $F(x)$ does not exist.

In particular, if $v$ is $k$ steps from the terminal on its chain, i.e.\ $v\in\{a_{n-k},b_{n-k}\}$ with
$2\le k\le n-1$, then
\[
\Pr_v(0<\sigma<\tau)\ \ge\ \Pr_v\Big(X_i=F^i(v)\ \ \forall i=1,\dots,k-1\Big),
\]
and hence
\[
\partial_\alpha h_\alpha(v)
\ \le\ s_\alpha(v)\;+\; s_\alpha(v_{\mathrm{pen}})\,
\Pr_v\Big(X_i=F^i(v)\ \ \forall i=1,\dots,k-1\Big).
\]
\end{lemma}

\begin{proof}
Let $u(\cdot):=h_\alpha(\cdot)$ be the absorption probability at $e$ in the long--jump--absorbing chain,
so $u(x)=\Pr_x[X_\tau=e]$ with $\tau:=\inf\{t\ge0:X_t\in\{a_n,b_n,e\}\}$.
Apply the source/trajectory representation (Lemma~\ref{lem:source_representation}) with $\theta=\alpha$:
\[
\partial_\alpha h_\alpha(v)=\partial_\alpha u(v)
=\E_v\!\left[\sum_{t=0}^{\tau-1} s_\alpha(X_t)\right],
\]
where $s_\alpha(x):=\sum_y (\partial_\alpha P(x,y))\,u(y)$ is the local $\alpha$--source.

By the explicit $\alpha$--source identity in the long--jump--absorbing chain (Lemma~\ref{lem:source_alpha_longjump_absorbing}),
\[
s_\alpha(x)=p_f(x)\big(u(F(x))-u(x)\big),
\]
with the convention $s_\alpha(x)=0$ if $F(x)$ does not exist. By monotonicity of $u$ along depth
(Lemma~\ref{lem:u_monotone_and_recurrence}), we have $u(F(x))\le u(x)$ whenever $F(x)$ exists, hence
$s_\alpha(x)\le 0$ for all transient $x$. Therefore
\[
\begin{aligned}
\partial_\alpha h_\alpha(v)
&=\E_v\!\left[\sum_{t=0}^{\tau-1} s_\alpha(X_t)\right]\\
&\le \E_v\!\left[s_\alpha(X_0)
  +s_\alpha(X_\sigma)\,\mathbf 1\{0<\sigma<\tau\}\right]\\
&= s_\alpha(v)+s_\alpha(v_{\mathrm{pen}})\,\Pr_v(0<\sigma<\tau),
\end{aligned}
\]
where $\sigma:=\inf\{t\ge 0:\ X_t=v_{\mathrm{pen}}\}$ and $v_{\mathrm{pen}}$ is the penultimate state on the same chain as $v$.
Indeed, on $\{0<\sigma<\tau\}$ we have $X_\sigma=v_{\mathrm{pen}}$, and for all other times $t\in\{0,\dots,\tau-1\}\setminus\{0,\sigma\}$
we have $s_\alpha(X_t)\le 0$, so discarding those terms only increases the sum.
\qedhere
\end{proof}

\begin{lemma}[Derivative estimate for $\beta$]
\label{lem:deriv_est_beta}
In the \emph{long--jump--absorbing chain}, with $\tau,\sigma,v_{\mathrm{pen}}$ as in
Lemma~\ref{lem:deriv_est_alpha}, the derivative $\partial_\beta h_\beta(v)$ satisfies
\[
\partial_\beta h_\beta(v)
=\E_v\!\left[\sum_{t=0}^{\tau-1} s_\beta(X_t)\right]
\ \ge\ s_\beta(v)\;+\; s_\beta(v_{\mathrm{pen}})\,\Pr_v(0<\sigma<\tau),
\]
where
\[
s_\beta(x)=p_b(x)\big(h_\beta(B(x))-h_\beta(x)\big)\ge 0,
\]
with the convention $s_\beta(x)=0$ if the (same--chain) backward neighbor $B(x)$ does not exist.

In particular, if $v$ is $k$ steps from the terminal on its chain, i.e.\ $v\in\{a_{n-k},b_{n-k}\}$ with
$2\le k\le n-1$, then
\[
\Pr_v(0<\sigma<\tau)\ \ge\ \Pr_v\Big(X_i=F^i(v)\ \ \forall i=1,\dots,k-1\Big),
\]
and hence
\[
\partial_\beta h_\beta(v)
\ \ge\ s_\beta(v)\;+\; s_\beta(v_{\mathrm{pen}})\,
\Pr_v\Big(X_i=F^i(v)\ \ \forall i=1,\dots,k-1\Big).
\]
\end{lemma}

\begin{proof}
The argument is analogous to that of Lemma~\ref{lem:deriv_est_alpha}; we include it here for completeness.

Let $u(\cdot):=h_\beta(\cdot)$ be the absorption probability at $e$ in the long--jump--absorbing chain,
so $u(x)=\Pr_x[X_\tau=e]$ with $\tau:=\inf\{t\ge0:X_t\in\{a_n,b_n,e\}\}$.
Apply the source/trajectory representation (Lemma~\ref{lem:source_representation}) with $\theta=\beta$:
\[
\begin{aligned}
\partial_\beta h_\beta(v)
&=\E_v\!\left[\sum_{t=0}^{\tau-1} s_\beta(X_t)\right],
\end{aligned}
\]
where $s_\beta(x):=\sum_y (\partial_\beta P(x,y))\,u(y)$ is the local $\beta$--source.

By the explicit $\beta$--source identity in the long--jump--absorbing chain (Lemma~\ref{lem:source_beta_longjump_absorbing}),
\[
s_\beta(x)=p_b(x)\big(u(B(x))-u(x)\big),
\]
with the convention $s_\beta(x)=0$ if $B(x)$ does not exist. By monotonicity of $u$ along depth
(Lemma~\ref{lem:u_monotone_and_recurrence}), we have $u(B(x))\ge u(x)$ whenever $B(x)$ exists, hence
$s_\beta(x)\ge 0$ for all transient $x$. Therefore the random sum in the trajectory representation is
almost surely nonnegative, and so for $\sigma:=\inf\{t\ge 0:\ X_t=v_{\mathrm{pen}}\}$,
\[
\begin{aligned}
\partial_\beta h_\beta(v)
&=\E_v\!\left[\sum_{t=0}^{\tau-1} s_\beta(X_t)\right]\\
&\ge \E_v\!\left[s_\beta(X_0)
  +s_\beta(X_\sigma)\,\mathbf 1\{0<\sigma<\tau\}\right]\\
&= s_\beta(v)+s_\beta(v_{\mathrm{pen}})\,\Pr_v(0<\sigma<\tau),
\end{aligned}
\]
since $X_0=v$ and on $\{0<\sigma<\tau\}$ we have $X_\sigma=v_{\mathrm{pen}}$.
\qedhere
\end{proof}

\begin{lemma}[Derivative estimate for $\gamma$]
\label{lem:deriv_est_gamma}
Assume the forward-drift condition \eqref{eq:drift_f_ge_b} and the setting of
Lemma~\ref{lem:gamma_source_lower_bounds}. Fix a start state \(v_0=a_{n-K}\)
at distance \(K\ge 1\) from the terminal, where \(K\) is a fixed constant independent of \(n\).
Let
\[
v_{\mathrm{pen}}:=a_{n-1},
\qquad
\tau:=\inf\{t\ge 0:\ X_t\in\{a_n,b_n,e\}\},
\qquad
\sigma:=\inf\{t\ge 0:\ X_t=v_{\mathrm{pen}}\}.
\]
Then for all sufficiently large \(n\) (depending on \(K,\alpha,\beta,\gamma\)),
\[
\partial_\gamma h_\gamma(v_0)
\ \ge\
\Big(s_\gamma(v_{\mathrm{pen}})\Big)\Pr_{v_0}(0<\sigma<\tau)
\;+\;
\frac{1}{2}\,s_\gamma(v_0)
\ >\
0,
\]
and
\[
\Pr_{v_0}(0<\sigma<\tau)
\ \ge\
\Pr_{v_0}\Big(X_i=F^i(v_0)\ \ \forall i=1,\dots,K-1\Big).
\]
where \(s_\gamma(v_0)\) is the local \(\gamma\)--source at the initial state
(as given explicitly in Lemma~\ref{lem:gamma_source_all_cases} and lower-bounded
in Lemma~\ref{lem:gamma_source_lower_bounds}).
\end{lemma}

\begin{proof}
Let \(u(v):=\Pr_v[X_\tau=e]\) denote the absorption probability at \(e\) in
the long--jump--absorbing chain, with \(\tau:=\inf\{t\ge 0: X_t\in T\}\) and
\(T:=\{a_n,b_n,e\}\). By the source/trajectory representation (Lemma~\ref{lem:source_representation}),
\begin{equation}
\label{eq:source_decomp_gamma}
\partial_\gamma u(v_0)
=
\E_{v_0}\!\left[\sum_{t=0}^{\tau-1} s_\gamma(X_t)\right].
\end{equation}

\medskip
\noindent
Since \(v_0=a_{n-K}\) with \(K\ge 1\) fixed, Lemma~\ref{lem:gamma_source_lower_bounds}
applies (using case \textup{(C)} when \(K=1\), and case \textup{(A)} when \(K\ge 2\)) and gives
\[
s_\gamma(v_0)\ \ge\ \frac{1}{2}\,q_{n-K}\,(1-u(v_0))\ >\ 0,
\]
where \(q_{n-K}=\widetilde P(a_{n-K},e)\) is the one-step absorption probability.
The factor \(1-u(v_0)\ge\eta>0\) is bounded below by a positive constant depending
only on \(K\) and the fixed parameters (since the event of \(K\) consecutive forward
moves has probability at least \((f_{\min})^K>0\)).

\medskip
\noindent
Let \(S_2:=\{a_2,b_2\}\). By Lemma~\ref{lem:gamma_source_lower_bounds}(A)--(C),
for all states \(x\notin S_2\) with \(x\notin T\), we have \(s_\gamma(x)\ge 0\).
Thus
\[
\E_{v_0}\!\left[\sum_{t=0}^{\tau-1} s_\gamma(X_t)\right]
\ \ge\
s_\gamma(v_0)
+
\E_{v_0}\!\left[s_\gamma(X_\sigma)\,\mathbf 1\{0<\sigma<\tau\}\right]
\;+\;
\E_{v_0}\!\left[\sum_{\substack{t:\,X_t\in S_2}} s_\gamma(X_t)\right].
\]

On the event \(\{0<\sigma<\tau\}\) we have \(X_\sigma=v_{\mathrm{pen}}=a_{n-1}\), hence
\[
\E_{v_0}\!\left[s_\gamma(X_\sigma)\,\mathbf 1\{0<\sigma<\tau\}\right]
=s_\gamma(v_{\mathrm{pen}})\,\Pr_{v_0}(0<\sigma<\tau).
\]
Moreover, since \(v_0=a_{n-K}\), the event of \(K-1\) consecutive forward steps implies \(\{0<\sigma<\tau\}\), so
\[
\Pr_{v_0}(0<\sigma<\tau)
\ \ge\
\Pr_{v_0}\Big(X_i=F^i(v_0)\ \ \forall i=1,\dots,K-1\Big)
=\prod_{j=0}^{K-2} p_f\!\big(F^j(v_0)\big).
\]

\medskip
\noindent
Under the forward-drift condition \(\alpha>\beta\), there exists \(\rho\in(0,1)\)
(depending only on \(\alpha,\beta,\gamma\)) such that \(p_b(x)/p_f(x)\le\rho\)
for all deep--middle states. By the gambler's--ruin estimate
(Lemma~\ref{lem:gambler_ruin_no_martingale}), the probability of reaching the second vertex
\(S_2=\{a_2,b_2\}\) from a start state at depth \(K\) before absorption satisfies
\[
\Pr_{v_0}(H_2)\ \le\ \rho^{\,n-K-2},
\]
where \(H_2:=\{\exists\,t<\tau: X_t\in S_2\}\).

\medskip
\noindent
To control the (potentially negative) contribution from visits to \(S_2\), we use the absolute
bound from Lemma~\ref{lem:gamma_source_lower_bounds}(D): for each \(x\in S_2\),
\[
|s_\gamma(x)|\ \le\ 4C_\star\,q_{\rm mid}\,\E_x[\tau].
\]
Let \(N_2:=\sum_{t<\tau}\mathbf 1\{X_t\in S_2\}\).
Then
\[
\left|\E_{v_0}\!\left[\sum_{t:\,X_t\in S_2} s_\gamma(X_t)\right]\right|
\ \le\
\sup_{x\in S_2}|s_\gamma(x)|\cdot\E_{v_0}[N_2]
\ \le\
4C_\star q_{\rm mid}\Big(\sup_{x\in S_2}\E_x[\tau]\Big)\E_{v_0}[N_2],
\]
where we used \(\E_{v_0}[N_2]\le\Pr_{v_0}(H_2)\cdot\sup_{x\in S_2}\E_x[\tau]\)
and the polynomial bound on hitting times from Lemma~\ref{lem:gamma_source_lower_bounds}.
Combining with \(\Pr_{v_0}(H_2)\le\rho^{\,n-K-2}\), we obtain
\[
\left|\E_{v_0}\!\left[\sum_{t:\,X_t\in S_2} s_\gamma(X_t)\right]\right|
\ \le\
C'\,q_{\rm mid}\,\mathrm{poly}(n)\cdot\rho^{\,n-K-2},
\]
for a constant \(C'>0\) depending only on \((\alpha,\beta,\gamma,K)\).

\medskip
\noindent
Combining the above bounds, we have
\[
\partial_\gamma u(v_0)
\ \ge\
s_\gamma(v_0)
+s_\gamma(v_{\mathrm{pen}})\,\Pr_{v_0}(0<\sigma<\tau)
- C'\,q_{\rm mid}\,\mathrm{poly}(n)\cdot\rho^{\,n-K-2}.
\]
Finally, note that the main term is itself proportional to \(q_{\rm mid}\): by the lower bound
above together with \eqref{eq:q_comparable},
\[
s_\gamma(v_0)\ \ge\ \frac{1}{2}\,q_{n-K}\,\eta\ \ge\ \frac{\eta}{2C_\star}\,q_{\rm mid}.
\]
Thus both the positive contribution \(s_\gamma(v_0)\) and the worst-case negative correction from \(S_2\)
carry a factor \(q_{\rm mid}\), but the latter is suppressed by the exponentially small prefactor
\(\mathrm{poly}(n)\rho^{\,n-K-2}\). Since \(\rho<1\), for all sufficiently large \(n\) we have
\[
C'\,q_{\rm mid}\,\mathrm{poly}(n)\cdot\rho^{\,n-K-2}\ \le\ \frac{1}{2}\,s_\gamma(v_0),
\]
and hence
\[
\partial_\gamma h_\gamma(v_0)\ \ge\ s_\gamma(v_{\mathrm{pen}})\,\Pr_{v_0}(0<\sigma<\tau)+\frac{1}{2}\,s_\gamma(v_0)\ >\ 0.
\]
\end{proof}

\subsection{Analysis of no-long-jump chain}
\label{sec:nl_chain_analysis}

In this subsection, we analyze the no-long-jump (NL) chain defined in \cref{def:nl-chain}, which captures the canonical chain's behavior conditioned on never making a long jump before absorption. The analysis mirrors that of \cref{sec:lj_absorbing_chain_lemmas}: we compute the source functions for the NL chain, establish symmetry properties, and derive derivative estimates in \cref{sec:derivative_estimates_nl}. 

We recall the base model assumptions from \cref{sec:dyn:setting}:

\begin{assumption}[Base model (from the main paper)]
\label{assump:base-model-paper}
For any input sequence $s_{1:L_0}$, let \(p_{\text{fwd}}\), \(p_{\text{bwd}}\), and \(p_{\text{switch}}\) denote the probabilities that the model performs a forward, backward, or switch step respectively at the first step (time \(t=L_0+1\)).
We assume that for any $s_{1:L_0}$, \(p_{\text{fwd}} > p_{\text{bwd}}+\fwdbound\) for some constant \(\fwdbound > 0\), and additionally \(p_{\text{switch}} > \switchbound\) for some \(\switchbound > 0\).
\end{assumption}

\medskip
\noindent\textbf{Setup (NL chain).}\label{sec:nl_setup}
Throughout this section we work with the \emph{no--long--jump (NL) chain} on the two--chain graph of length \(n\),
with absorbing set \(\{a_n,b_n\}\). We write \(p_{\mathrm{NL}}(v\to k)\) for its one--step transition probabilities.
Let \(D(\cdot)\) denote depth (distance from the absorbing endpoint on its chain) and let
\[
\tau:=\inf\{t\ge0:\ X_t\in\{a_n,b_n\}\},\qquad h(v):=\Pr_v^{NL}[X_\tau=a_n].
\]
For any parameter \(\theta\in\{\alpha,\beta,\gamma\}\) define the local NL source
\[
s_{\theta,\mathrm{NL}}(v):=\sum_{k\in\mathcal N(v)}\big(\partial_\theta p_{\mathrm{NL}}(v\to k)\big)\,h(k),
\qquad
\mathcal N(v):=\{k\in V\setminus\{v\}:\ |D(k)-D(v)|\le 1\}.
\]
We also write \(p_{\mathrm{sw}}(v):=\sum_{k\in\mathcal S(v)}p_{\mathrm{NL}}(v\to k)\) for the single--step switch probability
(where \(\mathcal S(v)\subseteq\mathcal N(v)\) are the cross--chain neighbors) and
\(p_{\mathrm{sw}}^{\max}:=\sup_{v\notin\{a_n,b_n\}}p_{\mathrm{sw}}(v)\).
Finally let
\begin{align*}
p_f(v)&:=\Pr\big(D(X_1)=D(v)-1\mid X_0=v\big),\\
p_b(v)&:=\Pr\big(D(X_1)=D(v)+1\mid X_0=v\big),\\
\delta&:=\min_{v\notin\{a_n,b_n\}}\big(p_f(v)-p_b(v)\big)>0,
\end{align*}
where probabilities are under the NL chain.

\begin{lemma}[Switch probability is $O(1/n)$ in the NL chain]
\label{lem:psw_order_1_over_L}
Consider the \emph{no--long--jump (NL) chain} on the two--chain graph of length \(n\) (Definition~\ref{def:nl-chain}), and write \(p_{\mathrm{sw}}(v)\) for the one-step switch probability in the NL chain (as in paragraph~\ref{sec:nl_setup} below).
Under Assumption~\ref{assump:base-model-paper} and the logit structure of \cref{sec:canonical_chain}, there exists a constant \(C>0\) (depending only on fixed parameters) such that for all nonterminal states \(v\),
\[
p_{\mathrm{sw}}(v)\ \le\ \frac{C}{n}.
\]
\end{lemma}

\begin{proof}
Fix a nonterminal state \(v\), and let f be the logit of the forward probability. Because the forward probability is bounded below by a constant \(\fwdbound>0\), we have that $e^{f}\geq \Theta(n)\fwdbound\cdot e^{\gamma}$. Because the logits of the switch targets are are all at most $\gamma$ (recall that $\gamma>0$), and in the No Long jump chain, there are at most 3 such targets, it follows that the switch probability is bounded above $3\cdot \frac{e^{\gamma}}{e^{f}}\leq 3\cdot \frac{e^{\gamma}}{\Theta(n)\fwdbound\cdot e^{\gamma}}\leq \frac{C}{n}$.\end{proof}

\begin{lemma}[NL transition scaling]
\label{lem:nl_transition_prob_compare}
Let \(p(v\to k)\) denote the one-step transition probability in the \emph{original} (unconditioned) chain, and
let \(p_{\mathrm{NL}}(v\to k)\) denote the one-step transition probability in the \emph{no--long--jump (NL) chain}.
Fix any nonterminal state \(v\) and any state \(k\neq v\) such that the transition \(v\to k\) is \emph{not} a long jump
(equivalently, \(|D(k)-D(v)|\le 1\)). Then
\[
p_{\mathrm{NL}}(v\to k)\ \le\ \frac{1}{\fwdbound}\,p(v\to k),
\]
where \(\fwdbound\) is the constant from Assumption~\ref{assump:base-model-paper}.
\end{lemma}

\begin{proof}
Let \(\mathcal N(v):=\{y\in V\setminus\{v\}:\ |D(y)-D(v)|\le 1\}\) be the set of non--long--jump targets from \(v\).
By Definition~\ref{def:nl-chain}, the NL kernel is obtained by conditioning on the event \(\{X_1\in\mathcal N(v)\}\), hence
\[
p_{\mathrm{NL}}(v\to k)=\Pr(X_1=k\mid X_0=v,\ X_1\in\mathcal N(v))=\frac{p(v\to k)}{\Pr(X_1\in\mathcal N(v)\mid X_0=v)}.
\]
Moreover, the forward move \(v\to F(v)\) is never a long jump, so \(\Pr(X_1\in\mathcal N(v)\mid X_0=v)\ge p(v\to F(v))\).
By Assumption~\ref{assump:base-model-paper}, \(p(v\to F(v))\ge \fwdbound\), hence \(\Pr(X_1\in\mathcal N(v)\mid X_0=v)\ge \fwdbound\),
which yields the claimed bound.
\end{proof}

\begin{lemma}[Same-side absorption in NL chain]
\label{lem:NL_same_side_to_one_pf_pb}
Fix \(k\ge 1\). In the notation of the NL setup above (paragraph~\ref{sec:nl_setup}), for \(v=a_{n-k}\) we have
\[
1-h(v)\ \le\ p_{\mathrm{sw}}^{\max}\cdot\frac{k}{\fwdbound}.
\]
Consequently, for fixed $k$, the same-side absorption probability tends to 1 as \(n\) tends to infinity in the NL chain.
\end{lemma}

\begin{proof}
Let $D(\cdot)$ be the depth--distance to the end and define, for each transient state $x\in V\setminus\{a_n,b_n\}$,
\begin{align*}
p_f(x)&:=\Pr\big(D(X_1)=D(x)-1\mid X_0=x\big),\\
p_b(x)&:=\Pr\big(D(X_1)=D(x)+1\mid X_0=x\big),\\
\delta&:=\min_{x\in V\setminus\{a_n,b_n\}}\big(p_f(x)-p_b(x)\big)>0.
\end{align*}

If $X_\tau=b_n$ starting from $a_{n-k}$, then the trajectory must switch from chain $A$ to chain $B$
at least once before absorption. Let $E$ be the event that at least one switch occurs before $\tau$,
and let $N_{\mathrm{sw}}$ be the number of switches before $\tau$. Then
\[
1-S^{NL}_n(a_{n-k})=\Pr(X_\tau=b_n)\le \Pr(E)\le \Pr(N_{\mathrm{sw}}\ge 1)\le \E[N_{\mathrm{sw}}].
\]
Moreover,
\[
\E[N_{\mathrm{sw}}]
=\E\Big[\sum_{t=0}^{\tau-1}\mathbf 1\{\text{switch at }t\}\Big]
\le \E\Big[\sum_{t=0}^{\tau-1} p_{\mathrm{sw}}(X_t)\Big]
\le p_{\mathrm{sw}}^{\max}\,\E[\tau].
\]
By Lemma~\ref{lem:exact-and-linear-upper} applied to the depth process (using that $p_f(x)-p_b(x)\ge\delta$ for all transient $x$), $\E[\tau]\le k/\delta$. Combining yields the stated bound. The consequence now follows from lemma \ref{lem:psw_order_1_over_L}.
\end{proof}

\begin{lemma}[Same-side value differences]
\label{lem:local_variation_by_psw_times_tau}
Fix \(k\ge 2\). In the notation of the NL setup above (paragraph~\ref{sec:nl_setup}), for every \(j\in\{1,\dots,k-1\}\),
\begin{align*}
|h(a_{n-j})-h(a_{n-(j+1)})|
&\le \max\!\Big(1-S^{NL}_n(a_{n-j}),\,1-S^{NL}_n(a_{n-(j+1)})\Big)\\
&\le p_{\mathrm{sw}}^{\max}\,\max\!\Big(\E_{a_{n-j}}^{NL}[\tau],\,\E_{a_{n-(j+1)}}^{NL}[\tau]\Big)\\
&\le p_{\mathrm{sw}}^{\max}\,\max\!\Big(\tfrac{D(a_{n-j})}{\delta},\,\tfrac{D(a_{n-(j+1)})}{\delta}\Big)
\;=\; p_{\mathrm{sw}}^{\max}\,\frac{j+1}{\delta},
\end{align*}
where \(D(a_{n-j})=j\).
\end{lemma}

\begin{proof}
Let $u_j:=S^{NL}_n(a_{n-j})\in[0,1]$. Then
\[
|u_j-u_{j+1}|=|(1-u_j)-(1-u_{j+1})|\le \max(1-u_j,\,1-u_{j+1}).
\]
Moreover, for any start state $x\in\{a_1,\dots,a_{n-1}\}$, the event $\{X_\tau=b_n\}$ implies that the trajectory
switches chains at least once before absorption, hence
\[
1-\Pr_x^{NL}(X_\tau=a_n)=\Pr_x^{NL}(X_\tau=b_n)\le \E_x^{NL}[N_{\mathrm{sw}}]
\le p_{\mathrm{sw}}^{\max}\,\E_x^{NL}[\tau],
\]
where the last inequality follows from
$\E_x^{NL}[N_{\mathrm{sw}}]\le \E_x^{NL}\big[\sum_{t=0}^{\tau-1}p_{\mathrm{sw}}(X_t)\big]\le p_{\mathrm{sw}}^{\max}\E_x^{NL}[\tau]$.
Applying this with $x=a_{n-j}$ and $x=a_{n-(j+1)}$ and combining with the first display gives the stated bound.
For the final inequality, apply Lemma~\ref{lem:exact-and-linear-upper} to the depth process (with drift gap $\delta$)
to obtain $\E_{a_{n-j}}^{NL}[\tau]\le D(a_{n-j})/\delta=j/\delta$ and
$\E_{a_{n-(j+1)}}^{NL}[\tau]\le D(a_{n-(j+1)})/\delta=(j+1)/\delta$.
\end{proof}

\subsection{Source computations in the no--long--jump (NL) chain}

\begin{lemma}[$\gamma$-source in NL chain]
\label{lem:gamma-source-nolj-mult}
In the notation of the NL setup above (paragraph~\ref{sec:nl_setup}), consider the local
\(\gamma\)--source \(s_{\gamma,\mathrm{NL}}(v)\) and the switch probability \(p_{\mathrm{sw}}(v)\).
Define \(\mathrm{osc}(h):=\sup_{x,y\in V}|h(x)-h(y)|\le 1\).
Under the logit structure of \cref{sec:canonical_chain} (after the global shift by \(-\gamma\)), the only \(\gamma\)--dependence among allowed NL moves is:
\begin{itemize}
\item If \(v\notin\{a_2,b_2,a_{n-1},b_{n-1}\}\), then \(\partial_\gamma o(v,k)=+1\) for \(k\in\mathcal S(v)\) and
\(\partial_\gamma o(v,k)=0\) for \(k\in\mathcal N(v)\setminus\mathcal S(v)\).
\item If \(v\in\{a_2,b_2\}\), then in addition \(\partial_\gamma o(v,v^- )=-1\) for the unique same--chain backward neighbor \(v^-\in\mathcal N(v)\).
\item If \(v\in\{a_{n-1},b_{n-1}\}\), then in addition \(\partial_\gamma o(v,v^+ )=-1\) for the unique same--chain forward neighbor \(v^+\in\mathcal N(v)\).
\end{itemize}
Then:

\medskip
\noindent\textup{(A) Generic vertices \(v\notin\{a_2,b_2,a_{n-1},b_{n-1}\}\).}
\[
 s_{\gamma,\mathrm{NL}}(v)
 =
 \sum_{k\in\mathcal S(v)}p_{\mathrm{NL}}(v\to k)\big(h(k)-h(v)\big),
 \qquad\text{hence}\qquad
 |s_{\gamma,\mathrm{NL}}(v)|\le p_{\mathrm{sw}}(v)\,\mathrm{osc}(h)\le p_{\mathrm{sw}}(v).
\]

\medskip
\noindent\textup{(B) Second vertex \(v\in\{a_2,b_2\}\).}
\[
 s_{\gamma,\mathrm{NL}}(v)
 =
 \sum_{k\in\mathcal S(v)}p_{\mathrm{NL}}(v\to k)\big(h(k)-h(v)\big)
 \; -\; p_{\mathrm{NL}}(v\to v^-)\big(h(v^-)-h(v)\big).
\]
Consequently,
\[
 |s_{\gamma,\mathrm{NL}}(v)|
 \le
p_{\mathrm{sw}}(v)\,\mathrm{osc}(h)
 \; +\;
p_{\mathrm{NL}}(v\to v^-)\,|h(v^-)-h(v)|
 \le
p_{\mathrm{sw}}(v)
 \; +\;
p_{\mathrm{NL}}(v\to v^-)\,|h(v^-)-h(v)|.
\]

\medskip
\noindent\textup{(C) Penultimate vertex \(v\in\{a_{n-1},b_{n-1}\}\).}
\[
 s_{\gamma,\mathrm{NL}}(v)
 =
 \sum_{k\in\mathcal S(v)}p_{\mathrm{NL}}(v\to k)\big(h(k)-h(v)\big)
 \; -\; p_{\mathrm{NL}}(v\to v^+)\big(h(v^+)-h(v)\big).
\]
Consequently,
\[
 |s_{\gamma,\mathrm{NL}}(v)|
 \le
p_{\mathrm{sw}}(v)\,\mathrm{osc}(h)
 \; +\;
p_{\mathrm{NL}}(v\to v^+)\,|h(v^+)-h(v)|
 \le
p_{\mathrm{sw}}(v)
 \; +\;
p_{\mathrm{NL}}(v\to v^+)\,|h(v^+)-h(v)|.
\]
\end{lemma}

\begin{proof}
Fix \(v\in U\). Abbreviate \(p_k:=p_{\mathrm{NL}}(v\to k)\) and \(o_k:=o(v,k)\) for \(k\in\mathcal N(v)\).
By the (restricted) softmax derivative identity (Lemma~\ref{lem:softmax_derivative_identity_aux}),
\[
 \partial_\gamma p_k
 =
 p_k\Big(\partial_\gamma o_k-\sum_{k'\in\mathcal N(v)}p_{k'}\,\partial_\gamma o_{k'}\Big).
\]
Therefore
\begin{align*}
s_{\gamma,\mathrm{NL}}(v)
&=\sum_{k\in\mathcal N(v)}\big(\partial_\gamma p_k\big)\,h(k)\\
&=\sum_{k\in\mathcal N(v)}p_k\,h(k)\Big(\partial_\gamma o_k-\sum_{k'\in\mathcal N(v)}p_{k'}\,\partial_\gamma o_{k'}\Big)\\
&=\sum_{k\in\mathcal N(v)}p_k\,h(k)\,\partial_\gamma o_k
-\Big(\sum_{k\in\mathcal N(v)}p_k\,h(k)\Big)\Big(\sum_{k'\in\mathcal N(v)}p_{k'}\,\partial_\gamma o_{k'}\Big).
\end{align*}
Since \(h\) is harmonic at transient states for the NL chain, \(h(v)=\E_v[h(X_1)]\), we can rewrite this in centered form as
\[
 s_{\gamma,\mathrm{NL}}(v)=\sum_{k\in\mathcal N(v)}p_k\,\big(\partial_\gamma o_k\big)\,\big(h(k)-h(v)\big).
\]
Now apply the case split on \(\partial_\gamma o(v,\cdot)\) listed in the statement:
\begin{itemize}
\item If \(v\notin\{a_2,b_2,a_{n-1},b_{n-1}\}\), then \(\partial_\gamma o(v,k)=1\) on \(k\in\mathcal S(v)\) and zero otherwise, giving (A).
\item If \(v\in\{a_2,b_2\}\), then in addition \(\partial_\gamma o(v,v^- )=-1\), giving (B).
\item If \(v\in\{a_{n-1},b_{n-1}\}\), then in addition \(\partial_\gamma o(v,v^+ )=-1\), giving (C).
\end{itemize}
The displayed upper bounds follow from the triangle inequality together with \(|h(k)-h(v)|\le \mathrm{osc}(h)\le 1\).
\end{proof}

\begin{lemma}[$\alpha$-source in NL chain]
\label{lem:alpha-source-nolj-simple}
In the notation of the NL setup above (paragraph~\ref{sec:nl_setup}), let \(v^+\) be the same--chain forward neighbor of \(v\) when it exists.
Let $v^+$ denote the unique same--chain forward neighbor of $v$ (so $v^+=a_{i+1}$ if
$v=a_i$ with $i\le n-1$, and similarly on chain $B$; in particular $a_{n-1}^+=a_n$ and
$b_{n-1}^+=b_n$). Then
\begin{equation}
\label{eq:alpha-source-simple}
s_{\alpha,\mathrm{NL}}(v)=p_{\mathrm{NL}}(v\to v^+)\Big(h(v^+)-h(v)\Big).
\end{equation}
Moreover,
\begin{equation}
\label{eq:alpha-source-bound}
\begin{aligned}
|s_{\alpha,\mathrm{NL}}(v)|
&\le p_{\mathrm{NL}}(v\to v^+)\,\mathrm{osc}_v(h)
 \le p_{\mathrm{NL}}(v\to v^+)\le 1,\\
\mathrm{osc}_v(h)
&:=\max_{k\in\mathcal N(v)}h(k)-\min_{k\in\mathcal N(v)}h(k)\le 1.
\end{aligned}
\end{equation}
\end{lemma}

\begin{proof}
Fix $v\in U$. Abbreviate $p_k:=p_{\mathrm{NL}}(v\to k)$ and $o_k:=o(v,k)$ for $k\in\mathcal N(v)$.
For restricted softmax,
\[
\partial_\alpha p_k
=
p_k\Big(\partial_\alpha o_k-\sum_{k'\in\mathcal N(v)}p_{k'}\,\partial_\alpha o_{k'}\Big).
\]
By the specified logits, $\alpha$ appears only on the forward edge $k=v^+$ with coefficient $+1$,
and nowhere else among allowed neighbors. Hence
\[
\partial_\alpha o_k=
\begin{cases}
1, & k=v^+,\\
0, & k\in\mathcal N(v)\setminus\{v^+\}.
\end{cases}
\]
Therefore,
\[
\sum_{k'\in\mathcal N(v)}p_{k'}\,\partial_\alpha o_{k'}=p_{\mathrm{NL}}(v\to v^+),
\qquad
\sum_{k\in\mathcal N(v)}p_k\,h(k)\,\partial_\alpha o_k=p_{\mathrm{NL}}(v\to v^+)\,h(v^+).
\]
Now expand
\[
s_{\alpha,\mathrm{NL}}(v)=\sum_{k\in\mathcal N(v)}(\partial_\alpha p_k)\,h(k)
=
\sum_{k\in\mathcal N(v)}p_k h(k)\partial_\alpha o_k
-
\Big(\sum_{k\in\mathcal N(v)}p_k h(k)\Big)\Big(\sum_{k'\in\mathcal N(v)}p_{k'}\partial_\alpha o_{k'}\Big),
\]
and substitute the two identities above to obtain \eqref{eq:alpha-source-simple}.

For \eqref{eq:alpha-source-bound}, note that $h$ is harmonic for the NL chain at transient $v$, so
$h(v)=\E_v[h(X_1)]=\sum_{k\in\mathcal N(v)}p_k\,h(k)$ is a convex combination of $\{h(k):k\in\mathcal N(v)\}$.
Hence $|h(v^+)-h(v)|\le \mathrm{osc}_v(h)\le 1$. Multiply by $p_{\mathrm{NL}}(v\to v^+)$.
\end{proof}

\begin{lemma}[$\beta$-source in NL chain]
\label{lem:beta-source-nolj-simple}
In the notation of the NL setup above (paragraph~\ref{sec:nl_setup}), let \(v^-\) be the same--chain backward neighbor of \(v\) when it exists.
Let $v^-$ denote the unique same--chain backward neighbor of $v$ when it exists
(so $v^-=a_{i-1}$ if $v=a_i$ with $i\ge 2$, and similarly on chain $B$).
Then for all $v\in U$,
\begin{equation}
\label{eq:beta-source-simple}
s_{\beta,\mathrm{NL}}(v)=
\begin{cases}
p_{\mathrm{NL}}(v\to v^-)\Big(h(v^-)-h(v)\Big), & \text{if }v^- \text{ exists},\\[0.6ex]
0, & \text{if }v\in\{a_1,b_1\}.
\end{cases}
\end{equation}
Moreover, whenever $v^-$ exists,
\begin{equation}
\label{eq:beta-source-bound}
|s_{\beta,\mathrm{NL}}(v)|\le p_{\mathrm{NL}}(v\to v^-)\,\mathrm{osc}_v(h)\le p_{\mathrm{NL}}(v\to v^-)\le 1,
\qquad
\mathrm{osc}_v(h)\le 1.
\end{equation}
\end{lemma}

\begin{proof}
Fix $v\in U$. If $v\in\{a_1,b_1\}$, the logits do not contain $\beta$ on any allowed edge,
so $\partial_\beta o(v,k)=0$ for all $k\in\mathcal N(v)$ and hence $s_{\beta,\mathrm{NL}}(v)=0$.

Otherwise, the specified logits place $\beta$ only on the single backward edge $k=v^-$,
with coefficient $+1$ (including the boundary case $v\in\{a_2,b_2\}$ where the back logit is
$(\beta-\gamma)$). Thus
\[
\partial_\beta o_k=
\begin{cases}
1, & k=v^-,\\
0, & k\in\mathcal N(v)\setminus\{v^-\}.
\end{cases}
\]
The rest of the argument is identical to the proof of Lemma~\ref{lem:alpha-source-nolj-simple}:
use the restricted softmax derivative, expand $s_{\beta,\mathrm{NL}}(v)$, and substitute
\[
\begin{aligned}
\sum_{k'\in\mathcal N(v)}p_{\mathrm{NL}}(v\to k')\,\partial_\beta o_{k'}
&=p_{\mathrm{NL}}(v\to v^-),\\
\sum_{k\in\mathcal N(v)}p_{\mathrm{NL}}(v\to k)\,h(k)\,\partial_\beta o_k
&=p_{\mathrm{NL}}(v\to v^-)\,h(v^-),
\end{aligned}
\]
to get \eqref{eq:beta-source-simple}. The bound \eqref{eq:beta-source-bound} follows from
$|h(v^-)-h(v)|\le \mathrm{osc}_v(h)\le 1$ (and $\E_v[h(X_1)]=h(v)$ by harmonicity).
\end{proof}

\subsubsection{Derivative estimates}
\label{sec:derivative_estimates_nl}

Having computed the source functions for the NL chain in the preceding lemmas, we now derive quantitative bounds on the derivatives of the same-side absorption probability.

\begin{lemma}[Truncation bound for $\partial_\beta u$]
\label{lem:beta_band_dominance_clean}
In the notation of the NL setup above (paragraph~\ref{sec:nl_setup}), let
\[
u(x):=\Pr^{NL}_x[X_\tau=a_n],
\qquad
u_\beta(x):=\partial_\beta u(x),
\]
and fix \(v=a_{n-k}\) with \(k\) independent of \(n\).
Assume the base model assumptions of Assumption~\ref{assump:base-model-paper} (in particular, a uniform drift gap \(\fwdbound>0\)).
Then there exist constants \(C,c>0\) (depending only on fixed parameters and \(\fwdbound\)) such that for all \(n\) large enough,
\[
|u_\beta(v)|
\ \le\
C\,\frac{\sqrt{n}}{\fwdbound}\,p_b^{(2\lfloor n^{1/4}\rfloor)}\,p_{\mathrm{sw}}^{\max}
\;+\;
C\,n\,\exp\!\left(-\frac{c^2}{8}\,n^{1/4}\right)\cdot \sup_{x\notin\{a_n,b_n\}}|s_{\beta,\mathrm{NL}}(x)|.
\]
\end{lemma}

\begin{proof}
\noindent
Fix \(K:=\lfloor n^{1/4}\rfloor\). By the trajectory/source representation (Lemma~\ref{lem:source_representation}) for the NL chain,
\[
u_\beta(v)=\E_v\!\left[\sum_{t=0}^{\tau-1} s_{\beta,\mathrm{NL}}(X_t)\right].
\]
Decompose on \(\{\tau\le K\}\) and \(\{\tau>K\}\). On \(\{\tau\le K\}\), the sum has at most \(K\) terms and the depth can increase by at most one per step, so every visited state
satisfies \(D(X_t)\le D(v)+t\le k+K\le 2K\) for all large \(n\). Therefore,
\[
|u_\beta(v)|
\le
K\cdot \sup_{x:\,D(x)\le 2K}\,|s_{\beta,\mathrm{NL}}(x)|.
\]
By Lemma~\ref{lem:beta-source-nolj-simple},
\(
|s_{\beta,\mathrm{NL}}(x)|\le p_b(x)\,|u(B(x))-u(x)|.
\)
By Lemma~\ref{lem:local_variation_by_psw_times_tau},
\(
|u(B(x))-u(x)|\le p_{\mathrm{sw}}^{\max}\E_x[\tau].
\)
Under the base model forward-bias gap \(\fwdbound>0\), the depth-to-terminal process has \(\E_x[\tau]\le D(x)/\fwdbound\le 2K/\fwdbound\) on the band \(D(x)\le 2K\).
Combining yields
\[
\sup_{x:\,D(x)\le 2K}|s_{\beta,\mathrm{NL}}(x)|
\ \le\
C\,p_b^{(2K)}\,p_{\mathrm{sw}}^{\max}\cdot\frac{2K}{\fwdbound},
\]
for a constant \(C>0\), hence \(|u_\beta(v)|\lesssim \frac{K^2}{\fwdbound}\,p_b^{(2K)}p_{\mathrm{sw}}^{\max}\).
For \(\{\tau>K\}\), write \(\sum_{t=0}^{\tau-1}=\sum_{t=0}^{K-1}+\sum_{t=K}^{\tau-1}\). The first part is bounded identically by the same \(K\)-step band argument above. For the tail, since \(\sup_x \E_x[\tau]\le n/\fwdbound\) under the drift gap and \(|s_{\beta,\mathrm{NL}}(x)|\le \sup_{y}|s_{\beta,\mathrm{NL}}(y)|\),
the contribution after time \(K\) is at most on the order of \(n\cdot \sup_y |s_{\beta,\mathrm{NL}}(y)|\).
Multiplying by \(\Pr_v(\tau>K)\) and using Lemma~\ref{lem:hoeffding_tail_absorption} (with \(t=\lfloor n^{1/4}\rfloor\)) gives
\[
\Pr_v(\tau>K)\cdot n\cdot \sup_y |s_{\beta,\mathrm{NL}}(y)|
\ \le\
n\cdot \exp\!\left(-\frac{c^2}{8}\,n^{1/4}\right)\cdot \sup_y |s_{\beta,\mathrm{NL}}(y)|,
\]
\qedhere
\end{proof}

\medskip
\noindent\emph{The next lemma is completely analogous to Lemma~\ref{lem:beta_band_dominance_clean} and is included for completeness.}
\begin{lemma}[Truncation bound for $\partial_\alpha u$]
\label{lem:alpha_band_dominance_clean}
In the notation of the NL setup above (paragraph~\ref{sec:nl_setup}), let
\[
u(x):=\Pr^{NL}_x[X_\tau=a_n],
\qquad
u_\alpha(x):=\partial_\alpha u(x),
\]
and fix \(v=a_{n-k}\) with \(k\) independent of \(n\).
Assume the base model assumptions of Assumption~\ref{assump:base-model-paper} (in particular, a uniform drift gap \(\fwdbound>0\)).
Then there exist constants \(C,c>0\) (depending only on fixed parameters and \(\fwdbound\)) such that for all \(n\) large enough,
\[
|u_\alpha(v)|
\ \le\
C\,\frac{\sqrt{n}}{\fwdbound}\,p_f^{(2\lfloor n^{1/4}\rfloor)}\,p_{\mathrm{sw}}^{\max}
\;+\;
C\,n\,\exp\!\left(-\frac{c^2}{8}\,n^{1/4}\right)\cdot \sup_{x\notin\{a_n,b_n\}}|s_{\alpha,\mathrm{NL}}(x)|,
\]
where \(p_f^{(m)}:=\sup\{p_{\mathrm{NL}}(x\to x^+):\ x\notin\{a_n,b_n\},\ D(x)\le m\}\).
\end{lemma}

\begin{proof}
\noindent
Fix \(K:=\lfloor n^{1/4}\rfloor\). By the trajectory/source representation (Lemma~\ref{lem:source_representation}) for the NL chain,
\[
u_\alpha(v)=\E_v\!\left[\sum_{t=0}^{\tau-1} s_{\alpha,\mathrm{NL}}(X_t)\right].
\]
Decompose on \(\{\tau\le K\}\) and \(\{\tau>K\}\). On \(\{\tau\le K\}\), the sum has at most \(K\) terms and the depth can increase by at most one per step, so every visited state
satisfies \(D(X_t)\le D(v)+t\le k+K\le 2K\) for all large \(n\). Therefore,
\[
|u_\alpha(v)|
\le
K\cdot \sup_{x:\,D(x)\le 2K}\,|s_{\alpha,\mathrm{NL}}(x)|.
\]
By Lemma~\ref{lem:alpha-source-nolj-simple},
\(
|s_{\alpha,\mathrm{NL}}(x)|\le p_{\mathrm{NL}}(x\to x^+)\,|u(x^+)-u(x)|.
\)
By Lemma~\ref{lem:local_variation_by_psw_times_tau} (applied to the appropriate chain),
\(
|u(x^+)-u(x)|\le p_{\mathrm{sw}}^{\max}\E_x[\tau].
\)
Under the base model forward-bias gap \(\fwdbound>0\), the depth-to-terminal process has \(\E_x[\tau]\le D(x)/\fwdbound\le 2K/\fwdbound\) on the band \(D(x)\le 2K\).
Combining yields
\[
\sup_{x:\,D(x)\le 2K}|s_{\alpha,\mathrm{NL}}(x)|
\ \le\
C\,p_f^{(2K)}\,p_{\mathrm{sw}}^{\max}\cdot\frac{2K}{\fwdbound},
\]
for a constant \(C>0\), hence \(|u_\alpha(v)|\lesssim \frac{K^2}{\fwdbound}\,p_f^{(2K)}p_{\mathrm{sw}}^{\max}\).
For \(\{\tau>K\}\), write \(\sum_{t=0}^{\tau-1}=\sum_{t=0}^{K-1}+\sum_{t=K}^{\tau-1}\). The first part is bounded identically by the same \(K\)-step band argument above. For the tail, since \(\sup_x \E_x[\tau]\le n/\fwdbound\) under the drift gap and \(|s_{\alpha,\mathrm{NL}}(x)|\le \sup_{y}|s_{\alpha,\mathrm{NL}}(y)|\),
the contribution after time \(K\) is at most on the order of \(n\cdot \sup_y |s_{\alpha,\mathrm{NL}}(y)|\).
Multiplying by \(\Pr_v(\tau>K)\) and using Lemma~\ref{lem:hoeffding_tail_absorption} (with \(t=\lfloor n^{1/4}\rfloor\)) gives
\[
\Pr_v(\tau>K)\cdot n\cdot \sup_y |s_{\alpha,\mathrm{NL}}(y)|
\ \le\
n\cdot \exp\!\left(-\frac{c^2}{8}\,n^{1/4}\right)\cdot \sup_y |s_{\alpha,\mathrm{NL}}(y)|,
\]
as needed.
\qedhere
\end{proof}

\medskip
\noindent\emph{The next lemma is completely analogous to Lemmas~\ref{lem:beta_band_dominance_clean} and~\ref{lem:alpha_band_dominance_clean} and is included for completeness.}
\begin{lemma}[Truncation bound for $\partial_\gamma u$]
\label{lem:gamma_band_dominance_clean}
In the notation of the NL setup above (paragraph~\ref{sec:nl_setup}), let
\[
u(x):=\Pr^{NL}_x[X_\tau=a_n],
\qquad
u_\gamma(x):=\partial_\gamma u(x),
\]
and fix \(v=a_{n-k}\) with \(k\) independent of \(n\).
Assume the base model assumptions of Assumption~\ref{assump:base-model-paper} (in particular, a uniform drift gap \(\fwdbound>0\)).
Then there exist constants \(C,c>0\) (depending only on fixed parameters and \(\fwdbound\)) such that for all \(n\) large enough,
\[
|u_\gamma(v)|
\ \le\
C\,n^{1/4}\,p_{\mathrm{sw}}^{\max}
\;+\;
C\,\frac{\sqrt{n}}{\fwdbound}\,p_f^{(2\lfloor n^{1/4}\rfloor)}\,p_{\mathrm{sw}}^{\max}
\;+\;
C\,n\,\exp\!\left(-\frac{c^2}{8}\,n^{1/4}\right)\cdot \sup_{x\notin\{a_n,b_n\}}|s_{\gamma,\mathrm{NL}}(x)|,
\]
where \(p_f^{(m)}:=\sup\{p_{\mathrm{NL}}(x\to x^+):\ x\notin\{a_n,b_n\},\ D(x)\le m\}\).
\end{lemma}

\begin{proof}
\noindent
Fix \(K:=\lfloor n^{1/4}\rfloor\). By the trajectory/source representation (Lemma~\ref{lem:source_representation}) for the NL chain,
\[
u_\gamma(v)=\E_v\!\left[\sum_{t=0}^{\tau-1} s_{\gamma,\mathrm{NL}}(X_t)\right].
\]
Decompose on \(\{\tau\le K\}\) and \(\{\tau>K\}\). On \(\{\tau\le K\}\), the sum has at most \(K\) terms and the depth can increase by at most one per step, so every visited state
satisfies \(D(X_t)\le D(v)+t\le k+K\le 2K\) for all large \(n\). Therefore,
\[
|u_\gamma(v)|
\le
K\cdot \sup_{x:\,D(x)\le 2K}\,|s_{\gamma,\mathrm{NL}}(x)|.
\]
For \(x\) with \(D(x)\le 2K\), Lemma~\ref{lem:gamma-source-nolj-mult} gives a uniform upper bound:
by part~(A) we have \(|s_{\gamma,\mathrm{NL}}(x)|\le p_{\mathrm{sw}}(x)\), while at the boundary \(x\in\{a_{n-1},b_{n-1}\}\) part~(C) yields
\[
|s_{\gamma,\mathrm{NL}}(x)|
\le p_{\mathrm{sw}}(x)+p_{\mathrm{NL}}(x\to x^+)\,|u(x^+)-u(x)|.
\]
By Lemma~\ref{lem:local_variation_by_psw_times_tau} (applied to the appropriate chain),
\(
|u(x^+)-u(x)|\le p_{\mathrm{sw}}^{\max}\E_x[\tau].
\)
Under the base model forward-bias gap \(\fwdbound>0\), the depth-to-terminal process has \(\E_x[\tau]\le D(x)/\fwdbound\le 2K/\fwdbound\) on the band \(D(x)\le 2K\).
Combining yields
\[
\sup_{x:\,D(x)\le 2K}|s_{\gamma,\mathrm{NL}}(x)|
\ \le\
C\,p_{\mathrm{sw}}^{\max}
\;+\;
C\,p_f^{(2K)}\,p_{\mathrm{sw}}^{\max}\cdot\frac{2K}{\fwdbound},
\]
and hence \(|u_\gamma(v)|\lesssim K\,p_{\mathrm{sw}}^{\max}+\frac{K^2}{\fwdbound}\,p_f^{(2K)}p_{\mathrm{sw}}^{\max}\).
For \(\{\tau>K\}\), write \(\sum_{t=0}^{\tau-1}=\sum_{t=0}^{K-1}+\sum_{t=K}^{\tau-1}\). The first part is bounded identically by the same \(K\)-step band argument above. For the tail, since \(\sup_x \E_x[\tau]\le n/\fwdbound\) under the drift gap and \(|s_{\gamma,\mathrm{NL}}(x)|\le \sup_{y}|s_{\gamma,\mathrm{NL}}(y)|\),
the contribution after time \(K\) is at most on the order of \(n\cdot \sup_y |s_{\gamma,\mathrm{NL}}(y)|\).
Multiplying by \(\Pr_v(\tau>K)\) and using Lemma~\ref{lem:hoeffding_tail_absorption} (with \(t=\lfloor n^{1/4}\rfloor\)) gives
\[
\Pr_v(\tau>K)\cdot n\cdot \sup_y |s_{\gamma,\mathrm{NL}}(y)|
\ \le\
n\cdot \exp\!\left(-\frac{c^2}{8}\,n^{1/4}\right)\cdot \sup_y |s_{\gamma,\mathrm{NL}}(y)|,
\]
as needed.
\qedhere
\end{proof}

\subsection{ Monotonicity of $\alpha,\beta,\gamma$ and derivative estimates}
\label{sec:monotonicity}

We establish monotonicity and estimates for the derivativesof $\alpha,\beta,\gamma$ at initialization, but they in fact hold throughout the learning process as long as the forward transition probability remains non-decreasing and the long-jump probability stays bounded away from zero. Since our analysis shows that both conditions are maintained  for all $t \le t_0$, the estimates below in fact hold for all $t \le t_0$.

\begin{lemma}[Transition mass at penultimate vertex]
\label{lem:max_back_and_lj_at_penultimate}
Assume \(\gamma>0\) and the logit structure of \cref{sec:canonical_chain} (after the global shift by \(-\gamma\)),
and consider the original chain (equivalently, the long--jump--absorbing chain before redirection of long jumps).
For a transient state \(x\notin\{a_n,b_n\}\), let
\begin{align*}
p_b(x)&:=\Pr(X_1=B(x)\mid X_0=x)
\quad(\text{set }p_b(x):=0\text{ if }B(x)\text{ does not exist}),\\
p(e;x)&:=\Pr(\text{a long jump at time }1\mid X_0=x),
\end{align*}
where \(B(x)\) is the same--chain backward neighbor when it exists. Then
there exists \(n_0\in\N\) (depending only on fixed parameters \(\alpha,\beta,\gamma\)) such that for all \(n\ge n_0\),
\[
\max_{x\notin\{a_n,b_n\}} p_b(x)\ \le\ 2\,p_b(a_{n-1}),
\qquad
\max_{x\notin\{a_n,b_n\}} p(e;x)\ \le\ 2\,p(e;a_{n-1}),
\]
and by chain-swap symmetry the same bounds hold with \(a_{n-1}\) replaced by \(b_{n-1}\).
\end{lemma}

\begin{proof}
Fix a chain (say \(A\)); the statement for chain \(B\) follows by chain-swap symmetry.
Write the one-step kernel as a softmax over logits: for each transient \(x\) and each allowed target \(y\),
\[
P(x,y)=\frac{e^{o(x,y)}}{Z(x)},
\qquad
Z(x):=\sum_{z} e^{o(x,z)},
\]
where the sum ranges over all original one-step targets \(z\) from \(x\) (including long-jump targets).
Therefore
\[
p_b(x)=\frac{e^{o(x,B(x))}}{Z(x)},
\qquad
p(e;x)=\sum_{z\in L(x)}\frac{e^{o(x,z)}}{Z(x)}
=\frac{\sum_{z\in L(x)}e^{o(x,z)}}{Z(x)}.
\]
Under the logit specification of \cref{sec:canonical_chain} (after the harmless global shift by \(-\gamma\)), the logits take only finitely many values,
and both the numerator \(\sum_{z\in L(x)}e^{o(x,z)}\) and the normalizer \(Z(x)\) can be written explicitly in each of the four boundary regimes.
In particular (see the enumeration underlying Lemma~\ref{lem:gamma_source_all_cases} for the corresponding target counts and denominators):
\begin{itemize}
\item \textup{Deep middle \(x=a_i\), \(2<i<n-1\):}
\[
Z_{\mathrm{mid}}=e^{\alpha}+e^{\beta}+4+(2n-7)e^{\gamma},\qquad
p_b(a_i)=\frac{e^{\beta}}{Z_{\mathrm{mid}}},\qquad
p(e;a_i)=\frac{(2n-10)e^{\gamma}}{Z_{\mathrm{mid}}}.
\]
\item \textup{Second vertex \(x=a_2\):}
\[
Z_{2}=e^{\alpha}+e^{(\beta-\gamma)}+3+(2n-6)e^{\gamma},\qquad
p_b(a_2)=\frac{e^{(\beta-\gamma)}}{Z_{2}},\qquad
p(e;a_2)=\frac{(2n-8)e^{\gamma}}{Z_{2}}.
\]
\item \textup{First vertex \(x=a_1\):}
\[
Z_{1}=e^{\alpha}+3+(2n-5)e^{\gamma},\qquad
p_b(a_1)=0,\qquad
p(e;a_1)=\frac{(2n-6)e^{\gamma}}{Z_{1}}.
\]
\item \textup{Penultimate vertex \(x=a_{n-1}\):}
\[
Z_{n-1}=e^{(\alpha-\gamma)}+e^{\beta}+3+(2n-6)e^{\gamma},\qquad
p_b(a_{n-1})=\frac{e^{\beta}}{Z_{n-1}},\qquad
p(e;a_{n-1})=\frac{(2n-8)e^{\gamma}}{Z_{n-1}}.
\]
\end{itemize}

Set \(C_{\max}:=e^{\alpha}+e^{\beta}+4\). For every transient \(x\) on chain \(A\), the normalizer has the form
\[
Z(x)=C_x + c_x e^{\gamma},
\qquad
0\le C_x\le C_{\max},
\qquad
c_x\in\{2n-7,\,2n-6,\,2n-5\},
\]
and at the penultimate vertex \(Z_{n-1}=C_{n-1}+(2n-6)e^{\gamma}\) with \(0\le C_{n-1}\le C_{\max}\).
Choose \(n_0\) so that \((2n_0-6)e^{\gamma}\ge 2C_{\max}\). Then for all \(n\ge n_0\),
\[
Z_{n-1}\ \le\ (2n-6)e^{\gamma}+C_{\max}\ \le\ \tfrac32(2n-6)e^{\gamma},
\qquad
Z(x)\ \ge\ (2n-7)e^{\gamma}.
\]

We first bound \(p_b\). The only vertices with a backward move are \(a_i\) for \(i\ge 2\). Among these,
the backward weight is \(e^{\beta}\) except at \(a_2\), where it is \(e^{(\beta-\gamma)}\le e^{\beta}\).
Hence for all \(x\) on chain \(A\),
\[
p_b(x)\ \le\ \frac{e^{\beta}}{(2n-7)e^{\gamma}}.
\]
On the other hand,
\[
p_b(a_{n-1})=\frac{e^{\beta}}{Z_{n-1}}
\ \ge\
\frac{e^{\beta}}{\tfrac32(2n-6)e^{\gamma}}.
\]
Therefore, for all \(n\ge n_0\),
\[
\frac{p_b(x)}{p_b(a_{n-1})}
\ \le\
\frac{\tfrac32(2n-6)}{(2n-7)}
\ \le\ 2,
\]
which proves \(\max_x p_b(x)\le 2p_b(a_{n-1})\) on chain \(A\).

Next we bound the long-jump mass. In each of the cases above we have \(p(e;x)=N_x e^{\gamma}/Z(x)\) with
\(N_x\in\{2n-6,2n-8,2n-10\}\). Thus for all \(x\) on chain \(A\) and all \(n\ge n_0\),
\[
p(e;x)\ \le\ \frac{(2n-6)e^{\gamma}}{(2n-7)e^{\gamma}}=\frac{2n-6}{2n-7}.
\]
Meanwhile,
\[
p(e;a_{n-1})=\frac{(2n-8)e^{\gamma}}{Z_{n-1}}
\ \ge\
\frac{(2n-8)e^{\gamma}}{\tfrac32(2n-6)e^{\gamma}}
=\frac{2}{3}\cdot\frac{2n-8}{2n-6}.
\]
For \(n\) large enough (in particular, for \(n\ge n_0\) after possibly increasing \(n_0\) by a fixed constant),
the ratio \(\frac{2n-6}{2n-7}\Big/\Big(\frac{2}{3}\cdot\frac{2n-8}{2n-6}\Big)\) is at most \(2\), and therefore
\(\max_x p(e;x)\le 2p(e;a_{n-1})\) on chain \(A\). The same argument applies on chain \(B\).
\end{proof}

Recall the derivative decomposition from Lemma~\ref{lem:decomp_NL_LJ_concise}. For a parameter
$\theta\in\{\alpha,\beta,\gamma\}$ and a start state $v$, we have
\[
\partial_\theta S_\theta(v)=-(\partial_\theta h_\theta(v))\Big(\bar S^{NL}_\theta(v)-\tfrac12\Big)+(1-h_\theta(v))\,\partial_\theta\bar S^{NL}_\theta(v),
\]
where:
\begin{itemize}
\item $S_\theta(v):=\Pr^\theta_v[X_\tau=a_n]$ is the same--side absorption probability in the original chain.
\item $h_\theta(v)$ is the probability of a long jump before absorption, which by definition equals the absorption
probability at $S_{LJ}$ in the long--jump--absorbing chain where long jumps are redirected to the absorbing state $S_{LJ}$.
\item $\bar S^{NL}_\theta(v):=\Pr^\theta_v(X_\tau=a_n\mid NL)$ is the same--side absorption probability conditional on
no long jump occurring (i.e.\ in the no--long--jump chain).
\item By symmetry, $\bar S^{LJ}_\theta(v)=\tfrac12$, so the long--jump conditional absorption probability equals $1/2$.
\end{itemize}

We will prove monotonicity results for each of the three parameters $\theta\in\{\alpha,\beta,\gamma\}$ by showing
that the term involving the long--jump--absorbing chain, $-(\partial_\theta h_\theta(v))\Big(\bar S^{NL}_\theta(v)-\tfrac12\Big)$,
dominates the remaining term $(1-h_\theta(v))\,\partial_\theta\bar S^{NL}_\theta(v)$ in the derivative decomposition.

\begin{lemma}[Monotonicity of $\alpha$]
\label{lem:monotonicity_alpha}
Fix a constant \(K\ge 1\) independent of \(n\) and initialize at \(v_0:=a_{n-K}\).
For all \(n\) large enough, the long--jump term in the decomposition of
Lemma~\ref{lem:decomp_NL_LJ_concise} dominates the NL term, in the sense that
\[
\partial_\alpha S_\alpha(v_0)
\ \ge\
\frac18(\partial_\alpha h_\alpha(v_0))\geq 0.
\]
Moreover, \(\partial_\alpha h_\alpha(v_0)\le 0\), so the right-hand side is nonnegative.
In particular, \(\partial_\alpha S_\alpha(v_0)\ge 0\), so the same--side absorption probability
is (near-terminal) monotonically increasing in \(\alpha\).
\end{lemma}

\begin{proof}
Recall the derivative decomposition (Lemma~\ref{lem:decomp_NL_LJ_concise}):
\[
\partial_\alpha S_\alpha(v_0)
=-(\partial_\alpha h_\alpha(v_0))\Big(\bar S^{NL}_\alpha(v_0)-\tfrac12\Big)
(1-h_\alpha(v_0))\,\partial_\alpha\bar S^{NL}_\alpha(v_0).
\]
We first lower bound the long--jump contribution.  By Lemma~\ref{lem:deriv_est_alpha},
together with Corollary~\ref{cor:di_lower_bound_forward_steps} and the base model assumption
\(p_{\mathrm{fwd}}\ge \fwdbound\) (Assumption~\ref{assump:base-model-paper}), there exists a constant
\(c_{\mathrm{LJ}}>0\) (depending only on fixed parameters and \(K\)) such that
\[
|\partial_\alpha h_\alpha(v_0)|\ \ge\ c_{\mathrm{LJ}}\cdot p(e),
\qquad
p(e):=P^{\mathrm{LJ\text{-}abs}}(v_0,e).
\]
Moreover, by Lemma~\ref{lem:NL_same_side_to_one_pf_pb} applied in the NL chain at depth \(K\),
for \(n\) large enough we have \(\bar S^{NL}_\alpha(v_0)\ge \tfrac34\), hence
\(\bar S^{NL}_\alpha(v_0)-\tfrac12\ge \tfrac14\).

We next upper bound the NL contribution.  By the truncation bound in Lemma~\ref{lem:alpha_band_dominance_clean}
(completely analogous to Lemma~\ref{lem:beta_band_dominance_clean}), we have
\[
\big|\partial_\alpha\bar S^{NL}_\alpha(v_0)\big|
\ \le\
C\,\frac{\sqrt{n}}{\fwdbound}\,p_f^{(2\lfloor n^{1/4}\rfloor)}\,p_{\mathrm{sw}}^{\max}
\;+\;
C\,n\,\exp\!\left(-\frac{c^2}{8}\,n^{1/4}\right)\cdot \sup_{x\notin\{a_n,b_n\}}|s_{\alpha,\mathrm{NL}}(x)|.
\]
Using Lemma~\ref{lem:local_variation_by_psw_times_tau}, \(\sup_x |s_{\alpha,\mathrm{NL}}(x)|\lesssim p_{\mathrm{sw}}^{\max}\sup_x \E_x^{NL}[\tau]\lesssim (n/\fwdbound)\,p_{\mathrm{sw}}^{\max}\)
and the exponential tail \(\exp(-\Omega(n^{1/4}))\), the last term is a negligible fraction of \(p_{\mathrm{sw}}^{\max}\).
Since \(p_f^{(2\lfloor n^{1/4}\rfloor)}\le 1\), we obtain \(\big|\partial_\alpha\bar S^{NL}_\alpha(v_0)\big|\le C\sqrt{n}\,p_{\mathrm{sw}}^{\max}\) for all \(n\) large enough.
Finally, note that the relevant comparison is with the one-step long-jump probability at the initialization,
\(p(e):=P^{\mathrm{LJ\text{-}abs}}(v_0,e)\), which depends on \(v_0\) through the local softmax denominator.
Nevertheless, the same \(1/n\) ratio holds uniformly. Indeed, from any state \(x\), the NL switch set has only \(O(1)\)
available targets (at most three), each with logit at most \(\gamma\), while the long-jump class consists of \(\Theta(n)\)
targets at the same \(\gamma\)-logit scale. One can check this separately in the deep middle (where there are \(2n+O(1)\)
long-jump targets) and at the penultimate vertex (where there are again \(2n+O(1)\) long-jump targets); in both cases the
ratio of the switch numerator to the long-jump numerator is \(O(1/n)\), and conditioning on \(NL\) only rescales by a factor
bounded away from \(0\) (since the forward move has probability at least \(\fwdbound\) by Assumption~\ref{assump:base-model-paper}).
Consequently, for all \(n\) large enough we have the domination \(p_{\mathrm{sw}}^{\max}\le C\,p(e)/n\). Therefore,
\[
\big|(1-h_\alpha(v_0))\,\partial_\alpha\bar S^{NL}_\alpha(v_0)\big|
\ \le\
\frac12\Big(-(\partial_\alpha h_\alpha(v_0))\Big(\bar S^{NL}_\alpha(v_0)-\tfrac12\Big)\Big),
\]
and substituting into the decomposition yields the claimed factor-\(\tfrac12\) dominance and \(\partial_\alpha S_\alpha(v_0)\ge 0\).
\end{proof}

\begin{lemma}[Monotonicity of $\beta$]
\label{lem:monotonicity_beta}
Fix a constant \(K\ge 1\) independent of \(n\) and initialize at \(v_0:=a_{n-K}\).
For all \(n\) large enough, the long--jump term in the decomposition of
Lemma~\ref{lem:decomp_NL_LJ_concise} dominates the NL term, in the sense that
\[
\partial_\beta S_\beta(v_0)
\ \le\
-\frac18\Big((\partial_\beta h_\beta(v_0))\Big)\leq 0.
\]
Moreover, \(\partial_\beta h_\beta(v_0)\ge 0\), so the right-hand side is nonpositive.
In particular, \(\partial_\beta S_\beta(v_0)\le 0\), so the same--side absorption probability
is (near-terminal) monotonically decreasing in \(\beta\).
\end{lemma}

\begin{proof}
Recall the derivative decomposition (Lemma~\ref{lem:decomp_NL_LJ_concise}):
\[
\partial_\beta S_\beta(v_0)
=-(\partial_\beta h_\beta(v_0))\Big(\bar S^{NL}_\beta(v_0)-\tfrac12\Big)
+(1-h_\beta(v_0))\,\partial_\beta\bar S^{NL}_\beta(v_0).
\]
We first lower bound the long--jump contribution.  By Lemma~\ref{lem:deriv_est_beta},
\[
\partial_\beta h_\beta(v_0)\ \ge\ s_\beta(v_0)\;+\;s_\beta(v_{\mathrm{pen}})\,\Pr_{v_0}(0<\sigma<\tau),
\]
where \(v_{\mathrm{pen}}:=a_{n-1}\) and \(\sigma:=\inf\{t\ge 0:\ X_t=v_{\mathrm{pen}}\}\).  We may drop the one of the two terms and get that regardless of whther $n=1$ or $n>1$,
\[
\partial_\beta h_\beta(v_0)\ \ge\ s_\beta(v_{\mathrm{pen}})\,\Pr_{v_0}(0<\sigma<\tau).
\]
Define the one-step long-jump probability at a state \(x\) in the long--jump--absorbing chain by
\[
p(e;x):=P^{\mathrm{LJ\text{-}abs}}(x,e).
\]
By the \(\beta\)-source identity \(s_\beta(x)=p_b(x)\big(u(B(x))-u(x)\big)\) (Lemma~\ref{lem:source_beta_longjump_absorbing}),
Corollary~\ref{cor:di_lower_bound_forward_steps} (applied at \(B(v_{\mathrm{pen}})=a_{n-1}\)), and the base model lower bound
\(p_{\mathrm{fwd}}\ge \fwdbound\) (Assumption~\ref{assump:base-model-paper}), we obtain that
\[
s_\beta(v_{\mathrm{pen}})\ \ge\ \fwdbound^{K}\cdot p_b(v_{\mathrm{pen}})\,p(e;v_{\mathrm{pen}}).
\]
Moreover, if \(v_0=a_{n-K}\) then the event of \(K-1\) consecutive forward steps implies \(\{0<\sigma<\tau\}\), hence
\[
\Pr_{v_0}(0<\sigma<\tau)\ \ge\ \fwdbound^{K-1}.
\]
Combining,
\[
\partial_\beta h_\beta(v_0)
\ \ge\
\fwdbound^{2K-1}\cdot p_b(v_{\mathrm{pen}})\,p(e;v_{\mathrm{pen}}).
\]
 Moreover, by
Lemma~\ref{lem:NL_same_side_to_one_pf_pb} applied in the NL chain at depth \(K\), for \(n\) large enough we have
\(\bar S^{NL}_\beta(v_0)\ge \tfrac34\), hence \(\bar S^{NL}_\beta(v_0)-\tfrac12\ge \tfrac14\).

We next upper bound the NL contribution.  By Lemma~\ref{lem:beta_band_dominance_clean},
\[
\big|\partial_\beta\bar S^{NL}_\beta(v_0)\big|
\ \le\
C\,\frac{\sqrt{n}}{\fwdbound}\,p_b^{(2\lfloor n^{1/4}\rfloor)}\,p_{\mathrm{sw}}^{\max}
\;+\;
C\,n\,\exp\!\left(-\frac{c^2}{8}\,n^{1/4}\right)\cdot \sup_{x\notin\{a_n,b_n\}}|s_{\beta,\mathrm{NL}}(x)|.
\]
We now control the NL source term.
By the NL \(\beta\)--source identity (Lemma~\ref{lem:beta-source-nolj-simple}) and the consecutive--value bound
(Lemma~\ref{lem:local_variation_by_psw_times_tau}), for every nonterminal \(x\) with backward neighbor \(B(x)\),
\[
|s_{\beta,\mathrm{NL}}(x)|
\le
p_{\mathrm{NL}}(x\to B(x))\,|h(B(x))-h(x)|
\le
p_{\mathrm{NL}}(x\to B(x))\,p_{\mathrm{sw}}^{\max}\,\E_x^{NL}[\tau].
\]
Therefore,
\[
\sup_x |s_{\beta,\mathrm{NL}}(x)|
\le
\Big(\sup_x p_{\mathrm{NL}}(x\to B(x))\Big)\,p_{\mathrm{sw}}^{\max}\,\Big(\sup_x \E_x^{NL}[\tau]\Big)
\le
\frac{n}{\fwdbound}\Big(\sup_x p_{\mathrm{NL}}(x\to B(x))\Big)\,p_{\mathrm{sw}}^{\max},
\]
using \(\sup_x \E_x^{NL}[\tau]\le n/\fwdbound\) under the drift gap.

Next we bound both \(\sup_x p_{\mathrm{NL}}(x\to B(x))\) and \(p_{\mathrm{sw}}^{\max}\) in terms of penultimate-vertex
probabilities in the original / LJ-absorbing chain.
By Lemma~\ref{lem:nl_transition_prob_compare}, NL transition probabilities are inflated by at most \(1/\fwdbound\) relative to
the original chain, so in particular
\[
\sup_x p_{\mathrm{NL}}(x\to B(x))
\le
\frac{1}{\fwdbound}\max_x p_b(x),
\qquad
p_{\mathrm{sw}}^{\max}
\le
\frac{1}{\fwdbound}\max_x p_{\mathrm{sw}}^{\mathrm{orig}}(x).
\]
Moreover, in the original chain the switch class consists of \(O(1)\) \(\gamma\)-logit targets while the long-jump class
has \(\Theta(n)\) \(\gamma\)-logit targets, hence \(p_{\mathrm{sw}}^{\mathrm{orig}}(x)\le C\,p(e;x)/n\) for all \(x\).
Combining with Lemma~\ref{lem:max_back_and_lj_at_penultimate} gives, for all \(n\) large enough,
\[
\sup_x p_{\mathrm{NL}}(x\to B(x))\ \le\ C\,p_b(v_{\mathrm{pen}}),
\qquad
p_{\mathrm{sw}}^{\max}\ \le\ C\,\frac{p(e;v_{\mathrm{pen}})}{n}.
\]
Also, \(p_b^{(2\lfloor n^{1/4}\rfloor)}\le \sup_x p_{\mathrm{NL}}(x\to B(x))\le C\,p_b(v_{\mathrm{pen}})\). Substituting these
bounds into the two-term upper bound from Lemma~\ref{lem:beta_band_dominance_clean} yields
\[
\big|\partial_\beta\bar S^{NL}_\beta(v_0)\big|
\le
C\,\frac{\sqrt{n}}{\fwdbound}\,p_b(v_{\mathrm{pen}})\cdot \frac{p(e;v_{\mathrm{pen}})}{n}
\le
\frac{C'}{\fwdbound\sqrt{n}}\,p_b(v_{\mathrm{pen}})\,p(e;v_{\mathrm{pen}})
\]
for all \(n\) large enough.

Therefore, for \(n\) large enough we have
\[
\big|(1-h_\beta(v_0))\,\partial_\beta\bar S^{NL}_\beta(v_0)\big|
\ \le\
\frac12\Big((\partial_\beta h_\beta(v_0))\Big(\bar S^{NL}_\beta(v_0)-\tfrac12\Big)\Big),
\]
and substituting into the decomposition yields
\(\partial_\beta S_\beta(v_0)\le \frac12\big(-(\partial_\beta h_\beta(v_0))(\bar S^{NL}_\beta(v_0)-\tfrac12)\big)\le 0\).
\end{proof}

\begin{lemma}[Monotonicity of $\gamma$]
\label{lem:monotonicity_gamma}
Fix a constant \(K\ge 1\) independent of \(n\) and initialize at \(v_0:=a_{n-K}\).
For all \(n\) large enough, the long--jump term in the decomposition of
Lemma~\ref{lem:decomp_NL_LJ_concise} dominates the NL term, in the sense that
\[
\partial_\gamma S_\gamma(v_0)
\ \le\
-\frac18\Big((\partial_\gamma h_\gamma(v_0))\Big)\leq 0.
\]
Moreover, \(\partial_\gamma h_\gamma(v_0)\ge 0\), so the right-hand side is nonpositive.
In particular, \(\partial_\gamma S_\gamma(v_0)\le 0\), so the same--side absorption probability
is (near-terminal) monotonically decreasing in \(\gamma\).
\end{lemma}

\begin{proof}
We analyze the derivative decomposition (Lemma~\ref{lem:decomp_NL_LJ_concise}):
\[
\partial_\gamma S_\gamma(v_0)
=-(\partial_\gamma h_\gamma(v_0))\Big(\bar S^{NL}_\gamma(v_0)-\tfrac12\Big)
+(1-h_\gamma(v_0))\,\partial_\gamma\bar S^{NL}_\gamma(v_0).
\]

We first lower bound the long--jump contribution in the same way as in the \(\beta\) case.  Let \(v_{\mathrm{pen}}:=a_{n-1}\) and
\(\sigma:=\inf\{t\ge 0:\ X_t=v_{\mathrm{pen}}\}\) in the long--jump--absorbing chain. By Lemma~\ref{lem:deriv_est_gamma},
\[
\partial_\gamma h_\gamma(v_0)\ \ge\ s_\gamma(v_{\mathrm{pen}})\,\Pr_{v_0}(0<\sigma<\tau)\;+\;\frac12\,s_\gamma(v_0)\ \ge\ s_\gamma(v_{\mathrm{pen}})\,\Pr_{v_0}(0<\sigma<\tau),
\]
since \(s_\gamma(v_0)\ge 0\). Moreover, if \(v_0=a_{n-K}\) then the event of \(K-1\) consecutive forward steps implies \(\{0<\sigma<\tau\}\), hence
\[
\Pr_{v_0}(0<\sigma<\tau)\ \ge\ \fwdbound^{K-1}
\]
by the base model lower bound \(p_{\mathrm{fwd}}\ge \fwdbound\) (Assumption~\ref{assump:base-model-paper}). Define the one-step long-jump probability at a state \(x\) in the
long--jump--absorbing chain by \(p(e;x):=P^{\mathrm{LJ\text{-}abs}}(x,e)\). At the penultimate vertex, Lemma~\ref{lem:gamma_source_lower_bounds}\textup{(C)} gives
\[
s_\gamma(v_{\mathrm{pen}})\ \ge\ p(e;v_{\mathrm{pen}})\,\bigl(1-u(v_{\mathrm{pen}})\bigr)\ \ge\ \fwdbound\cdot p(e;v_{\mathrm{pen}}),
\]
where \(u(x)=\Pr_x[X_\tau=e]\) and the last inequality uses \ref{cor:di_lower_bound_forward_steps}.
Combining these and considering cases where $n=1$ or $n>1$ yields \(\partial_\gamma h_\gamma(v_0)\ge 0.5 \cdot\cdot p(e;v_{\mathrm{pen}})\,\fwdbound^{K}\). Moreover, by Lemma~\ref{lem:NL_same_side_to_one_pf_pb}
applied in the NL chain at depth \(K\), for \(n\) large enough we have \(\bar S^{NL}_\gamma(v_0)\ge \tfrac34\), hence \(\bar S^{NL}_\gamma(v_0)-\tfrac12\ge \tfrac14\).

We next upper bound the NL contribution.  By Lemma~\ref{lem:gamma_band_dominance_clean},
\[
\big|\partial_\gamma\bar S^{NL}_\gamma(v_0)\big|
\ \le\
C\,n^{1/4}\,p_{\mathrm{sw}}^{\max}
\;+\;
C\,\frac{\sqrt{n}}{\fwdbound}\,p_f^{(2\lfloor n^{1/4}\rfloor)}\,p_{\mathrm{sw}}^{\max}
\;+\;
C\,n\,\exp\!\left(-\frac{c^2}{8}\,n^{1/4}\right)\cdot \sup_{x\notin\{a_n,b_n\}}|s_{\gamma,\mathrm{NL}}(x)|.
\]
We now control the three terms on the right-hand side using the same comparison-to-penultimate-vertex strategy as in the \(\beta\) case.

First, by Lemma~\ref{lem:nl_transition_prob_compare}, NL probabilities are inflated by at most \(1/\fwdbound\) relative to the original chain.
In the original chain, the switch class consists of \(O(1)\) \(\gamma\)-logit targets while the long-jump class has \(\Theta(n)\)
\(\gamma\)-logit targets, so \(p_{\mathrm{sw}}^{\mathrm{orig}}(x)\le C\,p(e;x)/n\) for all \(x\). Combining with
Lemma~\ref{lem:max_back_and_lj_at_penultimate} gives, for all \(n\) large enough,
\[
p_{\mathrm{sw}}^{\max}\ \le\ \frac{1}{\fwdbound}\max_x p_{\mathrm{sw}}^{\mathrm{orig}}(x)
\ \le\ C\,\frac{\max_x p(e;x)}{n}
\ \le\ C\,\frac{p(e;v_{\mathrm{pen}})}{n}.
\]

Next, we bound the NL \(\gamma\)-source. By the NL \(\gamma\)--source bounds (Lemma~\ref{lem:gamma-source-nolj-mult}),
each \(s_{\gamma,\mathrm{NL}}(x)\) is a linear combination of \(p_{\mathrm{sw}}(x)\) and (at the boundary) one additional
term involving a same--chain neighbor difference \(|h(x^\pm)-h(x)|\). By Lemma~\ref{lem:local_variation_by_psw_times_tau}
and \(\sup_x \E_x^{NL}[\tau]\le n/\fwdbound\) under the drift gap, these neighbor differences are at most
\(p_{\mathrm{sw}}^{\max}\sup_x \E_x^{NL}[\tau]\le (n/\fwdbound)\,p_{\mathrm{sw}}^{\max}\). Therefore,
\[
\sup_x |s_{\gamma,\mathrm{NL}}(x)|\ \le\ C\,\frac{n}{\fwdbound}\,p_{\mathrm{sw}}^{\max}.
\]

Substituting the bound on \(p_{\mathrm{sw}}^{\max}\) into the first two terms (using \(p_f^{(2\lfloor n^{1/4}\rfloor)}\le 1\)),
and the source bound into the exponentially-suppressed third term, yields
\[
\big|\partial_\gamma\bar S^{NL}_\gamma(v_0)\big|
\ \le\
C\,\frac{n^{1/4}}{n}\,p(e;v_{\mathrm{pen}})
\;+\;
C\,\frac{\sqrt{n}}{\fwdbound}\cdot\frac{1}{n}\,p(e;v_{\mathrm{pen}})
\;+\;
C\,n\,\exp\!\left(-\frac{c^2}{8}\,n^{1/4}\right)\cdot \frac{n}{\fwdbound}\cdot \frac{1}{n}\,p(e;v_{\mathrm{pen}}).
\]
In particular, for all \(n\) large enough,
\[
\big|\partial_\gamma\bar S^{NL}_\gamma(v_0)\big|
\ \le\
\frac{C'}{\fwdbound\sqrt{n}}\,p(e;v_{\mathrm{pen}}).
\]

Therefore, for \(n\) large enough we have
\[
\big|(1-h_\gamma(v_0))\,\partial_\gamma\bar S^{NL}_\gamma(v_0)\big|
\ \le\
\frac12\Big((\partial_\gamma h_\gamma(v_0))\Big(\bar S^{NL}_\gamma(v_0)-\tfrac12\Big)\Big),
\]
and substituting into the decomposition yields
\(\partial_\gamma S_\gamma(v_0)\le \frac12\big(-(\partial_\gamma h_\gamma(v_0))(\bar S^{NL}_\gamma(v_0)-\tfrac12)\big)\le 0\).
\end{proof}

\begin{lemma}[Near--terminal same--side derivative for $\gamma$]
\label{lem:ss-deriv-gamma-near-terminal-tbd_upper}
Fix a constant $K\ge 1$ independent of $n$, and initialize at
\[
v_0:=a_{n-K}.
\]
Write $p(e)$ for the one--step probability of a long jump out of $v_0$ in the original chain
(equivalently, $p(e)=P^{\mathrm{LJ\text{-}abs}}(v_0,e)$ in the long--jump--absorbing chain).
Write \(p(e;x):=P^{\mathrm{LJ\text{-}abs}}(x,e)\) and \(p(e)_{\max}:=\sup_{x\notin\{a_n,b_n\}}p(e;x)\).
Then
\[
\big|\partial_\gamma S_\gamma(v_0)\big|\ \le\ C_{1,\gamma}\,\frac{\sqrt{n}}{\fwdbound}\,p(e)_{\max}.
\]
\end{lemma}

\begin{proof}[Proof (upper bound)]
Write the derivative decomposition (Lemma~\ref{lem:decomp_NL_LJ_concise}) at $v_0$:
\[
\partial_\gamma S_\gamma(v_0)
=-(\partial_\gamma h_\gamma(v_0))\Big(\bar S^{NL}_\gamma(v_0)-\tfrac12\Big)
+(1-h_\gamma(v_0))\,\partial_\gamma\bar S^{NL}_\gamma(v_0).
\]
In the proof of Lemma~\ref{lem:monotonicity_gamma} (the domination step), for all $n$ large enough we showed that
\[
\big|(1-h_\gamma(v_0))\,\partial_\gamma\bar S^{NL}_\gamma(v_0)\big|
\ \le\ \tfrac12\,\big|(\partial_\gamma h_\gamma(v_0))\big|\cdot\Big|\bar S^{NL}_\gamma(v_0)-\tfrac12\Big|
\ \le\ \tfrac12\,\big|\partial_\gamma h_\gamma(v_0)\big|.
\]
Therefore,
\begin{equation}
\label{eq:ss_deriv_gamma_upper_reduce_to_LJ}
\big|\partial_\gamma S_\gamma(v_0)\big|
\ \le\ \tfrac32\,\big|\partial_\gamma h_\gamma(v_0)\big|.
\end{equation}

It remains to bound $|\partial_\gamma h_\gamma(v_0)|$ in the long--jump--absorbing chain.
Let $u(x):=\Pr_x[X_\tau=e]$ denote the absorption probability at $e$ (so $h_\gamma(\cdot)=u(\cdot)$), with
$\tau:=\inf\{t\ge 0: X_t\in\{a_n,b_n,e\}\}$. By the trajectory/source representation (Lemma~\ref{lem:source_representation}),
\[
\partial_\gamma h_\gamma(v_0)=\partial_\gamma u(v_0)=\E_{v_0}\!\left[\sum_{t=0}^{\tau-1} s_\gamma(X_t)\right],
\]
where $s_\gamma$ is the local $\gamma$--source in the long--jump--absorbing chain.
We use the standard band decomposition with $M:=\lfloor n^{1/4}\rfloor$.

On $\{\tau\le M\}$, all visited states satisfy $D(X_t)\le D(v_0)+t\le 2M$ for $n$ large enough (since $D(v_0)=K$ is fixed), hence
\[
\left|\sum_{t=0}^{\tau-1} s_\gamma(X_t)\right|\mathbf 1\{\tau\le M\}
\ \le\ M\cdot \sup_{x:\,D(x)\le 2M}|s_\gamma(x)|.
\]
To bound the band supremum, use the explicit $\gamma$--source formulas in the long--jump--absorbing chain (Lemma~\ref{lem:gamma_source_all_cases}).
In each of the finitely many cases (deep middle / first / second / penultimate), $s_\gamma(x)$ is a finite linear combination of terms of the form
\[
\widetilde P(x,y)\cdot w_x(y),
\]
where either $y=e$ (a long jump in the original chain, redirected to $e$) so $\widetilde P(x,y)=p(e;x)\le p(e)_{\max}$, or else $y$ is one of the $O(1)$
non--long--jump $\gamma$--dependent neighbors and then $\widetilde P(x,y)\leq 1$
and the coefficient multiplies a value term $w_x(y)$ which is a difference of absorption probabilities (or a centered version thereof).
On the band $D(\cdot)\le 2M$, Lemma~\ref{lem:lj_value_bound_pe_tau} together with the drift-gap bound $\E[\tau]\lesssim D/\fwdbound$ gives
\[
u(y)=\Pr_y[X_\tau=e]\ \lesssim\ \frac{M}{\fwdbound}\,p(e)_{\max},
\]
and hence every such value term $w_x(y)$ appearing in the case formulas is $O\!\big((M/\fwdbound)\,p(e)_{\max}\big)$.
Since each case has only $O(1)$ such terms, we conclude that
\[
\sup_{D(x)\le 2M}|s_\gamma(x)|\ \le\ C\,\frac{M}{\fwdbound}\,p(e)_{\max}
\]
for a constant $C$.
Therefore the contribution on $\{\tau\le M\}$ is at most $C\,\frac{M^2}{\fwdbound}\,p(e)_{\max}$.

On $\{\tau>M\}$, use the crude bound $\left|\sum_{t<\tau} s_\gamma(X_t)\right|\le \tau\sup_x|s_\gamma(x)|$ together with the tail bound
from Lemma~\ref{lem:hoeffding_tail_absorption} (with $t=M$). The same reasoning as above, with $D(\cdot)\le n$ instead of $2M$,
gives the global bound
\[
\sup_x |s_\gamma(x)|\ \le\ C\,\frac{n}{\fwdbound}\,p(e)_{\max}.
\]
Multiplying by $\Pr_{v_0}(\tau>M)=\exp(-\Omega(M))$ (with $M=\lfloor n^{1/4}\rfloor$) shows this tail contribution is negligible.
Thus,
\[
|\partial_\gamma h_\gamma(v_0)|\ \le\ C\,\frac{M^2}{\fwdbound}\,p(e)_{\max}
\ =\ C\,\frac{\sqrt{n}}{\fwdbound}\,p(e)_{\max}.
\]
Combining with \eqref{eq:ss_deriv_gamma_upper_reduce_to_LJ} yields the claimed upper bound.
\end{proof}

\subsection{Hard example contributions to derivative are negligible}
\label{sec:hard_examples}

Recall the derivative decomposition from Lemma~\ref{lem:decomp_NL_LJ_concise}. For a parameter
$\theta\in\{\alpha,\beta,\gamma\}$ and a start state $v$, we have
\[
\partial_\theta S_\theta(v)=-(\partial_\theta h_\theta(v))\Big(\bar S^{NL}_\theta(v)-\tfrac12\Big)+(1-h_\theta(v))\,\partial_\theta\bar S^{NL}_\theta(v),
\]
where:
\begin{itemize}
\item $S_\theta(v):=\Pr^\theta_v[X_\tau=a_n]$ is the same--side absorption probability in the original chain.
\item $h_\theta(v)$ is the probability of a long jump before absorption, which by definition equals the absorption
probability at $S_{LJ}$ in the long--jump--absorbing chain where long jumps are redirected to the absorbing state $S_{LJ}$.
\item $\bar S^{NL}_\theta(v):=\Pr^\theta_v(X_\tau=a_n\mid NL)$ is the same--side absorption probability conditional on
no long jump occurring (i.e.\ in the no--long--jump chain).
\item By symmetry, $\bar S^{LJ}_\theta(v)=\tfrac12$, so the long--jump conditional absorption probability equals $1/2$.
\end{itemize}

We will now argue using this decomposition that the addition of hard examples preserves the properties of the derivatives argued above. Specifically, we will show that the monotonicity and sign properties established for the derivatives $\partial_\theta S_\theta(v)$ continue to hold when hard examples are included in the training distribution.

\subsubsection{Upper bound on derivative for hard examples}
We will now bound the derivatives of the same-side absorption probability for hard examples. 

\begin{lemma}
  \label{lem:reinforce_covariance_and_CS}
  Let $S$ be finite and let $T\subseteq S$ be absorbing. Fix $A\subseteq T$ and a starting state $v\in S$.
  For $\theta\in\R$, let $P_\theta$ be a time-homogeneous Markov kernel on $S$ with absorbing set $T$.
  Let
  \[
  \tau:=\inf\{t\ge 0:\ X_t\in T\},\qquad
  R:=\mathbf 1\{X_\tau\in A\},\qquad
  u_\theta(v):=\E_v^\theta[R].
  \]
  For $x\notin T$ and $P_\theta(x,y)>0$, define the one-step score
  \[
  \phi_\theta(x,y):=\partial_\theta \log P_\theta(x,y),
  \]
  and the stopped cumulative score
  \[
  L:=\sum_{t=0}^{\tau-1}\phi_\theta(X_t,X_{t+1}).
  \]
  Assume that all differentiations and exchanges of derivative and expectation below are valid, and that
  $\E_v^\theta[L^2]<\infty$. Then $u_\theta(v)$ is differentiable and
  \[
  \partial_\theta u_\theta(v)
  =\E_v^\theta[R\,L]
  =\E_v^\theta\!\big[(R-u_\theta(v))\,L\big]
  =\Cov_v^\theta(R,L).
  \]
  Moreover, by Cauchy--Schwarz,
  \[
  \bigl|\partial_\theta u_\theta(v)\bigr|
  \le \sqrt{u_\theta(v)\bigl(1-u_\theta(v)\bigr)}\;\sqrt{\E_v^\theta[L^2]}.
  \]
  \end{lemma}
  
  \begin{proof}
  By definition $u_\theta(v)=\E_v^\theta[R]$ where $R=\mathbf 1\{X_\tau\in A\}$ depends on the trajectory
  only through the absorbed terminal state. Using the log-derivative trick (score-function identity),
  \[
  \partial_\theta u_\theta(v)
  =\partial_\theta \E_v^\theta[R]
  =\E_v^\theta\!\big[R\,\partial_\theta \log \Pr_v^\theta(X_0,\dots,X_\tau)\big].
  \]
  Since the chain is time-homogeneous and $T$ is absorbing, the (stopped) path likelihood factors as
  \[
  \Pr_v^\theta(X_0,\dots,X_\tau)
  =\prod_{t=0}^{\tau-1} P_\theta(X_t,X_{t+1}).
  \]
  Therefore
  \[
  \partial_\theta \log \Pr_v^\theta(X_0,\dots,X_\tau)
  =\sum_{t=0}^{\tau-1}\partial_\theta \log P_\theta(X_t,X_{t+1})
  =\sum_{t=0}^{\tau-1}\phi_\theta(X_t,X_{t+1})
  =L.
  \]
  Plugging this back gives the REINFORCE form
  \[
  \partial_\theta u_\theta(v)=\E_v^\theta[R\,L].
  \]
  
  Next we show that $L$ has mean zero (baseline fact). Condition on $X_t=x$ (with $x\notin T$). Then
  \[
  \E_v^\theta[\phi_\theta(X_t,X_{t+1})\mid X_t=x]
  =\sum_{y\in S} P_\theta(x,y)\,\partial_\theta \log P_\theta(x,y)
  =\sum_{y\in S}\partial_\theta P_\theta(x,y)
  =\partial_\theta\Big(\sum_{y\in S} P_\theta(x,y)\Big)=0.
  \]
  Summing over $t=0,\dots,\tau-1$ and using linearity of expectation yields
  \[
  \E_v^\theta[L]=\E_v^\theta\!\left[\sum_{t=0}^{\tau-1}\phi_\theta(X_t,X_{t+1})\right]=0.
  \]
  Hence, for the constant baseline $b:=u_\theta(v)=\E_v^\theta[R]$,
  \[
  \E_v^\theta[(R-b)L]
  =\E_v^\theta[R\,L]-b\,\E_v^\theta[L]
  =\E_v^\theta[R\,L].
  \]
  Therefore
  \[
  \partial_\theta u_\theta(v)
  =\E_v^\theta[R\,L]
  =\E_v^\theta[(R-u_\theta(v))\,L].
  \]
  Since $\E_v^\theta[L]=0$, this is exactly the covariance:
  \[
  \E_v^\theta[(R-u_\theta(v))\,L]
  =\E_v^\theta[R\,L]-\E_v^\theta[R]\E_v^\theta[L]
  =\Cov_v^\theta(R,L).
  \]
  
  Finally, apply Cauchy--Schwarz:
  \[
  \bigl|\partial_\theta u_\theta(v)\bigr|
  =\bigl|\E_v^\theta[(R-u_\theta(v))\,L]\bigr|
  \le \sqrt{\E_v^\theta[(R-u_\theta(v))^2]}\;\sqrt{\E_v^\theta[L^2]}.
  \]
  Because $R$ is Bernoulli with mean $u_\theta(v)$,
  \[
  \E_v^\theta[(R-u_\theta(v))^2]=\Var_v^\theta(R)=u_\theta(v)\bigl(1-u_\theta(v)\bigr).
  \]
  Substituting yields
  \[
  \bigl|\partial_\theta u_\theta(v)\bigr|
  \le \sqrt{u_\theta(v)\bigl(1-u_\theta(v)\bigr)}\;\sqrt{\E_v^\theta[L^2]}.
  \]
  \end{proof}
  We have the following decomposition for $E_\theta[L^{2}]$:
  \begin{lemma}
    \label{lem:L2_sum_one_step}
    Let $S$ be finite and let $T\subseteq S$ be absorbing. Fix $\theta\in\R$ and a starting state $v\in S$.
    Let $(X_t)_{t\ge 0}$ be the time-homogeneous Markov chain with kernel $P_\theta$ and absorption time
    \[
    \tau:=\inf\{t\ge 0:\ X_t\in T\}.
    \]
    For $x\notin T$ and $P_\theta(x,y)>0$ define the one-step score
    \[
    \phi_\theta(x,y):=\partial_\theta\log P_\theta(x,y),
    \]
    and define the stopped cumulative score
    \[
    L:=\sum_{t=0}^{\tau-1}\phi_\theta(X_t,X_{t+1}).
    \]
    Define the one-step Fisher-type quantity
    \[
    I_\theta(x):=\sum_{y\in S}P_\theta(x,y)\,\phi_\theta(x,y)^2
    \qquad (x\in S\setminus T),
    \]
    and set $I_\theta(x):=0$ for $x\in T$.
    Assume $\E_v^\theta[L^2]<\infty$ and $\E_v^\theta[L]=0$. Then
    \[
    \Var_v^\theta(L)=\E_v^\theta[L^2]
    =\E_v^\theta\!\left[\sum_{t=0}^{\tau-1} I_\theta(X_t)\right].
    \]
    In particular, if $I_\theta^\star:=\max_{x\in S\setminus T} I_\theta(x)$, then
    \[
    \Var_v^\theta(L)\le I_\theta^\star\,\E_v^\theta[\tau].
    \]
    \end{lemma}
    
    \begin{proof}
    Let $\mathcal F_t:=\sigma(X_0,\dots,X_t)$ and define the stopped increments
    \[
    D_t:=\phi_\theta(X_t,X_{t+1})\,\mathbf 1\{t<\tau\},
    \qquad t\ge 0,
    \]
    so that $L=\sum_{t\ge 0}D_t$.
    
    For each $t\ge 0$ we have $\E_v^\theta[D_t\mid \mathcal F_t]=0$.
    Indeed, on the event $\{t\ge \tau\}$ we have $D_t=0$. On the event $\{t<\tau\}$ we have $X_t\notin T$, and
    conditioning on $X_t=x$ gives
    \[
    \E_v^\theta[D_t\mid \mathcal F_t]
    =\E_v^\theta[\phi_\theta(X_t,X_{t+1})\mid X_t=x]
    =\sum_{y\in S}P_\theta(x,y)\,\partial_\theta\log P_\theta(x,y)
    =\sum_{y\in S}\partial_\theta P_\theta(x,y)=0.
    \]
    Hence $(D_t)_{t\ge 0}$ is a martingale-difference sequence.
    
    Expanding $L^2$ gives
    \[
    L^2=\left(\sum_{t\ge 0}D_t\right)^2
    =\sum_{t\ge 0}D_t^2 + 2\sum_{0\le s<t}D_sD_t.
    \]
    We claim $\E_v^\theta[D_sD_t]=0$ for all $0\le s<t$.
    Since $D_s$ is $\mathcal F_t$-measurable when $s<t$, and using the tower property,
    \[
    \E_v^\theta[D_sD_t]
    =\E_v^\theta\!\left[\E_v^\theta[D_sD_t\mid \mathcal F_t]\right]
    =\E_v^\theta\!\left[D_s\,\E_v^\theta[D_t\mid \mathcal F_t]\right]
    =0,
    \]
    by the martingale-difference property.
    Therefore
    \[
    \E_v^\theta[L^2]=\sum_{t\ge 0}\E_v^\theta[D_t^2].
    \]
    
    For each $t$,
    \[
    \E_v^\theta[D_t^2]
    =\E_v^\theta\!\left[\E_v^\theta[D_t^2\mid \mathcal F_t]\right]
    =\E_v^\theta\!\left[\mathbf 1\{t<\tau\}\,\E_v^\theta\!\big[\phi_\theta(X_t,X_{t+1})^2\mid X_t\big]\right].
    \]
    On $\{t<\tau\}$, $X_t\notin T$, and conditioning on $X_t=x$ yields
    \[
    \E_v^\theta\!\big[\phi_\theta(X_t,X_{t+1})^2\mid X_t=x\big]
    =\sum_{y\in S}P_\theta(x,y)\,\phi_\theta(x,y)^2
    =I_\theta(x).
    \]
    Hence $\E_v^\theta[D_t^2]=\E_v^\theta[\mathbf 1\{t<\tau\}I_\theta(X_t)]$, and summing over $t$ gives
    \[
    \E_v^\theta[L^2]
    =\sum_{t\ge 0}\E_v^\theta[\mathbf 1\{t<\tau\}I_\theta(X_t)]
    =\E_v^\theta\!\left[\sum_{t=0}^{\tau-1}I_\theta(X_t)\right].
    \]
    Since $\E_v^\theta[L]=0$, we have $\Var_v^\theta(L)=\E_v^\theta[L^2]$, proving the main identity.
    
    Finally, $I_\theta(X_t)\le I_\theta^\star$ for $t<\tau$, so
    \[
    \sum_{t=0}^{\tau-1}I_\theta(X_t)\le I_\theta^\star\,\tau,
    \]
    and taking expectations yields $\Var_v^\theta(L)\le I_\theta^\star\,\E_v^\theta[\tau]$.
    \end{proof}
    
    We now bound these single step terms for the LJ absorbing chain:

    \begin{lemma}
      \label{lem:LJ_uniform_one_step_fisher_all_params_streamlined}
      Work in the long-jump absorbing (LJ) chain with the softmax logit structure.
      For each $\theta\in\{\alpha,\beta,\gamma\}$ and each nonterminal state $x\notin T$, let
      \[
      I_\theta(x)\;:=\;\sum_{a\in\mathcal A(x)}P_{\alpha,\beta,\gamma}(x,a)\Big(\partial_\theta\log P_{\alpha,\beta,\gamma}(x,a)\Big)^2
      \]
      denote the one-step Fisher-type term. Then
      \[
      I_\theta(x)\le 1
      \qquad\text{for all } \theta\in\{\alpha,\beta,\gamma\}\text{ and all }x\notin T.
      \]
      Consequently, $I_\theta^\star:=\max_{x\notin T}I_\theta(x)\le 1$ for each $\theta\in\{\alpha,\beta,\gamma\}$.
      \end{lemma}
      
      \begin{proof}
      Fix $\theta\in\{\alpha,\beta,\gamma\}$ and a nonterminal state $x\notin T$. Write the allowed moves out of $x$
      as $\mathcal A(x)$ and write the outgoing probabilities in softmax form
      \[
      p_a(\theta):=P_{\alpha,\beta,\gamma}(x,a)=\frac{e^{\ell_a(\theta)}}{Z(\theta)},
      \qquad
      Z(\theta):=\sum_{b\in\mathcal A(x)}e^{\ell_b(\theta)}.
      \]
      In the LJ logit structure, after the usual normalization for $\gamma$ (subtracting a common $\gamma$ from all logits,
      which does not change the softmax probabilities), the dependence on $\theta$ is affine with coefficient bounded by $1$:
      \[
      \ell_a(\theta)=\ell_a^{(0)}+k_a\,\theta
      \qquad\text{with}\qquad |k_a|\le 1.
      \]
      Define $\bar k(\theta):=\sum_{b\in\mathcal A(x)}p_b(\theta)\,k_b$. Differentiating
      $\log p_a(\theta)=\ell_a(\theta)-\log Z(\theta)$ gives
      \[
      \partial_\theta\log p_a(\theta)
      =\partial_\theta\ell_a(\theta)-\frac{Z'(\theta)}{Z(\theta)}
      =k_a-\bar k(\theta),
      \]
      where we used
      \[
      \frac{Z'(\theta)}{Z(\theta)}
      =\frac{\sum_{b}e^{\ell_b(\theta)}\partial_\theta\ell_b(\theta)}{\sum_{b}e^{\ell_b(\theta)}}
      =\sum_{b}\frac{e^{\ell_b(\theta)}}{Z(\theta)}\,k_b
      =\bar k(\theta).
      \]
      Therefore
      \[
      I_\theta(x)
      =\sum_{a\in\mathcal A(x)}p_a(\theta)\bigl(\partial_\theta\log p_a(\theta)\bigr)^2
      =\sum_{a\in\mathcal A(x)}p_a(\theta)\,(k_a-\bar k(\theta))^2
      =\Var_{a\sim p(\theta)}(k_a).
      \]
      Finally, $\Var(k)\le \E[k^2]$ and $|k_a|\le 1$ imply
      \[
      \Var_{a\sim p(\theta)}(k_a)\le \E_{a\sim p(\theta)}[k_a^2]\le 1,
      \]
      hence $I_\theta(x)\le 1$, as claimed. Taking the maximum over $x\notin T$ yields $I_\theta^\star\le 1$.
      \end{proof}

\begin{lemma}[Fisher bound for LJ chain derivative]
\label{lem:lj_absorption_derivative_overall_bound}
Work in the long--jump--absorbing (LJ) chain (Definition~\ref{def:lj-absorbing-chain}) with absorbing set \(T=\{a_n,b_n,e\}\).
Fix \(\theta\in\{\alpha,\beta,\gamma\}\) and define the long-jump absorption probability
\[
u_\theta(v):=\Pr_v^\theta(X_\tau=e),
\qquad
\tau:=\inf\{t\ge 0:\ X_t\in T\}.
\]
Assume that for all starting states \(v\notin T\),
\[
\E_v^\theta[\tau]\ \le\ \frac{n}{\fwdbound},
\]
for some constant \(\fwdbound>0\) independent of \(n\).
Then, for every \(v\notin T\),
\[
\big|\partial_\theta u_\theta(v)\big|
\ \le\
\sqrt{u_\theta(v)\bigl(1-u_\theta(v)\bigr)}\cdot \sqrt{\frac{n}{\fwdbound}}
\]
\end{lemma}

\begin{proof}
Apply Lemma~\ref{lem:reinforce_covariance_and_CS} with \(A=\{e\}\). This gives
\[
|\partial_\theta u_\theta(v)|\le \sqrt{u_\theta(v)\bigl(1-u_\theta(v)\bigr)}\;\sqrt{\E_v^\theta[L^2]}.
\]
By Lemma~\ref{lem:L2_sum_one_step},
\[
\E_v^\theta[L^2]=\E_v^\theta\!\left[\sum_{t=0}^{\tau-1} I_\theta(X_t)\right]\le I_\theta^\star\,\E_v^\theta[\tau].
\]
Finally, Lemma~\ref{lem:LJ_uniform_one_step_fisher_all_params_streamlined} gives \(I_\theta^\star\le 1\), and by assumption
\(\E_v^\theta[\tau]\le n/\fwdbound\). Combining yields
\[
|\partial_\theta u_\theta(v)|
\le \sqrt{u_\theta(v)\bigl(1-u_\theta(v)\bigr)}\;\sqrt{\frac{n}{\fwdbound}}.
\]
\end{proof}

\begin{lemma}[Exponential decay of derivatives for hard examples]
\label{lem:exp_small_deriv_far_from_terminal}
Fix a constant \(a\in(0,1)\) and consider the original two--chain model of length \(n\), together with its
long--jump--absorbing chain (Definition~\ref{def:lj-absorbing-chain}).
Assume there exists a constant \(q_{\min}\in(0,1)\), independent of \(n\), such that the one-step long-jump probability satisfies
\[
p(e;x):=P^{\mathrm{LJ\text{-}abs}}(x,e)\ \ge\ q_{\min}
\qquad\text{for all nonterminal }x\notin\{a_n,b_n\}.
\]
Let \(v\) be any start state with depth \(D(v)\ge an\). Then for each parameter \(\theta\in\{\alpha,\beta,\gamma\}\) there exist constants
\(C,c>0\) (depending only on fixed parameters, \(a\), and \(q_{\min}\), but not on \(n\)) such that
\[
\big|\partial_\alpha S_\alpha(v)\big|\ \le\ C\,\frac{n}{\fwdbound}\,p_f^{\max}\,e^{-cn},
\qquad
\big|\partial_\beta S_\beta(v)\big|\ \le\ C\,\frac{n}{\fwdbound}\,p_b^{\max}\,e^{-cn},
\qquad
\big|\partial_\gamma S_\gamma(v)\big|\ \le\ C\,\frac{n}{\fwdbound}\,p(e)_{\max}\,e^{-cn},
\]
where \(\fwdbound>0\) is the drift-gap constant from Assumption~\ref{assump:base-model-paper},
\begin{align*}
p_f^{\max}&:=\sup_{x\notin\{a_n,b_n\}} \Pr(X_1=F(x)\mid X_0=x),\\
p_b^{\max}&:=\sup_{x\notin\{a_n,b_n\}} \Pr(X_1=B(x)\mid X_0=x),\\
p(e)_{\max}&:=\sup_{x\notin\{a_n,b_n\}} p(e;x),
\end{align*}
and \(F(x),B(x)\) denote the same--chain forward/backward neighbors when they exist (with the convention that the corresponding probability is \(0\) if the neighbor does not exist).
\end{lemma}

\begin{proof}
Let \(\tau:=\inf\{t\ge 0:\ X_t\in\{a_n,b_n\}\}\) be the absorption time in the original chain and let \(NL:=\{J>\tau\}\) be the event
of no long jump before absorption (with \(J\) as in Lemma~\ref{lem:decomp_NL_LJ_concise}).
In the long--jump--absorbing chain (Definition~\ref{def:lj-absorbing-chain}) with absorption time
\[
\tau_{LJ}:=\inf\{t\ge 0:\ X_t\in\{a_n,b_n,e\}\},
\]
one has \(NL=\{X_{\tau_{LJ}}\in\{a_n,b_n\}\}\) and \(LJ=\{X_{\tau_{LJ}}=e\}\).

We use the long--jump / no--long--jump \emph{derivative} decomposition (Lemma~\ref{lem:decomp_NL_LJ_concise}): for each
\(\theta\in\{\alpha,\beta,\gamma\}\),
\[
\partial_\theta S_\theta(v)
=-(\partial_\theta h_\theta(v))\Big(\bar S^{NL}_\theta(v)-\tfrac12\Big)
+(1-h_\theta(v))\,\partial_\theta\bar S^{NL}_\theta(v),
\]
where \(h_\theta(v)=\Pr_v^\theta(LJ)\) and \(1-h_\theta(v)=\Pr_v^\theta(NL)\).

Since \(D(v)\ge an\), on \(NL\) the depth can decrease by at most \(1\) per step, so \(\tau\ge an\). Under the assumption
\(p(e;x)\ge q_{\min}\) for every nonterminal \(x\), this gives
\[
\Pr_v^\theta(NL)
\le \Pr_v^\theta(\text{no hit of }e\text{ in the first }an\text{ steps})
\le (1-q_{\min})^{an}
\le e^{-cn}
\]
for some \(c>0\) depending only on \(a,q_{\min}\).

Using \(|\bar S^{NL}_\theta(v)-\tfrac12|\le \tfrac12\), we have
\[
\big|\partial_\theta S_\theta(v)\big|
\le \tfrac12\,\big|\partial_\theta h_\theta(v)\big|+\Pr_v^\theta(NL)\cdot \big|\partial_\theta\bar S^{NL}_\theta(v)\big|.
\]
Now \(\partial_\theta h_\theta(v)=\partial_\theta \Pr_v^\theta(LJ)\) is the derivative of the long-jump absorption probability in the
long--jump--absorbing chain (with absorbing set \(T=\{a_n,b_n,e\}\) and target \(A=\{e\}\)).
Therefore, applying Lemma~\ref{lem:lj_absorption_derivative_overall_bound} yields
\[
|\partial_\theta h_\theta(v)|
\le
c\,\sqrt{h_\theta(v)\bigl(1-h_\theta(v)\bigr)}\cdot \sqrt{\frac{n}{\fwdbound}}
=
c\,\sqrt{\Pr_v^\theta(LJ)\Pr_v^\theta(NL)}\cdot \sqrt{\frac{n}{\fwdbound}}
\le
C\,\sqrt{\frac{n}{\fwdbound}}\,e^{-c'n},
\]
where we used \(\Pr_v^\theta(NL)\le e^{-cn}\) and absorbed constants by possibly decreasing \(c\).

It remains to control \(|\partial_\theta \bar S^{NL}_\theta(v)|\) by the relevant one-step probability scale.
By the source identities in the NL chain (Lemmas~\ref{lem:alpha-source-nolj-simple},~\ref{lem:beta-source-nolj-simple}, and the corresponding
\(\gamma\)-source bound Lemma~\ref{lem:gamma-source-nolj-mult}), the NL sources satisfy
\(|s_{\alpha,NL}(x)|\le C\,p_f(x)\), \(|s_{\beta,NL}(x)|\le C\,p_b(x)\), and \(|s_{\gamma,NL}(x)|\le C\,p_{\mathrm{sw}}(x)\le C\,p(e)_{\max}\).
Therefore, by the trajectory/source representation (Lemma~\ref{lem:source_representation}) and the drift-gap absorption time bound
\(\E_v^{NL}[\tau]\le n/\fwdbound\), we have
\[
|\partial_\alpha \bar S^{NL}_\alpha(v)|\le C\,\frac{n}{\fwdbound}\,p_f^{\max},\qquad
|\partial_\beta \bar S^{NL}_\beta(v)|\le C\,\frac{n}{\fwdbound}\,p_b^{\max},\qquad
|\partial_\gamma \bar S^{NL}_\gamma(v)|\le C\,\frac{n}{\fwdbound}\,p(e)_{\max}.
\]
Multiplying by \(\Pr_v^\theta(NL)\le e^{-cn}\) and combining with the bound on \(|\partial_\theta h_\theta(v)|\) yields the stated result.
\end{proof}
\subsection{Easy+Hard examples derivative estimates}
\label{app:easy_hard_deriv_est}

We now apply the results of the previous subsections to derive our three main derivative bounds that will be used under the training-distribution assumptions of
Theorem~\ref{result:dyn_learn}.

\begin{lemma}[Easy+hard expected derivative bound for $\gamma$ (Theorem~\ref{result:dyn_learn})]
  \label{lem:easy_hard_expected_deriv_gamma}
  Assume the training distribution $\QQ$ satisfies the assumptions of Theorem~\ref{result:dyn_learn},
  and that there exists a constant $q_{\min}\in(0,1)$, independent of $n$, such that $P^{\mathrm{LJ\text{-}abs}}(x,e) \ge q_{\min}$ for all nonterminal $x$.
  Let $v_i:=a_i$ (equivalently, by symmetry one may take $v_i:=b_i$). Define
  \[
  p(e)_{\max}:=\max_{x\ \mathrm{nonterminal}} P^{\mathrm{LJ\text{-}abs}}(x,e).
  \]
  Then for all $n$ large enough,
  \[
  \E_{i\sim\QQ}\big[\partial_\gamma S_\gamma(v_i)\big]\le 0,
  \]
  and there exist constants $C_{1,\gamma},C_{2,\gamma}>0$ (independent of $n$ and $c$) such that
  \[ c \cdot
  C_{2,\gamma}\,p(e)_{\max}
  \ \le\
  \Big|\E_{i\sim\QQ}\big[\partial_\gamma S_\gamma(v_i)\big]\Big|
  \ \le\
  c \cdot C_{1,\gamma}\,\sqrt{n}\,p(e)_{\max}.
  \]
  \end{lemma}
  \begin{proof}
    Let $v_i := a_i$. Define the index sets
    \[
    \mathcal{E} := \{n-k, \ldots, n-1\} \quad \text{(easy)}, \qquad
    \mathcal{H} := \{1, \ldots, \lfloor(1-a)n\rfloor\} \quad \text{(hard)}.
    \]
    Since $\QQ$ has s-simplicity mass $c$ we have $\QQ(\mathcal{E}) = c$, and since $\QQ$ also has $\lceil a n \rceil$-simplicity mass $c$, we can decompose the expected derivative as
    \begin{equation}
    \label{eq:deriv_decomp_easy_hard_gamma}
    \E_{i \sim \QQ}[\partial_\gamma S_\gamma(v_i)]
    = \QQ(\mathcal{E}) \cdot \E_{i \sim \QQ}[\partial_\gamma S_\gamma(v_i) \mid i \in \mathcal{E}]
    + \QQ(\mathcal{H}) \cdot \E_{i \sim \QQ}[\partial_\gamma S_\gamma(v_i) \mid i \in \mathcal{H}].
    \end{equation}

    By Lemma~\ref{lem:monotonicity_gamma}, $\partial_\gamma S_\gamma(v_i) \le 0$ for all $n$ large enough, and moreover
    \[
    C_{2,\gamma}\,p(e)_{\max}
    \ \le\ 
    |\partial_\gamma S_\gamma(v_i)| 
    \ \le\ 
    C_{1,\gamma}\,\sqrt{n}\,p(e)_{\max},
    \]
    where the constants $C_{1,\gamma}, C_{2,\gamma} > 0$ depend only on $k$ (not on $n$ or $c$).
    Therefore, the easy-example conditional expectation satisfies
    \begin{equation}
    \label{eq:easy_gamma_bounds}
    C_{2,\gamma}\,p(e)_{\max}
    \ \le\ 
    \Big|\E_{i \sim \QQ}[\partial_\gamma S_\gamma(v_i) \mid i \in \mathcal{E}]\Big|
    \ \le\ 
    C_{1,\gamma}\,\sqrt{n}\,p(e)_{\max}.
    \end{equation}
    
    For each $i \in \mathcal{H}$, the depth satisfies $D(v_i) = n - i \ge an$.
    By Lemma~\ref{lem:exp_small_deriv_far_from_terminal}, for all $n$ large enough,
    \begin{equation}
    \label{eq:hard_gamma_bound}
    \big|\QQ(\mathcal{H}) \cdot \E_{i \sim \QQ}[\partial_\gamma S_\gamma(v_i) \mid i \in \mathcal{H}]\big|
    \le C\,\frac{n}{\fwdbound}\,p(e)_{\max}\,e^{-cn}.
    \end{equation}
    
    Combining \eqref{eq:easy_gamma_bounds} and \eqref{eq:hard_gamma_bound} via the triangle inequality, and redefining the constants $C_{1,\gamma}, C_{2,\gamma}$, we obtain the stated bounds.
    \end{proof}

  \begin{lemma}[Easy+hard expected derivative lower bound for $\beta$ (Theorem~\ref{result:dyn_learn})]
  \label{lem:easy_hard_expected_deriv_beta_upper}
  Assume the training distribution $\QQ$ satisfies the assumptions of Theorem~\ref{result:dyn_learn},
  and that there exists a constant $q_{\min}\in(0,1)$, independent of $n$, such that $P^{\mathrm{LJ\text{-}abs}}(x,e) \ge q_{\min}$ for all nonterminal $x$.
  Let $v_i:=a_i$ (equivalently, by symmetry one may take $v_i:=b_i$). Define
  \[
  p(e)_{\max}:=\max_{x\ \mathrm{nonterminal}} P^{\mathrm{LJ\text{-}abs}}(x,e),
  \qquad
  p(b)_{\max}:=\max_{x\ \mathrm{nonterminal}} p_b(x),
  \]
  where $p_b(x)$ denotes the one--step same--chain backward probability from $x$ in the original chain.
  Then for all $n$ large enough,
  \[
  \E_{i\sim\QQ}\big[\partial_\beta S_\beta(v_i)\big]\le 0,
  \]
  and there exists a constant $C_{2,\beta}>0$ (independent of $n$) such that
  \[
  \Big|\E_{i\sim\QQ}\big[\partial_\beta S_\beta(v_i)\big]\Big|
  \ \ge\ c \cdot
  C_{2,\beta}\,p(e)_{\max}\,p(b)_{\max}.
  \]
  \end{lemma}
  \begin{proof}
    As in Lemma~\ref{lem:easy_hard_expected_deriv_gamma}, we decompose
    \[
    \E_{i \sim \QQ}[\partial_\beta S_\beta(v_i)]
    = \QQ(\mathcal{E}) \cdot \E[\partial_\beta S_\beta(v_i) \mid i \in \mathcal{E}]
    + \QQ(\mathcal{H}) \cdot \E[\partial_\beta S_\beta(v_i) \mid i \in \mathcal{H}].
    \]
    For easy examples, Lemmas~\ref{lem:monotonicity_beta} and~\ref{lem:deriv_est_beta} give $\partial_\beta S_\beta(v_i) \le 0$ and $|\partial_\beta S_\beta(v_i)| \ge C_{2,\beta}\,p(e)_{\max}\,p(b)_{\max}$.
    For hard examples, Lemma~\ref{lem:exp_small_deriv_far_from_terminal} gives $|\partial_\beta S_\beta(v_i)| \le C\,\frac{n}{\fwdbound}\,p_b^{\max}\,e^{-cn}$.
    Since this is exponentially small, combining via the triangle inequality and redefining $C_{2,\beta}$ yields the stated bounds.
  \end{proof}
  
  \begin{lemma}[Easy+hard expected derivative lower bound for $\alpha$ (Theorem~\ref{result:dyn_learn})]
  \label{lem:easy_hard_expected_deriv_alpha_lower}
  Assume the training distribution $\QQ$ satisfies the assumptions of Theorem~\ref{result:dyn_learn},
  and that there exists a constant $q_{\min}\in(0,1)$, independent of $n$, such that $P^{\mathrm{LJ\text{-}abs}}(x,e) \ge q_{\min}$ for all nonterminal $x$.
  Let $v_i:=a_i$ (equivalently, by symmetry one may take $v_i:=b_i$). Define
  \[
  p(e)_{\max}:=\max_{x\ \mathrm{nonterminal}} P^{\mathrm{LJ\text{-}abs}}(x,e),
  \qquad
  p(f)_{\min}:=\min_{x\ \mathrm{nonterminal}} p_f(x),
  \]
  where $p_f(x)$ denotes the one--step same--chain forward probability from $x$ in the original chain.
  Then for all $n$ large enough,
  \[
  \E_{i\sim\QQ}\big[\partial_\alpha S_\alpha(v_i)\big]\ge 0,
  \]
  and there exists a constant $C_{2,\alpha}>0$ (independent of $n$) such that
  \[
  \Big|\E_{i\sim\QQ}\big[\partial_\alpha S_\alpha(v_i)\big]\Big|
  \ \ge c \cdot \
  C_{2,\alpha}\,p(e)_{\max}\,p(f)_{\min}.
  \]
  \end{lemma}
  \begin{proof}
    As in Lemma~\ref{lem:easy_hard_expected_deriv_gamma}, we decompose
    \[
    \E_{i \sim \QQ}[\partial_\alpha S_\alpha(v_i)]
    = \QQ(\mathcal{E}) \cdot \E[\partial_\alpha S_\alpha(v_i) \mid i \in \mathcal{E}]
    + \QQ(\mathcal{H}) \cdot \E[\partial_\alpha S_\alpha(v_i) \mid i \in \mathcal{H}].
    \]
    For easy examples, Lemmas~\ref{lem:monotonicity_alpha} and~\ref{lem:deriv_est_alpha} give $\partial_\alpha S_\alpha(v_i) \ge 0$ and $|\partial_\alpha S_\alpha(v_i)| \ge C_{2,\alpha}\,p(e)_{\max}\,p(f)_{\min}$.
    For hard examples, Lemma~\ref{lem:exp_small_deriv_far_from_terminal} gives $|\partial_\alpha S_\alpha(v_i)| \le C\,\frac{n}{\fwdbound}\,p_f^{\max}\,e^{-cn}$.
    Since this is exponentially small, combining via the triangle inequality and redefining $C_{2,\alpha}$ yields the stated bounds.
  \end{proof}

\subsection{Convergence rate}
\label{app:dyn_learn}
Main proof
We now present the proof of \cref{result:dyn_learn}, bringing together the components developed in the previous subsections.

\begin{proof}[Proof of \cref{result:dyn_learn}]
    By \cref{lem:symmetry_reduction}, the dynamics of the attention parameters are governed by the gradients of the reduced loss $S := \tilde L(\alpha,\beta,\gamma)$.
    Since the loss $S$ is an expectation over the starting vertex location $v$ (sampled from $Q$), we write $S = \E_{v \sim Q}[S_{v}(\alpha, \beta, \gamma)]$, where $S_{v}$ is the loss term for a specific $v$.
    Substituting this into the dynamics from \cref{lem:symmetry_reduction} and taking absolute values, we have with $|V| = 2n$:
    \begin{align*}
        |\dot\alpha(t)| &= \frac{1}{2n(2n-1)}\, \left| \E_{v} [\partial_\alpha S_{v}] \right|,\\
        |\dot\beta(t)| &= \frac{1}{2n(2n-1)}\, \left| \E_{v} [\partial_\beta S_{v}] \right|,\\
        |\dot\gamma(t)| &= \frac{1}{(2n-2)\cdot 2n(2n-1)}\, \left| \E_{v} [\partial_\gamma S_{v}] \right|.
    \end{align*}
    We apply \cref{lem:easy_hard_expected_deriv_alpha_lower,lem:easy_hard_expected_deriv_beta_upper,lem:easy_hard_expected_deriv_gamma} to bound the derivatives.
    These lemmas require the assumption that $p(e)_{\max}(t) \ge q_{\min}$ for some constant $q_{\min} > 0$ independent of $n$. Under this assumption, they imply $\E_{v \sim Q}[\partial_\alpha S_{v}] \ge 0$, $\E_{v \sim Q}[\partial_\beta S_{v}] \le 0$, and $\E_{v \sim Q}[\partial_\gamma S_{v}] \le 0$.
    Since \cref{lem:symmetry_reduction} gives $\dot\theta = -c_\theta \partial_\theta S$ for positive constants $c_\theta$, we have $\dot\alpha(t) \ge 0$, $\dot\beta(t) \le 0$, and $\dot\gamma(t) \le 0$. 

    We will from now assume that $p(e)_{\max}(t) \ge q_{\min}$ while analyzing the dynamics of the parameters $\alpha, \beta, \gamma$ until some time $t_0$ (to be determined later). Later, we will show that for all $t\leq t_0$ indeed $p(e)_{\max}(t) \ge q_{\min}$.
    
    We define the extremal probabilities at time $t$:
    \[
        p(e)_{\max}(t) := \max{x\ \mathrm{nonterminal}} P^{\mathrm{LJ\text{-}abs}}(x,e;t),
        \qquad
        p(f)_{\min}(t) := \min_{x\ \mathrm{nonterminal}} p_f(x;t),
    \]
    where the dependence on $t$ comes through the parameters $\alpha(t),\beta(t),\gamma(t)$.
    
    From \cref{lem:easy_hard_expected_deriv_alpha_lower,lem:easy_hard_expected_deriv_gamma}, the expected derivatives satisfy:
    \begin{itemize}
        \item $\big|\E_{v \sim Q}[\partial_\alpha S_{v}]\big| \ge c \cdot C_{2,\alpha}\,p(e)_{\max}(t)\,p(f)_{\min}(t)$,
        \item $ c \cdot C_{2,\gamma}\,p(e)_{\max}(t) \le \big|\E_{v \sim Q}[\partial_\gamma S_{v}]\big| \le  c \cdot C_{1,\gamma}\,\sqrt{n}\,p(e)_{\max}(t)$.
    \end{itemize}
    
    Next, we derive a lower bound on $|\dot\alpha(t)|$. From the dynamics and the lower bound on the expected $\alpha$-derivative:
    \[ |\dot\alpha(t)| \ge  \frac{c \cdot C_{2,\alpha}}{2n(2n-1)} \cdot p(e)_{\max}(t)\,p(f)_{\min}(t). \]
    
    By the softmax structure of the attention logits (\cref{lem:dyn:logits_cases_canonical}), there exist constants $D>0$ and $C>0$ (independent of $n$) such that
    \[
    p(f)_{\min}(0) = \frac{e^{\alpha(0)-D}}{e^{\alpha(0)-D}+C}.
    \]
    Moreover, since $\dot{\beta}, \dot{\gamma} \le 0$, we have that 
    \[
    p(f)_{\min}(t) \geq \frac{e^{\alpha(t)-D}}{e^{\alpha(t)-D}+C} =: g(t).
    \]
    By assumption (\cref{sec:dyn:setting}), at initialization we have $p(f)_{\min}(0) \geq \fwdbound$.
    
    Define
    \[
    q_{\min} := \frac{1}{8} \cdot   \frac{\varepsilon \, \fwdbound}{(1-\varepsilon)(1-\fwdbound)}  \cdot \switchbound,
    \]
    where $\switchbound > 0$ is the lower bound on the initial switch probability from \cref{sec:dyn:setting}.
    
    Under the assumption that $p(e)_{\max}(t) \ge q_{\min}$, we have
    \[ |\dot\alpha(t)| \ge \frac{c \cdot C_{2,\alpha}\,q_{\min}}{2n(2n-1)} \cdot p(f)_{\min}(t). \]
    Since $p(f)_{\min}(t) \geq g(t)$, we may apply \cref{lem:soft_max_ode_scaled_shifted} to conclude that there exists a time 
    \begin{align}
    t_0 &<  \frac{16\, n(2n-1)}{c \cdot C_{2,\alpha}\, \switchbound} \cdot  \frac{(1-\varepsilon)(1-\fwdbound)}{\varepsilon\, \fwdbound} \cdot \log\left(\frac{1-\varepsilon}{\varepsilon}\cdot\frac{1-\fwdbound}{\fwdbound}\right)\\
    &<  O(\frac{n^{2}}{c\cdot \varepsilon^{2}})
    \end{align}
    such that $p(f)_{\min}(t_0) \geq 1-\varepsilon$ and 
    \begin{equation}
    \alpha(t_0) - \alpha(0) < \log\left(\frac{1-\varepsilon}{\varepsilon}\cdot\frac{1-\fwdbound}{\fwdbound}\right). \label{eq:alpha_increase_bound}
    \end{equation} 

    All that's left to show is that indeed for any $t\leq t_0$, $p(e)_{\max}(t) \ge q_{\min}$ for some constant $q_{\min} > 0$ independent of $n$. Assume the contrary, then there exists a minimal time $0 \leq t_1 < t_0$ such that $p(e)_{\max}(t_1) < q_{\min}$. First, we now show that for any constant $1>a_1>0$ independent of $n$, to be determined later, for sufficiently large $n$, for any $t\in[0,t_1]$, $|\dot\gamma(t)| < a_1|\dot\alpha(t)|$.
    
    Using the bounds from the lemmas:
    \[ |\dot\alpha(t)| \ge \frac{c\cdot C_{2,\alpha}}{2n(2n-1)} \cdot p(e)_{\max}(t)\,p(f)_{\min}(t), \]
    and
    \[ |\dot\gamma(t)| \le \frac{c \cdot C_{1,\gamma}\,\sqrt{n}}{(2n-2) \cdot 2n(2n-1)} \cdot p(e)_{\max}(t). \]
    Taking the ratio and using that the $p(e)_{\max}(t)$ terms cancel:
    \[
        \frac{|\dot\gamma(t)|}{|\dot\alpha(t)|} \le \frac{ C_{1,\gamma}}{C_{2,\alpha} (2n-2)} \cdot \frac{\sqrt{n}}{p(f)_{\min}(t)}.
    \]
    By assumption (\cref{sec:dyn:setting}), the model is initialized such that $p(f)_{\min}(0) \ge \fwdbound$, together with the logit structure (\cref{lem:dyn:logits_cases_canonical}) and the fact that $\dot\alpha \ge 0$ and $\dot\beta, \dot\gamma \le 0$ implies that $p(f)_{\min}(t) \ge \fwdbound$ for all $t$ if $n$ is sufficiently large.
    
    Consequently, for sufficiently large $n$, we have $|\dot\gamma(t)| < a_1|\dot\alpha(t)|$.
    
    This inequality allows us to apply \cref{lem:dominated-change-deriv} with $x(t)=\alpha(t),y(t)=\gamma(t)$ and $c=1/a_1$. Note that $\dot\alpha \ge 0$, $\dot\gamma \le 0$, and $\dot\alpha \ge -\dot\gamma/a_1$.
    This yields
    \[ \alpha(t) - \alpha(0) \ge \frac{1}{a_1}(\gamma(0) - \gamma(t)). \]
    Rearranging and exponentiating gives
    \begin{equation}
        e^{\gamma(t)} \ge e^{\gamma(0) + a_1(\alpha(0)-\alpha(t))}. \label{eq:gamma_alpha_bound}
    \end{equation}
    By \cref{eq:alpha_increase_bound}, for all $t \le t_1$ we have that
    \[
    \alpha(t) - \alpha(0) < \log\left(\frac{1-\varepsilon}{\varepsilon}\cdot\frac{1-\fwdbound}{\fwdbound}\right),
    \]
    which is a constant independent of $n$. Therefore, we can choose $a_1 > 0$ sufficiently small (independent of $n$) such that for all $t \le t_0$,
    \[
    e^{\gamma(t)} \ge e^{\gamma(0) + a_1(\alpha(0)-\alpha(t))} \ge e^{\gamma(0) - a_1\log\left(\frac{1-\varepsilon}{\varepsilon}\cdot\frac{1-\fwdbound}{\fwdbound}\right)} \ge \frac{1}{2} e^{\gamma(0)}.
    \] 
    
    By the logit structure (\cref{lem:dyn:logits_cases_canonical}), and the fact that $\alpha(t)>\beta(t),\gamma(t),0$ at init and throughout the run (by signs of the derivatives), we have that for any $t\leq t_1$:
    \begin{align*}
    p(e)_{\max}(t) 
    &\geq \frac{1}{2} \cdot \frac{n e^{\gamma(t)}}{2n e^{\gamma(t)}+ e^{\beta(t)}+ e^{\alpha(t)}} \\
    &\geq \frac{1}{4} \cdot \frac{n e^{\gamma(0)}}{2n e^{\gamma(0)}+ e^{\beta(0)}+ e^{\alpha(0)+(\alpha(t)-\alpha(0))}} \\
    \end{align*}
    By \cref{eq:alpha_increase_bound}, we have $\alpha(t)-\alpha(0) < \log\left(\frac{1-\varepsilon}{\varepsilon}\cdot\frac{1-\fwdbound}{\fwdbound}\right)$:
    \begin{align*}
        p(e)_{\max}(t) \geq & \frac{1}{4} \cdot 
         \frac{\varepsilon \, \fwdbound}{(1-\varepsilon)(1-\fwdbound)}  \cdot \frac{n e^{\gamma(0)}}{2n e^{\gamma(0)}+ e^{\beta(0)}+ e^{\alpha(0)}} \\
        \geq & \frac{1}{8} \cdot 
          \frac{\varepsilon \, \fwdbound}{(1-\varepsilon)(1-\fwdbound)}  \cdot p(e)_{\max}(0) \\
        \geq & \frac{1}{8} \cdot 
          \frac{\varepsilon \, \fwdbound}{(1-\varepsilon)(1-\fwdbound)}  \cdot \switchbound
        \;=\; q_{\min}.
    \end{align*}
    Thus $p(e)_{\max}(t) \geq q_{\min}$ for all $t \le t_1$, which contradicts the assumption that $p(e)_{\max}(t_1) < q_{\min}$, as required.
\end{proof}

\subsection{Extension to discrete-time gradient descent}
\label{app:discretization}

We briefly justify the claim in Footnote~\ref{fn:discretization} (\cref{sec:dyn:traverse_chain}) that the gradient flow guarantees of \cref{result:dyn_learn} extend to discrete-time policy gradient with a fixed step size~$\eta > 0$.

By \cref{lem:symmetry_reduction}, the full attention-matrix dynamics reduce to dynamics on the three scalar parameters $(\alpha, \beta, \gamma)$.
The analysis in \cref{sec:monotonicity,app:easy_hard_deriv_est} establishes that the expected gradients satisfy definite sign conditions---$\partial_\alpha S \geq 0$, $\partial_\beta S \leq 0$, $\partial_\gamma S \leq 0$ (\cref{lem:monotonicity_alpha,lem:monotonicity_beta,lem:monotonicity_gamma})---together with quantitative lower bounds on their magnitudes (\cref{lem:easy_hard_expected_deriv_alpha_lower,lem:easy_hard_expected_deriv_beta_upper,lem:easy_hard_expected_deriv_gamma}), provided two conditions hold: (i)~the forward bias assumption $p_{\mathrm{fwd}} > p_{\mathrm{bwd}} + \fwdbound$ (\cref{assump:base-model-paper}), and (ii)~the long-jump probability remains bounded away from zero, $p(e)_{\max}(t) \geq q_{\min}$.
Crucially, these sign conditions are properties of the \emph{loss landscape} at the current parameter values, not of the continuous-time parametrization.
Therefore they apply equally to discrete gradient descent with any step size $\eta > 0$: at each iteration, $\alpha_{k+1} = \alpha_k + \eta\,\nabla_\alpha \RR \geq \alpha_k$, and similarly $\beta_{k+1} \leq \beta_k$ and $\gamma_{k+1} \leq \gamma_k$.
The monotonicity of the parameters is thus an immediate corollary of the gradient sign conditions, with no ODE approximation required.
Moreover, the same self-reinforcing structure used in the gradient flow proof (\cref{app:dyn_learn}) carries over: since $\alpha$ is non-decreasing and $\beta, \gamma$ are non-increasing, condition~(i) is maintained (in fact strengthened) along the discrete trajectory; and the verification that condition~(ii) holds throughout $[0, t_0]$ (carried out in the final part of the proof of \cref{result:dyn_learn} in \cref{app:dyn_learn}) depends only on the monotonicity of the parameters and the bounded total change in $\alpha$, both of which transfer identically to the discrete setting.
The quantitative lower bounds on gradient magnitudes from \cref{lem:easy_hard_expected_deriv_alpha_lower,lem:easy_hard_expected_deriv_gamma} then imply that the number of discrete iterations to reach the target accuracy is $O(n^2 c^{-1} \varepsilon^{-2} \eta^{-1})$---polynomial in all relevant quantities.

	\section{Proof of \cref{result:dyn_vanishing}}
\label{app:dyn_vanishing}

As in \cref{app:dyn_learn}, we define the loss as $\mathcal{L} := 1 - S$, where $S$ is the same-side absorption probability (the reward). Gradient flow on $\mathcal{L}$ is equivalent to gradient ascent on $S$.

The proof proceeds in four stages. First, we establish general total-variation bounds showing that small parameter changes induce small changes in trajectory distributions (\cref{lem:tv_data_processing_finite}--\cref{cor:traj_tv_from_params}). Second, we bound the absorption time of the Markov chain (\cref{lem:geom_tail_from_pmin}--\cref{lem:plj_param_lipschitz}). Third, we show that under the hard-example assumption, the loss gradient is exponentially small in the chain length~$n$ (\cref{lem:uniform_exp_small_loss_grad}) at initialization. Finally, we combine these ingredients, showing that for exponentially long times, the derivative bounds remain valid and hence the induced distribution over rollouts remains statistically indistinguishable from the rollout distribution at initialization.

\subsection{Total-variation bounds for trajectory distributions}
\label{subsec:tv_bounds}

The first lemma records a data processing (pushforward contraction) property of total variation on finite spaces.

\begin{lemma}[Data processing for total variation on finite spaces]
\label{lem:tv_data_processing_finite}
Let $\Omega,\mathcal S$ be finite sets. Let $P,Q$ be probability measures on $\Omega$, and let
$f:\Omega\to\mathcal S$ be any map. Denote the pushforward (image) measures by
\[
(P\circ f^{-1})(B):=P\big(f^{-1}(B)\big),\qquad (Q\circ f^{-1})(B):=Q\big(f^{-1}(B)\big),
\qquad B\subseteq \mathcal S.
\]
Then
\[
\|P\circ f^{-1}-Q\circ f^{-1}\|_{\TV}\ \le\ \|P-Q\|_{\TV}.
\]
\end{lemma}

\begin{proof}
By definition of total variation on a finite space,
\[
\|P\circ f^{-1}-Q\circ f^{-1}\|_{\TV}
=\sup_{B\subseteq \mathcal S}\big|(P\circ f^{-1})(B)-(Q\circ f^{-1})(B)\big|
=\sup_{B\subseteq \mathcal S}\big|P(f^{-1}(B))-Q(f^{-1}(B))\big|.
\]
Since $f^{-1}(B)\subseteq \Omega$ for every $B\subseteq \mathcal S$, we have
\[
\sup_{B\subseteq \mathcal S}\big|P(f^{-1}(B))-Q(f^{-1}(B))\big|
\ \le\
\sup_{A\subseteq \Omega}\big|P(A)-Q(A)\big|
=\|P-Q\|_{\TV},
\]
which proves the claim.
\end{proof}

The next lemma bounds the total variation between joint distributions sharing a common marginal.

\begin{lemma}
  \label{lem:tv_product_common_marginal}
  Let $\X,\Y$ be finite sets. Let $\Pi,\Pi'$ be probability measures on $\X\times\Y$ with the same first
  marginal $\mu$ on $\X$, i.e.
  \[
  \Pi(x,\Y)=\Pi'(x,\Y)=\mu(x)\qquad \forall\,x\in\X.
  \]
  For each $x$ with $\mu(x)>0$, let
  \[
  K_x(\cdot):=\Pi(\cdot\mid X=x),\qquad L_x(\cdot):=\Pi'(\cdot\mid X=x)
  \]
  denote the conditional distributions on $\Y$. (For $x$ with $\mu(x)=0$, define $K_x=L_x$ arbitrarily.)
  Then
  \[
  \|\Pi-\Pi'\|_{\TV}
  \ \le\
  \sum_{x\in\X}\mu(x)\,\|K_x-L_x\|_{\TV}.
  \]
  \end{lemma}
  
  \begin{proof}
  For any event $E\subseteq \X\times\Y$, write the $x$-section
  \[
  E_x:=\{y\in\Y:\ (x,y)\in E\}\subseteq \Y.
  \]
  We have that
  \[
  \Pi(E)=\sum_{x\in\X}\Pi(\{x\}\times E_x)
  =\sum_{x\in\X}\Pi(X=x)\,\Pi(E_x\mid X=x)
  =\sum_{x\in\X}\mu(x)\,K_x(E_x),
  \]
  and similarly
  \[
  \Pi'(E)=\sum_{x\in\X}\mu(x)\,L_x(E_x),
  \]
  since the first marginal is the same $\mu$ for both measures. Therefore
  \[
  \Pi(E)-\Pi'(E)=\sum_{x\in\X}\mu(x)\,\big(K_x(E_x)-L_x(E_x)\big).
  \]
  Taking absolute values and using $|\sum a_x|\le \sum |a_x|$ gives
  \[
  |\Pi(E)-\Pi'(E)|
  \le \sum_{x\in\X}\mu(x)\,\big|K_x(E_x)-L_x(E_x)\big|.
  \]
  For each fixed $x$, the set $E_x\subseteq\Y$ is an admissible event in the definition of total
  variation on $\Y$, hence
  \[
  \big|K_x(E_x)-L_x(E_x)\big|
  \le \sup_{B\subseteq \Y}\big|K_x(B)-L_x(B)\big|
  = \|K_x-L_x\|_{\TV}.
  \]
  Substituting yields
  \[
  |\Pi(E)-\Pi'(E)|
  \le \sum_{x\in\X}\mu(x)\,\|K_x-L_x\|_{\TV}.
  \]
  Finally, take a supremum over all events $E\subseteq\X\times\Y$ to obtain
  \[
  \|\Pi-\Pi'\|_{\TV}
  =\sup_{E\subseteq \X\times\Y}|\Pi(E)-\Pi'(E)|
  \le \sum_{x\in\X}\mu(x)\,\|K_x-L_x\|_{\TV},
  \]
  as claimed.
  \end{proof}
  
  We now extend the above lemma to the case of trajectories sampled from Markov chains.
\begin{lemma}[TV bound for trajectories]
  \label{lem:tv_traj_from_kernel}
  Let \(\X\) be a finite state space and let \(P,Q\) be Markov kernels on \(\X\).
  Fix an initial distribution \(\mu\) on \(\X\). For \(T\ge 1\), let
  \(\mathbb P^{(T)}_{\mu,P}\) and \(\mathbb P^{(T)}_{\mu,Q}\) denote the laws on \(\X^{T+1}\) of the
  length-\((T+1)\) trajectories \((X_0,\dots,X_T)\) with \(X_0\sim\mu\) and one-step transitions
  given by \(P\) and \(Q\), respectively. Then
  \[
  \big\|\mathbb P^{(T)}_{\mu,P}-\mathbb P^{(T)}_{\mu,Q}\big\|_{\TV}
  \ \le\
  \sum_{t=0}^{T-1}\sup_{x\in\X}\big\|P(x,\cdot)-Q(x,\cdot)\big\|_{\TV}
  \ =\
  T\cdot \sup_{x\in\X}\big\|P(x,\cdot)-Q(x,\cdot)\big\|_{\TV}.
  \]
  \end{lemma}
  
  \begin{proof}
    For \(t=0,1,\dots,T\), let \(R^{(t)}\) be the law on \(\X^{T+1}\) of the trajectory
    \((X_0,\dots,X_T)\) obtained by starting from \(X_0\sim\mu\), using \(P\) for the first \(t\)
    transitions, and using \(Q\) for the remaining transitions. Thus
    \(R^{(0)}=\mathbb P^{(T)}_{\mu,Q}\) and \(R^{(T)}=\mathbb P^{(T)}_{\mu,P}\). By the triangle inequality,
    \[
    \big\|\mathbb P^{(T)}_{\mu,P}-\mathbb P^{(T)}_{\mu,Q}\big\|_{\TV}
    \le \sum_{t=0}^{T-1}\|R^{(t+1)}-R^{(t)}\|_{\TV}.
    \]
    
    Fix \(t\in\{0,\dots,T-1\}\). Define \(Y:=(X_{t+1},\dots,X_T)\in\X^{T-t}\).
    Under both \(R^{(t)}\) and \(R^{(t+1)}\), the first \(t\) transitions are governed by \(P\), hence the
    marginal law of \(X_t\) is the same. Therefore conditioning on \(X_t\) and averaging yields
    \[
    \|R^{(t+1)}-R^{(t)}\|_{\TV}
    \le \sum_{x\in\X}\Pr_{R^{(t)}}(X_t=x)\,
    \Big\|R^{(t+1)}(\cdot\mid X_t=x)-R^{(t)}(\cdot\mid X_t=x)\Big\|_{\TV}.
    \]
    
    Define a Markov kernel \(M\) from \(\X\) to \(\X^{T-t}\) as follows: for each \(y\in\X\),
    \(M(y,\cdot)\) is the trajectory distribution on \(\X^{T-t}\) of length \((T-t)\) for the \(Q\)-chain
    \((Z_0,\dots,Z_{T-t-1})\) started from \(Z_0=y\).
    We use the standard composition notation: if \(\nu\) is a distribution on \(\X\), then \(\nu M\) denotes
    the induced distribution on \(\X^{T-t}\) obtained by first sampling \(Y_0\sim \nu\) and then sampling a
    trajectory \(Y\sim M(Y_0,\cdot)\); equivalently, for any \(A\subseteq \X^{T-t}\),
    \[
    (\nu M)(A)=\sum_{y\in\X}\nu(y)\,M(y,A).
    \]
    Then, conditional on \(X_t=x\),
    under \(R^{(t)}\) we have \(X_{t+1}\sim Q(x,\cdot)\) and \(Y\) has distribution \(Q(x,\cdot)M\), whereas
    under \(R^{(t+1)}\) we have \(X_{t+1}\sim P(x,\cdot)\) and \(Y\) has distribution \(P(x,\cdot)M\).
    Hence
    \[
    \Big\|R^{(t+1)}(\cdot\mid X_t=x)-R^{(t)}(\cdot\mid X_t=x)\Big\|_{\TV}
    =
    \|P(x,\cdot)M-Q(x,\cdot)M\|_{\TV}.
    \]
    
    Finally, total variation is contractive under a common Markov kernel (by Lemma~\ref{lem:tv_data_processing_finite}): for any distributions \(\nu,\eta\)
    on \(\X\),
    \[
    \|\nu M-\eta M\|_{\TV}\le \|\nu-\eta\|_{\TV}.
    \]
    Applying this with \(\nu=P(x,\cdot)\) and \(\eta=Q(x,\cdot)\) yields
    \[
    \|P(x,\cdot)M-Q(x,\cdot)M\|_{\TV}\le \|P(x,\cdot)-Q(x,\cdot)\|_{\TV}.
    \]
    Substituting into the bound on \(\|R^{(t+1)}-R^{(t)}\|_{\TV}\) gives
    \[
    \|R^{(t+1)}-R^{(t)}\|_{\TV}
    \le \sum_{x\in\X}\Pr_{R^{(t)}}(X_t=x)\,\|P(x,\cdot)-Q(x,\cdot)\|_{\TV}
    \le \sup_{x\in\X}\|P(x,\cdot)-Q(x,\cdot)\|_{\TV}.
    \]
    
    Summing over \(t=0,\dots,T-1\) completes the proof:
    \[
    \big\|\mathbb P^{(T)}_{\mu,P}-\mathbb P^{(T)}_{\mu,Q}\big\|_{\TV}
    \le T\cdot \sup_{x\in\X}\|P(x,\cdot)-Q(x,\cdot)\|_{\TV}.
    \]
    \end{proof}
    
  We now prove a slight extension of the above lemma to the case of trajectories sampled from Markov chains with absorbing states, and hence the trajectory length is not fixed.
  \begin{lemma}
  \label{lem:tv_path_to_absorption_trunc}
Let $\X$ be finite. Let $P,Q$ be Markov kernels on $\X$ and let $T\subseteq \X$ be absorbing for both $P$ and $Q$.
Let $\tau:=\inf\{t\ge 0:\ X_t\in T\}$. For an integer $M\ge 0$, define the truncated stopping time
$\tau\wedge M$ and consider the corresponding stopped paths $X_{0:\tau}$ and $X_{0:\tau\wedge M}$.
Let $\mathcal L_P(X_{0:\tau})$ and $\mathcal L_Q(X_{0:\tau})$ denote the laws of $X_{0:\tau}$ under $P$ and $Q$
(respectively), for the same initial law $\mu$.
Then
\[
\big\|\mathcal L_P(X_{0:\tau})-\mathcal L_Q(X_{0:\tau})\big\|_{\TV}
\ \le\
M\cdot \sup_{x\in\X}\|P(x,\cdot)-Q(x,\cdot)\|_{\TV}
\;+\;\Pr_P(\tau>M)\;+\;\Pr_Q(\tau>M).
\]
  \end{lemma}
  
  \begin{proof}
Let $\mathbb P^{(M)}_{\mu,P}$ and $\mathbb P^{(M)}_{\mu,Q}$ be the laws of $(X_0,\dots,X_M)$ under kernels
$P$ and $Q$, respectively. The truncated stopped path $X_{0:\tau\wedge M}$ is a deterministic function
of $(X_0,\dots,X_M)$,
 hence by the data processing inequality for total variation (Lemma~\ref{lem:tv_data_processing_finite}),
\[
\big\|\mathcal L_P(X_{0:\tau\wedge M})-\mathcal L_Q(X_{0:\tau\wedge M})\big\|_{\TV}
\ \le\
\big\|\mathbb P^{(M)}_{\mu,P}-\mathbb P^{(M)}_{\mu,Q}\big\|_{\TV}.
\]
By Lemma~\ref{lem:tv_traj_from_kernel},
\[
\big\|\mathbb P^{(M)}_{\mu,P}-\mathbb P^{(M)}_{\mu,Q}\big\|_{\TV}
\ \le\
M\cdot \sup_{x\in\X}\|P(x,\cdot)-Q(x,\cdot)\|_{\TV},
\]
and therefore
\[
\big\|\mathcal L_P(X_{0:\tau\wedge M})-\mathcal L_Q(X_{0:\tau\wedge M})\big\|_{\TV}
\ \le\
M\cdot \sup_{x\in\X}\|P(x,\cdot)-Q(x,\cdot)\|_{\TV}.
\]

It remains to compare the true stopped path $X_{0:\tau}$ with the truncated one $X_{0:\tau\wedge M}$.
On the event $\{\tau\le M\}$ we have the identity $X_{0:\tau}=X_{0:\tau\wedge M}$. Hence for any event $A$
in the stopped-path sigma-field,
\begin{equation}
\label{eq:one_sided_A}
\Pr_P(X_{0:\tau}\in A)
\le
\Pr_P(X_{0:\tau\wedge M}\in A)\;+\;\Pr_P(\tau>M),
\end{equation}
and similarly
\begin{equation}
\label{eq:one_sided_A_Q}
\Pr_Q(X_{0:\tau}\in A)
\le
\Pr_Q(X_{0:\tau\wedge M}\in A)\;+\;\Pr_Q(\tau>M).
\end{equation}
Applying \eqref{eq:one_sided_A} to the complement $A^c$ and using $\Pr(\cdot)=1-\Pr(\cdot^c)$ yields the
corresponding reverse inequality:
\[
\Pr_P(X_{0:\tau}\in A)
\ge
\Pr_P(X_{0:\tau\wedge M}\in A)\;-\;\Pr_P(\tau>M),
\]
and likewise
\[
\Pr_Q(X_{0:\tau}\in A)
\ge
\Pr_Q(X_{0:\tau\wedge M}\in A)\;-\;\Pr_Q(\tau>M).
\]
Combining upper and lower bounds, we obtain for every $A$,
\[
\big|\Pr_P(X_{0:\tau}\in A)-\Pr_Q(X_{0:\tau}\in A)\big|
\le
\big|\Pr_P(X_{0:\tau\wedge M}\in A)-\Pr_Q(X_{0:\tau\wedge M}\in A)\big|
+\Pr_P(\tau>M)+\Pr_Q(\tau>M).
\]
Taking a supremum over $A$ gives
\[
\big\|\mathcal L_P(X_{0:\tau})-\mathcal L_Q(X_{0:\tau})\big\|_{\TV}
\ \le\
\big\|\mathcal L_P(X_{0:\tau\wedge M})-\mathcal L_Q(X_{0:\tau\wedge M})\big\|_{\TV}
+\Pr_P(\tau>M)+\Pr_Q(\tau>M).
\]
Substituting the earlier bound on the truncated stopped-path TV completes the proof.
  \end{proof}
  
  The following lemma is a standard result about the total variation distance between distributions induced by softmax over logits.
  \begin{lemma}[Softmax Lipschitz bound]
  \label{lem:softmax_lipschitz_tv}
  Let \(\mathcal Y\) be finite. For logits \(o,o'\in\mathbb R^{\mathcal Y}\), define
  \[
  \sigma(o)(y):=\frac{e^{o(y)}}{\sum_{z\in\mathcal Y}e^{o(z)}}.
  \]
  Then
  \[
  \|\sigma(o)-\sigma(o')\|_{\TV}\ \le\ \|o-o'\|_{\infty}.
  \]
  \end{lemma}
  
  \begin{proof}
  Let \(o_t:=(1-t)o+to'\) and \(p_t:=\sigma(o_t)\). Then
  \[
  \sigma(o')-\sigma(o)=\int_0^1 \frac{d}{dt}p_t\,dt.
  \]
  By the softmax derivative identity (Lemma~\ref{lem:softmax_derivative_identity_aux}) with parameter \(t\), each coordinate satisfies
  \[
  \Big|\frac{d}{dt}p_t(y)\Big|
  = p_t(y)\Big| (o'(y)-o(y))-\sum_{z}p_t(z)(o'(z)-o(z))\Big|
  \le 2\,p_t(y)\,\|o'-o\|_\infty.
  \]
  Summing over \(y\) gives
  \[
  \Big\|\frac{d}{dt}p_t\Big\|_1\le 2\|o'-o\|_\infty,
  \]
  hence
  \[
  \|\sigma(o')-\sigma(o)\|_{\TV}
  =\frac12\|\sigma(o')-\sigma(o)\|_1
  \le \frac12\int_0^1 \Big\|\frac{d}{dt}p_t\Big\|_1\,dt
  \le \|o'-o\|_\infty.
  \]
  \end{proof}
  
  We now apply the above lemmas to the case of Markov kernels induced by the logit structure (\cref{sec:canonical_chain}).
  
  \begin{lemma}[Kernel TV bound]
  \label{lem:kernel_tv_from_params}
  Let \(P_{\alpha,\beta,\gamma}\) and \(P_{\alpha',\beta',\gamma'}\) be Markov kernels induced
  by the logit structure with parameters \((\alpha,\beta,\gamma)\) and \((\alpha',\beta',\gamma')\)
  (see \cref{sec:canonical_chain}).
  Then
  \[
  \sup_x \big\|P_{\alpha,\beta,\gamma}(x,\cdot)-P_{\alpha',\beta',\gamma'}(x,\cdot)\big\|_{\TV}
  \ \le\
  |\alpha-\alpha'|+|\beta-\beta'|+|\gamma-\gamma'|.
  \]
  \end{lemma}
  
  \begin{proof}
  Fix \(x\). By the logit structure, \(P_{\alpha,\beta,\gamma}(x,\cdot)\) is a softmax over logits
  \(o_{\alpha,\beta,\gamma}(x,\cdot)\). By inspection of the logit structure (including boundary cases),
  each coordinate of \(o(x,\cdot)\) changes by at most
  \(|\alpha-\alpha'|+|\beta-\beta'|+|\gamma-\gamma'|\) when the parameters are changed, hence
  \[
  \|o_{\alpha,\beta,\gamma}(x,\cdot)-o_{\alpha',\beta',\gamma'}(x,\cdot)\|_\infty
  \le |\alpha-\alpha'|+|\beta-\beta'|+|\gamma-\gamma'|.
  \]
  Lemma~\ref{lem:softmax_lipschitz_tv} implies the stated bound after taking
  a supremum over \(x\).

  \end{proof}

Combining the kernel TV bound with the trajectory TV lemmas yields the following corollary.

  \begin{corollary}[Trajectory TV bound from parameter changes]
  \label{cor:traj_tv_from_params}
  Let \(P_{\theta}\) and \(P_{\theta'}\) be Markov kernels (induced by the logit structure) at
  \(\theta=(\alpha,\beta,\gamma)\) and \(\theta'=(\alpha',\beta',\gamma')\), with the same initial
  distribution \(\mu\). Then for every \(T\ge 1\),
  \[
  \big\|\mathbb P^{(T)}_{\mu,P_{\theta}}
  -
  \mathbb P^{(T)}_{\mu,P_{\theta'}}\big\|_{\TV}
  \ \le\
  T\,\|\theta-\theta'\|_1.
  \]
  If \(T\subseteq\X\) is absorbing and \(\tau\) is the corresponding hitting time, then for every
  integer \(M\ge 0\),
  \[
  \big\|p_{\text{roll}}(\theta)
  -
  p_{\text{roll}}(\theta')\big\|_{\TV}
  \ \le\
  M\,\|\theta-\theta'\|_1
  +\Pr_{\theta}(\tau>M)+\Pr_{\theta'}(\tau>M).
  \]
  \end{corollary}
  
  \begin{proof}
  Combine Lemma~\ref{lem:tv_traj_from_kernel} with Lemma~\ref{lem:kernel_tv_from_params} for the
  fixed-horizon bound, and combine Lemma~\ref{lem:tv_path_to_absorption_trunc} with
  Lemma~\ref{lem:kernel_tv_from_params} for the stopped-path bound.
  \end{proof}

\subsection{Absorption time bounds}
\label{subsec:absorption_time}

The stopped-path TV bound in Corollary~\ref{cor:traj_tv_from_params} involves tail probabilities \(\Pr(\tau > M)\), which we control via a geometric bound on the absorption time.

\begin{lemma}[Geometric absorption tail]
\label{lem:geom_tail_from_pmin}
  Let \(P\) be a Markov kernel on \(\X\) with terminal states \(\{a_n,b_n\}\),
  and let \(\tau:=\inf\{t\ge 0:\ X_t\in\{a_n,b_n\}\}\).
  If for every non-terminal state \(x\), the two-step absorption probability satisfies \(P^2(x,\{a_n,b_n\})\ge p\) for some \(p\in(0,1]\), then for every integer \(m\ge 0\),
\[
\Pr(\tau>2m)\le (1-p)^m,
\qquad\text{and hence}\qquad
\mathbb E[\tau]\le \frac{2}{p}.
\]
\end{lemma}

\begin{proof}
  For each \(t\ge 0\), on the event \(\{\tau>2t\}\) we have \(X_{2t}\notin\{a_n,b_n\}\) and therefore
\[
\Pr(\tau>2(t+1)\mid \tau>2t)
\le \Pr(X_{2(t+1)}\notin\{a_n,b_n\}\mid X_{2t}\notin\{a_n,b_n\})
\le 1-p.
\]
  Thus \(\Pr(\tau>2(t+1))\le (1-p)\Pr(\tau>2t)\), and iterating gives \(\Pr(\tau>2m)\le (1-p)^m\).
  The expectation bound follows from \(\mathbb E[\tau]\le 2\sum_{m\ge 0}\Pr(\tau>2m)\).
\end{proof}

  To apply Lemma~\ref{lem:geom_tail_from_pmin}, we must verify that the two-step absorption probability is lower bounded at initialization.
  The next lemma establishes this via a two-step argument: from any non-terminal state, one can first perform a long jump to a penultimate state, then move forward to the terminal.
  
  \begin{lemma}[Initial lower bound on two-step absorption probability]
  \label{lem:pmin_init_lower_bound}
  Under the base model assumptions of \cref{sec:dyn:setting}---specifically that \(p_{\mathrm{fwd}} > p_{\mathrm{bwd}} + \fwdbound\) and \(p_{\mathrm{switch}} > \switchbound\)---the minimum two-step absorption probability satisfies
  \[
  p_{\mathrm{abs}}^{(2)}(\theta_0)\ :=\ \min_{x\notin\{a_n,b_n\}}P_{\theta_0}^2(x,\{a_n,b_n\})\ \ge\ \frac{c}{n}
  \]
  for some constant \(c = c(\switchbound, \fwdbound) > 0\) and all \(n\) sufficiently large.
  \end{lemma}
  
  \begin{proof}
  Fix a non-terminal vertex \(x\). We show that with probability \(\Omega(1/n)\), the chain reaches a terminal state in exactly two steps via the following path: first, a long jump to a penultimate state \(a_{n-1}\) or \(b_{n-1}\); then, a forward step to the corresponding terminal \(a_n\) or \(b_n\).
  
  \emph{Step 1: Long jump to penultimate state.}
  By the logit structure (\cref{sec:canonical_chain}), all distant targets share the same logit \(\gamma\).
  The penultimate states \(\{a_{n-1}, b_{n-1}\}\) are among the \(\Theta(n)\) distant targets from any non-terminal \(x\).
  Since the switch probability \(p_{\mathrm{switch}}(x) > \switchbound\) and all distant targets have equal logit, the probability of jumping to \(\{a_{n-1}, b_{n-1}\}\) is at least \(\Omega(\switchbound/n)\).
  
  \emph{Step 2: Forward step to terminal.}
  From a penultimate state \(a_{n-1}\) (or \(b_{n-1}\)), the forward step leads to the terminal \(a_n\) (or \(b_n\)).
  By the base model assumption, the forward probability satisfies \(p_{\mathrm{fwd}} > p_{\mathrm{bwd}} + \fwdbound\), which implies \(p_{\mathrm{fwd}} \ge \fwdbound/2\) (since probabilities are non-negative).
  
  Combining both steps, the two-step absorption probability is at least
  \[
  P_{\theta_0}^2(x,\{a_n,b_n\})\ \ge\ \Omega\Big(\frac{\switchbound}{n}\Big) \cdot \Omega(\fwdbound)\ =\ \Omega\Big(\frac{1}{n}\Big).
  \]
  Taking the minimum over all non-terminal \(x\) yields the result.
  \end{proof}
  
  Having established that \(p_{\mathrm{abs}}^{(2)}(\theta_0) > 0\) at initialization, we need to ensure it remains positive throughout gradient flow.
  The next lemma shows that \(p_{\mathrm{abs}}^{(2)}\) is Lipschitz in the parameters, so it cannot drop too quickly as \(\theta\) evolves.
  
  \begin{lemma}[Lipschitz stability of \(p_{\mathrm{abs}}^{(2)}\) in parameters]
  \label{lem:pmin_param_lipschitz}
  Let \(P_{\theta}\) be the Markov kernel (induced by the logit structure) at parameters
  \(\theta=(\alpha,\beta,\gamma)\), and define
  \[
  p_{\mathrm{abs}}^{(2)}(\theta)\ :=\ \min_{x\notin\{a_n,b_n\}}P_{\theta}^2(x,\{a_n,b_n\}).
  \]
  Then for all \(\theta,\theta'\),
  \[
  p_{\mathrm{abs}}^{(2)}(\theta)\ \ge\ p_{\mathrm{abs}}^{(2)}(\theta')\ -\ C\|\theta-\theta'\|_1
  \]
  for some constant \(C > 0\).
  \end{lemma}
  
  \begin{proof}
  Fix \(x\notin\{a_n,b_n\}\). The two-step transition probability can be written as
  \[
  P_{\theta}^2(x,\{a_n,b_n\}) = \sum_{y} P_{\theta}(x,y) P_{\theta}(y,\{a_n,b_n\}).
  \]
  By the triangle inequality and Lemma~\ref{lem:kernel_tv_from_params}, changing \(\theta\) to \(\theta'\) affects each one-step kernel by at most \(\|\theta-\theta'\|_1\) in total variation.
  Since the two-step probability is a composition of two kernels, the change in \(P^2(x,\{a_n,b_n\})\) is bounded by \(C\|\theta-\theta'\|_1\) for some constant \(C\).
  Taking the minimum over all non-terminal \(x\) yields the result.
  \end{proof}

  For the gradient bounds below, we also need stability of the \emph{long-jump probability}---the probability of jumping to a distant state in one step.
  
  \begin{lemma}[Lipschitz stability of long-jump probability]
  \label{lem:plj_param_lipschitz}
  Let \(P_{\theta}\) be the Markov kernel (induced by the logit structure) at parameters \(\theta=(\alpha,\beta,\gamma)\).
  For each non-terminal state \(x\), let \(\mathrm{LJ}(x) := \{y : |D(y) - D(x)| \ge 2\}\) denote the set of long-jump targets, and define
  \[
  p_{\mathrm{lj}}(\theta)\ :=\ \min_{x\notin\{a_n,b_n\}}P_{\theta}(x,\mathrm{LJ}(x)).
  \]
  Then for all \(\theta,\theta'\),
  \[
  p_{\mathrm{lj}}(\theta)\ \ge\ p_{\mathrm{lj}}(\theta')\ -\ \|\theta-\theta'\|_1.
  \]
  \end{lemma}
  
  \begin{proof}
  Fix \(x\notin\{a_n,b_n\}\). By definition of total variation,
  \[
  \big|P_{\theta}(x,\mathrm{LJ}(x))-P_{\theta'}(x,\mathrm{LJ}(x))\big|
  \le
  \big\|P_{\theta}(x,\cdot)-P_{\theta'}(x,\cdot)\big\|_{\TV}.
  \]
  By Lemma~\ref{lem:kernel_tv_from_params}, the right-hand side is at most \(\|\theta-\theta'\|_1\).
  Taking the minimum over all non-terminal \(x\) yields the result.
  \end{proof}
  
\subsection{Exponentially small gradients for hard examples}
\label{subsec:exp_small_grad}

The final ingredient shows that when all starting states are far from the terminal, the loss gradient is exponentially small.
  
  \begin{lemma}[Exponentially small gradient for hard starts]
  \label{lem:uniform_exp_small_loss_grad}
  Fix constants \(a\in(0,1)\) and \(q_0\in(0,1)\). Let \(\mathcal H\subseteq\X\) be a set of start states with \(D(v)\ge an\) for all \(v\in\mathcal H\), let \((w_v)_{v\in\mathcal H}\) be nonnegative weights summing to~\(1\), and define
  \[
  \mathcal L(\theta)\ :=\ \sum_{v\in\mathcal H} w_v\,S(v;\theta),
  \qquad
  \mathcal U_{q_0}\ :=\ \Big\{\theta:\ p_{\mathrm{lj}}(\theta)\ge q_0\Big\}.
  \]
  Then there exist constants \(C,c>0\) (independent of \(n\)) such that for every \(\theta\in\mathcal U_{q_0}\),
  \[
  \|\nabla \mathcal L(\theta)\|_1 \ \le\ C\,\frac{n}{\fwdbound}\,e^{-cn}.
  \]
  \end{lemma}
  
  \begin{proof}
  Fix \(\theta\in\mathcal U_{q_0}\), so \(p_{\mathrm{lj}}(\theta)\ge q_0\). By Lemma~\ref{lem:exp_small_deriv_far_from_terminal} (with the long-jump probability lower bounded by \(q_0\)) and the trivial bounds \(p_f^{\max}, p_b^{\max}, p_{\mathrm{lj}}^{\max}\le 1\), for each \(\vartheta\in\{\alpha,\beta,\gamma\}\),
  \[
  \sup_{v\in\mathcal H}\big|\partial_{\vartheta}S(v;\theta)\big|
  \ \le\
  C\,\frac{n}{\fwdbound}\,e^{-cn}.
  \]
  Since \(\mathcal L\) is a convex combination, the triangle inequality gives \(|\partial_\vartheta \mathcal L| \le \sup_v |\partial_\vartheta S(v;\theta)|\). Summing over \(\vartheta\) yields the stated bound.
  \end{proof}
  
\subsection{Proof of \cref{result:dyn_vanishing}}
\label{subsec:bootstrap_main_proof}

We now assemble the above lemmas to prove the main result.
  
  \begin{proof}[Proof of \cref{result:dyn_vanishing}]
    By Lemma~\ref{lem:pmin_init_lower_bound}, the base model assumptions ensure that
    \(q_0 := p_{\mathrm{abs}}^{(2)}(\theta_0) \ge c/n > 0\) for some constant \(c = c(\switchbound, \fwdbound) > 0\) and \(n\) sufficiently large.
    By an analogous argument (with the same \(\switchbound\) bound on switch probability), the long-jump probability satisfies \(p_{\mathrm{lj}}(\theta_0) \ge c'/n\) for some constant \(c' > 0\); without loss of generality, assume \(c' \ge c\) so that \(p_{\mathrm{lj}}(\theta_0) \ge q_0\).
    
    Define the safe region
    \[
    \mathcal U\ :=\ \{\theta:\ p_{\mathrm{abs}}^{(2)}(\theta)\ge q_0/2 \text{ and } p_{\mathrm{lj}}(\theta)\ge q_0/2\},
    \]
    and the exit time
    \[
    t^\star\ :=\ \inf\{t\ge 0:\ \theta(t)\notin \mathcal U\},
    \]
    with the convention \(t^\star=\infty\) if the set is empty. By definition, for every \(s\in[0,t^\star)\)
    we have \(\theta(s)\in\mathcal U\), and hence Lemma~\ref{lem:uniform_exp_small_loss_grad} (applied with \(q_0/2\)
    in place of \(q_0\)) yields
    \[
    \|\dot\theta(s)\|_1=\|\nabla\mathcal L(\theta(s))\|_1 \ \le\ C\,\frac{n}{\fwdbound}\,e^{-cn}
    \qquad (0\le s<t^\star).
    \]
    Consequently, for every \(s\ge 0\),
    \[
    \|\theta(s\wedge t^\star)-\theta(0)\|_1
    \le \int_0^{s\wedge t^\star} \|\dot\theta(u)\|_1\,du
    \le (s\wedge t^\star)\cdot \Big(C\,\frac{n}{\fwdbound}\,e^{-cn}\Big).
    \]
    Applying Lemmas~\ref{lem:pmin_param_lipschitz} and~\ref{lem:plj_param_lipschitz} with \(\theta=\theta(s\wedge t^\star)\) and \(\theta'=\theta(0)\) gives
    \begin{align*}
    p_{\mathrm{abs}}^{(2)}(\theta(s\wedge t^\star))
    &\ge
    q_0-(s\wedge t^\star)\cdot \Big(C'\cdot C\,\frac{n}{\fwdbound}\,e^{-cn}\Big),\\
    p_{\mathrm{lj}}(\theta(s\wedge t^\star))
    &\ge
    q_0-(s\wedge t^\star)\cdot \Big(C\,\frac{n}{\fwdbound}\,e^{-cn}\Big).
    \end{align*}
    In particular, if \(s \le \frac{q_0}{2}\cdot \frac{\fwdbound}{C''n}\,e^{cn}\) for \(C'' := \max(C'C, C)\), then both bounds remain at least \(q_0/2\),
    so \(\theta(s\wedge t^\star)\in\mathcal U\). If \(t^\star\le s\), then by definition of \(t^\star\) we would have
    \(\theta(t^\star)\notin\mathcal U\), a contradiction. Therefore \(t^\star>s\).
    
    Setting
    \[
    t_{\mathrm{boot}}(n)\ :=\ \frac{q_0}{4}\cdot \frac{\fwdbound}{C''n}\,e^{cn}\ =\ \frac{c\,\fwdbound}{4C''n^2}\,e^{cn},
    \]
    we conclude that for every \(t\in[0,t_{\mathrm{boot}}(n)]\) one has \(t<t^\star\), and hence
    \[
    p_{\mathrm{abs}}^{(2)}(\theta(u))\ \ge\ \frac{q_0}{2}, \quad p_{\mathrm{lj}}(\theta(u))\ \ge\ \frac{q_0}{2}\qquad \forall\,u\in[0,t].
    \]

  Now fix \(t\in[0,t_{\mathrm{boot}}(n)]\) and let \(\tau:=\inf\{t\ge 0:\ X_t\in \{a_n,b_n\}\}\).
  Because \(p_{\mathrm{abs}}^{(2)}(\theta(0))=q_0\) and \(p_{\mathrm{abs}}^{(2)}(\theta(t))\ge q_0/2\), Lemma~\ref{lem:geom_tail_from_pmin} yields
  \[
  \Pr_{\theta(0)}(\tau>2M)\le (1-q_0/2)^M,
  \qquad
  \Pr_{\theta(t)}(\tau>2M)\le (1-q_0/2)^M
  \qquad\forall\,M\ge 0.
  \]
  
  Finally, apply the stopped-path bound in Corollary~\ref{cor:traj_tv_from_params} with
  \(\theta=\theta(0)\) and \(\theta'=\theta(t)\):
  \[
  \big\|p_{\text{roll}}(\theta(t))-p_{\text{roll}}(\theta(0))\big\|_{\TV}
  \ \le\
  M\,\|\theta(t)-\theta(0)\|_1 \;+\; \Pr_{\theta(t)}(\tau>M)+\Pr_{\theta(0)}(\tau>M).
  \]
  Using the tail estimates (noting that \(\Pr(\tau>M)\le \Pr(\tau>2\lfloor M/2\rfloor)\le (1-q_0/2)^{\lfloor M/2\rfloor}\)) and the drift bound
  \(\|\theta(t)-\theta(0)\|_1\le t\cdot \big(C\frac{n}{\fwdbound}e^{-cn}\big)\) gives, for all \(M\ge 0\),
  \[
  \big\|p_{\text{roll}}(\theta(t))-p_{\text{roll}}(\theta(0))\big\|_{\TV}
  \ \le\
  M\,t\cdot \Big(C\,\frac{n}{\fwdbound}\,e^{-cn}\Big)\;+\;2(1-q_0/2)^{M/2}.
  \]
  
  We choose \(M\) to (approximately) minimize the right-hand side. Treat \(M\ge 0\) as a real variable and write
  \[
  f(M)\ :=\ A\,M + 2e^{-\lambda M},
  \]
  where \(A:=t\big(C\frac{n}{\fwdbound}e^{-cn}\big)\) and \(\lambda:=-\frac{1}{2}\log(1-q_0/2)>0\).
  Differentiating gives
  \[
  f'(M)\ =\ A - 2\lambda e^{-\lambda M},\qquad
  f''(M)\ =\ 2\lambda^2 e^{-\lambda M}\ >\ 0,
  \]
  so \(f\) is convex and has a unique minimizer whenever it exists in \([0,\infty)\).
  If \(A<2\lambda\), the minimizer satisfies \(f'(M^\star)=0\), i.e.
  \[
  e^{-\lambda M^\star}\ =\ \frac{A}{2\lambda}
  \qquad\Longleftrightarrow\qquad
  M^\star\ =\ \frac{1}{\lambda}\log\Big(\frac{2\lambda}{A}\Big).
  \]
  Substituting \(e^{-\lambda M^\star}=A/(2\lambda)\) yields the optimized value
  \[
  f(M^\star)
  \ =\ A\cdot \frac{1}{\lambda}\log\Big(\frac{2\lambda}{A}\Big)\ +\ 2\cdot \frac{A}{2\lambda}
  \ =\ \frac{A}{\lambda}\Big(1+\log\Big(\frac{2\lambda}{A}\Big)\Big).
  \]
  Returning to integer \(M\), set \(M_{\mathrm{opt}}:=\lceil M^\star\rceil\). Since \(f\) is convex and
  \(f'\) is \(A\)-Lipschitz, one has \(f(M_{\mathrm{opt}})\le f(M^\star)+A\). In particular, for all
  \(A<2\lambda\),
  \[
  \big\|p_{\text{roll}}(\theta(t))-p_{\text{roll}}(\theta(0))\big\|_{\TV}
  \ \le\ f(M_{\mathrm{opt}})
  \ \le\ \frac{A}{\lambda}\Big(1+\log\Big(\frac{2\lambda}{A}\Big)\Big)\ +\ A.
  \]
  Equivalently, writing \(x:=A/(2\lambda)\in(0,1)\), this becomes
  \[
  f(M_{\mathrm{opt}})\ \le\ 2x\big(1+\log(1/x)\big)\ +\ 2\lambda x.
  \]
  
  Consequently, to guarantee the total-variation bound is at most \(\varepsilon\), it suffices to enforce
  \[
  \frac{A}{\lambda}\Big(1+\log\Big(\frac{2\lambda}{A}\Big)\Big)\ +\ A\ \le\ \varepsilon,
  \]
  where \(A=t\big(C\frac{n}{\fwdbound}e^{-cn}\big)\) and \(\lambda=-\frac{1}{2}\log(1-q_0/2)\).
  
  Since \(q_0 = c/n\), we have \(\lambda = -\frac{1}{2}\log(1-c/(2n)) \approx c/(4n)\) for large \(n\), i.e., \(\lambda = \Theta(1/n)\).
  
  To match the bound in \cref{result:dyn_vanishing}, we verify that \(t < \varepsilon^2 \cdot e^{\kappa n}\) implies the above bound.
  For \(t = \varepsilon^2 \cdot e^{\kappa n}\), we have
  \[
  A\ =\ \varepsilon^2 \cdot e^{\kappa n} \cdot \frac{Cn}{\fwdbound}\,e^{-cn}
  \ =\ \varepsilon^2 \cdot \frac{Cn}{\fwdbound}\,e^{-(c-\kappa)n}.
  \]
  The logarithmic factor satisfies
  \[
  \log\Big(\frac{2\lambda}{A}\Big)
  \ =\ \log(2\lambda) - \log(\varepsilon^2) - \log\Big(\frac{Cn}{\fwdbound}\Big) + (c-\kappa)n
  \ =\ (c-\kappa)n + O\big(\log n + \log(1/\varepsilon)\big).
  \]
  Substituting into the bound gives
  \[
  \frac{A}{\lambda}\Big(1+\log\Big(\frac{2\lambda}{A}\Big)\Big) + A
  \ \le\ \frac{A}{\lambda}\cdot (c-\kappa)n \;+\; O\Big(\frac{A}{\lambda}\big(\log n + \log(1/\varepsilon)\big)\Big) \;+\; A.
  \]
  Since \(\lambda = \Theta(1/n)\), we have \(A/\lambda = \Theta(n) \cdot A\). The dominant term is thus
  \[\frac{A}{\lambda}(c-\kappa)n = O(1) \cdot \varepsilon^2 \cdot \frac{C(c-\kappa)n^3}{\fwdbound}\,e^{-(c-\kappa)n}.\]
  Since \(n^3 e^{-(c-\kappa)n}\to 0\) as \(n\to\infty\) for any \(\kappa < c\), all terms vanish exponentially.
  Specifically, for any \(\kappa \in (0,c)\), there exists \(n_0 = n_0(a, \switchbound, \varepsilon)\) such that for all \(n \ge n_0\), the left-hand side is at most \(\varepsilon\).
  Setting \(\kappa := c/2\) yields the claimed result.
  
  Finally, we verify that the time scale \(t < \varepsilon^2 e^{\kappa n}\) lies within the bootstrap regime \(t \le t_{\mathrm{boot}}(n)\). Since \(t_{\mathrm{boot}}(n) = \frac{q_0 \fwdbound}{4C'Cn}\,e^{cn}\) and \(q_0 = c/n\), we have \(t_{\mathrm{boot}}(n) = \frac{c\,\fwdbound}{4C'Cn^2}\,e^{cn}\). With \(\kappa = c/2\),
  \[
  \frac{\varepsilon^2 e^{\kappa n}}{t_{\mathrm{boot}}(n)}\ =\ \frac{4C'C\varepsilon^2 n^2}{c\,\fwdbound}\,e^{-(c-\kappa)n}\ =\ \frac{4C'C\varepsilon^2 n^2}{c\,\fwdbound}\,e^{-(c/2)n}\ \to\ 0
  \]
  as \(n\to\infty\). Thus, for \(n\ge n_0\) (possibly enlarging \(n_0\)), we have \(\varepsilon^2 e^{\kappa n} < t_{\mathrm{boot}}(n)\), so the analysis above is valid for all \(t < \varepsilon^2 e^{\kappa n}\).
\end{proof}

\section{Supporting Lemmas}
\label{app:lemmas}

\begin{lemma}
\label{lem:symmetry_reduction}
Under the assumptions and notation of \cref{sec:dyn,sec:task}. In particular, the distribution $D^Q_m$ is invariant under vertex permutations, and the value matrix $A(t)$ evolves by gradient flow
\[
\frac{d}{dt}A(t) = -\nabla_A L(D^Q_m, A(t)),
\]

For $\alpha,\beta,\gamma \in \mathbb{R}$, define $A(\alpha,\beta,\gamma)$ on the vertex--edge block by
\[
A(\alpha,\beta,\gamma)_{v,(u,w)} =
\begin{cases}
\alpha & \text{if } v = u,\\
\beta  & \text{if } v = w,\\
\gamma & \text{if } v \notin \{u,w\},
\end{cases}
\quad v,u,w \in V,\ u\neq w,
\]
with all other entries defined as in \cref{sec:dyn}.

Let $A(t)$ solve gradient flow with $A(0) = A(\alpha_0,\beta_0,\gamma_0)$ for some $\alpha_0,\beta_0,\gamma_0$. Then for all $t \ge 0$ there exist scalars $\alpha(t),\beta(t),\gamma(t)$ such that
\[
A_{v,(u,w)}(t) =
\begin{cases}
\alpha(t) & \text{if } v = u,\\
\beta(t)  & \text{if } v = w,\\
\gamma(t) & \text{if } v \notin \{u,w\},
\end{cases}
\quad \forall\, v\in V,\ (u,w)\in E,\ u\neq w.
\]

Define the reduced loss
\[
\tilde L(\alpha,\beta,\gamma)
:= L\big(D^Q_m, A(\alpha,\beta,\gamma)\big).
\]
The scalars $\alpha(t),\beta(t),\gamma(t)$ satisfy
\begin{align*}
\dot\alpha(t) &= -\frac{1}{2n(2n-1)}\,\partial_\alpha \tilde L\big(\alpha(t),\beta(t),\gamma(t)\big),\\
\dot\beta(t) &= -\frac{1}{2n(2n-1)}\,\partial_\beta \tilde L\big(\alpha(t),\beta(t),\gamma(t)\big),\\
\dot\gamma(t) &= -\frac{1}{(2n-2)\cdot 2n(2n-1)}\,\partial_\gamma \tilde L\big(\alpha(t),\beta(t),\gamma(t)\big).
\end{align*}
\end{lemma}

\begin{proof}
For a permutation $\pi$ of the vertex set $V$, define an operator $T_\pi$ on matrices $A$ by
\[
\big(T_\pi A\big)_{v,(u,w)}
:= A_{\pi^{-1}(v),\,(\pi^{-1}(u),\pi^{-1}(w))},
\quad v,u,w\in V.
\]
Equivalently, $(T_\pi A)_{\pi(v),(\pi(u),\pi(w))} = A_{v,(u,w)}$. This corresponds to consistently relabelling all vertices by $\pi$.

By construction of $D^Q_m$ and the RL loss, relabelling vertices does not change the distribution over inputs or the loss, hence
\[
L(D^Q_m, T_\pi A) = L(D^Q_m, A)
\quad\text{for all }A,\pi.
\]
Differentiating gives gradient equivariance
\[
\nabla_A L(D^Q_m, T_\pi A) = T_\pi\,\nabla_A L(D^Q_m, A).
\]

Now let $A(t)$ solve gradient flow. Then
\[
\frac{d}{dt}\big(T_\pi A(t)\big)
= T_\pi \frac{d}{dt}A(t)
= -T_\pi \nabla_A L(D^Q_m, A(t))
= -\nabla_A L(D^Q_m, T_\pi A(t)).
\]
So $T_\pi A(t)$ is also a solution of the same ODE.

At $t=0$, the symmetric initialization from \cref{sec:dyn:setting} gives
\[
A_{v,(u,w)}(0) =
\begin{cases}
\alpha_0 & v=u,\\
\beta_0  & v=w,\\
\gamma_0 & v\notin\{u,w\},
\end{cases}
\]
which is clearly invariant under any permutation of vertex labels, hence $T_\pi A(0) = A(0)$ for all $\pi$. By uniqueness of solutions to the ODE, we must have
\[
T_\pi A(t) = A(t)\quad\text{for all }t\ge 0,\ \pi.
\]
Thus $A(t)$ always lies in the fixed-point subspace
\[
\mathcal{S} := \{A : T_\pi A = A\ \text{for all }\pi\}.
\]

To describe $\mathcal{S}$, consider triples $(v,u,w)\in V^3$ with $u\neq w$. Two such triples lie in the same orbit under the action of permutations if and only if they have the same equality pattern among $\{v,u,w\}$. There are exactly three possibilities:
\begin{enumerate}
    \item $v = u$, $v \neq w$;
    \item $v = w$, $v \neq u$;
    \item $v,u,w$ are all distinct.
\end{enumerate}

Given any two triples with the same pattern, we can construct a permutation $\pi$ mapping the first triple to the second (e.g., map $u\mapsto u'$, $w\mapsto w'$, and extend arbitrarily), so these are indeed the three orbits. If $A\in\mathcal{S}$, then for any two entries $A_{v,(u,w)}$ and $A_{v',(u',w')}$ from the same orbit we must have
\[
A_{v,(u,w)} = A_{v',(u',w')},
\]
because some $\pi$ sends one triple to the other and $A = T_\pi A$. Therefore there exist scalars $\alpha,\beta,\gamma$ such that
\[
A_{v,(u,w)} =
\begin{cases}
\alpha & v = u,\\
\beta  & v = w,\\
\gamma & v\notin\{u,w\},
\end{cases}
\]
which is exactly the form $A(\alpha,\beta,\gamma)$. This gives the representation of $A(t)$ in terms of $\alpha(t),\beta(t),\gamma(t)$.

For the dynamics, write $A(t) = A(\alpha(t),\beta(t),\gamma(t))$ and define $\tilde L$ as above. By the chain rule,
\[
\partial_\alpha \tilde L(\alpha,\beta,\gamma)
= \sum_{v=u,\,u\neq w}
\frac{\partial}{\partial A_{v,(u,w)}} L(D^Q_m, A(\alpha,\beta,\gamma)),
\]
and similarly for $\beta,\gamma$. Gradient flow gives
\[
\frac{d}{dt}A_{v,(u,w)}(t)
= -\,\frac{\partial}{\partial A_{v,(u,w)}} L(D^Q_m, A(t)).
\]

By symmetry, all entries with $v=u$ share the same value $\alpha(t)$ and the same time derivative $\dot\alpha(t)$. Summing the previous identity over all such entries and comparing with the expression for $\partial_\alpha \tilde L$ yields
\[
N_\alpha\,\dot\alpha(t)
= -\partial_\alpha \tilde L\big(\alpha(t),\beta(t),\gamma(t)\big),
\]
where $N_\alpha$ is the number of pairs $(v,(u,w))$ with $v=u$ and $u\neq w$. For each ordered pair $(u,w)$ with $u\neq w$, there is exactly one $v$ with $v=u$, so
\[
N_\alpha = |E_{\text{dir}}|
= |V|(|V|-1),
\]
where $E_{\text{dir}} = \{(u,w)\in V^2 : u\neq w\}$. The same reasoning gives $N_\beta = |V|(|V|-1)$, and for $\gamma$, for each $(u,w)$ there are $|V|-2$ vertices $v\notin\{u,w\}$, so
\[
N_\gamma = (|V|-2)|V|(|V|-1).
\]

Dividing by these constants gives the claimed ODEs for $\alpha,\beta,\gamma$.
\end{proof}

\begin{lemma}[Softmax derivative]
\label{lem:softmax_derivative_identity_aux}
Let \(\mathcal Y\) be a finite set. For each \(\theta\in\R\), let \(o_\theta:\mathcal Y\to\R\) be
differentiable in \(\theta\), and define the softmax distribution \(P_\theta\) on \(\mathcal Y\) by
\[
P_\theta(y):=\frac{\exp(o_\theta(y))}{\sum_{z\in\mathcal Y}\exp(o_\theta(z))}.
\]
Write \(d_\theta(y):=\partial_\theta o_\theta(y)\). Then for every \(y\in\mathcal Y\),
\[
\partial_\theta P_\theta(y)
:=
P_\theta(y)\Big(d_\theta(y)-\sum_{z\in\mathcal Y}P_\theta(z)\,d_\theta(z)\Big).
\]
Equivalently, for any function \(w:\mathcal Y\to\R\),
\[
\sum_{y\in\mathcal Y}\big(\partial_\theta P_\theta(y)\big)\,w(y)
:=
\sum_{y\in\mathcal Y}P_\theta(y)\,d_\theta(y)\,w(y)
-
\Big(\sum_{y\in\mathcal Y}P_\theta(y)\,d_\theta(y)\Big)
\Big(\sum_{y\in\mathcal Y}P_\theta(y)\,w(y)\Big),
\]
i.e.
\[
\sum_{y\in\mathcal Y}\big(\partial_\theta P_\theta(y)\big)\,w(y)
:=
\mathrm{Cov}_{Y\sim P_\theta}\!\big(d_\theta(Y),w(Y)\big).
\]
\end{lemma}

\begin{proof}
Fix \(\theta\) and abbreviate \(o(y):=o_\theta(y)\), \(P(y):=P_\theta(y)\), and
\(d(y):=\partial_\theta o_\theta(y)\). Let
\[
Z:=\sum_{z\in\mathcal Y} e^{o(z)},
\qquad\text{so}\qquad
P(y)=\frac{e^{o(y)}}{Z}.
\]
Differentiate using the quotient rule:
\[
\partial_\theta P(y)
:=
\frac{\partial_\theta e^{o(y)}}{Z}
-\frac{e^{o(y)}}{Z^2}\,\partial_\theta Z.
\]
Since \(\partial_\theta e^{o(y)}=e^{o(y)}d(y)\) and
\[
\partial_\theta Z=\sum_{z\in\mathcal Y}\partial_\theta e^{o(z)}
:=\sum_{z\in\mathcal Y}e^{o(z)}d(z),
\]
we obtain
\[
\partial_\theta P(y)
:=
\frac{e^{o(y)}}{Z}d(y)
-\frac{e^{o(y)}}{Z}\cdot\frac{1}{Z}\sum_{z\in\mathcal Y}e^{o(z)}d(z).
\]
Recognizing \(e^{o(y)}/Z=P(y)\) and \(e^{o(z)}/Z=P(z)\) gives
\[
\partial_\theta P(y)
:=
P(y)\Big(d(y)-\sum_{z\in\mathcal Y}P(z)d(z)\Big),
\]
which proves the first identity.

Multiplying this by \(w(y)\) and summing over \(y\in\mathcal Y\) yields
\[
\sum_{y\in\mathcal Y}(\partial_\theta P(y))w(y)
:=
\sum_{y\in\mathcal Y}P(y)d(y)w(y)
-\Big(\sum_{z\in\mathcal Y}P(z)d(z)\Big)\Big(\sum_{y\in\mathcal Y}P(y)w(y)\Big),
\]
which is the stated centered form. The covariance reformulation follows from
\(\mathrm{Cov}(A,B)=\E[AB]-\E[A]\E[B]\) with \(Y\sim P_\theta\).
\end{proof}

\begin{lemma}[Elementary gambler's--ruin bound (no martingales)]
\label{lem:gambler_ruin_no_martingale}
Let $(D_t)_{t\ge0}$ be a Markov chain on $\{0,1,\dots,K\}$ with absorbing boundaries $\{0,K\}$.
For each $i\in\{1,\dots,K-1\}$ assume
\[
\Pr(D_{t+1}=i-1\mid D_t=i)=p_f(i),\qquad
\Pr(D_{t+1}=i+1\mid D_t=i)=p_b(i),
\]
and $\Pr(D_{t+1}=i\mid D_t=i)=1-p_f(i)-p_b(i)$ (holds allowed).
Assume $p_f(i),p_b(i)>0$ and that there exists $\rho\in(0,1)$ such that for all $i\in\{1,\dots,K-1\}$,
\[
\frac{p_b(i)}{p_f(i)}\ \le\ \rho.
\]
Let $\tau_0:=\inf\{t\ge0:D_t=0\}$ and $\tau_K:=\inf\{t\ge0:D_t=K\}$.
Then for every $k\in\{0,1,\dots,K\}$,
\[
\Pr_k(\tau_K<\tau_0)\ \le\ \rho^{\,K-k}.
\]
\end{lemma}

\begin{proof}
Let $h(i):=\Pr_i(\tau_K<\tau_0)$. Then $h(0)=0$, $h(K)=1$, and for $1\le i\le K-1$,
conditioning on one step gives
\[
h(i)=p_f(i)h(i-1)+p_b(i)h(i+1)+\bigl(1-p_f(i)-p_b(i)\bigr)h(i),
\]
hence
\begin{equation}\label{eq:flux_ruin_aux}
p_f(i)\bigl(h(i)-h(i-1)\bigr)=p_b(i)\bigl(h(i+1)-h(i)\bigr).
\end{equation}
With $\Delta_i:=h(i)-h(i-1)$, \eqref{eq:flux_ruin_aux} implies
$\Delta_{i+1}=\frac{p_f(i)}{p_b(i)}\Delta_i$, hence
\[
\Delta_i=\Delta_1\prod_{m=1}^{i-1}\frac{p_f(m)}{p_b(m)}.
\]
Since $h(k)=\sum_{i=1}^k\Delta_i$ and $h(K)=\sum_{i=1}^K\Delta_i=1$, we get
\[
h(k)=
\frac{\sum_{i=1}^k \prod_{m=1}^{i-1}\frac{p_f(m)}{p_b(m)}}
{\sum_{i=1}^K \prod_{m=1}^{i-1}\frac{p_f(m)}{p_b(m)}}.
\]
Using $p_f(m)/p_b(m)\ge \rho^{-1}$ in both numerator and denominator yields
\[
h(k)\le
\frac{\sum_{i=1}^k \rho^{-(i-1)}}{\sum_{i=1}^K \rho^{-(i-1)}}
:=
\rho^{K-k}\cdot
\frac{\sum_{j=0}^{k-1}\rho^j}{\sum_{j=0}^{K-1}\rho^j}
\le \rho^{K-k}.
\]
\end{proof}

\begin{lemma}[Expected absorption time]
\label{lem:exact-and-linear-upper}
Let $(X_t)_{t\ge 0}$ be a Markov chain on $\{0,1,\dots,N\}$ such that $N$ is
absorbing, from each $i\in\{1,\dots,N-1\}$ the chain moves to $i+1$ with
probability $p$ and to $i-1$ with probability $q$, where $p,q\ge 0$, $p+q=1$ and
$p>q$, and from $0$ the chain moves to $1$ with probability $1$.
Let $T_N:=\inf\{t\ge 0: X_t=N\}$. Then for every $k\in\{0,1,\dots,N\}$,
\[
\mathbb E_k[T_N]
:=\frac{N-k}{p-q}
+\frac{2pq}{(p-q)^2}\left[\left(\frac{q}{p}\right)^N-\left(\frac{q}{p}\right)^k\right],
\]
and in particular,
\[
\mathbb E_k[T_N]\le \frac{N-k}{p-q}.
\]
\end{lemma}

\begin{proof}
For $i\in\{0,1,\dots,N\}$ define $m_i:=\mathbb E_i[T_N]$.
If $1\le i\le N-1$, conditioning on the first step and using the Markov property
gives
\begin{equation}
\label{eq:recurrence2}
m_i = 1 + p\,m_{i+1} + q\,m_{i-1}.
\end{equation}
At $N$ we have $m_N=0$ since $N$ is absorbing, and at $0$ we have
\begin{equation}
\label{eq:boundary2}
m_0 = 1 + m_1
\end{equation}
since the chain moves deterministically from $0$ to $1$ in one step.

Set $\delta:=p-q>0$ and rewrite \eqref{eq:recurrence2} as
\begin{equation}
\label{eq:inhom2}
p\,m_{i+1} - m_i + q\,m_{i-1} = -1,\qquad 1\le i\le N-1.
\end{equation}
A particular solution is given by $m_i^{\mathrm{part}}=-i/\delta$, because
\[
p\left(-\frac{i+1}{\delta}\right)-\left(-\frac{i}{\delta}\right)
+q\left(-\frac{i-1}{\delta}\right)
:=\frac{-p+q}{\delta}=-1
\]
using $p+q=1$.
For the associated homogeneous equation
\[
p\,u_{i+1}-u_i+q\,u_{i-1}=0,
\]
a trial solution $u_i=r^i$ leads to the characteristic equation
$p r^2-r+q=0$, whose two roots are $r=1$ and $r=q/p$.
Writing $\rho:=q/p\in(0,1)$, the general homogeneous solution is
$u_i=A+B\rho^i$, hence every solution to \eqref{eq:inhom2} has the form
\begin{equation}
\label{eq:general2}
m_i = A + B\rho^i - \frac{i}{\delta}.
\end{equation}

Imposing $m_N=0$ in \eqref{eq:general2} yields
\[
0 = A + B\rho^N - \frac{N}{\delta}
\quad\Longrightarrow\quad
A = \frac{N}{\delta} - B\rho^N.
\]
Next, \eqref{eq:boundary2} gives $m_0=1+m_1$. From \eqref{eq:general2},
\[
m_0=A+B,\qquad m_1=A+B\rho-\frac{1}{\delta},
\]
so
\[
A+B = 1 + \left(A+B\rho-\frac{1}{\delta}\right)
\quad\Longrightarrow\quad
B(1-\rho)=1-\frac{1}{\delta}.
\]
Now $1-\rho=1-q/p=(p-q)/p=\delta/p$ and, since $p+q=1$, we have
$1-\frac{1}{\delta}=\frac{\delta-1}{\delta}=\frac{(p-q)-(p+q)}{\delta}
=-\frac{2q}{\delta}$. Therefore
\[
B\cdot \frac{\delta}{p} = -\frac{2q}{\delta}
\quad\Longrightarrow\quad
B=-\frac{2pq}{\delta^2}.
\]
Substituting back gives
\[
A=\frac{N}{\delta}+\frac{2pq}{\delta^2}\rho^N,
\]
and hence for $k\in\{0,1,\dots,N\}$,
\[
m_k
:=\frac{N-k}{\delta}+\frac{2pq}{\delta^2}\bigl(\rho^N-\rho^k\bigr).
\]
This is the stated closed form after recalling $\delta=p-q$ and $\rho=q/p$.

To see that this solution is the unique solution to
\eqref{eq:recurrence2}--\eqref{eq:boundary2}, suppose $\tilde m$ is another
solution and let $d_i:=m_i-\tilde m_i$. Then $d_N=0$, $d_0=d_1$, and $d$ solves
the homogeneous recurrence $p d_{i+1}-d_i+q d_{i-1}=0$, so $d_i=A'+B'\rho^i$.
The condition $d_0=d_1$ implies $A'+B'=A'+B'\rho$, hence $B'(1-\rho)=0$ and
therefore $B'=0$, and then $d_N=0$ forces $A'=0$, so $d\equiv 0$.

Finally, since $\rho\in(0,1)$ and $k\le N$, we have $\rho^N\le \rho^k$ and thus
$\rho^N-\rho^k\le 0$. Because $2pq/\delta^2\ge 0$, the correction term
$\frac{2pq}{\delta^2}(\rho^N-\rho^k)$ is nonpositive, and therefore
\[
\mathbb E_k[T_N]=m_k \le \frac{N-k}{\delta}=\frac{N-k}{p-q}.
\]
\end{proof}

\begin{lemma}[Absorption time tail bound]
\label{lem:hoeffding_tail_absorption}
Let $(X_t)_{t\ge 0}$ be a nearest--neighbor random walk with
\[
\Pr(X_{t+1}=X_t+1)=p,\qquad \Pr(X_{t+1}=X_t-1)=q=1-p,
\qquad p-q\ge c>0,
\]
run on a chain of length $L$ with the \emph{right endpoint absorbing}.  Assume the walk starts at distance
$k$ from the absorbing endpoint (i.e.\ $X_0=L-k$), where $k$ is a fixed constant independent of $L$.
Let
\[
\tau:=\inf\{t\ge 0:\ X_t=L\}
\]
be the absorption time at the right edge. Then for any $t\ge 1$,
\[
\Pr(\tau>t)
\;\le\;
\Pr\!\left(\sum_{i=1}^t \xi_i \le k-1\right),
\qquad
\xi_i\in\{+1,-1\},\ \ \E[\xi_i]=p-q=:\mu\ (\ge c),
\]
and by Hoeffding's inequality,
\[
\Pr(\tau>t)
\;\le\;
\exp\!\left(-\frac{(\mu t-(k-1))^2}{2t}\right)
\;\le\;
\exp\!\left(-\frac{(c t-(k-1))^2}{2t}\right).
\]
In particular, taking $t=\lfloor L^{1/4}\rfloor$ and using that $k$ is constant, for all $L$ large enough
(e.g.\ $L^{1/4}\ge 2(k-1)/c$),
\[
\Pr\!\big(\tau>\lfloor L^{1/4}\rfloor\big)
\;\le\;
\exp\!\left(-\frac{c^2}{8}\,L^{1/4}\right).
\]
\end{lemma}

\begin{proof}
Write \(X_t=X_0+\sum_{i=1}^t \xi_i\) with i.i.d.\ increments \(\xi_i\in\{+1,-1\}\) of mean \(\mu=p-q\ge c\).
If \(\tau>t\) then \(X_t\le L-1\), hence \(\sum_{i=1}^t \xi_i=X_t-X_0\le (L-1)-(L-k)=k-1\), proving the first inequality.
Hoeffding's inequality for bounded i.i.d.\ variables yields
\[
\Pr\!\left(\sum_{i=1}^t \xi_i - \mu t \le -( \mu t-(k-1))\right)
\le \exp\!\left(-\frac{(\mu t-(k-1))^2}{2t}\right),
\]
and \(\mu\ge c\) gives the stated simplification. For \(t=\lfloor L^{1/4}\rfloor\) and \(L^{1/4}\ge 2(k-1)/c\),
we have \(\mu t-(k-1)\ge ct/2\), giving the final display.
\end{proof}

\begin{lemma}[Dominated change under one--sided derivative domination]
\label{lem:dominated-change-deriv}
Let $x,y:[0,T]\to\R$ be differentiable functions with initial values
$x(0)=a$ and $y(0)=b$. Assume that for every $t\in[0,T]$,
\[
x'(t)\ge 0,\qquad y'(t)\le 0,
\]
and there exists a constant $c>0$ such that
\[
x'(t)\ \ge\ c\bigl(-y'(t)\bigr).
\]
Then for every $t\in[0,T]$,
\[
x(t)-a \ \ge\ c\bigl(b-y(t)\bigr),
\]
equivalently,
\[
x(t)+c\,y(t)\ \ge\ a+c\,b.
\]
\end{lemma}

\begin{proof}
Define $F(t):=x(t)+c\,y(t)$. Then $F$ is differentiable and
\[
F'(t)=x'(t)+c\,y'(t)\ \ge\ 0
\qquad\text{for all }t\in[0,T],
\]
by the assumed inequality $x'(t)\ge c(-y'(t))$. Hence $F$ is nondecreasing on $[0,T]$,
so
\[
x(t)+c\,y(t)=F(t)\ \ge\ F(0)=x(0)+c\,y(0)=a+c\,b.
\]
Rearranging gives $x(t)-a\ge c(b-y(t))$, as claimed.
\end{proof}
\begin{lemma}
    \label{lem:soft_max_ode_scaled_shifted}
    Fix $C>0$, $K>0$, $\varepsilon\in(0,1)$, $d>0$, and $D\in\mathbb{R}$.  
    Let $\alpha\in C^{1}([0,\infty))$ satisfy
    \[
    \dot{\alpha}(t) > K \qquad \text{for all } t\ge 0,
    \]
    and define
    \[
    g(t) := \frac{e^{\alpha(t)-D}}{e^{\alpha(t)-D}+C}\in(0,1).
    \]
    Let $f:[0,\infty)\to(0,1)$ be continuous and assume
    \[
    f(0)=g(0)
    \qquad\text{and}\qquad
    f(t)>g(t)\ \ \text{for all } t>0.
    \]
    If $f(0)<1-\varepsilon$, then there exists $T_\varepsilon\in(0,\bar T_\varepsilon)$ such that
    \[
    f(T_\varepsilon)=1-\varepsilon,
    \]
    where
    \[
    \bar T_\varepsilon := \frac{1}{K}\log\!\left(\frac{1-\varepsilon}{\varepsilon}\cdot\frac{1-f(0)}{f(0)}\right).
    \]
    Moreover,
    \[
    \alpha(T_\varepsilon)-\alpha(0)\;<\;\log\!\left(\frac{1-\varepsilon}{\varepsilon}\cdot\frac{1-f(0)}{f(0)}\right).
    \]
    \end{lemma}
    
    \begin{proof}
    Set
    \[
    \Delta_\varepsilon:=\log\!\left(\frac{1-\varepsilon}{\varepsilon}\cdot\frac{1-f(0)}{f(0)}\right),
    \qquad
    \bar T_\varepsilon:=\frac{\Delta_\varepsilon}{dK}.
    \]
    From $f(0)=g(0)=\frac{e^{\alpha(0)-D}}{e^{\alpha(0)-D}+C}$ we obtain
    \[
    e^{\alpha(0)-D}=C\frac{f(0)}{1-f(0)},
    \qquad\text{hence}\qquad
    \alpha(0)-D=\log\!\left(C\frac{f(0)}{1-f(0)}\right).
    \]
    Consequently,
    \[
    (\alpha(0)-D)+\Delta_\varepsilon
    =\log\!\left(C\frac{f(0)}{1-f(0)}\right)
    +\log\!\left(\frac{1-\varepsilon}{\varepsilon}\cdot\frac{1-f(0)}{f(0)}\right)
    =\log\!\left(C\frac{1-\varepsilon}{\varepsilon}\right).
    \]
    
    Since $\dot\alpha(t)>K$ for all $t\ge 0$, integration yields
    \[
    \alpha(t)-\alpha(0)=\int_0^t \dot\alpha(s)\,ds > Kt \qquad (t>0).
    \]
    Subtracting $D$ gives
    \[
    \alpha(t)-D > (\alpha(0)-D)+Kt \qquad (t>0).
    \]
    In particular,
    \[
    \alpha(\bar T_\varepsilon)-D
    >(\alpha(0)-D)+K\bar T_\varepsilon
    =(\alpha(0)-D)+\Delta_\varepsilon
    =\log\!\left(C\frac{1-\varepsilon}{\varepsilon}\right).
    \]
    The map $x\mapsto \frac{e^{x}}{e^{x}+C}$ is strictly increasing, hence
    \[
    g(\bar T_\varepsilon)
    >\frac{e^{\log(C(1-\varepsilon)/\varepsilon)}}{e^{\log(C(1-\varepsilon)/\varepsilon)}+C}
    =1-\varepsilon.
    \]
    By assumption $f(t)>g(t)$ for $t>0$, thus $f(\bar T_\varepsilon)>g(\bar T_\varepsilon)>1-\varepsilon$, while $f(0)<1-\varepsilon$.
    By continuity of $f$, the intermediate value theorem implies the existence of
    $T_\varepsilon\in(0,\bar T_\varepsilon)$ such that $f(T_\varepsilon)=1-\varepsilon$.
    
    Finally, since $T_\varepsilon>0$ and $f(T_\varepsilon)>g(T_\varepsilon)$, we have $g(T_\varepsilon)<1-\varepsilon$.
    By monotonicity of $x\mapsto \frac{e^{x}}{e^{x}+C}$ it follows that
    \[
    \alpha(T_\varepsilon)-D < \log\!\left(C\frac{1-\varepsilon}{\varepsilon}\right).
    \]
    Subtracting $\alpha(0)-D=\log\!\left(C\frac{f(0)}{1-f(0)}\right)$ yields
    \[
    \alpha(T_\varepsilon)-\alpha(0)
    <
    \log\!\left(C\frac{1-\varepsilon}{\varepsilon}\right)
    -\log\!\left(C\frac{f(0)}{1-f(0)}\right)
    =
    \log\!\left(\frac{1-\varepsilon}{\varepsilon}\cdot\frac{1-f(0)}{f(0)}\right).
    \]
    \end{proof}

\begin{lemma}[Logit computation]
\label{lem:dyn:logits_cases_canonical}
Fix $n \ge 3$. Consider an input sequence $s_{1:L}$ sampled autoregressively as described in \cref{sec:task}, such that the prefix $s_{1:2n-2}$ contains the multiset of edges $G$ of two disjoint directed chains $\chaina=(a_1,\ldots,a_n)$ and $\chainb=(b_1,\ldots,b_n)$, and all subsequent tokens $s_{2n-1}, \ldots, s_L$ are vertices in $\chaina \cup \chainb$ (in particular, the last token $v = s_L$ is a vertex).

Assume the transformer's parameters satisfy:
\begin{enumerate}
    \item The attention matrix has the symmetric form $A=A(\alpha,\beta,\gamma)$ on the vertex--edge block as in \cref{lem:symmetry_reduction} and $\gamma$ on all vertex--vertex blocks.
    \item The value matrix has the following form: there exist $\alpha',\beta',\gamma'\in\R$ such that for every edge token $(u,w)\in\edgevocab$ and any output token $k\in\vocab$,
\[
V_{(u,w),k}=
\begin{cases}
\alpha' & \text{if } k=u,\\
\beta' & \text{if } k=w,\\
\gamma' & \text{otherwise},
\end{cases}
\qquad\text{and}\qquad
V_{v,:}=0\ \text{for all vertex tokens }v\in\vset.
\]
\end{enumerate}

For Softmax Attention, let $Z_v = \sum_{j=1}^L \exp(A_{s_L,s_j})$ be the partition function.
We define the weight function $w(\mu)$ as:
\[
w(\mu) = \begin{cases}
\mu & \text{for Linear Attention}, \\
e^\mu/Z_v & \text{for Softmax Attention}.
\end{cases}
\]
Note that for any token $x$, the attention weight from $v$ to $x$ is given by $w(A_{v,x})$.

Then the logits $H^{(1)}_{L, :}$ depend on $v$, $G$, and in the softmax case, the partition function $Z_v$ (we denote them by $o(v, G, Z_v)$). In the case of Linear Attention, the dependence on $Z_v$ is absent, and we simply write $o(v, G)$. Furthermore, $o(v,G, Z_v)_k = \Psi_k + C$, where $C$ is a constant independent of $k$, we define $\Delta_\alpha \coloneqq \alpha'-\gamma'$, $\Delta_\beta \coloneqq \beta'-\gamma'$, and $\Psi_k$ is given by the following cases (and symmetrically when swapping $\chaina$ and $\chainb$).

\paragraph*{Case 1: $v$ is a middle vertex.}
If $v=a_i$ for some $i\in\{3,\ldots,n-2\}$ (and symmetrically if $v=b_i$), then:
\[
\Psi_k =
\begin{cases}
\Delta_\alpha w(\gamma) + \Delta_\beta w(\alpha), & k=a_{i+1},\\
\Delta_\alpha w(\beta) + \Delta_\beta w(\gamma), & k=a_{i-1},\\
\Delta_\alpha w(\alpha) + \Delta_\beta w(\beta), & k=a_i,\\
\Delta_\alpha w(\gamma), & k\in\{a_1,b_1\},\\
\Delta_\beta w(\gamma), & k\in\{a_n,b_n\},\\
(\Delta_\alpha + \Delta_\beta)w(\gamma), & k\in(\chaina\cup\chainb)\setminus\{a_{i-1},a_i,a_{i+1},a_1,a_n,b_1,b_n\},\\
0, & k\in\edgevocab.
\end{cases}
\]

\paragraph*{Case 2: $v$ is the penultimate vertex.}
If $v=a_{n-1}$ (and symmetrically if $v=b_{n-1}$), then:
\[
\Psi_k =
\begin{cases}
\Delta_\beta w(\alpha), & k=a_n,\\
\Delta_\alpha w(\beta) + \Delta_\beta w(\gamma), & k=a_{n-2},\\
\Delta_\alpha w(\alpha) + \Delta_\beta w(\beta), & k=a_{n-1},\\
\Delta_\alpha w(\gamma), & k\in\{a_1,b_1\},\\
\Delta_\beta w(\gamma), & k\in\{b_n\},\\
(\Delta_\alpha + \Delta_\beta)w(\gamma), & k\in(\chaina\cup\chainb)\setminus\{a_{n-2},a_{n-1},a_n,a_1,b_1,b_n\},\\
0, & k\in\edgevocab.
\end{cases}
\]

\paragraph*{Case 3: $v$ is the second vertex (and not penultimate).}
If $m\ge 4$ and $v=a_2$ (and symmetrically if $v=b_2$), then:
\[
\Psi_k =
\begin{cases}
\Delta_\alpha w(\gamma) + \Delta_\beta w(\alpha), & k=a_3,\\
\Delta_\alpha w(\beta), & k=a_1,\\
\Delta_\alpha w(\alpha) + \Delta_\beta w(\beta), & k=a_2,\\
\Delta_\beta w(\gamma), & k\in\{a_n,b_n\},\\
\Delta_\alpha w(\gamma), & k\in\{b_1\},\\
(\Delta_\alpha + \Delta_\beta)w(\gamma), & k\in(\chaina\cup\chainb)\setminus\{a_{1},a_{2},a_{3},a_n,b_1,b_n\},\\
0, & k\in\edgevocab.
\end{cases}
\]

\paragraph*{Case 4: $v$ is the first vertex (and not penultimate).}
If $m\ge 3$ and $v=a_1$ (and symmetrically if $v=b_1$), then:
\[
\Psi_k =
\begin{cases}
\Delta_\alpha w(\gamma) + \Delta_\beta w(\alpha), & k=a_2,\\
\Delta_\alpha w(\alpha), & k=a_1,\\
\Delta_\beta w(\gamma), & k\in\{a_n,b_n\},\\
\Delta_\alpha w(\gamma), & k\in\{b_1\},\\
(\Delta_\alpha + \Delta_\beta)w(\gamma), & k\in(\chaina\cup\chainb)\setminus\{a_{1},a_{2},a_n,b_1,b_n\},\\
0, & k\in\edgevocab.
\end{cases}
\]

\paragraph*{Case 5: $v$ is a terminal vertex.}
If $v=a_n$ (and symmetrically if $v=b_n$), then:
\[
\Psi_k =
\begin{cases}
\Delta_\beta w(\beta), & k=a_n,\\
\Delta_\alpha w(\beta) + \Delta_\beta w(\gamma), & k=a_{n-1},\\
\Delta_\alpha w(\gamma), & k\in\{a_1,b_1\},\\
\Delta_\beta w(\gamma), & k\in\{b_n\},\\
(\Delta_\alpha + \Delta_\beta)w(\gamma), & k\in(\chaina\cup\chainb)\setminus\{a_{n-1},a_n,a_1,b_1,b_n\},\\
0, & k\in\edgevocab.
\end{cases}
\]
\end{lemma}

\begin{proof}
Let $s_{1:L}$ denote the input sequence with $s_L = v$.
From the transformer architecture (\cref{sec:setup:trans}), the output logits at position $L$ are given by the attention-weighted sum of value vectors. Specifically:
\[
    H^{(1)}_{L,:} = \sum_{j=1}^L \text{weight}(s_L, s_j) V_{s_j, :},
\]
where $\text{weight}(s_L, s_j)$ is the attention weight of token $s_j$ when attending from $s_L$.
Using our definition of $w(\mu)$ and the attention matrix entries $A_{s_L, s_j}$, this weight is exactly $w(A_{s_L, s_j})$.

By our assumption on the value matrix, $V_{s_j, :} = 0$ whenever $s_j$ is a vertex token.
Since the sequence $s_{1:L}$ consists of edge tokens (the set $G$) and vertex tokens, the sum over $j=1,\ldots,L$ restricts to the indices corresponding to edges in $G$:
\[
    H^{(1)}_{L,:} = \sum_{e \in G} w(A_{v,e}) V_{e,:}.
\]
This expression confirms that the logits depend on $v$, $G$, and implicitly on the context via $Z_v$ (in the Softmax case), justifying the notation $o(v, G, Z_v)$.

For a specific output token $k \in \vocab$, we have:
\[
    o(v, G, Z_v)_k = \sum_{e \in G} w(A_{v,e}) V_{e,k}.
\]
Using the symmetric structure of $V$ where $V_{e,k} = \alpha'$ if $k$ is the source of $e$, $\beta'$ if $k$ is the target of $e$, and $\gamma'$ otherwise (including when $k$ is an edge token), we substitute:
\begin{align*}
    o(v,G, Z_v)_k &= \sum_{e=(u,w) \in G} w(A_{v,e}) \left( \alpha' \mathbbm{1}\{k=u\} + \beta' \mathbbm{1}\{k=w\} + \gamma' \mathbbm{1}\{k \notin \{u,w\}\} \right) \\
    &= \alpha' \sum_{e \in G, \text{src}(e)=k} w(A_{v,e}) + \beta' \sum_{e \in G, \text{tgt}(e)=k} w(A_{v,e}) + \gamma' \sum_{e \in G, k \notin e} w(A_{v,e}).
\end{align*}
Let $W_{out}(v,k) = \sum_{e \in G, \text{src}(e)=k} w(A_{v,e})$ be the total attention weight on edges outgoing from $k$, and $W_{in}(v,k) = \sum_{e \in G, \text{tgt}(e)=k} w(A_{v,e})$ be the total weight on edges incoming to $k$.
The total weight is $W_{total}(v) = \sum_{e \in G} w(A_{v,e})$.
We can rewrite the third term using the identity $\mathbbm{1}\{k \notin e\} = 1 - \mathbbm{1}\{k=\text{src}(e)\} - \mathbbm{1}\{k=\text{tgt}(e)\}$:
\begin{align*}
    \gamma' \sum_{e \in G, k \notin e} w(A_{v,e}) &= \gamma' \left( W_{total}(v) - W_{out}(v,k) - W_{in}(v,k) \right).
\end{align*}
Substituting this back:
\begin{align*}
    o(v,G, Z_v)_k &= \alpha' W_{out}(v,k) + \beta' W_{in}(v,k) + \gamma' \left( W_{total}(v) - W_{out}(v,k) - W_{in}(v,k) \right) \\
    &= (\alpha' - \gamma') W_{out}(v,k) + (\beta' - \gamma') W_{in}(v,k) + \gamma' W_{total}(v).
\end{align*}
Defining $\Delta_\alpha = \alpha' - \gamma'$, $\Delta_\beta = \beta' - \gamma'$, and $C = \gamma' W_{total}(v)$ (which is independent of $k$), we get:
\[
    o(v,G, Z_v)_k = \Delta_\alpha W_{out}(v,k) + \Delta_\beta W_{in}(v,k) + C.
\]
We define $\Psi_k = \Delta_\alpha W_{out}(v,k) + \Delta_\beta W_{in}(v,k)$.
It remains to compute $W_{out}(v,k)$ and $W_{in}(v,k)$ in each of the configurations listed in the lemma statement. Since $G$ is the disjoint union of two directed chains, each $k\in(\chaina\cup\chainb)$ has at most one outgoing edge $(k,\mathrm{succ}(k))$ (absent iff $k\in\{a_n,b_n\}$) and at most one incoming edge $(\mathrm{pred}(k),k)$ (absent iff $k\in\{a_1,b_1\}$). For such edges, by the assumed form of $A(\alpha,\beta,\gamma)$ on the vertex--edge block,
\[
A_{v,(k,\mathrm{succ}(k))} =
\begin{cases}
\alpha, & v=k,\\
\beta, & v=\mathrm{succ}(k),\\
\gamma, & \text{otherwise,}
\end{cases}
\qquad
A_{v,(\mathrm{pred}(k),k)} =
\begin{cases}
\alpha, & v=\mathrm{pred}(k),\\
\beta, & v=k,\\
\gamma, & \text{otherwise.}
\end{cases}
\]
Applying $w(\cdot)$ and substituting into $\Psi_k=\Delta_\alpha W_{out}(v,k)+\Delta_\beta W_{in}(v,k)$ gives exactly the case-by-case expressions stated in the lemma (with the $\chainb$ cases obtained by symmetry).
\end{proof}

	\section{Additional Experiments and Experimental Details}
\label{app:exp_details}

\subsection{Theoretically Inspired Setting}
\label{sec:exp:theo}

We corroborate our theoretical findings from \cref{sec:how_learn,sec:dyn} in the chain identification task of \cref{sec:task}.
Each example is specified by a dependency graph consisting of two directed chains, together with a starting vertex and a queried terminal vertex.
The model must output a chain of intermediate variables that identifies the correct dependency path.
We train single-layer linear Transformers, as defined in \cref{sec:setup}, without the simplifying assumptions from \cref{sec:dyn:setting}: the attention and value matrices are trained jointly with policy gradient, implemented using REINFORCE~\cite{Williams2004SimpleSG}, from a small Gaussian initialization. Full experimental details and hyperparameters are provided in \cref{app:exp_theo_details}.

\begin{table}[t]
\caption{In line with our theory (\cref{sec:dyn}), on the chain identification task (\cref{sec:task}), a Transformer trained with simple examples achieves high accuracy through the efficient chain traversal (reasoning) algorithm, thereby demonstrating an implicit bias towards efficiency (\cref{sec:how_learn}).  
In this table, ``Test Acc.'' and ``Chain Trav.'' refer to the fraction of test examples that are solved accurately and using the chain traversal algorithm, respectively.  
For further details on this experiment see \cref{sec:exp:theo:main,app:exp_theo_details}.}
\label{tab:main_results}
\centering
\small
\begin{tabular}{lcc}
\toprule
Chain size \(n\) & Test Acc. & Chain Trav. \\
\midrule
4 & 100\% & 100\% \\
8 & 100\% & 99.3\% \\
12 & 99.3\% & 94.2\% \\
\bottomrule
\end{tabular}
\end{table}

\subsubsection{Emergence of Efficient Reasoning}
\label{sec:exp:theo:main}

\cref{tab:main_results} reports the results of an experiment where:
\emph{(i)}~the chain length $n$ takes a value in $\{ 4 , 8 , 12 \}$;
\emph{(ii)}~the training data distribution is drawn from a uniform distribution over $[ n - 1 ]$ (\cref{sec:setup:dist}), meaning in particular that it has simplicity mass greater than zero (\cref{def:tail_mass});
and
\emph{(iii)}~the test data distribution is identical to the training data distribution.
As can be seen in the table, in line with our theory (\cref{sec:dyn:traverse_chain}), the learned Transformers achieve extremely high test accuracy through the efficient chain traversal algorithm (\cref{sec:setup:traversal}), thereby demonstrating an implicit bias towards efficiency (recall that, as proven in \cref{sec:how_learn:reasoning_sufficient}, algorithms far less efficient than chain traversal also solve the chain identification task). 

\subsubsection{Out-of-Distribution Generalization}
\label{sec:exp:theo:gen}

To demonstrate our theoretical prediction of out-of-distribution generalization (\cref{sec:dyn:traverse_chain}), we conduct an experiment where:
\emph{(i)}~the chain length $n$ equals~$12$;
\emph{(ii)}~the training data distribution is drawn from a uniform distribution over $\{ n - 4 , n - 3 , n - 2 , n - 1 \}$ (\cref{sec:setup:dist}), meaning in particular that it has $4$-simplicity mass of one (\cref{def:tail_mass}), i.e., all of its examples require at most four reasoning steps;
and
\emph{(iii)}~the test data distribution is drawn from a uniform distribution over $\{ n - s , n - s + 1 , \ldots , n - 1 \}$, where $s \in \{ 4 , 8 , 11 \}$.
In this experiment, when $s = 4$ the test data distribution is identical to the training data distribution, whereas when $s$ is larger the test data distribution entails examples more complex than those seen in training.
The results of the experiment were perfect: under all test data distributions, the learned Transformer solved all test examples using the efficient chain traversal algorithm (\cref{sec:setup:traversal}).~\footnote{Interestingly, these results surpass those in \cref{sec:exp:theo:main}. This is another example (on top of others in \cref{sec:exp,sec:exp:theo}) where replacing in-distribution complex examples with simple examples may be more effective than training on the target distribution.}
This confirms that the Transformer indeed generalizes out of distribution, including to complex examples unseen during training.
For further details on the experiment see \cref{app:exp_theo_details}.

\subsubsection{Solving Complex Tasks Requires Training on Simple Tasks}
\label{sec:exp:theo:vanishing}

\cref{fig:vanishing} reports the results of an experiment where:
\emph{(i)}~the chain length $n$ equals~$5$;
\emph{(ii)}~the training data distribution is drawn from a uniform distribution over $[ k ]$ (\cref{sec:setup:dist}), where $k \in [ n - 1 ]$;
and
\emph{(iii)}~the test data distribution is drawn from a distribution that places all probability mass on~$1$.
In this experiment the test data distribution entails only the most complex examples (those that require the maximal number of reasoning steps), whereas the training data distribution entails only complex examples when $k$ is small, and a mix of complex and simple examples when $k$ is large.
As can be seen in the figure, in line with our theory (\cref{sec:dyn:traverse_chain,sec:dyn:vanishing_gradients}), the accuracy of the Transformer on complex test examples is higher when the training data includes more simple examples, despite such simple examples inducing a distribution shift.

\begin{figure}[t]
    \centering
    \includegraphics[width=0.9\columnwidth]{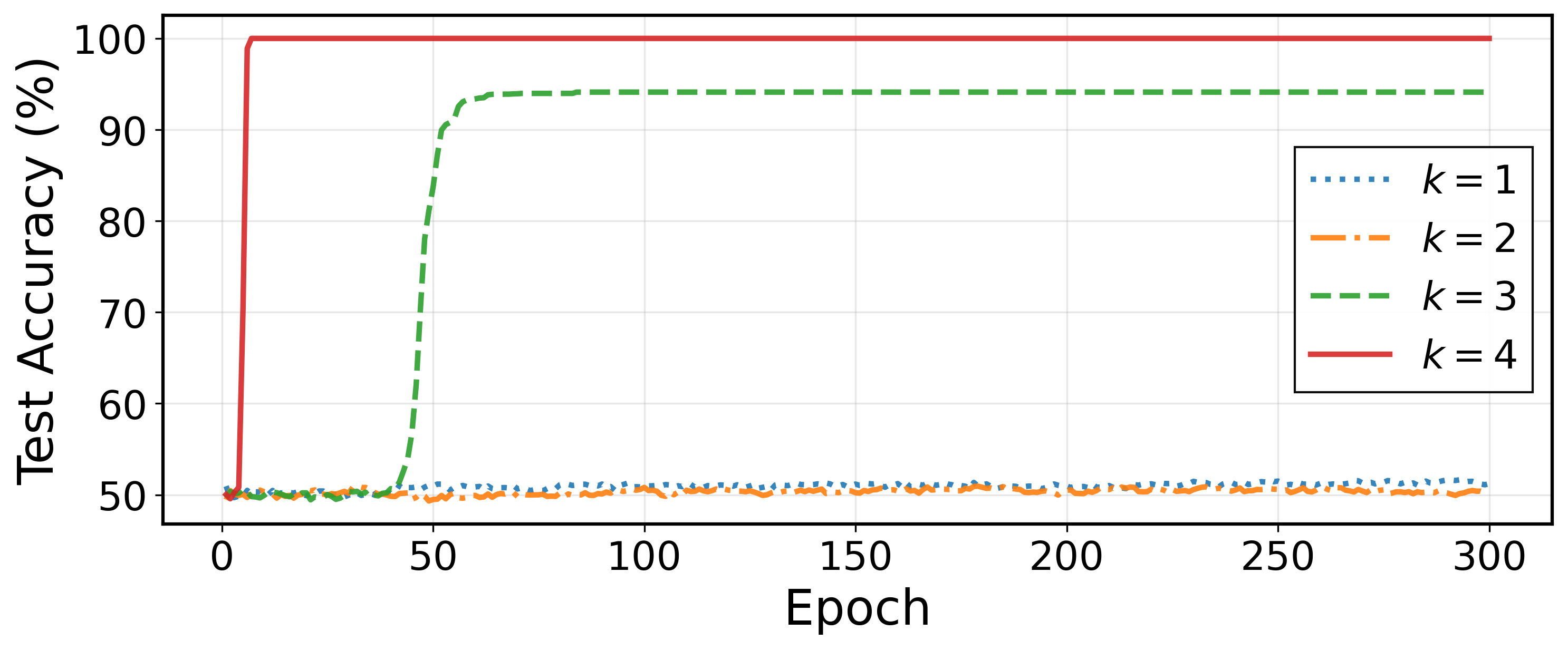}
    \caption{In line with our theory (\cref{sec:dyn}), on the chain identification task (\cref{sec:task}), the accuracy of a Transformer on complex test examples is higher when its training data includes simple examples, despite such simple examples inducing a distribution shift.
    This figure shows the accuracy on complex test examples throughout training, for different training data distributions: ones entailing only complex examples (small $k$), and ones entailing a mix of complex and simple examples (large $k$).
    For further details on this experiment see \cref{sec:exp:theo:vanishing,app:exp_theo_details}.}
    \label{fig:vanishing}
\end{figure}

\subsubsection{Experimental Details}
\label{app:exp_theo_details}
We use the following implementation throughout the theoretically inspired experiments.
We use a single-layer transformer with attention matrix \(A \in \R^{|\vocab| \times |\vocab|}\) and value matrix \(V \in \R^{|\vocab| \times |\vocab|}\), as defined in \cref{sec:setup:trans}.
We initialize both matrices with small random values sampled from a normal distribution with standard deviation \(10^{-8}\).
We optimize \(\L(\dist, \theta)\) using the REINFORCE algorithm~\cite{Williams2004SimpleSG} with the advantage function computed as the difference between the reward and the average reward across the batch.
We use the Adam optimizer with learning rate \(0.01\).
For each chain size, we train on \(1{,}000{,}000\) samples over \(300\) epochs with batch size \(50{,}000\), and validate on \(1{,}000\) samples.
During evaluation, we measure test accuracy and chain traversal correctness using greedy decoding (temperature \(0\)).
For experiments with a single chain size (see \cref{sec:exp:theo:main,sec:exp:theo:vanishing}), we set max new tokens to \(n + 1\) where \(n\) is the chain size; for generalization experiments across multiple chain sizes (see \cref{sec:exp:theo:gen}), we use \(\max(\text{test chain sizes}) + 1\).
All experiments were implemented using a single NVIDIA RTX A6000 GPU.

\subsection{Real-World Experiments}

All real-world pre-training experiments were implemented using eight NVIDIA RTX A6000 GPUs.

\label{app:exp_real_details}
\label{sec:exp:real}

\subsubsection{Synthetic Pre-Training Setting}
\label{app:exp_synthetic_pretraining}

This section provides full experimental context for the \emph{synthetic pre-training} setting of \cref{sec:exp} (Table~\ref{tab:exp_results}).
This setting follows the controlled reasoning framework of \citet{zhang2025interplaypretrainingmidtrainingrl}, and differs from the real-world pre-training setting in \cref{app:exp_realistic_pretraining} in one central respect: the base model has already been pre-trained on chain-of-thought traces from the same synthetic reasoning family.
It therefore tests our prediction in a setting where the model has task-specific reasoning primitives before RL, but must still generalize to examples requiring more reasoning steps than those encountered during pre-training.

\medskip
\noindent\textbf{Task and Complexity.}
\par\smallskip
We use the synthetic arithmetic reasoning task introduced by \citet{zhang2025interplaypretrainingmidtrainingrl}.
Each problem is generated from a directed dependency graph \(\mathcal{G}=(\mathcal{V},\mathcal{E})\), whose nodes represent intermediate quantities and whose directed edges represent arithmetic dependencies.
The graph is rendered into a natural-language math problem, and the model must produce a parseable chain-of-thought solution that computes the designated answer node.
Following \citet{zhang2025interplaypretrainingmidtrainingrl}, we measure the complexity of an instance by the number of arithmetic operations,
\[
\mathrm{op}(\mathcal{G}) = |\mathcal{E}|,
\]
which controls the amount of compositional reasoning required to solve the problem.

\medskip
\noindent\textbf{Base Model and Pre-Training.}
\par\smallskip
The base model is the same \(100\)M-parameter decoder-only Qwen2.5-style Transformer used by \citet{zhang2025interplaypretrainingmidtrainingrl}.
It is pre-trained for \(10\)B tokens on synthetic chain-of-thought data from the task family above, using their pre-training recipe: \(20\%\) examples with two to four arithmetic operations, \(30\%\) examples with five to seven arithmetic operations, and \(50\%\) examples with eight to ten arithmetic operations.
Thus, the model is exposed during pre-training to reasoning traces up to roughly ten operations.

\medskip
\noindent\textbf{Post-Training Distributions.}
\par\smallskip
We keep the architecture, optimizer, GRPO setup, prompt format, reward, and other hyperparameters identical to \citet{zhang2025interplaypretrainingmidtrainingrl}; the only aspect varied for our comparison is the post-training data distribution.
For a target complexity \(n\), \emph{Hard-\(n\)} denotes post-training only on examples with \(\mathrm{op}(\mathcal{G})=n\).
By contrast, \emph{Mixed-\(n\)} denotes a distribution that samples uniformly from operation counts between the edge of the pre-training range and the target complexity, i.e.\ \(\mathrm{op}(\mathcal{G}) \in \{10,11,\ldots,n\}\).

All runs use $200$ GRPO optimization steps. Training-time validation is performed every ten steps on the same \emph{Hard-$20$} subset used for final evaluation.

\medskip
\noindent\textbf{Evaluation.}
\par\smallskip

We use the final-answer checker implemented in the code of \citet{zhang2025interplaypretrainingmidtrainingrl}. An output is counted as correct if its extracted final answer matches the ground truth.
We report \texttt{pass@128} accuracy on the corresponding \emph{Hard-\(n\)} test distribution; in Table~\ref{tab:exp_results}, this is \emph{Hard-\(20\)}.
For each test example, we sample $128$ completions at temperature $1$ from the final checkpoint and count the example as solved if at least one final answer is correct. The model retains its $2{,}048$-token context, with prompt and response limits of $1{,}024$ tokens each. The prompt-length filter retains $197$ of the $200$ \emph{Hard-$20$} test examples; all synthetic pre-training results use this same subset.

\subsubsection{Real-World Pre-Training Setting}
\label{app:exp_realistic_pretraining}
\label{app:h2-real-world}

This section provides full experimental details for the \emph{real-world pre-training} setting of \cref{sec:exp}.
Unlike the synthetic pre-training setting (\cref{app:exp_synthetic_pretraining}), in which the base model was pre-trained on chain-of-thought traces of the specific reasoning task, here RL is applied directly to standard open-weight language models with broad language pre-training but \emph{no prior task-specific exposure}.

We fine-tune two such base models, Qwen2.5-3B-Instruct
\citep{qwen2025qwen25technicalreport} (\href{https://huggingface.co/Qwen/Qwen2.5-3B-Instruct}{Qwen/Qwen2.5-3B-Instruct})
and Llama-3.2-3B-Instruct
(\href{https://huggingface.co/meta-llama/Llama-3.2-3B-Instruct}{meta-llama/Llama-3.2-3B-Instruct}),
using GRPO \citep{shao2024deepseekmathpushinglimitsmathematical}.
The code builds on
\href{https://github.com/philschmid/deep-learning-pytorch-huggingface/blob/main/training/mini-deepseek-r1-aha-grpo.ipynb}{philschmid/deep-learning-pytorch-huggingface}
and is adapted to TRL~1.2 / Transformers~5.5.

This appendix provides comprehensive details for reproducing the real-world experiments. 
Section~\ref{app:h2-task} describes the mathematical reasoning task, including the equation 
system construction, graph generation procedure, complexity measures, and the mixed/hard data 
distributions used for training. Section~\ref{app:h2-training-setting} details the training 
procedure, including the architecture and reasoning framework, hyperparameter settings 
(Table~\ref{tab:h2-hparams}), and the prompt template used for both fine-tuning and evaluation. 
Finally, Section~\ref{app:h2-results} presents the results for both base models (Qwen and Llama), 
showing learning curves across different data distributions and providing concrete examples of 
model completions from the evaluation set.

\medskip
\noindent\textbf{Training Data.}\phantomsection\label{app:h2-task}
\par\smallskip
\textbf{Task description.}
We construct a mathematical reasoning task in which the input is presented as
natural-language text containing a \emph{shuffled} list of affine equations
over the binary field $\mathrm{GF}(2)$, of the form
\[
x_v \;=\; \Big(\!\!\sum_{u \in \mathrm{Pa}(v)}\! x_u \;+\; c_v\Big) \bmod 2,
\qquad c_v \in \{-1, 0, +1\},
\]
where \(\mathrm{Pa}(v)\) denotes the set of parent variables whose values are needed to compute \(x_v\).
Each input also contains one or more \emph{constant assignments} $x_r = c_r$
($c_r \in \{0,1\}$) for the root variables, and a query for a specified
\emph{target variable} $x_t$
(see \cref{fig:h2-example,fig:h2-llama-example} for examples of inputs and their
corresponding labels).
The dependency structure of the equation system forms a directed acyclic
graph (DAG) with up to two parents per non-root node ($\mathrm{in\text{-}deg} \le 2$).
The model is trained to compute the value of $x_t$, which is located at the
sink of the DAG.

\textbf{Complexity measure.}
For a given example, let the ancestor graph of the target be the subgraph containing all variables that can influence \(x_t\).
We define the example's complexity as the depth of this ancestor graph, i.e., the length of the longest directed path from a constant-assigned root variable to the target.

\textbf{Graph generation.}
We sample DAGs conditioned on this depth.
For target depth \(D-1\), we partition the variables into \(D\) ordered buckets and force one ``spine'' path to pass through one variable in each bucket, ending at the target.
This guarantees an ancestor path of the desired length.
We then add additional edges only from earlier buckets to later buckets, preserving acyclicity and the in-degree cap of two, so the graph may contain branching dependencies while retaining controlled depth.
After sampling the graph, we randomize variable labels, sample root values uniformly from \(\{0,1\}\), assign offsets \(c_v\in\{-1,0,+1\}\), and shuffle the resulting equations before concatenating them into the prompt.

\textbf{Data distributions (mixed/hard).}
For a maximum dependency depth $n$, we denote by $n$-Mixed the mixed data
distribution in which the ancestor-graph depth is sampled uniformly between
$1$ and $n-1$, and by $n$-Hard the hard data distribution that contains only
examples whose ancestor-graph depth is exactly $n-1$.
Throughout we use $15$ variables.

\medskip
\noindent\textbf{Training Setting.}\phantomsection\label{app:h2-training-setting}
\par\smallskip
\textbf{Architecture and training.}
For each base model, we fine-tune with a method inspired by DeepSeek-R1
\citep{deepseekai2025deepseekr1incentivizingreasoningcapability}.
The LLM is prompted to output its reasoning process within
\texttt{<think>} and \texttt{</think>} tags, followed by the final answer
within \texttt{<answer>} and \texttt{</answer>} tags.
Crucially, unlike training in DeepSeek-R1, we provide \emph{no explicit
instructions on how to reason} (e.g., no prompts to ``think step-by-step'');
the LLM must discover effective reasoning strategies on its own.
We employ GRPO \citep{shao2024deepseekmathpushinglimitsmathematical}, where the reward depends solely
on the correctness of the final answer and adherence to the specified output
format, with no supervision provided for the intermediate reasoning steps.

\textbf{Hyperparameters.}
We fine-tune Qwen2.5-3B-Instruct and Llama-3.2-3B-Instruct using GRPO with
the following configuration. We use a KL coefficient $\beta = 10^{-4}$,
generate $8$ completions per prompt with sampling temperature
$0.9$ during training, and set the max completion length to $1024$ tokens.
We optimize with AdamW using a learning rate of $5\!\times\!10^{-7}$
with a cosine learning-rate scheduler and a warmup ratio of $0.03$.
Training is performed in bfloat16 with FlashAttention~2 across
$8$ GPUs (DeepSpeed ZeRO-3, vLLM colocated on every training rank).
We use a per-device batch size of $8$ generations and $8$ gradient
accumulation steps, training for $800$ optimization steps in total
(Table~\ref{tab:h2-hparams} summarizes the hyperparameters). Evaluation is performed over $1{,}000$ samples drawn from the
$15$-Hard distribution (i.e.\ ancestor-graph depth fixed at $14$).
All test evaluations use greedy decoding (temperature~$0$) and a max
completion length of $2{,}048$ tokens.

\begin{table}[t]
\centering
\small
\caption{Hyperparameters used for fine-tuning Qwen2.5-3B-Instruct and
Llama-3.2-3B-Instruct on the equations-reasoning task.}
\label{tab:h2-hparams}
\begin{tabular}{ll}
\toprule
\textbf{Hyperparameter} & \textbf{Value} \\
\midrule
RL algorithm                       & GRPO \\
KL coefficient $\beta$             & $10^{-4}$ \\
Generations per prompt             & $8$ \\
Sampling temperature (training)    & $0.9$ \\
Max completion length (training)   & $1{,}024$ tokens \\
Max prompt length (training)       & $512$ tokens \\
Optimizer                          & AdamW \\
Learning rate                      & $5\!\times\!10^{-7}$ \\
LR schedule                        & cosine, warmup ratio $0.03$ \\
Per-device batch size              & $8$ generations \\
Gradient accumulation steps        & $8$ \\
Optimization steps                 & $800$ \\
Number of GPUs                     & $8$ \\
Precision                          & bfloat16 + FlashAttention~2 \\
Distributed strategy               & DeepSpeed ZeRO-3, vLLM colocated \\
\midrule
Variables per example $n$          & $15$ \\
Max in-degree                      & $2$ \\
Edge probability                   & $0.2$ \\
Modulus                            & $2$ \\
Constant range $c_v$               & $\{-1, 0, +1\}$, $\Pr(c_v\!=\!0)=1{-}0.6\overline{6}$ \\
\midrule
Eval samples                       & $1{,}000$ \\
Eval distribution                  & $15$-Hard (depth $=14$) \\
Decoding (eval)                    & greedy (temperature $0$) \\
Max completion length (eval)       & $2{,}048$ tokens \\
\bottomrule
\end{tabular}
\end{table}

\textbf{Prompt template.}
The exact prompt template used for fine-tuning and evaluation is shown in
Figure~\ref{fig:h2-prompt}. The \texttt{\{equation\_system\}} placeholder is
replaced at runtime with a serialized, shuffled equation system as described
in Appendix~\ref{app:h2-task}.

\begin{figure}[t]
\centering
\begin{quote}
\small
\hrule
\medskip
\begin{verbatim}
[
{
     "role": "system",
     "content": "You are a helpful assistant that can solve systems of equations.",
},
{
     "role": "user",
     "content": (
     f"Given this equation system: {dag_input}\n\nFind the value of the "
     "target variable. Show your reasoning in <think> </think> tags, and "
     "provide your final answer in <answer> </answer> tags, for example "
     "<answer>12</answer>."
     ),
},
{
     "role": "assistant",
     "content": "I'll think about the equation system and solve it.\n<think>",
},
]
\end{verbatim}
\medskip
\hrule
\end{quote}
\caption{Prompt template used for fine-tuning and evaluation in the
real-world experiments. The \texttt{\{equation\_system\}} placeholder is
replaced with the (shuffled) equation system at runtime; the assistant
turn is forced to begin with the marker
\texttt{"I'll think about the equation system and solve it.{\textbackslash}n<think>"}
so that completions deterministically start inside the \texttt{<think>}
block.}
\label{fig:h2-prompt}
\end{figure}

\medskip
\noindent\textbf{Results.}\phantomsection\label{app:h2-results}
\par\smallskip
\textbf{Training accuracy.}
\Cref{fig:h2-training-curves} reports the training curves for Qwen2.5-3B-Instruct and Llama-3.2-3B-Instruct.
For both base models, post-training on \(15\)-Mixed improves accuracy on the \(15\)-Hard evaluation distribution relative to post-training only on \(15\)-Hard, matching the qualitative behavior predicted by our theory. \footnote{The plots are produced by sliding a window of the most recent \(300\) training prompts over the chronological log of completions, computing the average correctness inside each window, and smoothing the resulting curve with a \(300\)-sample moving average. Each \(x\)-axis tick corresponds to a single optimization step (\(512\) generations / \(64\) unique prompts per step). The semi-transparent line is the unsmoothed window accuracy; the solid line is the smoothed trend.}

\Cref{tab:h2-llama-results} additionally reports the final Llama test accuracies in the same format as the real-world pre-training rows of \cref{tab:exp_results}.

\begin{figure}[t]
\centering
\begin{minipage}{0.48\linewidth}
\centering
\includegraphics[width=\linewidth]{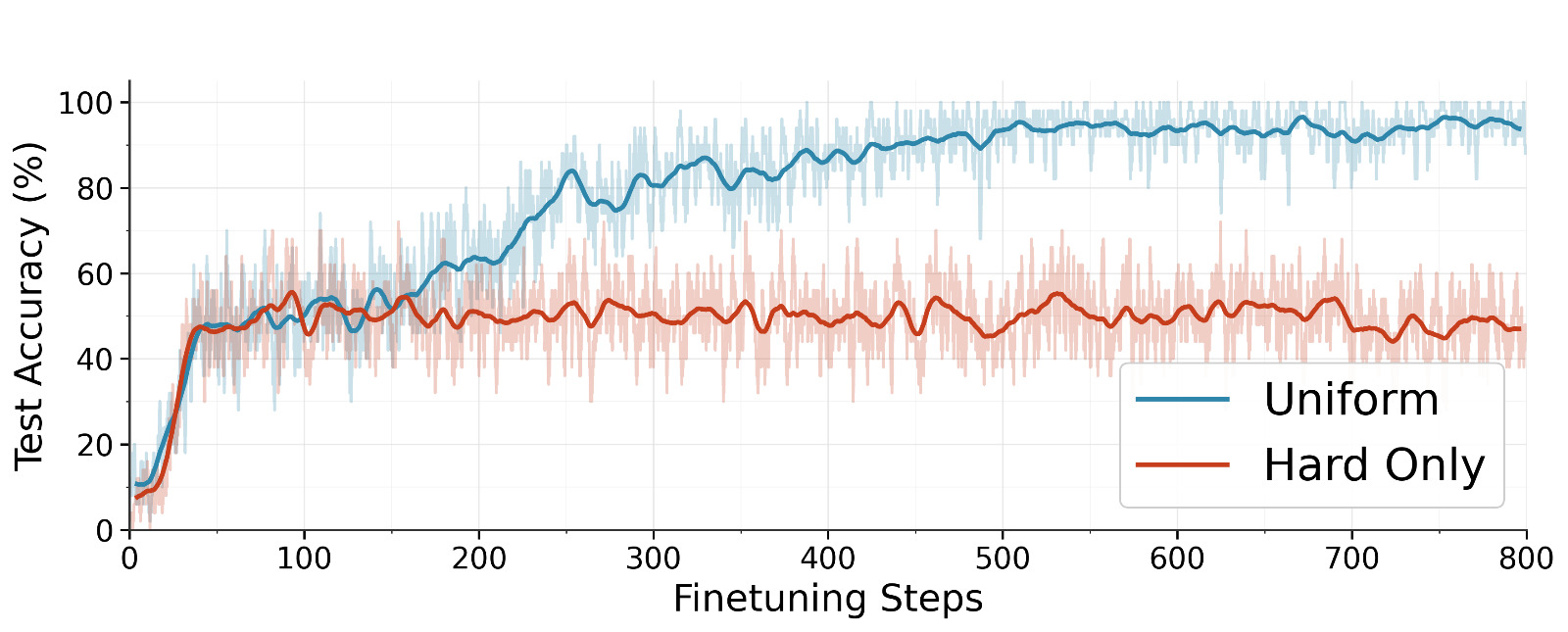}
\textbf{(a) Qwen2.5-3B-Instruct}
\end{minipage}
\hfill
\begin{minipage}{0.48\linewidth}
\centering
\includegraphics[width=\linewidth]{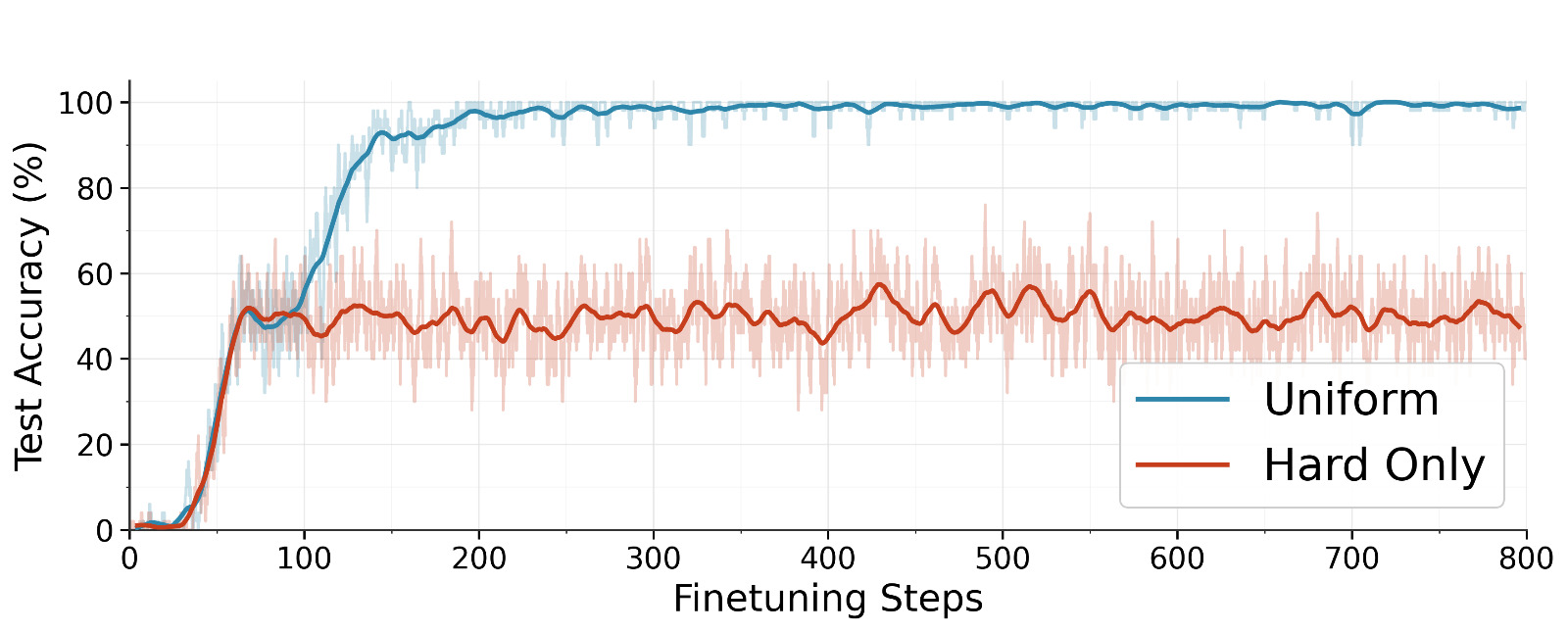}
\textbf{(b) Llama-3.2-3B-Instruct}
\end{minipage}
\caption{Training curves in the real-world pre-training setting. In both panels, the model is evaluated on \(15\)-Hard while being post-trained either on \(15\)-Mixed (blue) or \(15\)-Hard (red).}
\label{fig:h2-training-curves}
\end{figure}

\begin{table}[t]
\centering
\small
\caption{Same setting as the real-world pre-training experiment in \cref{tab:exp_results}, but using Llama-3.2-3B-Instruct as the real-world pre-trained base model. ``Test Acc.'' is the fraction of \(15\)-Hard test examples solved correctly.}
\label{tab:h2-llama-results}
\begin{tabular}{lc}
\toprule
Training Distribution & Test Acc. (\emph{Hard-\(15\)}) \\
\midrule
\emph{Hard-\(15\)}  & 50.2\% \\
\emph{Mixed-\(15\)} & 99.3\% \\
\emph{Mixed-\(10\)} & 98.6\% \\
\bottomrule
\end{tabular}
\end{table}

\textbf{Example completions.}
\Cref{fig:h2-example,fig:h2-llama-example} show randomly drawn \(15\)-Hard evaluation examples together with actual greedy-decoded completions from Qwen2.5-3B-Instruct and Llama-3.2-3B-Instruct respectively.
Both examples illustrate the learned strategy: identify the relevant dependencies and then compute the target value forward.

\begin{figure}[t]
\centering
\begin{quote}
\small
\hrule
\medskip
\textbf{Input.}
\begin{verbatim}
x_12 = (x_2 + x_15 + 1) mod 2.   x_11 = (x_10) mod 2.
x_6  = (x_1 + x_15 - 1) mod 2.   x_3  = (x_4 + 1) mod 2.
x_8  = (x_12 + x_14 - 1) mod 2.  x_13 = (x_9 + x_10 + 1) mod 2.
x_2  = (x_3 + x_10) mod 2.       x_1  = (x_8 + x_10 + 1) mod 2.
x_9  = (x_6 + x_15) mod 2.       x_4  = (x_5 + x_7 - 1) mod 2.
x_5  = (x_7 + x_14) mod 2.       x_15 = (x_11 + 1) mod 2.
@x_10 = 1.@                      x_14 = (x_7 + x_11) mod 2.
x_7  = (x_15 + 1) mod 2.         !Find x_13 (mod 2).!
\end{verbatim}
\textbf{Solution.}
\begin{verbatim}
<think>Starting from the known values:
x_10 = 1.
x_15 = (x_11 + 1) mod 2.
Since x_11 = (x_10) mod 2 and x_10 = 1, then x_11 = 1.
So, x_15 = (1 + 1) mod 2 = 0.
Now we know:
x_7  = (x_15 + 1) mod 2 = (0 + 1) mod 2 = 1.
x_14 = (x_7 + x_11) mod 2 = (1 + 1) mod 2 = 0.
x_5  = (x_7 + x_14) mod 2 = (1 + 0) mod 2 = 1.
x_4  = (x_5 + x_7 - 1) mod 2 = (1 + 1 - 1) mod 2 = 1.
x_3  = (x_4 + 1) mod 2 = (1 + 1) mod 2 = 0.
x_2  = (x_3 + x_10) mod 2 = (0 + 1) mod 2 = 1.
x_8  = (x_12 + x_14 - 1) mod 2.
To find x_12, we use: x_12 = (x_2 + x_15 + 1) mod 2
                           = (1 + 0 + 1) mod 2 = 0.
So, x_8 = (0 + 0 - 1) mod 2 = 1.
Finally, x_1 = (x_8 + x_10 + 1) mod 2
             = (1 + 1 + 1) mod 2 = 3 mod 2 = 1.
Now we can find x_9: x_9 = (x_6 + x_15) mod 2.
We need to find x_6 first: x_6 = (x_1 + x_15 - 1) mod 2.
Using x_1 = 1 and x_15 = 0, x_6 = (1 + 0 - 1) mod 2 = 0.
So, x_9 = (0 + 0) mod 2 = 0.
And finally, x_13 = (x_9 + x_10 + 1) mod 2
                  = (0 + 1 + 1) mod 2 = 2 mod 2 = 0.</think>
<answer>0</answer>
\end{verbatim}
\medskip
\hrule
\end{quote}
\caption{A \(15\)-Hard evaluation example and the actual greedy-decoded completion of Qwen2.5-3B-Instruct fine-tuned on \(15\)-Mixed.}
\label{fig:h2-example}
\end{figure}

\begin{figure}[t]
\centering
\begin{quote}
\small
\hrule
\medskip
\textbf{Input.}
\begin{verbatim}
x_3  = (x_1 + x_9 + 1) mod 2.    x_6  = (x_5) mod 2.
x_4  = (x_10 + x_12 + 1) mod 2.  x_2  = (x_1 + x_13 - 1) mod 2.
x_8  = (x_1 + 1) mod 2.          x_5  = (x_12 + 1) mod 2.
x_1  = (x_6) mod 2.              x_14 = (x_7) mod 2.
@x_12 = 1.@                      x_13 = (x_12 + x_15 + 1) mod 2.
x_15 = (x_8 + 1) mod 2.          x_10 = (x_8 + x_14) mod 2.
x_9  = (x_11 + x_15 + 1) mod 2.  x_7  = (x_2 - 1) mod 2.
x_11 = (x_4 - 1) mod 2.          !Find x_3 (mod 2).!
\end{verbatim}
\textbf{Solution.}
\begin{verbatim}
<think>
To solve for x_3, we need to start from the bottom and work our way up.
We are given x_12 = 1.
x_15 = (x_8 + 1) mod 2.
x_8 = (x_1 + 1) mod 2.
x_1 = (x_6) mod 2.
x_6 = (x_5) mod 2.
x_5 = (x_12 + 1) mod 2 = (1 + 1) mod 2 = 0.
x_6 = (x_5) mod 2 = (0) mod 2 = 0.
x_1 = (x_6) mod 2 = (0) mod 2 = 0.
x_8 = (x_1 + 1) mod 2 = (0 + 1) mod 2 = 1.
x_15 = (x_8 + 1) mod 2 = (1 + 1) mod 2 = 0.
x_13 = (x_12 + x_15 + 1) mod 2 = (1 + 0 + 1) mod 2 = 0.
x_11 = (x_4 - 1) mod 2.
x_4 = (x_10 + x_12 + 1) mod 2.
x_10 = (x_8 + x_14) mod 2.
x_14 = (x_7) mod 2.
x_7 = (x_2 - 1) mod 2.
x_2 = (x_1 + x_13 - 1) mod 2.
x_1 = 0.
x_13 = 0.
x_2 = (x_1 + x_13 - 1) mod 2 = (0 + 0 - 1) mod 2 = 1.
x_7 = (x_2 - 1) mod 2 = (1 - 1) mod 2 = 0.
x_14 = (x_7) mod 2 = (0) mod 2 = 0.
x_10 = (x_8 + x_14) mod 2 = (1 + 0) mod 2 = 1.
x_4 = (x_10 + x_12 + 1) mod 2 = (1 + 1 + 1) mod 2 = 1.
x_11 = (x_4 - 1) mod 2 = (1 - 1) mod 2 = 0.
x_9 = (x_11 + x_15 + 1) mod 2 = (0 + 0 + 1) mod 2 = 1.
x_3 = (x_1 + x_9 + 1) mod 2 = (0 + 1 + 1) mod 2 = 0.
</think>
<answer>0</answer>
\end{verbatim}
\medskip
\hrule
\end{quote}
\caption{A \(15\)-Hard evaluation example and the actual greedy-decoded completion of Llama-3.2-3B-Instruct fine-tuned on \(15\)-Mixed.}
\label{fig:h2-llama-example}
\end{figure}

\medskip
\noindent\textbf{LLM Annotation for Efficient Reasoning.}\phantomsection\label{app:llm_annotation}
\par\smallskip
To verify that the LLM learns the efficient step-by-step reasoning algorithm described in \cref{sec:exp:real:efficient}, we employ automated LLM annotation.
We use the real-world pre-training setting for this analysis, since in the synthetic pre-training setting the model is already exposed during pre-training to chain-of-thought traces that demonstrate the efficient algorithm.
We use GPT-5.1 with reasoning effort set to ``low'' to classify whether each solution follows the target algorithm.
We sample \(100\) correct solutions from the \(15\)-Uniform evaluation set and apply binary classification to each.
The judge LLM is provided with the system prompt shown in \cref{fig:judge_prompt}.

\begin{figure}[t]
\centering
\small
\fbox{\begin{minipage}{0.95\textwidth}
\texttt{You are a strict binary classifier.}\\[0.5em]
\texttt{Given:}\\[0.5em]
\texttt{(A) a system of equations and a request to find a target variable, and}\\[0.5em]
\texttt{(B) a solution writeup produced by a model,}\\[0.5em]
\texttt{Output True iff the writeup's solving strategy is:}\\[0.5em]
\texttt{1) backtrace from the target variable through dependencies until}\\
\texttt{\ \ \ reaching an explicit constant assignment,}\\[0.5em]
\texttt{2) then compute forward along that dependency chain to obtain the target.}\\[0.5em]
\texttt{Otherwise output False. Judge the described procedure (not correctness).}\\[0.5em]
\texttt{Output exactly one word: True or False.}
\end{minipage}}
\caption{System prompt provided to GPT-5.1 for classifying whether solutions follow the efficient step-by-step reasoning algorithm.}
\label{fig:judge_prompt}
\end{figure}

\clearpage
\subsubsection{\texorpdfstring{Simplicity and Data Diversity}{Simplicity and Data Diversity}}
\label{app:exp_simplicity_diversity}

To isolate the effect of simple examples from increased data diversity, we ran an ablation study in both pre-training settings of \cref{sec:exp}. Within each setting, we held the number of distinct complexity levels, the mass on the hardest level, and the training budget fixed. Including simple examples substantially improves accuracy on hard instances, indicating that this benefit is not explained by the number of distinct complexity levels.

We use the same models, training procedures, and evaluation as in \cref{sec:exp}. As in that section, an example's \emph{complexity} measures the extent of reasoning required to compute its answer. Within each setting, we post-train on two distributions with the same number of distinct complexity levels: one includes simple examples, whereas the other contains only harder examples. We sample complexities uniformly from $\{10,13,17,20\}$ and $\{17,18,19,20\}$ in the synthetic pre-training setting, and from $\{5,10,15\}$ and $\{13,14,15\}$ in the real-world pre-training setting. \Cref{tab:exp_simplicity_diversity} reports accuracy on \emph{Hard-$n$}, consisting of examples of complexity $n$, with $n=15$ in the real-world pre-training setting and $n=20$ in the synthetic pre-training setting.

\begin{table}[h]
    \centering
    \caption{Test accuracy on \emph{Hard-$n$} under post-training distributions with the same number of distinct complexity levels. Within each setting, the first row includes simple examples and the second contains only harder examples.}
    \label{tab:exp_simplicity_diversity}
    \small
    \setlength{\tabcolsep}{6pt}
    \begin{tabular}{@{} lc !{\vrule width 0.5pt} lc @{}}
    \toprule
    \multicolumn{2}{c}{\textbf{Real-World Pre-Training}} & \multicolumn{2}{c}{\textbf{Synthetic Pre-Training}} \\
    \cmidrule(r{4pt}){1-2} \cmidrule(l{4pt}){3-4}
    \shortstack{Training Dist.\\Complexity Levels} & Test Acc. (\emph{Hard-$15$}) & \shortstack{Training Dist.\\Complexity Levels} & Test Acc. (\emph{Hard-$20$}) \\
    \midrule
    $\{5,10,15\}$ & $96.80\%$ & $\{10,13,17,20\}$ & $80.20\%$ \\
    $\{13,14,15\}$ & $52.90\%$ & $\{17,18,19,20\}$ & $48.22\%$ \\
    \bottomrule
    \end{tabular}
\end{table}

	\ifdefined\NEURIPS
		\ifdefined\NOCHECKLIST\else
			\clearpage
			\renewcommand{\answerYes}[1][]{[Yes]#1}
\renewcommand{\answerNo}[1][]{[No]#1}
\renewcommand{\answerNA}[1][]{[N/A]#1}
\section*{NeurIPS Paper Checklist}

\begin{enumerate}

\item {\bf Claims}
    \item[] Question: Do the main claims made in the abstract and introduction accurately reflect the paper's contributions and scope?
    \item[] Answer: \answerYes{}
    \item[] Justification: The abstract and introduction state the paper's theoretical claims about outcome-based RL learning efficient reasoning from sufficiently simple examples, and the empirical claims are limited to the synthetic and real-world mathematical reasoning experiments reported in the paper.

\item {\bf Limitations}
    \item[] Question: Does the paper discuss the limitations of the work performed by the authors?
    \item[] Answer: \answerYes{}
    \item[] Justification: We discuss limitations in a dedicated Limitations section, including the stylized theoretical setting, assumptions of the analysis, and the restricted scope of the real-world experiments.

\item {\bf Theory assumptions and proofs}
    \item[] Question: For each theoretical result, does the paper provide the full set of assumptions and a complete (and correct) proof?
    \item[] Answer: \answerYes{}
    \item[] Justification: The main theoretical results state their assumptions in the main text, and the complete proofs are provided in the appendix.

    \item {\bf Experimental result reproducibility}
    \item[] Question: Does the paper fully disclose all the information needed to reproduce the main experimental results of the paper to the extent that it affects the main claims and/or conclusions of the paper (regardless of whether the code and data are provided or not)?
    \item[] Answer: \answerYes{}
    \item[] Justification: The main text and appendix describe the experimental tasks, data distributions, training algorithms, hyperparameters, prompts, evaluation protocols, and compute setup used for the reported experiments.

\item {\bf Open access to data and code}
    \item[] Question: Does the paper provide open access to the data and code, with sufficient instructions to faithfully reproduce the main experimental results, as described in supplemental material?
    \item[] Answer: \answerNo{}
    \item[] Justification: The paper provides detailed reproducibility information and states that code for reproducing the experiments will be made available for the camera-ready version, but the current manuscript does not yet provide an open code release.

\item {\bf Experimental setting/details}
    \item[] Question: Does the paper specify all the training and test details (e.g., data splits, hyperparameters, how they were chosen, type of optimizer) necessary to understand the results?
    \item[] Answer: \answerYes{}
    \item[] Justification: The experimental section and appendix specify the task construction, train/test distributions, optimizer, RL setup, decoding and evaluation procedures, prompt template, and hyperparameters.

\item {\bf Experiment statistical significance}
    \item[] Question: Does the paper report error bars suitably and correctly defined or other appropriate information about the statistical significance of the experiments?
    \item[] Answer: \answerNo{}
    \item[] Justification: The reported experiments are long-running single-run demonstrations intended to test the qualitative behavior predicted by the theory; the paper does not report error bars or statistical significance tests.

\item {\bf Experiments compute resources}
    \item[] Question: For each experiment, does the paper provide sufficient information on the computer resources (type of compute workers, memory, time of execution) needed to reproduce the experiments?
    \item[] Answer: \answerYes{}
    \item[] Justification: The appendix reports the number of GPUs and distributed training setup used for the real-world experiments, along with the relevant training hyperparameters.

\item {\bf Code of ethics}
    \item[] Question: Does the research conducted in the paper conform, in every respect, with the NeurIPS Code of Ethics \url{https://neurips.cc/public/EthicsGuidelines}?
    \item[] Answer: \answerYes{}
    \item[] Justification: The research uses synthetic mathematical reasoning tasks and open-weight models in a standard experimental setting, and does not involve human subjects, private data, or deployed systems.

\item {\bf Broader impacts}
    \item[] Question: Does the paper discuss both potential positive societal impacts and negative societal impacts of the work performed?
    \item[] Answer: \answerNA{}
    \item[] Justification: This is foundational research on the training dynamics of reasoning in Transformers and does not introduce a deployed system, dataset, or model with direct societal impact.

\item {\bf Safeguards}
    \item[] Question: Does the paper describe safeguards that have been put in place for responsible release of data or models that have a high risk for misuse (e.g., pre-trained language models, image generators, or scraped datasets)?
    \item[] Answer: \answerNA{}
    \item[] Justification: We do not release any datasets, generators, or models.

\item {\bf Licenses for existing assets}
    \item[] Question: Are the creators or original owners of assets (e.g., code, data, models), used in the paper, properly credited and are the license and terms of use explicitly mentioned and properly respected?
    \item[] Answer: \answerYes{}
    \item[] Justification: The paper cites the prior synthetic reasoning framework, the open-weight base models used in the real-world experiments, and the codebase on which the implementation builds, with links where applicable.

\item {\bf New assets}
    \item[] Question: Are new assets introduced in the paper well documented and is the documentation provided alongside the assets?
    \item[] Answer: \answerNA{}
    \item[] Justification: The paper does not release a new dataset, benchmark, or model as a standalone asset.

\item {\bf Crowdsourcing and research with human subjects}
    \item[] Question: For crowdsourcing experiments and research with human subjects, does the paper include the full text of instructions given to participants and screenshots, if applicable, as well as details about compensation (if any)?
    \item[] Answer: \answerNA{}
    \item[] Justification: Our research does not involve crowdsourcing or human subjects.

\item {\bf Institutional review board (IRB) approvals or equivalent for research with human subjects}
    \item[] Question: Does the paper describe potential risks incurred by study participants, whether such risks were disclosed to the subjects, and whether Institutional Review Board (IRB) approvals (or an equivalent approval/review based on the requirements of your country or institution) were obtained?
    \item[] Answer: \answerNA{}
    \item[] Justification: Our research does not involve human subjects.

\item {\bf Declaration of LLM usage}
    \item[] Question: Does the paper describe the usage of LLMs if it is an important, original, or non-standard component of the core methods in this research? Note that if the LLM is used only for writing, editing, or formatting purposes and does \emph{not} impact the core methodology, scientific rigor, or originality of the research, declaration is not required.
    \item[] Answer: \answerYes{}
    \item[] Justification:  We use GPT-5.1 to annotate whether generated solutions follow the efficient reasoning algorithm. The annotation procedure and judge prompt are described in \cref{app:llm_annotation}.

\end{enumerate}

		\fi
	\fi

\end{document}